\documentclass[11pt]{article}
\usepackage[T1]{fontenc}
\usepackage[utf8]{inputenc}
\usepackage{geometry}
\usepackage{color}
\usepackage{mathrsfs}
\usepackage{enumitem}
\usepackage{amsmath}
\usepackage{amsthm}
\usepackage{amssymb}
\usepackage{stmaryrd}
\usepackage[unicode=true,
 bookmarks=false,
 breaklinks=false,pdfborder={0 0 0},pdfborderstyle={},backref=section,colorlinks=true]
 {hyperref}
\hypersetup{
 pdfborderstyle={},linkcolor=blue, citecolor=blue}

\makeatletter
\theoremstyle{plain}
\newtheorem{thm}{\protect\theoremname}
\theoremstyle{definition}
\newtheorem{example}[thm]{\protect\examplename}
\theoremstyle{remark}
\newtheorem{rem}[thm]{\protect\remarkname}
\theoremstyle{plain}
\newtheorem{assumption}{\protect\assumptionname}
\theoremstyle{plain}
\newtheorem{lem}[thm]{\protect\lemmaname}
\theoremstyle{definition}
\newtheorem{defn}[thm]{\protect\definitionname}
\theoremstyle{plain}
\newtheorem{cor}[thm]{\protect\corollaryname}
\theoremstyle{plain}
\newtheorem{prop}[thm]{\protect\propositionname}

\providecommand{\corollaryname}{Corollary}
\providecommand{\definitionname}{Definition}
\providecommand{\examplename}{Example}
\providecommand{\lemmaname}{Lemma}
\providecommand{\propositionname}{Proposition}
\providecommand{\theoremname}{Theorem}

\usepackage[alphabetic,nobysame]{amsrefs}

\allowdisplaybreaks[3]

\makeatother

\providecommand{\assumptionname}{Assumption}
\providecommand{\corollaryname}{Corollary}
\providecommand{\definitionname}{Definition}
\providecommand{\examplename}{Example}
\providecommand{\lemmaname}{Lemma}
\providecommand{\propositionname}{Proposition}
\providecommand{\remarkname}{Remark}
\providecommand{\theoremname}{Theorem}

\begin{document}
\title{A Rigorous Framework for the Mean Field Limit\\of Multilayer Neural
Networks \thanks{A conference version \cite{pham2021global} of the work appears in
ICLR 2021.}}
\author{Phan-Minh Nguyen\thanks{The Voleon Group. The majority of this work was done when P.-M. Nguyen
was at Department of Electrical Engineering, Stanford University.}$\quad$and Huy Tuan Pham\thanks{Department of Mathematics, Stanford University. This work was done
in part while H. T. Pham was at the University of Cambridge.} \thanks{Author ordering is randomized.}}
\maketitle
\begin{abstract}
We develop a mathematically rigorous framework for multilayer neural
networks in the mean field regime. As the network's widths increase,
the network's learning trajectory is shown to be well captured by
a meaningful and dynamically nonlinear limit (the \textit{mean field}
limit), which is characterized by a system of ODEs. Our framework
applies to a broad range of network architectures, learning dynamics
and network initializations. Central to the framework is the new idea
of a \textit{neuronal embedding}, which comprises of a non-evolving
probability space that allows to embed neural networks of arbitrary
widths.

Using our framework, we prove several properties of large-width multilayer
neural networks. Firstly we show that independent and identically
distributed initializations cause strong degeneracy effects on the
network's learning trajectory when the network's depth is at least
four. Secondly we obtain several global convergence guarantees for
feedforward multilayer networks under a number of different setups.
These include two-layer and three-layer networks with independent
and identically distributed initializations, and multilayer networks
of arbitrary depths with a special type of correlated initializations
that is motivated by the new concept of \textit{bidirectional diversity}.
Unlike previous works that rely on convexity, our results admit non-convex
losses and hinge on a certain universal approximation property, which
is a distinctive feature of infinite-width neural networks and is
shown to hold throughout the training process. Aside from being the
first known results for global convergence of multilayer networks
in the mean field regime, they demonstrate flexibility of our framework
and incorporate several new ideas and insights that depart from the
conventional convex optimization wisdom.
\end{abstract}
\tableofcontents{}

\section{Introduction}

A major outstanding theoretical challenge in deep learning is the
understanding of the learning dynamics of multilayer neural networks.
A precise characterization of the learning trajectory is typically
hard, primarily owing to the highly nonlinear and complex structure
of deep learning architectures, which departs from convex optimization
even when the loss function is convex. Recent progresses tackle this
challenge with one simplification: they consider networks whose widths
are very large, ideally approaching infinity. In particular, under
suitable conditions, as the width increases, the network's behavior
during training is expected to be captured by a meaningful limit.

One such type of analysis exploits exchangeability of neurons. \cite{mei2018mean,chizat2018,rotskoff2018neural,sirignano2018mean}
show that under a suitable scaling limit, the learning dynamics of
wide two-layer neural networks can be captured by a Wasserstein gradient
flow of a probability measure over weights. In this limit -- which
is usually referred to as the \textit{mean field (MF) limit}, the
network weights evolve nonlinearly with time. The MF scaling of two-layer
networks require a certain normalization to be applied to the last
layer, together with a learning rate that compensates for this normalization.
The MF limit under the Wasserstein gradient flow formulation has led
to a fruitful line of research that explains and uncovers interesting
properties of two-layer networks, such as their optimization efficacy.
Let us delve into a few further high-level details of the two-layer
case, before discussing the interesting challenge in the multilayer
case.

\subsection{Two-layer MF network: a brief overview}

Let us informally present a sampled subset of interesting results
from this line of works. To fix ideas, we consider the usual two-layer
neural network:
\[
\hat{\mathbf{y}}_{\text{2-layer}}\left(x;W\right)=\frac{1}{n}\sum_{i=1}^{n}w_{2,i}\sigma\left(\left\langle w_{1,i},x\right\rangle \right).
\]
Here $x\in\mathbb{R}^{d}$ is the input, $\sigma:\;\mathbb{R}\to\mathbb{R}$
is a nonlinear activation function, and $W=\left\{ w_{1,i},w_{2,i}\right\} _{i\in\left[n\right]}$
is the set of weights with $w_{1,i}\in\mathbb{R}^{d}$ and $w_{2,i}\in\mathbb{R}$,
for $i\in\left[n\right]$ the set of integers from $1$ and $n$.
This network has $n$ neurons; $n$ is also referred to as the width
of the network. The scaling factor $1/n$ is special to the MF scaling
in two-layer networks and we will see its role shortly.

\paragraph*{An ``infinite-width'' representation.}

One key idea in this line of work is to introduce the following representation:
\[
\hat{y}_{\text{2-layer}}\left(x;\mu\right)=\int w_{2}\sigma\left(\left\langle w_{1},x\right\rangle \right)\mu(dw_{1},dw_{2}),
\]
where $\mu$ is a probability measure on $\mathbb{R}^{d+1}$. It is
easy to see that by choosing $\mu=\frac{1}{n}\sum_{i=1}^{n}\delta_{w_{1,i},w_{2,i}}$
the empirical measure over the weights $W$, one identifies $\hat{y}_{\text{2-layer}}\left(x;\mu\right)=\hat{\mathbf{y}}_{\text{2-layer}}\left(x;W\right)$.
This identification is possible thanks to the previously mentioned
scaling factor $n$.

One way to rationalize this representation is as follows: there is
a special symmetry in the two-layer neural network, in which
\[
\hat{\mathbf{y}}_{\text{2-layer}}\left(x;W\right)=\hat{\mathbf{y}}_{\text{2-layer}}(x;\{w_{1,\Pi(i)},w_{2,\Pi(i)}\}_{i\in\left[n\right]})
\]
for any permutation $\Pi$ on the set of integers $\left[n\right]$.
The representation via $\mu$ is a neat way to factor our this symmetry
and capture the exchangeability of neurons. Of course, one is not
restricted to only empirical measures for $\mu$. Therefore this representation
allows one to reason about two-layer neural networks with \textit{arbitrary}
widths. In other words, it gives us the ability to take the infinite-width
limit $n\to\infty$. This is an important observation that is central
to this line of works.

\paragraph*{The learning dynamics at infinite width: the MF limit.}

We are interested in understanding the learning dynamics of the network
in the infinite-width limit. Consider the continuous-time gradient
descent learning rule (with respect to $W$) for the loss $\ell$:
\[
\frac{d}{dt}W(t)=-n\nabla_{W}\mathbb{E}_{Z}\left[\ell(Y,\hat{\mathbf{y}}_{\text{2-layer}}\left(X;W(t)\right))\right].
\]
Here $t$ denotes the time and $Z=(X,Y)$ a random variable that represents
the training data. Note the scaling factor $n$, which compensates
for the previously mentioned factor $1/n$ and therefore allows the
learning update to be on the ``correct'' order. To see this, we
rewrite the learning rule:
\begin{align*}
\frac{d}{dt}w_{1,i}(t) & =-\mathbb{E}_{Z}\left[\partial_{2}\ell(Y,\hat{\mathbf{y}}_{\text{2-layer}}\left(X;W(t)\right))\cdot w_{2,i}(t)\sigma'\left(\left\langle w_{1,i}(t),X\right\rangle \right)X\right],\\
\frac{d}{dt}w_{2,i}(t) & =-\mathbb{E}_{Z}\left[\partial_{2}\ell(Y,\hat{\mathbf{y}}_{\text{2-layer}}\left(X;W(t)\right))\cdot\sigma\left(\left\langle w_{1,i}(t),X\right\rangle \right)\right],
\end{align*}
where $\partial_{2}\ell$ is the derivative of $\ell$ w.r.t. the
second variable. In this form of the learning rule, one sees that
if $w_{1,i}(t)$ and $w_{2,i}(t)$ all have magnitudes on order $O(1)$
independent of $n$, then so are their updates $\frac{d}{dt}w_{1,i}(t)$
and $\frac{d}{dt}w_{2,i}(t)$. Hence if they are initialized to be
on this order, one can expect to see the same order of weights and
weight movements at any finite time $t$. This is a feature of the
MF scaling.

Suppose that the initialization is sampled $(w_{1,i}(0),w_{2,i}(0))\sim\mu_{0}$
independently for each $i\in\left[n\right]$, for a probability measure
$\mu_{0}$ on $\mathbb{R}^{d+1}$. We would like to study the empirical
measure over the weights $W(t)$:
\[
\mu_{t}^{n}=\frac{1}{n}\sum_{i=1}^{n}\delta_{w_{1,i}(t),w_{2,i}(t)}.
\]
At $t=0$, it is a standard result that $\mu_{0}^{n}\to\mu_{0}$ weakly
as $n\to\infty$, under suitable regularity conditions. We are interested
in a similar statement for any time $t$. To that end, we recall the
``infinite-width'' representation $\hat{y}_{\text{2-layer}}\left(x;\mu\right)$
and introduce the following distributional dynamics in the Wasserstein
space of probability measures on $\mathbb{R}^{d+1}$:
\[
\partial_{t}\mu_{t}(w_{1},w_{2})={\rm div}\Big(\mu_{t}(w_{1},w_{2})\nabla_{(w_{1},w_{2})}\Psi(w_{1},w_{2};\mu_{t})\Big),
\]
in which $\Psi(w_{1},w_{2};\mu)=\mathbb{E}_{Z}\left[\partial_{2}\ell(Y,\hat{y}_{\text{2-layer}}\left(X;\mu\right))\cdot w_{2}\sigma\left(\left\langle w_{1},X\right\rangle \right)\right]$
and the initialization is $\mu_{0}$. This dynamics is a Wasserstein
gradient flow. Prior works \cite{mei2018mean,chizat2018} prove the
following type of result:
\begin{thm}[Two-layer MF network with distributional representation, $n\to\infty$,
informal and simplified]
\label{thm:intro_MF}Under suitable regularity conditions, for any
finite constant $T$, as $n\to\infty$, $\mu_{t}^{n}\to\mu_{t}$ weakly
and uniformly over $t\in\left[0,T\right]$.
\end{thm}

The precise statement includes a quantitative convergence rate, more
realistic learning rules such as discrete-time stochastic gradient
descent and other variations. Theorem \ref{thm:intro_MF} formalizes
the notion of an infinite-width limit: we call $\mu_{t}$ the MF limit.

\paragraph*{An application of the MF limit: proving global convergence.}

Theorem \ref{thm:intro_MF} conveys an interesting message: one can
study the width-$n$ neural network by analyzing the MF limit $\mu_{t}$.
One success story is the study of optimization efficacy. In particular,
\cite{chizat2018} proves the following type of result:
\begin{thm}[Two-layer MF network with distributional representation, $t\to\infty$,
informal and simplified]
\label{thm:intro_global_opt}Suppose that the support of $\mu_{0}$
is $\mathbb{R}^{d+1}$ (i.e. it has full support at initialization)
and the loss $\ell$ is convex in the second variable. Under suitable
regularity and convergence conditions, as $t\to\infty$,
\[
\mu_{t}\to\inf_{\mu}\mathbb{E}_{Z}\left[\ell(Y,\hat{y}_{\text{2-layer}}\left(X;\mu\right))\right].
\]
\end{thm}

This global convergence result affirms positively the message that
taking $n\to\infty$ under the MF limit can lead to meaningful learning.
Similar global convergence results have been established for different
types of learning rules and, in special occasions, with quantitative
convergence rates. To understand the significance of this result,
we note a remarkable feature of the MF limit $\mu_{t}$: it represents
a genuinely nonlinear dynamics. To contrast the situation, another
line of works (e.g. \cite{jacot2018neural,chizat2018note}) show that
under a different scaling, in the infinite-width limit, the neural
network is equivalent to a parameterized model which is linear in
its parameter. In that scaling regime, the learning dynamics hence
simplifies into a linear dynamics and consequently it is relatively
clear how to attain global convergence using the usual convex optimization
wisdom. The MF limit is distinct in this sense, but it also comes
with a nontrivial problem: insights from convex optimization may no
longer apply. This is indeed the case in the proof of the global convergence
results of \cite{chizat2018,mei2018mean}.

We refer to \cite{nguyen2020phdthesis,bach2021gradient} for further
overview discussions on two-layer MF neural networks. See also Section
\ref{sec:Discussion} for a partial list of works.

\subsection{Multilayer MF network: the challenge and our contributions}

Recall, an important milestone is to find a representation that allows
to interpolate to an infinite-width limit. For two-layer networks,
by exploiting exchangeability among neurons, one can achieve this
goal and use the representation to successfully analyze properties
of the neural networks in the infinite-width limit. In multilayer
networks, exchangeability is, however, not a priori obvious and hence
poses a highly non-trivial challenge. In particular, the presence
of intermediate layers exhibits multiple symmetry groups with intertwined
actions on the model. To illustrate the point, let us consider a simple
three-layer fully-connected neural network which assumes the following
form (modulo scaling factors):
\[
\hat{\mathbf{y}}_{\text{3-layer}}\left(x\right)=\sum_{i=1}^{n_{2}}w_{3,i}\sigma\bigg(\sum_{j=1}^{n_{1}}w_{2,ij}\sigma\left(\left\langle w_{1,j},x\right\rangle \right)\bigg),
\]
for a set of parameters $\left\{ w_{3,i},w_{2,ij},w_{1,j}\right\} _{i\in\left[n_{2}\right],\;j\in\left[n_{1}\right]}$.
In the matrix notation:
\[
\hat{\mathbf{y}}_{\text{3-layer}}\left(x\right)=w_{3}^{\top}\sigma\left(W_{2}\sigma\left(W_{1}x\right)\right).
\]
Under any two permutations $\Pi_{1}:\;\left[n_{1}\right]\to\left[n_{1}\right]$
and $\Pi_{2}:\;\left[n_{2}\right]\to\left[n_{2}\right]$, we recognize:
\[
\hat{\mathbf{y}}_{\text{3-layer}}\left(x\right)=w_{3}^{\top}\Pi_{2}^{\top}\sigma\left(\Pi_{2}W_{2}\Pi_{1}^{\top}\sigma\left(\Pi_{1}W_{1}x\right)\right).
\]
The fact that the weight matrix $W_{2}$ in the middle layer is under
the simultaneous influence of both actions $\Pi_{1}$ and $\Pi_{2}$,
is what makes the three-layer case specifically and the multilayer
case in general different from the two-layer case, more challenging
and at the same time also a highly interesting problem. With this
blocker on the strategy to extend the two-layer case, even the goal
of obtaining a representation that captures networks with arbitrary
widths becomes less approachable. Indeed prior attempts in \cite{nguyen2019mean,araujo2019mean,sirignano2019mean}
arrive at quite complex solutions or require a certain strong assumption
that leads to undesirable properties (see Section \ref{sec:Discussion}),
and yet these attempts already have to do away with the Wasserstein
gradient flow formulation.

In short, finding a suitable formulation that is amenable to the infinite-width
limit-taking procedure, simultaneous at all layers, requires innovation
beyond the Wasserstein gradient flow idea of the two-layer case. The
formulation should faithfully describe settings where nonlinear and
meaningful learning trajectories take place. To compound the difficulty,
a useful formulation should lend a way to analyze properties of multilayer
neural networks in the infinite-width limit, for instance, how well
these networks could be optimized despite the strong presence of nonlinearity
and the lack of convexity. These are the considerations one ought
to keep in mind when tackling the challenge.

This work responses to this challenge with the proposal of a mathematically
rigorous framework for the MF limit of multilayer neural networks.
The framework is built on an innovative idea of a \textsl{neuronal
embedding}. More importantly, using this framework, we prove several
properties of multilayer networks, which incorporate new insights
and ideas. Specifically, our key contributions can be summarized as
follows:
\begin{itemize}
\item \textbf{(Sections \ref{sec:Framework}, \ref{sec:Existence-MF} and
\ref{sec:Main-result}) }We develop a framework for the MF limit of
multilayer neural networks under stochastic gradient descent (SGD)
training and suitable scalings. We introduce the concept of a neuronal
embedding, which comprises of a non-evolving probability space that
can embed neural networks of arbitrary widths. In this framework,
the MF limit is described by a system of ordinary differential equations
(ODEs), which govern the evolutions of different functions that represent
the weights at different layers and are adapted to the given neuronal
embedding. The complete framework is described in Section \ref{sec:Framework}
and the well-posedness of the MF limit is proven in Section \ref{sec:Existence-MF}.
Our main result in this thread is stated in Section \ref{sec:Main-result},
where the MF limit is proven to track closely characteristics of a
wide multilayer network under SGD training, with quantitative bounds
on the required widths.

In fact, our framework is quite general, admits a broad variety of
initialization schemes (including, but not limited to, independent
and identically distributed (i.i.d.) initializations) and operates
in Hilbert spaces. This allows for firstly describing the MF behavior
for generic multilayer setups (including fully-connected and convolutional
networks in Euclidean spaces that are common in practice), and secondly
obtaining \textit{dimension-free} quantitative bounds.
\item \textbf{(Section \ref{sec:iid_init})} Using the neuronal embedding
framework, we uncover strong degeneracy properties caused by i.i.d.
initializations. Specifically we prove that with at least four layers,
the MF limits, and hence the neural networks, are substantially simplified
under i.i.d. initializations: at an intermediate layer, each weight
evolves as a function of only time, its own initialization and the
initial biases associated with its connected neurons. An implication
is that when the initial biases are constant, different intermediate
layers evolve independently of each other. Remarkably, for common
neural network architectures, all weights (or biases) at each intermediate
layer then evolve by translation: they differ from their respective
initializations by the same deterministic amount, and the effective
number of parameters at each intermediate layer thus collapses to
only one.
\item Our framework allows to study the optimization efficacy of multilayer
neural networks trained under SGD in the infinite-width limit. In
particular:
\begin{itemize}
\item \textbf{(Section \ref{sec:global_convergence_iid})} We prove convergence
to the global optimum for two-layer and three-layer networks under
i.i.d. initializations, with suitable regularity conditions and convergence
assumptions. Some of these assumptions are mild and natural in neural
network learning. The key convergence assumption in this section turns
out to be necessary for global convergence to hold, i.e. it is impossible
to attain global convergence if this convergence assumption fails.
\item \textbf{(Section \ref{sec:global_convergence_general})} Avoiding
the degeneracy effect of i.i.d. initializations, we prove global convergence
for multilayer networks of arbitrary depths under a special type of
correlated initializations and a similar set of assumptions. Here
we introduce the new concept of \textsl{bidirectional diversity}.
\item \textbf{(Section \ref{sec:Global-convergence-ms})} We also establish
global convergence in the above settings under Morse-Sard conditions
that are usually assumed in the literature for MF two-layer networks.
This demonstrates flexibility of our framework: it can handle situations
where the two-layer Wasserstein gradient flow formulation works, as
well as situations where such formulation finds difficulty.
\end{itemize}
Two novel features that our global convergence results have in common
are firstly the role of a certain universal approximation property
which is natural of nonlinear neural networks, and secondly the admission
of non-convex losses. Importantly the universal approximation property
is shown to hold at any finite training time (but not necessarily
at infinite time) via topological invariance arguments. These new
insights signal the departure from conventional wisdoms of convex
optimization.

The idea of bidirectional diversity that we introduce in Section \ref{sec:global_convergence_general}
strikes directly to the universal approximation insight. Roughly speaking,
it helps ``propagating'' the universal approximation property from
the first layer to the second last layer. This is to be contrasted
with i.i.d.-initialized networks: universal approximation at the first
layer suffices when there are few layers, but as the number of layers
increases, due to degeneracy by i.i.d. initializations, the middle
layers become a bottleneck that generally prohibits universal approximation
to be propagated. Bidirectional diversity aims to break this bottleneck.
\end{itemize}
We defer a more technical discussion on the related literature to
Section \ref{sec:Discussion}. Proofs of several intermediate results
are deferred to the appendices. Readers who are interested in global
convergence of networks with more than three layers may skip directly
to Sections \ref{sec:global_convergence_general} and \ref{sec:Global-convergence-ms},
which we have made relatively self-contained with minimal references
to the previous sections.

\subsection{Notations}

For an integer $n$, we use $\left[n\right]$ to denote the set $\left\{ 1,...,n\right\} $.
We shall use $\left\langle \cdot,\cdot\right\rangle $ and $\left|\cdot\right|$
to indicate respectively the inner product and its induced norm for
a Hilbert space, and $\left|\cdot\right|$ to indicate the absolute
value for $\mathbb{R}$. We use $\sigma_{{\rm alg}}\left(U\right)$
to denote the sigma-algebra generated by a random variable $U$. We
write ${\rm cl}\left(S\right)$ to denote the closure of a set $S$
in a topological space. We use $K$ to denote a generic absolute constant
that may change from line to line. For a probability space $\left(\Omega,{\cal F},P\right)$,
we will suppress the presence of the sigma-algebra ${\cal F}$ wherever
unimportant. Given two events $\mathcal{E}$ and $\mathcal{E}'$,
we say that ${\cal E}'$ occurs with probability at least $1-\delta$
on the event ${\cal E}$ and write $\mathbb{P}({\cal E}';{\cal E})\ge1-\delta$
if $\mathbb{P}((\lnot{\cal E}')\cap{\cal E})\le\delta$. 

\section{A General Framework\label{sec:Framework}}

In this section, we describe our setup of a general multilayer neural
network with a generalized (stochastic) learning dynamics. In particular,
it covers several common neural network architectures as well as the
SGD training dynamics. We then describe the corresponding MF limit.

\subsection{Multilayer neural network and generalized learning dynamics\label{subsec:NN}}

We consider the following generalized neural network with $L$ layers:
\begin{align}
\hat{\mathbf{y}}\left(k,x\right)\equiv\hat{\mathbf{y}}\left(x;\mathbf{W}\left(k\right)\right) & =\phi_{L+1}\left({\bf H}_{L}\left(k,x,1\right)\right),\label{eq:NN}
\end{align}
in which we define ${\bf H}_{L}\left(k,x,1\right)$ recursively: 
\begin{align*}
{\bf H}_{1}\left(k,x,j_{1}\right)\equiv{\bf H}_{1}\left(x,j_{1};\mathbf{W}\left(k\right)\right) & =\phi_{1}\left({\bf w}_{1}(k,j_{1}),x\right),\qquad j_{1}\in\left[n_{1}\right],\\
{\bf H}_{i}\left(k,x,j_{i}\right)\equiv{\bf H}_{i}\left(x,j_{i};\mathbf{W}\left(k\right)\right) & =\frac{1}{n_{i-1}}\sum_{j_{i-1}=1}^{n_{i-1}}\phi_{i}\left({\bf w}_{i}\left(k,j_{i-1},j_{i}\right),{\bf b}_{i}\left(k,j_{i}\right),{\bf H}_{i-1}\left(k,x,j_{i-1}\right)\right),\\
 & \qquad\qquad j_{i}\in\left[n_{i}\right],\;i=2,...,L.
\end{align*}
The above equations describe the forward pass in the neural network.
We explain the quantities in the following:
\begin{itemize}
\item $x\in\mathbb{X}$ is the input, and $\mathbb{X}$ is the input space.
\item $k\in\mathbb{N}_{\geq0}$ is the (discrete) time.
\item $\mathbf{W}\left(k\right)=\left\{ \mathbf{w}_{1}\left(k,\cdot\right),\mathbf{w}_{i}\left(k,\cdot,\cdot\right),\mathbf{b}_{i}\left(k,\cdot\right),\;\;i=2,...,L\right\} $
is the collection of neural network parameters (weights and biases)
at time $k$.
\item $\mathbf{w}_{1}:\;\mathbb{N}_{\geq0}\times\left[n_{1}\right]\to\mathbb{W}_{1}$
is the weight of the first layer (which also includes the bias). Similarly
for $i=2,...,L$, $\mathbf{w}_{i}:\;\mathbb{N}_{\geq0}\times\left[n_{i-1}\right]\times\left[n_{i}\right]\to\mathbb{W}_{i}$
and $\mathbf{b}_{i}:\;\mathbb{N}_{\geq0}\times\left[n_{i}\right]\to\mathbb{B}_{i}$
are the weight and bias of the $i$-th layer. Here $n_{i}$ is the
number of neurons at the $i$-th layer, $\mathbb{W}_{i}$ and $\mathbb{B}_{i}$
are separable Hilbert spaces, and we take $n_{L}=1$.
\item $\phi_{1}:\;\mathbb{W}_{1}\times\mathbb{X}\to\mathbb{H}_{1}$, $\phi_{i}:\;\mathbb{W}_{i}\times\mathbb{B}_{i}\times\mathbb{H}_{i-1}\to\mathbb{H}_{i}$
for $i=2,...,L$, and $\phi_{L+1}:\;\mathbb{H}_{L}\to\hat{\mathbb{Y}}$,
where again $\mathbb{H}_{i}$ and $\hat{\mathbb{Y}}$ are separable
Hilbert spaces. 
\end{itemize}
In other words, the network $\hat{\mathbf{y}}\left(k,x\right)$ is
a state-dependent mapping that takes $x$ as input and is dependent
on the state $\mathbf{W}\left(k\right)$, which is allowed to vary
with time $k$.

The network is trained by the following (discrete-time) stochastic
learning dynamics. At each time $k$, we draw independently a data
sample $z\left(k\right)=\left(x\left(k\right),y\left(k\right)\right)\sim{\cal P}$,
where ${\cal P}$ is the data distribution on $\mathbb{X}\times\mathbb{Y}$
and $y\left(k\right)\in\mathbb{Y}$ a separable Hilbert space . Given
an initialization $\mathbf{W}\left(0\right)$, we update $\mathbf{W}\left(k\right)$
into $\mathbf{W}\left(k+1\right)$ as follows: 
\begin{align*}
{\bf w}_{1}\left(k+1,j_{1}\right) & ={\bf w}_{1}\left(k,j_{1}\right)-\epsilon\xi_{1}^{\mathbf{w}}\left(k\epsilon\right)\Delta_{1}^{{\bf w}}\left(k,z\left(k\right),j_{1}\right),\qquad\forall j_{1}\in\left[n_{1}\right],\\
{\bf w}_{i}\left(k+1,j_{i-1},j_{i}\right) & ={\bf w}_{i}\left(k,j_{i-1},j_{i}\right)-\epsilon\xi_{i}^{\mathbf{w}}\left(k\epsilon\right)\Delta_{i}^{{\bf w}}\left(k,z\left(k\right),j_{i-1},j_{i}\right),\\
{\bf b}_{i}\left(k+1,j_{i}\right) & ={\bf b}_{i}\left(k,j_{i}\right)-\epsilon\xi_{i}^{\mathbf{b}}\left(k\epsilon\right)\Delta_{i}^{{\bf b}}\left(k,z\left(k\right),j_{i}\right),\qquad\forall j_{i-1}\in\left[n_{i-1}\right],\;j_{i}\in\left[n_{i}\right],\;i=2,...,L.
\end{align*}
We explain the quantities in the following: 
\begin{itemize}
\item $\epsilon\in\mathbb{R}_{>0}$ is the learning rate, and $\xi_{i}^{\mathbf{w}}$
and $\xi_{i}^{\mathbf{b}}$ are mappings from $\mathbb{R}$ to $\mathbb{R}$,
representing the different learning rate schedules for each of the
weights and biases. Note that we allow the learning rate schedules
to take non-positive values.
\item To define the updates $\Delta_{i}^{\mathbf{w}}$ and $\Delta_{i}^{\mathbf{b}}$
requires additional definitions. Firstly, for $z=\left(x,y\right)$,
we define: 
\[
\Delta_{L}^{{\bf H}}\left(k,z,1\right)\equiv\Delta_{L}^{{\bf H}}\left(z,1;\mathbf{W}\left(k\right)\right)=\sigma_{L}^{\mathbf{H}}\left(y,\hat{\mathbf{y}}\left(k,x\right),{\bf H}_{L}\left(k,x,1\right)\right).
\]
Then we define recursively: 
\begin{align*}
\Delta_{i}^{{\bf w}}\left(k,z,j_{i-1},j_{i}\right) & \equiv\Delta_{i}^{{\bf w}}\left(z,j_{i-1},j_{i};\mathbf{W}\left(k\right)\right)\\
 & =\sigma_{i}^{\mathbf{w}}\left(\Delta_{i}^{{\bf H}}\left(k,z,j_{i}\right),{\bf w}_{i}\left(k,j_{i-1},j_{i}\right),{\bf b}_{i}\left(k,j_{i}\right),{\bf H}_{i}\left(k,x,j_{i}\right),{\bf H}_{i-1}\left(k,x,j_{i-1}\right)\right),\\
\Delta_{i}^{{\bf b}}\left(k,z,j_{i}\right) & \equiv\Delta_{i}^{{\bf b}}\left(z,j_{i};\mathbf{W}\left(k\right)\right)\\
 & =\frac{1}{n_{i-1}}\sum_{j_{i-1}=1}^{n_{i-1}}\sigma_{i}^{\mathbf{b}}\left(\Delta_{i}^{{\bf H}}\left(k,z,j_{i}\right),{\bf w}_{i}\left(k,j_{i-1},j_{i}\right),{\bf b}_{i}\left(k,j_{i}\right),{\bf H}_{i}\left(k,x,j_{i}\right),{\bf H}_{i-1}\left(k,x,j_{i-1}\right)\right),\\
\Delta_{i-1}^{{\bf H}}\left(k,z,j_{i-1}\right) & \equiv\Delta_{i-1}^{{\bf H}}\left(z,j_{i-1};\mathbf{W}\left(k\right)\right)\\
 & =\frac{1}{n_{i}}\sum_{j_{i}=1}^{n_{i}}\sigma_{i-1}^{\mathbf{H}}\left(\Delta_{i}^{{\bf H}}\left(k,z,j_{i}\right),{\bf w}_{i}\left(k,j_{i-1},j_{i}\right),{\bf b}_{i}\left(k,j_{i}\right),{\bf H}_{i}\left(k,x,j_{i}\right),{\bf H}_{i-1}\left(k,x,j_{i-1}\right)\right),\\
 & \qquad i=L,...,2,\\
\Delta_{1}^{{\bf w}}\left(k,z,j_{1}\right) & \equiv\Delta_{1}^{{\bf w}}\left(z,j_{1};\mathbf{W}\left(k\right)\right)=\sigma_{1}^{\mathbf{w}}\left(\Delta_{1}^{{\bf H}}\left(k,z,j_{i}\right),{\bf w}_{1}\left(k,j_{1}\right),x\right),
\end{align*}
in which the functions are: 
\begin{align*}
\sigma_{L}^{\mathbf{H}} & :\;\mathbb{Y}\times\hat{\mathbb{Y}}\times\mathbb{H}_{L}\to\hat{\mathbb{H}}_{L},\\
\sigma_{i}^{\mathbf{w}} & :\;\hat{\mathbb{H}}_{i}\times\mathbb{W}_{i}\times\mathbb{B}_{i}\times\mathbb{H}_{i}\times\mathbb{H}_{i-1}\to\mathbb{W}_{i},\\
\sigma_{i}^{\mathbf{b}} & :\;\hat{\mathbb{H}}_{i}\times\mathbb{W}_{i}\times\mathbb{B}_{i}\times\mathbb{H}_{i}\times\mathbb{H}_{i-1}\to\mathbb{B}_{i},\\
\sigma_{i-1}^{\mathbf{H}} & :\;\hat{\mathbb{H}}_{i}\times\mathbb{W}_{i}\times\mathbb{B}_{i}\times\mathbb{H}_{i}\times\mathbb{H}_{i-1}\to\hat{\mathbb{H}}_{i-1},\qquad i=L,...,2,\\
\sigma_{1}^{\mathbf{w}} & :\;\hat{\mathbb{H}}_{1}\times\mathbb{W}_{1}\times\mathbb{X}\to\mathbb{W}_{1},
\end{align*}
for separable Hilbert spaces $\hat{\mathbb{H}}_{i}$. Note that the
above equations describe the backward pass in the neural network. 
\end{itemize}
The introduced framework is quite general, while certain assumptions
can be further relaxed. We observe that several common network architectures
and training processes can be cast as special cases.
\begin{example}[Fully-connected networks]
\label{exa:fully-connected}We describe the simple setting of a fully-connected
network with 1-dimensional output and an activation function $\varphi_{i}:\;\mathbb{R}\to\mathbb{R}$
at the $i$-th layer. Specifically the network output assumes the
form:
\[
\hat{\mathbf{y}}\left(x;\mathbf{W}\right)=\frac{1}{n_{L-1}}\left\langle \mathbf{w}_{L},\varphi_{L-1}\left(\mathbf{b}_{L-1}+\frac{1}{n_{L-2}}\mathbf{W}_{L-1}\varphi_{L-2}\left(...\varphi_{1}\left(\mathbf{W}_{1}\left[\begin{array}{c}
x\\
1
\end{array}\right]\right)\right)\right)\right\rangle +\mathbf{b}_{L},
\]
in which $x\in\mathbb{R}^{d}$, $\mathbf{W}=\left\{ \mathbf{w}_{L},\mathbf{W}_{L-1},...,\mathbf{W}_{1},\mathbf{b}_{L},...,\mathbf{b}_{2}\right\} $,
$\mathbf{w}_{L}\in\mathbb{R}^{n_{L-1}}$, $\mathbf{W}_{i}\in\mathbb{R}^{n_{i}\times n_{i-1}}$
with $n_{0}=d+1$, $n_{L}=1$ and $\mathbf{b}_{i}\in\mathbb{R}^{n_{i}}$.
This case fits into our framework with $\mathbb{X}=\mathbb{R}^{d}$,
$\mathbb{W}_{1}=\mathbb{R}^{d+1}$, $\mathbb{H}_{1}=\mathbb{R}$ and
$\mathbb{W}_{i}=\mathbb{B}_{i}=\mathbb{H}_{i}=\mathbb{Y}=\hat{\mathbb{Y}}=\mathbb{R}$
for $2\leq i\leq L$. We also have: 
\begin{align*}
\phi_{1}\left(w,x\right) & =\left\langle w_{1:d},x\right\rangle +w_{d+1},\\
\phi_{i}\left(w,b,h\right) & =w\varphi_{i-1}\left(h\right)+b,\quad2\leq i\leq L,\\
\phi_{L+1}\left(h\right) & =h.
\end{align*}
Consider the regularized loss function:
\begin{align*}
{\rm Loss}\left(\mathbf{W};z\right) & ={\cal L}\left(y,\hat{y}\left(x;\mathbf{W}\right)\right)+\frac{1}{n_{1}}\sum_{j_{1}=1}^{n_{1}}\Phi_{1}\left(w_{1,j_{1}}\right)+\frac{1}{n_{L-1}}\sum_{j_{L-1}=1}^{n_{L-1}}\Phi_{L}\left(w_{L,j_{L-1}}\right)\\
 & \qquad+\sum_{i=2}^{L-1}\left(\frac{1}{n_{i-1}n_{i}}\sum_{j_{i-1}=1}^{n_{i-1}}\sum_{j_{i}=1}^{n_{i}}\Phi_{i}\left(w_{i,j_{i-1}j_{i}}\right)\right)+\sum_{i=2}^{L}\left(\frac{1}{n_{i}}\sum_{j_{i}=1}^{n_{i}}\Psi_{i}\left(b_{i,j_{i}}\right)\right),
\end{align*}
where ${\cal L}:\;\mathbb{R}\times\mathbb{R}\to\mathbb{R}_{\geq0}$,
$\Phi_{i}:\;\mathbb{R}\to\mathbb{R}_{\geq0}$ for $i\geq1$, $\Phi_{1}:\;\mathbb{R}^{d+1}\to\mathbb{R}_{\geq0}$,
$\Psi_{i}:\;\mathbb{R}\to\mathbb{R}_{\geq0}$, $w_{1,j_{1}}$ is the
$j_{1}$-th row of $\mathbf{W}_{1}$, $w_{i,j_{i-1}j_{i}}$ is the
$\left(j_{i-1},j_{i}\right)$-th entry of $\mathbf{W}_{i}$ for $2\leq i\leq L-1$,
$w_{L,j_{L-1}}$ is the $j_{L-1}$-th entry of $\mathbf{w}_{L}$,
and $b_{i,j_{i}}$ is the $j_{i}$-th entry of $\mathbf{b}_{i}$.
If we train the network by SGD w.r.t. this loss, then $\hat{\mathbb{H}}_{i}=\mathbb{R}$
and 
\begin{align*}
\sigma_{L}^{\boldsymbol{H}}\left(y,\hat{y},h\right) & =\partial_{2}{\cal L}\left(y,\hat{y}\right),\\
\sigma_{i}^{\mathbf{w}}\left(\Delta,w,b,g,h\right) & =\Delta\varphi_{i-1}\left(h\right)+\Phi_{i}'\left(w\right),\\
\sigma_{i}^{\mathbf{b}}\left(\Delta,w,b,g,h\right) & =\Delta+\Psi'_{i}\left(b\right),\\
\sigma_{i-1}^{\mathbf{H}}\left(\Delta,w,b,g,h\right) & =\Delta w\varphi'_{i-1}\left(h\right),\qquad2\leq i\leq L,\\
\sigma_{1}^{\mathbf{w}}\left(\Delta,w,x\right) & =\Delta\left[\begin{array}{c}
x\\
1
\end{array}\right]+\nabla\Phi_{1}\left(w\right).
\end{align*}
Observe that when there is no regularization (i.e. no $\Phi_{i}$
and $\Psi_{i}$), $\sigma_{i}^{\mathbf{w}}\left(\Delta,w,b,g,h\right)$
is independent of $w$ and $b$, and the same holds for $\sigma_{i}^{\mathbf{b}}$
and $\sigma_{1}^{\mathbf{w}}$.
\end{example}

\begin{example}[Convolutional networks]
\label{exa:conv}

Our framework can also describe networks that are not of the fully-connected
type. For illustration, we consider the first two layers of a convolutional
network with an activation $\varphi:\;\mathbb{R}\to\mathbb{R}$ and
pooling operation ${\rm pool}\left(\cdot\right)$; a description of
the complete network (which may contain fully-connected layers) can
be done in a similar fashion to Example \ref{exa:fully-connected}.
Here $\mathbb{X}=\left(\mathbb{R}^{p\times p}\right)^{n_{c}}$, where
in the context of a square image input, $p$ is the number of pixels
per row and $n_{c}$ is the number of channels (which is $3$ for
RGB images and $1$ for gray-scale images). We take $\mathbb{W}_{1}=\mathbb{R}^{f_{1}\times f_{1}\times n_{c}}\times\mathbb{R}$
and $\mathbb{W}_{2}=\mathbb{R}^{f_{2}\times f_{2}}$, where $f_{1}$
and $f_{2}$ are the filter sizes, $\mathbb{B}_{2}=\mathbb{R}$, $\mathbb{H}_{1}=\mathbb{R}^{p_{1}\times p_{1}}$
and $\mathbb{H}_{2}=\mathbb{R}^{p_{2}\times p_{2}}$. Then:
\begin{align*}
\phi_{1}\left(\left(w,b\right),x\right) & =w*x+b\boldsymbol{1}_{p_{1}},\\
\phi_{2}\left(w,b,h\right) & =w*{\rm pool}\left(\varphi\left(h\right)\right)+b\mathbf{1}_{p_{2}},
\end{align*}
where $*$ denotes (strided) convolution and $\boldsymbol{1}_{p_{i}}$
is an all-one matrix in $\mathbb{R}^{p_{i}\times p_{i}}$. The dimensions
$p_{1}$ and $p_{2}$ are determined by the actual convolution operation,
its stride size, its padding type and the input size. In this context,
$n_{1}$ and $n_{2}$ are the numbers of filters at the first and
second layer respectively. One can also specify the forms of $\sigma_{i}^{\mathbf{w}}$,
$\sigma_{i}^{\mathbf{b}}$ and $\sigma_{i}^{\mathbf{H}}$ upon the
choice of a loss function, with SGD training.
\end{example}

The examples of fully-connected and convolutional neural networks
serve as the main motivation to study the generalized neural network
model as described. In both of these examples, the spaces are finite-dimensional
Euclidean spaces, while in the generalized model, the spaces are allowed
to be infinite-dimensional. Similarly while SGD w.r.t. a loss function
is the typical choice of learning dynamics for these examples, in
our framework, the learning dynamics is more general. We shall see
that the key ideas hold regardless of the specific details. In particular,
the ultimate goal is to understand properties of $\mathbf{W}\left(k\right)$
in the limit of large $n_{i}$ and small $\epsilon$, via a limiting
object that is well-defined and has an explicit form. To this end,
we introduce the mean field limit in the next section.

\subsection{Mean field limit\label{subsec:MF}}

We now describe the mean field (MF) limit. Given a probability space
$\left(\Omega,{\cal F},P\right)=\prod_{i=1}^{L}\left(\Omega_{i},{\cal F}_{i},P_{i}\right)$
with $\Omega_{L}=\left\{ 1\right\} $, we independently sample $C_{i}\sim P_{i}$,
$1\leq i\leq L$. From here onwards, we hide the sigma-algebras ${\cal F}$,
${\cal F}_{i}$ wherever unimportant. In the following, we use $\mathbb{E}_{C_{i}}$
to denote the expectation w.r.t. the random variable $C_{i}\sim P_{i}$
and $c_{i}$ to denote a dummy variable $c_{i}\in\Omega_{i}$. The
space $\left(\Omega,P\right)$ is key to our MF formulation and is
referred to as the \textit{neuronal ensemble}. The choice of the neuronal
ensemble bridges the connection between the earlier described neural
network and the MF limit; this connection shall be established later
in Section \ref{sec:Main-result}. For the moment we treat the MF
limit as an independent object from the neural network.

Given the neuronal ensemble, we obtain the MF limit as follows. It
entails the following quantity: 
\begin{align*}
\hat{y}\left(t,x\right)\equiv\hat{y}\left(x;W\left(t\right)\right) & =\phi_{L+1}\left(H_{L}\left(t,x,1\right)\right),
\end{align*}
in which $H_{L}\left(t,x,1\right)$ is computed recursively: 
\begin{align*}
H_{1}\left(t,x,c_{1}\right)\equiv H_{1}\left(x,c_{1};W\left(t\right)\right) & =\phi_{1}\left(w_{1}\left(t,c_{1}\right),x\right),\qquad\forall c_{1}\in\Omega_{1},\\
H_{i}\left(t,x,c_{i}\right)\equiv H_{i}\left(x,c_{i};W\left(t\right)\right) & =\mathbb{E}_{C_{i-1}}\left[\phi_{i}\left(w_{i}\left(t,C_{i-1},c_{i}\right),b_{i}\left(t,c_{i}\right),H_{i-1}\left(t,x,C_{i-1}\right)\right)\right],\\
 & \qquad\forall c_{i}\in\Omega_{i},\;i=2,...,L.
\end{align*}
This corresponds to the forward pass of the neural network. We note
the similarity with the corresponding quantities of the neural network:
\begin{itemize}
\item $x\in\mathbb{X}$ is the input and $t\in\mathbb{R}_{\geq0}$ is the
(continuous) time.
\item $W\left(t\right)=\left\{ w_{1}\left(t,\cdot\right),w_{i}\left(t,\cdot,\cdot\right),b_{i}\left(t,\cdot\right),\;\;i=2,...,L\right\} $
is the collection of MF parameters at time $t$. 
\item $w_{1}:\;\mathbb{R}_{\geq0}\times\Omega_{1}\to\mathbb{W}_{1}$, and
for $i=2,...,L$, $w_{i}:\;\mathbb{R}_{\geq0}\times\Omega_{i-1}\times\Omega_{i}\to\mathbb{W}_{i}$
and $b_{i}:\;\mathbb{R}_{\geq0}\times\Omega_{i}\to\mathbb{B}_{i}$. 
\end{itemize}
In correspondence with the neural network's learning dynamics for
$\mathbf{W}\left(k\right)$, the MF limit also entails a continuous-time
evolution dynamics for $W\left(t\right)$. This dynamics takes the
form of a system of ODEs, which we refer to as the \textit{MF ODEs},
given an initialization $W\left(0\right)$:
\begin{align*}
\frac{\partial}{\partial t}w_{1}\left(t,c_{1}\right) & =-\xi_{1}^{\mathbf{w}}\left(t\right)\mathbb{E}_{Z}\left[\Delta_{1}^{w}\left(t,Z,c_{1}\right)\right],\qquad\forall c_{1}\in\Omega_{1},\\
\frac{\partial}{\partial t}w_{i}\left(t,c_{i-1},c_{i}\right) & =-\xi_{i}^{\mathbf{w}}\left(t\right)\mathbb{E}_{Z}\left[\Delta_{i}^{w}\left(t,Z,c_{i-1},c_{i}\right)\right],\\
\frac{\partial}{\partial t}b_{i}\left(t,c_{i}\right) & =-\xi_{i}^{\mathbf{b}}\left(t\right)\mathbb{E}_{Z}\left[\Delta_{i}^{b}\left(t,Z,c_{i}\right)\right],\qquad\forall c_{i-1}\in\Omega_{i-1},\;c_{i}\in\Omega_{i},\;i=2,...,L,
\end{align*}
where $\mathbb{E}_{Z}$ denotes the expectation w.r.t. the data $Z=\left(X,Y\right)\sim{\cal P}$,
and the update quantities are defined by the following recursion:
\begin{align*}
\Delta_{L}^{H}\left(t,z,1\right) & \equiv\Delta_{L}^{H}\left(z,1;W\left(t\right)\right)=\sigma_{L}^{\mathbf{H}}\left(y,\hat{y}\left(t,x\right),H_{L}\left(t,x,1\right)\right),\\
\Delta_{i}^{w}\left(t,z,c_{i-1},c_{i}\right) & \equiv\Delta_{i}^{w}\left(z,c_{i-1},c_{i};W\left(t\right)\right)\\
 & =\sigma_{i}^{\mathbf{w}}\left(\Delta_{i}^{H}\left(t,z,c_{i}\right),w_{i}\left(t,c_{i-1},c_{i}\right),b_{i}\left(t,c_{i}\right),H_{i}\left(t,x,c_{i}\right),H_{i-1}\left(t,x,c_{i-1}\right)\right),\\
\Delta_{i}^{b}\left(t,z,c_{i}\right) & \equiv\Delta_{i}^{b}\left(z,c_{i};W\left(t\right)\right)\\
 & =\mathbb{E}_{C_{i-1}}\left[\sigma_{i}^{\mathbf{b}}\left(\Delta_{i}^{H}\left(t,z,c_{i}\right),w_{i}\left(t,C_{i-1},c_{i}\right),b_{i}\left(t,c_{i}\right),H_{i}\left(t,x,c_{i}\right),H_{i-1}\left(t,x,C_{i-1}\right)\right)\right],\\
\Delta_{i-1}^{H}\left(t,z,c_{i-1}\right) & \equiv\Delta_{i-1}^{H}\left(z,c_{i-1};W\left(t\right)\right)\\
 & =\mathbb{E}_{C_{i}}\left[\sigma_{i-1}^{\mathbf{H}}\left(\Delta_{i}^{H}\left(t,z,C_{i}\right),w_{i}\left(t,c_{i-1},C_{i}\right),b_{i}\left(t,C_{i}\right),H_{i}\left(t,x,C_{i}\right),H_{i-1}\left(t,x,c_{i-1}\right)\right)\right],\\
 & \qquad i=L,...,2,\\
\Delta_{1}^{w}\left(t,z,c_{1}\right) & \equiv\Delta_{1}^{w}\left(z,c_{1};W\left(t\right)\right)=\sigma_{1}^{\mathbf{w}}\left(\Delta_{1}^{H}\left(t,z,c_{1}\right),w_{1}\left(t,c_{1}\right),x\right).
\end{align*}
This recursion corresponds to the backward pass of the neural network.
\begin{rem}
The definition of a MF limit model $\hat{y}\left(t,x\right)$ based
on the neuronal ensemble $\left(\Omega,P\right)$ gives a way to define
a large class of neural networks that encapsulates networks of arbitrary
sizes. More specifically, let us write $W=\left\{ w_{1},\,w_{i},\,b_{i}:\;i=2,...,L\right\} $
in place $W\left(t\right)$ and $\hat{y}\left(x;W,\Omega,P\right)$
in place of $\hat{y}\left(t,x\right)$ to ignore the time $t$ and
make explicit the dependency on the neuronal ensemble. Similarly here
let us also write $\mathbf{W}=\left\{ \mathbf{w}_{1},\,\mathbf{w}_{i},\,\mathbf{b}_{i}:\;i=2,...,L\right\} $
in place of $\mathbf{W}\left(k\right)$ and $\hat{\mathbf{y}}\left(x;\mathbf{W},n_{1},...,n_{L}\right)$
in place of $\hat{\mathbf{y}}\left(k,x\right)$. Then by defining
the class $\mathsf{NN}_{\infty}=\left\{ \hat{y}\left(\cdot;W,\Omega,P\right)\right\} _{W,\Omega,P}$
that is indexed by $\left(W,\Omega,P\right)$ while fixing other parameters
(such as the number of layers $L$), one sees that any finite-sized
neural network $\hat{\mathbf{y}}\left(\cdot;\mathbf{W},n_{1},...,n_{L}\right)$
belongs to $\mathsf{NN}_{\infty}$. This correspondence can be seen
by the following identification: $\Omega=\prod_{i=1}^{L}\Omega_{i}$
with $\Omega_{i}=\left[n_{i}\right]$, $P=\prod_{i=1}^{L}P_{i}$ with
$P_{i}$ a uniform measure on $\left[n_{i}\right]$, and
\begin{align*}
w_{i}\left(j_{i-1},j_{i}\right) & =\mathbf{w}_{i}\left(j_{i-1},j_{i}\right),\qquad\forall j_{i-1}\in\Omega_{i-1}=\left[n_{i-1}\right],\;j_{i}\in\Omega_{i}=\left[n_{i}\right],\\
b_{i}\left(j_{i}\right) & =\mathbf{b}_{i}\left(j_{i}\right),\qquad\forall j_{i}\in\Omega_{i}=\left[n_{i}\right],\\
w_{1}\left(j_{1}\right) & =\mathbf{w}_{1}\left(j_{1}\right),\qquad\forall j_{1}\in\Omega_{1}=\left[n_{1}\right],
\end{align*}
for $2\leq i\leq L$. In particular, there exists $\hat{y}\left(\cdot;W,\Omega,P\right)\in\mathsf{NN}_{\infty}$
such that
\[
\hat{y}\left(\cdot;W,\Omega,P\right)=\hat{\mathbf{y}}\left(\cdot;\mathbf{W},n_{1},...,n_{L}\right).
\]
More generally one may observe that a similar correspondence holds
for both the forward pass and the backward pass; for example,
\begin{align*}
H_{i}\left(x,j_{i};W,\Omega,P\right) & =\mathbf{H}_{i}\left(x,j_{i};\mathbf{W},n_{1},...,n_{L}\right),\\
\Delta_{i}^{H}\left(z,j_{i};W,\Omega,P\right) & =\Delta_{i}^{{\bf H}}\left(z,j_{i};\mathbf{W},n_{1},...,n_{L}\right),\qquad\forall j_{i}\in\Omega_{i}=\left[n_{i}\right],
\end{align*}
where the quantities are rewritten forms of $H_{i}\left(t,x,c_{i}\right)$,
$\Delta_{i}^{H}\left(t,z,c_{i}\right)$, $\mathbf{H}_{i}\left(k,x,j_{i}\right)$
and $\Delta_{i}^{\mathbf{H}}\left(k,z,c_{i}\right)$ respectively.
As such, roughly speaking, the dynamics of any finite-sized neural
network can be identified with a MF dynamics, modulo the differences
in time discretization and stochastic sampling of the data. The same
observation is made in \cite{wojtowytsch2020banach}, which instead
studies it from the function space approximation perspective.
\end{rem}

\subsection{Preliminaries\label{subsec:Preliminaries}}

We describe several preliminaries that are necessary for the next
steps. First we consider several structural assumptions.
\begin{assumption}
\label{enu:Assump_lrSchedule}The learning rate schedules are bounded
and Lipschitz: 
\begin{align*}
\max_{1\leq i\leq L}\left|\xi_{i}^{\mathbf{w}}\left(t\right)\right|,\quad\max_{2\leq i\leq L}\left|\xi_{i}^{\mathbf{b}}\left(t\right)\right| & \leq K,\\
\max_{1\leq i\leq L}\left|\xi_{i}^{\mathbf{w}}\left(t\right)-\xi_{i}^{\mathbf{w}}\left(t'\right)\right|,\quad\max_{2\leq i\leq L}\left|\xi_{i}^{\mathbf{b}}\left(t\right)-\xi_{i}^{\mathbf{b}}\left(t'\right)\right| & \leq K\left|t-t'\right|.
\end{align*}
\end{assumption}

\begin{assumption}[\textit{Forward pass assumptions}]
\label{enu:Assump_forward} $\phi_{1}$ satisfies: 
\[
\left|\phi_{1}\left(w,x\right)-\phi_{1}\left(w',x\right)\right|\leq K\left|w-w'\right|,
\]
for all $w,w'\in\mathbb{W}_{1}$ and for ${\cal P}$-almost every
$x$. For $i=2,...,L$, $\phi_{i}$ satisfies: 
\begin{align*}
\left|\phi_{i}\left(w,b,h\right)\right| & \leq K\left(1+\left|w\right|+\left|b\right|\right),\\
\left|\phi_{i}\left(w,b,h\right)-\phi_{i}\left(w',b',h'\right)\right| & \leq K\left(1+\left|w\right|+\left|w'\right|+\left|b\right|+\left|b'\right|\right)\left|h-h'\right|\\
 & \quad+K\left(\left|w-w'\right|+\left|b-b'\right|\right),
\end{align*}
for all $w,w'\in\mathbb{W}_{i}$, $b,b'\in\mathbb{B}_{i}$, and $h,h'\in\mathbb{H}_{i-1}$.
Finally $\phi_{L+1}$ satisfies: 
\[
\left|\phi_{L+1}\left(h\right)-\phi_{L+1}\left(h'\right)\right|\leq K\left|h-h'\right|,
\]
for all $h,h'\in\mathbb{H}_{L}$.
\end{assumption}

\begin{assumption}[\textit{Backward pass assumptions}]
\label{enu:Assump_backward} $\sigma_{1}^{\mathbf{w}}$ satisfies:
\begin{align*}
\left|\sigma_{1}^{\mathbf{w}}\left(\Delta,w,x\right)\right| & \leq K\left(1+\left|\Delta\right|\right),\\
\left|\sigma_{1}^{\mathbf{w}}\left(\Delta,w,x\right)-\sigma_{1}^{\mathbf{w}}\left(\Delta',w',x\right)\right| & \leq K\left(\left|\Delta-\Delta'\right|+\left|w-w'\right|\right),
\end{align*}
for all $w,w'\in\mathbb{W}_{1}$, $\Delta,\Delta'\in\hat{\mathbb{H}}_{1}$
and for ${\cal P}$-almost every $x$. For $i=2,...,L$, $\sigma_{i}^{\mathbf{w}}$
and $\sigma_{i}^{\mathbf{b}}$ satisfy the following growth bounds:
\[
\max\left(\left|\sigma_{i}^{\mathbf{w}}\left(\Delta,w,b,g,h\right)\right|,\;\left|\sigma_{i}^{\mathbf{b}}\left(\Delta,w,b,g,h\right)\right|\right)\leq K\left(1+\left|\Delta\right|\right),
\]
as well as the following perturbation bounds: 
\begin{align*}
 & \max\Big(\left|\sigma_{i}^{\mathbf{w}}\left(\Delta,w,b,g,h\right)-\sigma_{i}^{\mathbf{w}}\left(\Delta',w',b',g',h'\right)\right|,\\
 & \qquad\left|\sigma_{i}^{\mathbf{b}}\left(\Delta,w,b,g,h\right)-\sigma_{i}^{\mathbf{b}}\left(\Delta',w',b',g',h'\right)\right|\Big)\\
 & \leq K\left(1+\left|\Delta\right|+\left|\Delta'\right|\right)\left|h-h'\right|+K\left(\left|\Delta-\Delta'\right|+\left|w-w'\right|+\left|b-b'\right|+\left|g-g'\right|\right).
\end{align*}
For $i=2,...,L$, $\sigma_{i-1}^{\mathbf{H}}$ satisfies the growth
bound: 
\[
\left|\sigma_{i-1}^{\mathbf{H}}\left(\Delta,w,b,g,h\right)\right|\leq K\left(1+\left|\Delta\right|\right)\left(1+\left|w\right|+\left|b\right|\right),
\]
and the perturbation bound: 
\begin{align*}
 & \left|\sigma_{i-1}^{\mathbf{H}}\left(\Delta,w,b,g,h\right)-\sigma_{i-1}^{\mathbf{H}}\left(\Delta',w',b',g',h'\right)\right|\\
 & \leq K\left(1+\left|w\right|+\left|w'\right|+\left|b\right|+\left|b'\right|\right)\left|\Delta-\Delta'\right|\\
 & \quad+K\left(1+\left|\Delta\right|+\left|\Delta'\right|\right)\left(\left|w-w'\right|+\left|b-b'\right|\right)\\
 & \quad+K\left(1+\left|\Delta\right|+\left|\Delta'\right|\right)\left(1+\left|w\right|+\left|w'\right|+\left|b\right|+\left|b'\right|\right)\left(\left|g-g'\right|+\left|h-h'\right|\right).
\end{align*}
Finally $\sigma_{L}^{\mathbf{H}}$ satisfies: 
\[
\left|\sigma_{L}^{\mathbf{H}}\left(y,\hat{y},h\right)\right|\leq K,\qquad\left|\sigma_{L}^{\mathbf{H}}\left(y,\hat{y},h\right)-\sigma_{L}^{\mathbf{H}}\left(y,\hat{y}',h'\right)\right|\leq K\left(\left|h-h'\right|+\left|\hat{y}-\hat{y}'\right|\right),
\]
for ${\cal P}$-almost every $y$.
\end{assumption}

\begin{rem}
We remark that these assumptions can be relaxed, e.g. $\phi_{i}\left(w,b,h\right)$
may be allowed to grow super-linearly with the variables, at the expense
of suitable additional assumptions\footnote{Indeed this has been done in our previous iterate of the paper, posted
on arXiv.}. Here we pay attention to a simpler setting, which covers neural
network setups of interest that are relevant to Sections \ref{sec:global_convergence_iid}
and \ref{sec:global_convergence_general}.
\end{rem}

We also equip the neural network and its MF limit with several norms.
In particular, we define for the neural network parameters: 
\begin{align*}
\interleave{\bf w}_{i}\interleave_{t} & =\bigg(\frac{1}{n_{i-1}n_{i}}\sum_{j_{i-1}=1}^{n_{i-1}}\sum_{j_{i}=1}^{n_{i}}\sup_{s\leq t}\left|{\bf w}_{i}\left(\left\lfloor s/\epsilon\right\rfloor ,j_{i-1},j_{i}\right)\right|^{50}\bigg)^{1/50},\\
\interleave{\bf b}_{i}\interleave_{t} & =\bigg(\frac{1}{n_{i}}\sum_{j_{i}=1}^{n_{i}}\sup_{s\leq t}\left|{\bf b}_{i}\left(\left\lfloor s/\epsilon\right\rfloor ,j_{i}\right)\right|^{50}\bigg)^{1/50},\qquad i=2,...,L,\\
\interleave{\bf w}_{1}\interleave_{t} & =\Big(\frac{1}{n_{1}}\sum_{j_{1}=1}^{n_{1}}\sup_{s\leq t}\left|\mathbf{w}_{1}\left(\left\lfloor s/\epsilon\right\rfloor ,j_{1}\right)\right|^{50}\Big)^{1/50}.
\end{align*}
We also introduce the notation:
\[
\interleave\mathbf{W}\interleave_{t}=\max\left(\max_{1\leq i\leq L}\interleave{\bf w}_{i}\interleave_{t},\;\max_{2\leq i\leq L}\interleave{\bf b}_{i}\interleave_{t}\right).
\]
We also have similarly for the MF limit: 
\begin{align*}
\interleave w_{i}\interleave_{t} & =\mathbb{E}\left[\sup_{s\leq t}\left|w_{i}\left(s,C_{i-1},C_{i}\right)\right|^{50}\right]^{1/50},\\
\interleave b_{i}\interleave_{t} & =\mathbb{E}\left[\sup_{s\leq t}\left|b_{i}\left(s,C_{i}\right)\right|^{50}\right]^{1/50},\qquad i=2,...,L,\\
\interleave w_{1}\interleave_{t} & =\mathbb{E}\left[\sup_{s\le t}\left|w_{1}\left(s,C_{1}\right)\right|^{50}\right]^{1/50},
\end{align*}
as well as
\[
\interleave W\interleave_{t}=\max\left(\max_{1\leq i\leq L}\interleave w_{i}\interleave_{t},\;\max_{2\leq i\leq L}\interleave b_{i}\interleave_{t}\right).
\]
For convenience, let us define the quantities:
\begin{align*}
\mathsf{max}_{t}^{w}\left(W\right) & =\max_{2\leq i\leq L}\sup_{s\leq t}\left|w_{i}\left(s,C_{i-1},C_{i}\right)\right|,\\
\mathsf{max}_{t}^{b}\left(W\right) & =\max_{2\leq i\leq L}\sup_{s\leq t}\left|b_{i}\left(s,C_{i}\right)\right|,
\end{align*}
which are random variables. Note that $\mathsf{max}_{t}^{w}\left(W\right)$
does not involve $w_{1}$.

For a set of MF parameters $W$, we define 
\begin{align*}
\left\Vert W\right\Vert _{t} & =\max\left(\max_{1\leq i\leq L}\left\Vert w_{i}\right\Vert _{t},\;\max_{2\leq i\leq L}\left\Vert b_{i}\right\Vert _{t}\right),\\
\left\Vert w_{i}\right\Vert _{t} & =\mathbb{E}\left[\sup_{s\leq t}\left|w_{i}\left(s,C_{i-1},C_{i}\right)\right|^{2}\right]^{1/2},\\
\left\Vert b_{i}\right\Vert _{t} & =\mathbb{E}\left[\sup_{s\leq t}\left|b_{i}\left(s,C_{i}\right)\right|^{2}\right]^{1/2},\qquad i=2,...,L,\\
\left\Vert w_{1}\right\Vert _{t} & =\mathbb{E}\left[\sup_{s\leq t}\left|w_{1}\left(s,C_{1}\right)\right|^{2}\right]^{1/2}.
\end{align*}
Note that this defines a norm on the space of MF parameters. As such,
we can define the following distance for two sets of MF parameters
$W$ and $W'$:
\begin{align}
\left\Vert W-W'\right\Vert _{t} & =\max\left(\max_{1\leq i\leq L}\left\Vert w_{i}-w_{i}'\right\Vert _{t},\;\max_{2\leq i\leq L}\left\Vert b_{i}-b_{i}'\right\Vert _{t}\right),\label{eq:def_MF_distance}\\
\left\Vert w_{i}-w_{i}'\right\Vert _{t} & =\mathbb{E}\left[\sup_{s\leq t}\left|w_{i}\left(s,C_{i-1},C_{i}\right)-w_{i}'\left(s,C_{i-1},C_{i}\right)\right|^{2}\right]^{1/2},\nonumber \\
\left\Vert b_{i}-b_{i}'\right\Vert _{t} & =\mathbb{E}\left[\sup_{s\leq t}\left|b_{i}\left(s,C_{i}\right)-b_{i}'\left(s,C_{i}\right)\right|^{2}\right]^{1/2},\qquad i=2,...,L,\nonumber \\
\left\Vert w_{1}-w_{1}'\right\Vert _{t} & =\mathbb{E}\left[\sup_{s\leq t}\left|w_{1}\left(s,C_{1}\right)-w_{1}'\left(s,C_{1}\right)\right|^{2}\right]^{1/2}.\nonumber 
\end{align}

\section{Existence and Uniqueness of the Solution of the MF ODEs\label{sec:Existence-MF}}

We study the well-posedness of the solution of the MF ODEs introduced
in Section \ref{subsec:MF}. For this purpose specifically, we consider
the following sub-Gaussian norm for $w_{i}$, $i\geq2$:
\[
\left\llbracket w_{i}\right\rrbracket _{\psi,t}=\sqrt{50}\sup_{m\geq1}\frac{1}{\sqrt{m}}\mathbb{E}\left[\sup_{s\leq t}\left|w_{i}\left(s,C_{i-1},C_{i}\right)\right|^{m}\right]^{1/m},\qquad i=2,...,L,
\]
and accordingly define
\[
\left\llbracket W\right\rrbracket _{\psi,t}=\max\left(\max_{2\leq i\leq L}\left\llbracket w_{i}\right\rrbracket _{\psi,t},\;\max_{2\leq i\leq L}\interleave b_{i}\interleave_{t},\;\interleave w_{1}\interleave_{t}\right).
\]
The factor $\sqrt{50}$ is for the convenience that $\left\llbracket w_{i}\right\rrbracket _{\psi,t}\geq\interleave w_{i}\interleave_{t}$
and hence $\left\llbracket W\right\rrbracket _{\psi,t}\geq\interleave W\interleave_{t}$.

Denote by $\frak{W}_{T}$ the space of MF parameters $W$ such that
$\|W\|_{T}<\infty$. Given a terminal time $T\geq0$ and an initialization
$W\left(0\right)$, we define the mapping $F$ that associates $W'\in\mathfrak{W}_{T}$
with 
\begin{align*}
F\left(W'\right)\left(t,c_{1},...,c_{L}\right) & =\Big\{ F_{1}^{w}\left(W'\right)\left(t,c_{1}\right),\;F_{2}^{w}\left(W'\right)\left(t,c_{1},c_{2}\right),\;F_{2}^{b}\left(W'\right)\left(t,c_{2}\right),\\
 & \qquad...,F_{L}^{w}\left(W'\right)\left(t,c_{L-1},c_{L}\right),\;F_{L}^{b}\left(W'\right)\left(t,c_{L}\right)\Big\},
\end{align*}
in which
\begin{align*}
F_{1}^{w}\left(W'\right)\left(t,c_{1}\right) & =w_{1}\left(0,c_{1}\right)-\int_{0}^{t}\xi_{1}^{\mathbf{w}}\left(s\right)\mathbb{E}_{Z}\left[\Delta_{1}^{w}\left(Z,c_{1};W'\left(s\right)\right)\right]ds,\\
F_{i}^{w}\left(W'\right)\left(t,c_{i-1},c_{i}\right) & =w_{i}\left(0,c_{i-1},c_{i}\right)-\int_{0}^{t}\xi_{i}^{\mathbf{w}}\left(s\right)\mathbb{E}_{Z}\left[\Delta_{i}^{w}\left(Z,c_{i-1},c_{i};W'\left(s\right)\right)\right]ds,\\
F_{i}^{b}\left(W'\right)\left(t,c_{i}\right) & =b_{i}\left(0,c_{i}\right)-\int_{0}^{t}\xi_{i}^{\mathbf{b}}\left(s\right)\mathbb{E}_{Z}\left[\Delta_{i}^{b}\left(Z,c_{i};W'\left(s\right)\right)\right]ds,\qquad i=2,...,L.
\end{align*}
Observe that at initialization $F\left(W'\right)\left(0,\cdot,...,\cdot\right)=W\left(0\right)$,
whereas the quantities in the above time integrals are computed w.r.t.
$W'$. In the following, when referring to a solution $W$ to the
MF ODEs on $[0,T]$, we mean an element of $\mathfrak{W}_{T}$ satisfying
$F(W)=W$. We say that $W$ is a solution to the MF ODEs on $t\in[0,\infty)$
if its restriction to $[0,T]$ is a solution to the MF ODEs on $[0,T]$
for all $T>0$. 
\begin{thm}
\label{thm:existence ODE}Assume that the initialization $W\left(0\right)$
of the MF ODEs satisfies $\left\llbracket W\right\rrbracket _{\psi,0}\leq K$.
Then under Assumptions \ref{enu:Assump_lrSchedule}-\ref{enu:Assump_backward},
there exists a unique solution to the MF ODEs on $t\in[0,\infty)$.
\end{thm}

The rest of this section is devoted to the proof of this theorem.
To prove the theorem, we first collect a useful a priori estimate.
\begin{lem}
\label{lem:bounds MF a priori}Under Assumptions \ref{enu:Assump_lrSchedule}
and \ref{enu:Assump_backward}, given an initialization $W\left(0\right)$,
a solution $W$ to the MF ODEs, if exists, must satisfy that for any
$t\in[0,\infty)$, 
\[
\interleave W\interleave_{t},\quad\max_{1\leq i\leq L}\mathbb{E}\left[\sup_{s\leq t}\underset{Z\sim{\cal P}}{{\rm ess\text{-}sup}}\left|\Delta_{i}^{H}\left(Z,C_{i};W\left(s\right)\right)\right|^{50}\right]^{1/50}\leq K^{\kappa_{L}}\left(1+t^{\kappa_{L}}\right)\left(1+\interleave W\interleave_{0}^{\kappa_{L}}\right),
\]
where $\kappa_{L}=K^{L}$ for some constant $K>1$ sufficiently large.

A similar result holds for $\left\llbracket \cdot\right\rrbracket _{\psi,t}$
norm. Under Assumptions \ref{enu:Assump_lrSchedule} and \ref{enu:Assump_backward},
given an initialization $W\left(0\right)$, for any $t\in[0,\infty)$,
there exists $K_{0}\left(t\right)\geq1$ of the form
\[
K_{0}\left(t\right)=K^{\kappa_{L}}\left(1+t^{\kappa_{L}}\right)\left(1+\left\llbracket W\right\rrbracket _{\psi,0}^{\kappa_{L}}\right),
\]
where $\kappa_{L}=K^{L}$ for some constant $K>1$ sufficiently large,
such that the following holds. A solution $W$ to the MF ODEs, if
exists, must satisfy that for any $t\in[0,\infty)$, $\interleave W\interleave_{t}\leq\left\llbracket W\right\rrbracket _{\psi,t}\leq K_{0}\left(t\right)$.
Furthermore, by assuming $\left\llbracket W\right\rrbracket _{\psi,0}<\infty$,
for any $B\geq0$,
\[
\mathbb{P}\left(\mathsf{max}_{t}^{w}\left(W\right)\geq K_{0}\left(t\right)B\right)\leq2Le^{1-K_{1}B^{2}},
\]
for some universal constant $K_{1}>0$.
\end{lem}

Recall the bounds in Lemma \ref{lem:bounds MF a priori} are given
by $K_{0}\left(t\right)$, which is a function of the initialization
$W\left(0\right)$ and non-decreasing with $t$. These a priori bounds
lead us to consider the following spaces, given an initialization
$W\left(0\right)$ and an arbitrary terminal time $T>0$:
\begin{itemize}
\item The space ${\cal W}_{T}$ of MF parameters $W'=\left\{ W'\left(t\right)\right\} _{t\leq T}=\left\{ w_{1}'\left(t,\cdot\right),w_{i}'\left(t,\cdot,\cdot\right),b_{i}'\left(t,\cdot\right),\;\;i=2,...,L\right\} _{t\leq T}$
such that
\[
\interleave W'\interleave_{T}\leq K_{0}\left(T\right).
\]
\item The space ${\cal W}_{T}^{0}\subset{\cal W}_{T}$ of MF parameters
$W'\in{\cal W}_{T}$ such that
\begin{align*}
\left\llbracket W'\right\rrbracket _{\psi,T} & \leq K_{0}\left(T\right),\\
\mathbb{P}\left(\mathsf{max}_{T}^{w}\left(W'\right)\geq K_{0}\left(T\right)B\right) & \leq2Le^{1-K_{1}B^{2}}\quad\forall B\geq0,
\end{align*}
and $W'\left(0\right)=W\left(0\right)$ (and hence every elements
$W'$ in ${\cal W}_{T}^{0}$ share the same initialization $W\left(0\right)$).
It is easy to see that ${\cal W}_{T}^{0}\subset{\cal W}_{T}$ is valid
since $\interleave W'\interleave_{T}\leq\left\llbracket W'\right\rrbracket _{\psi,T}$.
\end{itemize}
We equip these spaces with the metric $\left(W',W''\right)\mapsto\left\Vert W'-W''\right\Vert _{T}$.
By Lemma \ref{lem:bounds MF a priori}, we know that any solution
$W$ to the MF ODEs, if exists, must belong to ${\cal W}_{T}^{0}$.

The proof of Theorem \ref{thm:existence ODE} follows from a Picard-type
iteration. It is easy to see that a solution to the MF ODEs is a fixed
point of $F$ and vice versa. Also note that by the same argument
of Lemma \ref{lem:bounds MF a priori}, one can prove the following:
\begin{lem}
\label{lem:a priori mapping F}Under Assumptions \ref{enu:Assump_lrSchedule}
and \ref{enu:Assump_backward}, for any $W'\in{\cal W}_{T}^{0}$,
$F\left(W'\right)\in{\cal W}_{T}^{0}$.
\end{lem}

We have the following key result:
\begin{lem}
\label{lem:difference MF}For a given $B\geq0$, consider two collections
of MF parameters $W',W''\in{\cal W}_{T}$ such that
\begin{align*}
\mathbb{P}\left(\mathsf{max}_{T}^{w}\left(W'\right)\geq K_{0}\left(T\right)B\right) & \leq2Le^{1-K_{1}B^{2}},\\
\mathbb{P}\left(\mathsf{max}_{T}^{w}\left(W''\right)\geq K_{0}\left(T\right)B\right) & \leq2Le^{1-K_{1}B^{2}}.
\end{align*}
Under Assumptions \ref{enu:Assump_lrSchedule}-\ref{enu:Assump_backward},
for any $t\leq T$,
\[
\left\Vert F\left(W'\right)-F\left(W''\right)\right\Vert _{t}\leq\left(KK_{0}\left(T\right)\right)^{2L+2}\int_{0}^{t}\left(\left(1+B\right)\left\Vert W'-W''\right\Vert _{s}+\sqrt{L}e^{-K_{1}B^{2}/2}\right)ds.
\]
\end{lem}

We are now ready to prove Theorem \ref{thm:existence ODE}.
\begin{proof}[Proof of Theorem \ref{thm:existence ODE}]
We perform a Picard-type iteration argument. Consider an arbitrary
finite $T\geq0$. Consider $W',W''\in{\cal W}_{T}^{0}$. From Lemma
\ref{lem:difference MF}:
\begin{align*}
\left\Vert F\left(W'\right)-F\left(W''\right)\right\Vert _{t} & \leq\left(KK_{0}\left(T\right)\right)^{2L+2}\left(\left(1+B\right)\int_{0}^{t}\left\Vert W'-W''\right\Vert _{s}ds+T\sqrt{L}e^{-K_{1}B^{2}/2}\right)\\
 & \equiv k_{1}\left(1+B\right)\int_{0}^{t}\left\Vert W'-W''\right\Vert _{s}ds+k_{2}e^{-k_{3}B^{2}},
\end{align*}
for any $B>0$. By Lemma \ref{lem:a priori mapping F}, $F$ maps
${\cal W}_{T}^{0}$ to ${\cal W}_{T}^{0}$. As such, we can iterate
this inequality to obtain:
\begin{align*}
 & \left\Vert F^{\left(m\right)}\left(W'\right)-F^{\left(m\right)}\left(W''\right)\right\Vert _{T}\\
 & \leq k_{1}\left(1+B\right)\int_{0}^{T}\left\Vert F^{\left(m-1\right)}\left(W'\right)-F^{\left(m-1\right)}\left(W''\right)\right\Vert _{T_{2}}dT_{2}+k_{2}e^{-k_{3}B^{2}}\\
 & \leq k_{1}^{2}\left(1+B\right)^{2}\int_{0}^{T}\int_{0}^{T_{2}}\left\Vert F^{\left(m-2\right)}\left(W'\right)-F^{\left(m-2\right)}\left(W''\right)\right\Vert _{T_{3}}\mathbb{I}\left(T_{2}\leq T\right)dT_{3}dT_{2}\\
 & \qquad+k_{2}\sum_{\ell=1}^{2}\frac{\left(Tk_{1}k_{2}\left(1+B\right)\right)^{\ell-1}}{\ell!}e^{-k_{3}B^{2}}\\
 & \ldots\\
 & \leq k_{1}^{m}\left(1+B\right)^{m}\int_{0}^{T}\int_{0}^{T_{2}}...\int_{0}^{T_{m}}\left\Vert W'-W''\right\Vert _{T_{m+1}}\mathbb{I}\left(T_{m}\leq...\leq T_{2}\leq T\right)dT_{m+1}...dT_{2}\\
 & \qquad+k_{2}\sum_{\ell=1}^{m}\frac{\left(Tk_{1}k_{2}\left(1+B\right)\right)^{\ell-1}}{\ell!}e^{-k_{3}B^{2}}\\
 & \leq\frac{1}{m!}T^{m}k_{1}^{m}\left(1+B\right)^{m}\left\Vert W'-W''\right\Vert _{T}+k_{2}e^{Tk_{1}k_{2}\left(1+B\right)-k_{3}B^{2}}\\
 & \leq\frac{1}{m!}T^{m}k_{1}^{m}\left(1+\sqrt{m}\right)^{m}\left\Vert W'-W''\right\Vert _{T}+k_{2}e^{Tk_{1}k_{2}\left(1+\sqrt{m}\right)-k_{3}m},
\end{align*}
where we choose $B=\sqrt{m}$ in the last display. Note that since
$\interleave W\interleave_{0}<\infty$, $K_{0}\left(T\right)$ and
hence $k_{1},k_{2}$ are finite for finite $T$. By substituting $W''=F\left(W'\right)$,
we obtain:
\[
\sum_{m=1}^{\infty}\left\Vert F^{\left(m+1\right)}\left(W'\right)-F^{\left(m\right)}\left(W'\right)\right\Vert _{T}=\sum_{m=1}^{\infty}\left\Vert F^{\left(m\right)}\left(W''\right)-F^{\left(m\right)}\left(W'\right)\right\Vert _{T}<\infty.
\]
Hence as $m\to\infty$, $F^{\left(m\right)}\left(W'\right)$ converges
in $\|\cdot\|_{T}$ to a limit $W\in\frak{W}_{T}$, which is a fixed
point of $F$. By Lemma \ref{lem:bounds MF a priori}, $W$ belongs
to ${\cal W}_{T}^{0}$. 

The uniqueness of the fixed point comes from the above estimate, since
if $W'$ and $W''$ are fixed points of $F$ then they are both in
${\cal W}_{T}^{0}$, and 
\begin{align*}
\text{\ensuremath{\left\Vert W'-W''\right\Vert }}_{T} & =\left\Vert F^{\left(m\right)}\left(W'\right)-F^{\left(m\right)}\left(W''\right)\right\Vert _{T}\\
 & \leq\frac{1}{m!}T^{m}k_{1}^{m}\left(1+\sqrt{m}\right)^{m}\left\Vert W'-W''\right\Vert _{T}+k_{2}e^{Tk_{1}k_{2}\left(1+\sqrt{m}\right)-k_{3}m},
\end{align*}
and one can take $m$ arbitrarily large. This proves that the solution
exists and is unique on $t\in\left[0,T\right]$. Since $T$ is arbitrary,
we have existence and uniqueness of the solution to the MF ODEs on
the time interval $[0,\infty)$.
\end{proof}
The proofs of the lemmas are in Appendix \ref{sec:Remaining-proofs-existence-MF}.

\section{Main Result: Connection between Neural Network and MF Limit\label{sec:Main-result}}

\subsection{Neuronal Embedding and the Coupling Procedure\label{subsec:Neuronal-Embedding}}

\paragraph*{Neuronal embedding.}

To formalize a connection between the neural network and its MF limit,
we consider their initializations. In practical scenarios, to set
the initial parameters $\mathbf{W}\left(0\right)$ of the neural network,
one typically randomizes $\mathbf{W}\left(0\right)$ according to
some distributional law $\rho$. We note that since the neural network
is defined w.r.t. a set of finite integers ${\bf n}=\left\{ n_{1},...,n_{L}\right\} $
that represents its size, so is $\rho$. In the context of infinite-width
limits of neural networks, we would like to accommodate a sequence
of neural networks of diverging sizes ${\bf n}$ (where $n_{1},...,n_{L-1}\to\infty$
and $n_{L}=1$). As such, it is useful to also consider a family $\mathsf{Init}$
of initialization laws, each of which is indexed by the set of finite
integers ${\bf n}=\left\{ n_{1},...,n_{L}\right\} $ (with $n_{L}=1$):
\begin{align*}
\mathsf{Init} & =\{\rho:\;\rho\text{ is the initialization law of a neural network of size {\bf n}=\ensuremath{\left\{  n_{1},...,n_{L}\right\} } },\\
 & \qquad n_{1},...,n_{L}\in\mathbb{N}_{>0},\;n_{L}=1\}.
\end{align*}
We make the following crucial definitions.
\begin{defn}[Unit neuronal embedding]
\label{def:neuronal_embedding-unit}Given an initialization law $\rho$
of a neural network of size ${\bf n}=\left\{ n_{1},\dots,n_{L}\right\} $
(where $n_{L}=1$), we call $\left(\Omega,P,\left\{ w_{i}^{0}\right\} _{i\in\left[L\right]},\left\{ b_{i}^{0}\right\} _{2\leq i\leq L}\right)$
a \textit{unit neuronal embedding} for this neural network if there
exists a sampling rule $\overline{P}_{{\bf n}}=\prod_{i=1}^{L}\overline{P}_{n_{i}}$
such that the following hold:

\begin{enumerate}
\item $\left(\Omega,P\right)=\prod_{i=1}^{L}\left(\Omega_{i},P_{i}\right)$
a product space and $\Omega_{L}=\{1\}$. We recall that $\left(\Omega,P\right)$
is called a neuronal ensemble.
\item $\overline{P}_{n_{i}}$ is a distribution over $\Omega_{i}^{n_{i}}$
whose marginals are given by $P_{i}$. Note it is not necessary that
$\overline{P}_{n_{i}}$ is factored as a product of $P_{i}$'s.
\item The deterministic functions $w_{1}^{0}:\;\Omega_{1}\to\mathbb{W}_{1}$,
$w_{i}^{0}:\;\Omega_{i-1}\times\Omega_{i}\to\mathbb{W}_{i}$ and $b_{i}^{0}:\;\Omega_{i}\to\mathbb{B}_{i}$,
$2\leq i\leq L$ are such that if --- with an abuse of notations
--- we sample $\left(C_{i}\left(j_{i}\right)\right)_{i\in[L],j_{i}\in\left[n_{i}\right]}\sim\overline{P}_{{\bf n}}$,
then 
\[
{\rm Law}\left(w_{1}^{0}\left(C_{1}\left(j_{1}\right)\right),\;w_{i}^{0}\left(C_{i-1}\text{\ensuremath{\left(j_{i-1}\right)}},C_{i}\left(j_{i}\right)\right),\;b_{i}^{0}\left(C_{i}\text{\ensuremath{\left(j_{i}\right)}}\right),\;\;j_{1}\in\left[n_{1}\right],\;j_{i}\in\left[n_{i}\right],\;i=2,...,L\right)=\rho.
\]
\end{enumerate}
\end{defn}

\begin{defn}[Neuronal embedding]
\label{def:neuronal_embedding}Given a family of initialization laws
$\mathsf{Init}$, we call $\left(\Omega,P,\left\{ w_{i}^{0}\right\} _{i\in\left[L\right]},\left\{ b_{i}^{0}\right\} _{2\leq i\leq L}\right)$
a \textit{neuronal embedding }for $\mathsf{Init}$ if it is a unit
neuronal embedding for any law $\rho$ in $\mathsf{Init}$.
\end{defn}

On one hand, we concern chiefly with the notion of a neuronal embedding,
which carries the idea of infinite-width limits. On the other hand,
the unit neuronal embedding -- as a standalone notion -- is useful
when one is to obtain a quantitative (finite-width) result, such as
Theorem \ref{thm:gradient descent coupling} and Corollary \ref{cor:gradient descent quality}
below. Note also that if the family $\mathsf{Init}$ contains only
one initialization law, then a unit neuronal embedding for this law
is obviously a neuronal embedding for $\mathsf{Init}$. We shall thus
routinely refer to a unit neuronal embedding as a neuronal embedding,
whenever there is no risk of confusion.

\paragraph*{$\eta$-independence.}

An important aspect of the neuronal embedding is the sampling rule
$\overline{P}_{{\bf n}}$. The product structure $\overline{P}_{{\bf n}}=\prod_{i=1}^{L}\overline{P}_{n_{i}}$
implies layer-wise independence. At each layer $i\in\left[L\right]$,
a canonical example of a sampling rule is one in which the samples
are i.i.d., i.e. $\overline{P}_{n_{i}}=P_{i}\times...\times P_{i}$
($n_{i}$-time product). In fact, we shall require a weaker condition,
given in the following.
\begin{defn}[$\eta$-independence]
We say that $(X_{1},\dots,X_{n})\in\Omega_{0}^{n}$ are \textit{$\eta$-independent}
if for all $1$-bounded functions $f$ that maps from $\Omega_{0}$
to a separable Hilbert space, for any $i\in\left[n\right]$, almost
surely,
\[
\left|\mathbb{E}\left[f(X_{i})\middle|\left\{ X_{i'},\;i'<i\right\} \right]-\mathbb{E}\left[f(X_{i})\right]\right|\le\eta.
\]
\end{defn}

\begin{assumption}[$\bar{\eta}$-independence for neuronal embedding]
\label{assump:neuronal-embedding}Let $\bar{\eta}=\left(\eta_{1},...,\eta_{L-1}\right)$
where $\eta_{i}=n_{i}^{-0.501}$. For the neuronal embedding in Definition
\ref{def:neuronal_embedding} (or Definition \ref{def:neuronal_embedding-unit}),
for each index $\mathbf{n}$ in the family $\mathsf{Init}$, the sampling
rule $\overline{P}_{{\bf n}}$ satisfies that $\left(C_{i}\left(j_{i}\right)\right)_{j_{i}\in\left[n_{i}\right]}\sim\overline{P}_{n_{i}}$
are $\eta_{i-1}$-independent for all $i\in\left[L-1\right]$. In
this case, we say the neuronal embedding satisfies $\bar{\eta}$-independence.
\end{assumption}

It is easy to see that in the canonical example where $\overline{P}_{n_{i}}=P_{i}\times...\times P_{i}$
for all $i\in\left[L-1\right]$ and all indices $\mathbf{n}$ from
$\mathsf{Init}$, the above assumption is trivially satisfied; that
is, any $n$ independent random variables are $n^{-c}$-independent
with $c=\infty$.

\begin{rem}
\label{rem:Init_family}When $\mathsf{Init}$ contains more than one
law, if a neuronal embedding exists, then $\mathsf{Init}$ must satisfy
a certain consistency property. For instance, under the canonical
example where $\overline{P}_{n_{i}}=P_{i}\times...\times P_{i}$ for
all $i\in\left[L-1\right]$ and all indices $\mathbf{n}$ from $\mathsf{Init}$,
if a neuronal embedding with this sampling rule exists, then the following
must hold. Suppose that $\rho$ indexed by $\left\{ n_{1},...,n_{L}\right\} $
and $\rho'$ indexed by $\left\{ n_{1}',...,n_{L}'\right\} $ are
elements of $\mathsf{Init}$ such that $n_{1}\leq n_{1}'$, ..., $n_{L-1}\leq n_{L-1}'$,
and suppose that
\[
{\rm Law}\left(\mathbf{w}_{1}'\left(0,j_{1}\right),\mathbf{w}_{i}'\left(0,j_{i-1},j_{i}\right),\mathbf{b}_{i}'\left(0,j_{i}\right):\;j_{i}\in\left[n_{i}'\right],\;i=1,...,L\right)=\rho'.
\]
Then we must have that
\[
{\rm Law}\left(\mathbf{w}_{1}'\left(0,j_{1}\right),\mathbf{w}_{i}'\left(0,j_{i-1},j_{i}\right),\mathbf{b}_{i}'\left(0,j_{i}\right):\;j_{i}\in S_{i},\;i=1,...,L\right)=\rho,
\]
for any collection of $L$ sets $S_{i}$, $i=1,...,L$, where each
$S_{i}$ is a subset of $\left[n_{i}'\right]$ with size $\left|S_{i}\right|=n_{i}$.
\end{rem}

\paragraph*{Coupling procedure.}

To proceed, we perform the following \textit{coupling procedure}:
\begin{enumerate}
\item Given a family of initialization laws $\mathsf{Init}$, let $\left(\Omega,P,\left\{ w_{i}^{0}\right\} _{i\in\left[L\right]},\left\{ b_{i}^{0}\right\} _{2\leq i\leq L}\right)$
be a neuronal embedding of $\mathsf{Init}$.
\item We form the MF ODEs' initialization $W\left(0\right)$ by setting
$w_{1}\left(0,\cdot\right)=w_{1}^{0}\left(\cdot\right)$, $w_{i}\left(0,\cdot,\cdot\right)=w_{i}^{0}\left(\cdot,\cdot\right)$
and $b_{i}\left(0,\cdot\right)=b_{i}^{0}\left(\cdot\right)$ for $2\leq i\leq L$.
With this initialization, we obtain the MF limit's trajectory $W\left(t\right)$,
for $t\in\mathbb{R}_{\geq0}$, according to the neuronal ensemble
$\left(\Omega,P\right)$.
\item Given $\mathbf{n}=\left\{ n_{1},...,n_{L}\right\} $, we find a sampling
rule $\overline{P}_{\mathbf{n}}=\prod_{i=1}^{L}\overline{P}_{n_{i}}$.
For each $i\in\left[L\right]$, we sample $\left(C_{i}\left(j_{1}\right),\dots,C_{i}\left(j_{n_{i}}\right)\right)\sim\overline{P}_{n_{i}}$.
We then form the neural network initialization $\mathbf{W}\left(0\right)$
by setting $\mathbf{w}_{1}\left(0,j_{1}\right)=w_{1}^{0}\left(C_{1}\left(j_{1}\right)\right)$,
$\mathbf{w}_{i}\left(0,j_{i-1},j_{i}\right)=w_{i}^{0}\left(C_{i-1}\left(j_{i-1}\right),C_{i}\left(j_{i}\right)\right)$
and $\mathbf{b}_{i}\left(0,j_{i}\right)=b_{i}^{0}\left(C_{i}\left(j_{i}\right)\right)$
for $j_{1}\in\left[n_{1}\right]$, $j_{i}\in\left[n_{i}\right]$,
$2\leq i\leq L$. With this initialization, we obtain the neural network's
trajectory $\mathbf{W}\left(k\right)$ for $k\in\mathbb{N}_{\geq0}$,
with the data $z\left(k\right)$ being generated independently of
$C_{i}\left(j_{i}\right)$'s and hence $\mathbf{W}\left(0\right)$.
\end{enumerate}
Hence we see that the connection is formalized on the basis of the
initialization, and in particular, the neuronal ensemble $\left(\Omega,P\right)$.
Note that $W\left(t\right)$ is a deterministic trajectory for $t\in\mathbb{R}_{\geq0}$
and is independent of $\left\{ n_{1},...,n_{L}\right\} $, whereas
$\mathbf{W}\left(k\right)$ is random for all $k\in\mathbb{N}_{\geq0}$
due to the randomness of $C_{i}\left(j_{i}\right)$ and the generation
of the training data $z\left(k\right)$. We define a measure of closeness
between $\mathbf{W}\left(\left\lfloor t/\epsilon\right\rfloor \right)$
and $W\left(t\right)$ for the whole interval $t\in\left[0,T\right]$:
\begin{align}
\mathscr{D}_{T}\left(W,\mathbf{W}\right)=\max\Bigg( & \max_{2\leq i\leq L}\bigg(\frac{1}{n_{i-1}n_{i}}\sum_{j_{i-1}=1}^{n_{i-1}}\sum_{j_{i}=1}^{n_{i}}\sup_{t\le T}\left|{\bf w}_{i}\left(\left\lfloor t/\epsilon\right\rfloor ,j_{i-1},j_{i}\right)-w_{i}\left(t,C_{i-1}\left(j_{i-1}\right),C_{i}\left(j_{i}\right)\right)\right|^{2}\bigg)^{1/2},\nonumber \\
 & \max_{2\leq i\leq L}\bigg(\frac{1}{n_{i}}\sum_{j_{i}=1}^{n_{i}}\sup_{t\le T}\left|{\bf b}_{i}\left(\left\lfloor t/\epsilon\right\rfloor ,j_{i}\right)-b_{i}\left(t,C_{i}\left(j_{i}\right)\right)\right|^{2}\bigg)^{1/2},\nonumber \\
 & \bigg(\frac{1}{n_{1}}\sum_{j_{1}=1}^{n_{1}}\sup_{t\le T}\left|{\bf w}_{1}\left(\left\lfloor t/\epsilon\right\rfloor ,j_{1}\right)-w_{1}\left(t,C_{1}\left(j_{1}\right)\right)\right|^{2}\bigg)^{1/2}\Bigg).\label{eq:dist_W}
\end{align}
Note that by definition, $\mathscr{D}_{T}\left(W,\mathbf{W}\right)$
is a random quantity due to the randomness of $\left\{ C_{i}\left(j_{i}\right)\right\} _{i\in\left[L\right]}$
and $\left\{ \mathbf{W}\left(\left\lfloor t/\epsilon\right\rfloor \right)\right\} _{t\in\left[0,T\right]}$.

The idea of the coupling procedure is closely related to the ``propagation
of chaos'' argument \cite{sznitman1991topics}. Here, instead of
playing the role of a proof technique, the coupling serves as a vehicle
to establish the connection between the neural network's trajectory
and the MF trajectory on the basis of the neuronal embedding.

\subsection{Main Theorem\label{subsec:Main-Theorem}}

Let us consider an assumption on the initialization:
\begin{assumption}[Initialization]
\label{assump:init}The functions $w_{i}^{0}$ and $b_{i}^{0}$ of
the neuronal embedding satisfy the conditions:
\begin{align*}
\max_{2\leq i\leq L}\sup_{m\geq1}\frac{1}{\sqrt{m}}\mathbb{E}\left[\left|w_{i}^{0}\left(C_{i-1},C_{i}\right)\right|^{m}\right]^{1/m} & \leq K,\\
\max_{2\leq i\leq L}\sup_{m\geq1}\frac{1}{\sqrt{m}}\mathbb{E}\left[\left|b_{i}^{0}\left(C_{i}\right)\right|^{m}\right]^{1/m} & \leq K,\\
\sup_{m\geq1}\frac{1}{\sqrt{m}}\mathbb{E}\left[\left|w_{1}^{0}\left(C_{1}\right)\right|^{m}\right]^{1/m} & \leq K.
\end{align*}
As such, following the coupling procedure, the initialization $W\left(0\right)$
of the MF ODEs satisfies $\interleave W\interleave_{0}\leq K<\infty$.
\end{assumption}

We are now ready to state the main theorem.
\begin{thm}
\label{thm:gradient descent coupling}Given a family $\mathsf{Init}$
of initialization laws and a tuple of positive integers $\left\{ n_{1},...,n_{L}\right\} $
with $n_{L}=1$, perform the coupling procedure as described in Section
\ref{subsec:Neuronal-Embedding}. Under Assumptions \ref{enu:Assump_lrSchedule}-\ref{enu:Assump_backward},
\ref{assump:init} and \ref{assump:neuronal-embedding}, there exist
constants $c_{1}\in\left(0,0.5\right)$ and $c_{2}\in\left(0,1/52\right)$,
such that for any $\delta>0$, $L\geq1$ and $T\in\epsilon\mathbb{N}_{\geq0}$,
the following holds. There exist $n^{*}=n^{*}\left(T,L,c_{1},c_{2}\right)\geq1$
and $\epsilon^{*}=\epsilon^{*}\left(T,L,c_{1},c_{2}\right)\leq1$
such that for any $n_{\min}\geq n^{*}$ and $\epsilon\in\left(0,\epsilon^{*}\right)$,
\begin{align*}
\mathbb{P}\bigg(\mathscr{D}_{T}\left(W,\mathbf{W}\right) & \geq K\left(n_{\min}^{-c_{1}}+\epsilon^{c_{1}}\right)\sqrt{\log\left(\frac{1}{\delta}n_{\max}^{2}+e\right)}\bigg)\leq2\delta+KLn_{\max}\exp\left(-Kn_{\min}^{c_{2}}\right).
\end{align*}
Here $n_{\min}=\min_{1\leq j\leq L-1}n_{j}$ and $n_{\max}=\max_{1\leq j\leq L}n_{j}$.
\end{thm}

Roughly speaking, with $n_{i}=\Theta\left(n\right)$ for $i\in\left[L-1\right]$
and $\epsilon\ll1/\log\left(n\right)$, we have $\mathbf{W}\left(\left\lfloor t/\epsilon\right\rfloor \right)\approx W\left(t\right)$
for all $t\in\left[0,T\right]$ and large $n$. We note that the exponents
$c_{1}$ and $c_{2}$ are independent of the terminal time $T$ and
the number of layers $L$. It is an interesting task to derive explicit
constant values for $c_{1}$ and $c_{2}$, which we have not done
given the complex dependency of these exponents on other hidden constants
in our current analysis.
\begin{rem}
Under the stronger assumption of boundedness of the initial weight
distributions at all except the first layer, in our work's previous
preprint, we show that a similar result to Theorem \ref{thm:gradient descent coupling}
holds with $c_{1}=0.5$. There an even stronger result is achieved,
in which we define $\mathscr{D}_{T}\left(W,\mathbf{W}\right)$ via
$L^{\infty}$ distance, instead of $L^{2}$ distance as done in Eq.
(\ref{eq:dist_W}). 
\end{rem}

The theorem gives a connection between $\mathbf{W}\left(\left\lfloor t/\epsilon\right\rfloor \right)$,
which involves finitely many neurons, and the MF limit $W\left(t\right)$,
whose description is independent of the number of neurons. It lends
a way to extract properties of the neural network in the many-neurons
limit.
\begin{cor}
\label{cor:gradient descent quality}Consider any test function $\psi:\mathbb{H}_{i}\to\mathbb{S}$
which is $K$-Lipschitz and $K$-bounded, i.e.
\[
\left|\psi\left(h\right)-\psi\left(h'\right)\right|\leq K\left|h-h'\right|,\qquad\left|\psi\left(h\right)\right|\leq K,
\]
where $\mathbb{S}$ is a separable Hilbert space. Under the same setting
as Theorem \ref{thm:gradient descent coupling}, for any $\delta>0$,
we have with probability at least $1-3\delta-KLn_{\max}\exp\left(-Kn_{\min}^{c_{2}}\right)$,
\[
\sup_{t\le T}\left|\frac{1}{n_{i}}\sum_{j_{i}=1}^{n_{i}}\mathbb{E}_{Z}\left[\psi\left({\bf H}_{i}\left(\left\lfloor t/\epsilon\right\rfloor ,X,j_{i}\right)\right)\right]-\mathbb{E}_{Z}\mathbb{E}_{C_{i}}\left[\psi\left(H_{i}\left(t,X,C_{i}\right)\right)\right]\right|=\tilde{O}\left(n_{\min}^{-c_{1}}+\epsilon^{c_{1}}\right),
\]
where $\tilde{O}$ hides the dependency on $T$, $L$ and $\delta$
as well as the logarithmic factors $\log n_{\max}$ and $\log\left(1/\epsilon\right)$.
Furthermore, for any test function $\psi:\mathbb{Y}\times\hat{\mathbb{Y}}\to\mathbb{S}$
which is $K$-Lipschitz in the second variable, uniformly in the first
variable,
\[
\sup_{t\leq T}\left|\mathbb{E}_{Z}\left[\psi\left(Y,\hat{\mathbf{y}}\left(\left\lfloor t/\epsilon\right\rfloor ,X\right)\right)\right]-\mathbb{E}_{Z}\left[\psi\left(Y,\hat{y}\left(t,X\right)\right)\right]\right|=\tilde{O}\left(n_{\min}^{-c_{1}}+\epsilon^{c_{1}}\right),
\]
with probability at least $1-2\delta-KLn_{\max}\exp\left(-Kn_{\min}^{c_{2}}\right)$.
\end{cor}

As per Remark \ref{rem:Init_family}, we note that the statements
in Theorem \ref{thm:gradient descent coupling} and Corollary \ref{cor:gradient descent quality}
have explicit quantitative dependence on the hidden widths $n_{i}$,
and hence one may consider $\mathsf{Init}$ that contains only one
initialization law.

We observe that while the MF trajectory $W\left(t\right)$ is defined
as per the choice of the neuronal embedding $\left(\Omega,P,\left\{ w_{i}^{0}\right\} _{i\in\left[L\right]},\left\{ b_{i}^{0}\right\} _{2\leq i\leq L}\right)$,
which may not be unique. On the other hand, the neural network's trajectory
$\mathbf{W}\left(t\right)$ depends on the randomization of the initial
parameters $\mathbf{W}\left(0\right)$ according to an initialization
law from the family $\mathsf{Init}$ (as well as the data $z\left(t\right)$)
and hence is independent of this choice. Another corollary of Theorem
\ref{thm:gradient descent coupling} is that given the same family
$\mathsf{Init}$, the MF trajectory is insensitive to the choice of
the neuronal embedding of $\mathsf{Init}$.
\begin{cor}
\label{cor:MF_insensitivity}Consider a family $\mathsf{Init}$ of
initialization laws, such that it contains a sequence of indices $\left\{ \left\{ n_{1}\left(n\right),...,n_{L}\left(n\right)\right\} :\;n\in\mathbb{N}\right\} $
in which $n_{\min}\left(n\right)\to\infty$ and $n_{\min}^{-c}\left(n\right)\log n_{\max}\left(n\right)\to0$
as $n\to\infty$ for any $c>0$, with $n_{\min}\left(n\right)=\min_{1\leq i\le L-1}n_{i}\left(n\right)$
and $n_{\max}\left(n\right)=\max_{1\leq i\le L-1}n_{i}\left(n\right)$.

Let $W\left(t\right)$ and $\hat{W}\left(t\right)$ be two MF trajectories
associated with two choices of neuronal embeddings of $\mathsf{Init}$,
$\left(\Omega,P,\left\{ w_{i}^{0}\right\} _{i\in\left[L\right]},\left\{ b_{i}^{0}\right\} _{2\leq i\leq L}\right)$
and $\left(\hat{\Omega},\hat{P},\left\{ \hat{w}_{i}^{0}\right\} _{i\in\left[L\right]},\left\{ \hat{b}_{i}^{0}\right\} _{2\leq i\leq L}\right)$
respectively. Suppose that both neuronal embeddings satisfy Assumptions
\ref{assump:neuronal-embedding} and \ref{assump:init}. Let us also
assume Assumptions \ref{enu:Assump_lrSchedule}-\ref{enu:Assump_backward}.

For any $T\in\mathbb{R}_{\geq0}$ and any set of positive integers
$\left\{ n_{1},...,n_{L}\right\} $ with $n_{L}=1$, if we independently
sample $U_{i}\left(j_{i}\right)\sim P_{i}$ and $\hat{U}_{i}\left(j_{i}\right)\sim\hat{P}_{i}$
for $j_{i}\in\left[n_{i}\right]$ and $i\in\left[L\right]$, then
${\rm Law}\left({\cal W}\left(n_{1},...,n_{L},T\right)\right)={\rm Law}(\hat{{\cal W}}\left(n_{1},...,n_{L},T\right))$,
where ${\cal W}\left(n_{1},...,n_{L},T\right)$ denotes the following
collection on $W\left(t\right)$:
\begin{align*}
{\cal W}\left(n_{1},...,n_{L},T\right) & =\Big\{ w_{1}\left(t,U_{1}\left(j_{1}\right)\right),\;w_{i}\left(t,U_{i-1}\left(j_{i-1}\right),U_{i}\left(j_{i}\right)\right),\;b_{i}\left(t,U_{i}\left(j_{i}\right)\right):\\
 & \qquad j_{i}\in\left[n_{i}\right],\;i\in\left[L\right],\;t\in\left[0,T\right]\Big\},
\end{align*}
and $\hat{{\cal W}}\left(n_{1},...,n_{L},T\right)$ denotes a similar
collection on $\hat{W}\left(t\right)$.
\end{cor}

In the case $L=2$, by looking at the induced distribution of $\left(w_{1}\left(t,C_{1}\right),w_{2}\left(t,C_{1},1\right)\right)$
over $C_{1}\sim P_{1}$, we immediately recover the distributional
equation in \cite{mei2018mean} describing the MF limit.
\begin{cor}
\label{cor:MF-twolayers}Assume the same setting as Theorem \ref{thm:gradient descent coupling},
and let us consider $L=2$. For simplicity, let us disregard the bias
of the second layer by considering $\xi_{2}^{\mathbf{b}}\left(\cdot\right)=0$
and $b_{2}\left(0,\cdot\right)=0$. Assume $\mathbb{W}_{1}=\mathbb{R}^{d_{1}}$,
$\mathbb{W}_{2}=\mathbb{R}^{d_{2}}$ for some integers $d_{1},d_{2}>0$.
Let $\rho_{t}$ denote the law of $\left(w_{1}\left(t,C_{1}\right),w_{2}\left(t,C_{1},1\right)\right)$
over $C_{1}\sim P_{1}$. Then $\rho_{t}$ satisfies the following
distributional partial differential equation in the weak sense:
\[
\partial_{t}\rho_{t}\left(u_{1},u_{2}\right)={\rm div}\left[\rho_{t}\left(u_{1},u_{2}\right)G\left(u_{1},u_{2};\rho_{t}\right)\right],
\]
in which
\[
G\left(u_{1},u_{2};\rho_{t}\right)=\left[\begin{array}{c}
\xi_{1}^{\mathbf{w}}\left(t\right)\mathbb{E}_{Z}\left[\underline{\Delta}_{1}^{w}\left(u_{1},u_{2};Z,\rho_{t}\right)\right]\\
\xi_{2}^{\mathbf{w}}\left(t\right)\mathbb{E}_{Z}\left[\underline{\Delta}_{2}^{w}\left(u_{1},u_{2};Z,\rho_{t}\right)\right]
\end{array}\right],
\]
and we define
\begin{align*}
\underline{H}_{2}\left(x;\rho_{t}\right) & =\int\phi_{2}\left(u_{2},0,\phi_{1}\left(u_{1},x\right)\right)d\rho_{t}\left(u_{1},u_{2}\right),\\
\hat{\underline{y}}\left(x;\rho_{t}\right) & =\phi_{3}\left(H_{2}\left(x;\rho_{t}\right)\right),\\
\underline{\Delta}_{2}^{H}\left(z;\rho_{t}\right) & =\sigma_{2}^{\mathbf{H}}\left(y,\hat{\underline{y}}\left(x;\rho_{t}\right),\underline{H}_{2}\left(x;\rho_{t}\right)\right),\\
\underline{\Delta}_{2}^{w}\left(u_{1},u_{2};z,\rho_{t}\right) & =\sigma_{2}^{\mathbf{w}}\left(\underline{\Delta}_{2}^{H}\left(z;\rho_{t}\right),u_{2},0,\underline{H}_{2}\left(x;\rho_{t}\right),\phi_{1}\left(u_{1},x\right)\right),\\
\underline{\Delta}_{1}^{H}\left(u_{1},u_{2};z,\rho_{t}\right) & =\sigma_{1}^{\mathbf{H}}\left(\underline{\Delta}_{2}^{H}\left(z;\rho_{t}\right),u_{2},0,\underline{H}_{2}\left(x;\rho_{t}\right),\phi_{1}\left(u_{1},x\right)\right),\\
\underline{\Delta}_{1}^{w}\left(u_{1},u_{2};z,\rho_{t}\right) & =\sigma_{1}^{\mathbf{w}}\left(\underline{\Delta}_{1}^{H}\left(u_{1},u_{2};z,\rho_{t}\right),u_{1},x\right).
\end{align*}
In particular, for any $\delta>0$ and any $K$-Lipschitz and $K$-bounded
test function $\psi:\mathbb{W}_{1}\times\mathbb{W}_{2}\to\mathbb{S}$,
where $\mathbb{S}$ is a separable Hilbert space,
\[
\sup_{t\le T}\left|\frac{1}{n_{1}}\sum_{j_{1}=1}^{n_{1}}\psi\left(\mathbf{w}_{1}\left(\left\lfloor t/\epsilon\right\rfloor ,j_{1}\right),\mathbf{w}_{2}\left(\left\lfloor t/\epsilon\right\rfloor ,j_{1},1\right)\right)-\int\psi\left(u_{1},u_{2}\right)d\rho_{t}\left(u_{1},u_{2}\right)\right|=\tilde{O}\left(n_{1}^{-c_{1}}+\epsilon^{c_{1}}\right).
\]
with probability at least $1-3\delta-Kn_{1}\exp\left(-Kn_{1}^{c_{2}}\right)$,
where $\tilde{O}$ hides the dependency on $T$ and $\delta$ as well
as the logarithmic factors $\log n_{1}$ and $\log\left(1/\epsilon\right)$.
Similarly, for any test function $\psi:\mathbb{Y}\times\hat{\mathbb{Y}}\to\mathbb{S}$
which is $K$-Lipschitz in the second variable, uniformly in the first
variable,
\[
\sup_{t\le T}\left|\mathbb{E}_{Z}\left[\psi\left(Y,\hat{{\bf y}}\left(\left\lfloor t/\epsilon\right\rfloor ,X\right)\right)\right]-\mathbb{E}_{Z}\left[\psi\left(Y,\hat{\underline{y}}\left(X;\rho_{t}\right)\right)\right]\right|=\tilde{O}\left(n_{1}^{-c_{1}}+\epsilon^{c_{1}}\right),
\]
with probability at least $1-3\delta-Kn_{1}\exp\left(-Kn_{1}^{c_{2}}\right)$.
\end{cor}

\subsection{Proof of Theorem \ref{thm:gradient descent coupling}}

We construct an auxiliary trajectory, which we call the \textit{particle
ODEs}: 
\begin{align*}
\frac{\partial}{\partial t}\tilde{w}_{1}\left(t,j_{1}\right) & =-\xi_{1}^{\mathbf{w}}\left(t\right)\mathbb{E}_{Z}\left[\Delta_{1}^{\mathbf{w}}\left(Z,j_{1};\tilde{W}\left(t\right)\right)\right],\qquad\forall j_{1}\in\left[n_{1}\right],\\
\frac{\partial}{\partial t}\tilde{w}_{i}\left(t,j_{i-1},j_{i}\right) & =-\xi_{i}^{\mathbf{w}}\left(t\right)\mathbb{E}_{Z}\left[\Delta_{i}^{\mathbf{w}}\left(Z,j_{i-1},j_{i};\tilde{W}\left(t\right)\right)\right],\\
\frac{\partial}{\partial t}\tilde{b}_{i}\left(t,j_{i}\right) & =-\xi_{i}^{\mathbf{b}}\left(t\right)\mathbb{E}_{Z}\left[\Delta_{i}^{\mathbf{b}}\left(Z,j_{i};\tilde{W}\left(t\right)\right)\right],\qquad\forall j_{i-1}\in\left[n_{i-1}\right],\;j_{i}\in\left[n_{i}\right],\;i=2,...,L,
\end{align*}
in which $\tilde{W}\left(t\right)=\left\{ \tilde{w}_{1}\left(t,\cdot\right),\tilde{w}_{i}\left(t,\cdot,\cdot\right),\tilde{b}_{i}\left(t,\cdot\right),\;\;i=2,...,L\right\} $,
and $t\in\mathbb{R}_{\geq0}$. We specify the initialization $\tilde{W}\left(0\right)$:
$\tilde{w}_{1}\left(0,j_{1}\right)=w_{1}^{0}\left(C_{1}\left(j_{1}\right)\right)$,
$\tilde{w}_{i}\left(0,j_{i-1},j_{i}\right)=w_{i}^{0}\left(C_{i-1}\left(j_{i-1}\right),C_{i}\left(j_{i}\right)\right)$
and $\tilde{b}_{i}\left(0,j_{i}\right)=b_{i}^{0}\left(C_{i}\left(j_{i}\right)\right)$.
That is, it shares the same initialization with the neural network
one $\mathbf{W}\left(0\right)$, and hence is coupled with the neural
network and the MF ODEs. Roughly speaking, the particle ODEs are continuous-time
trajectories of finitely many neurons, averaged over the data distribution.
We note that $\tilde{W}\left(t\right)$ is random for all $t\in\mathbb{R}_{\geq0}$
due to the randomness of $C_{i}\left(j_{i}\right)$'s.

The existence and uniqueness of the solution to the particle ODEs
follows from the same proof as in Theorem \ref{thm:existence ODE},
which we shall not repeat here\footnote{On a more technical note, we can view the particle ODEs as a new system
of MF ODEs whose neuronal ensemble $\left(\Omega_{{\rm new}},P_{{\rm new}}\right)=\prod_{i=1}^{L}\left(\Omega_{i,{\rm new}},P_{i,{\rm new}}\right)$
takes the following specific form: $\Omega_{i,{\rm new}}=\left\{ C_{i}\left(1\right),...,C_{i}\left(n_{i}\right)\right\} $
and $P_{i,{\rm new}}$ is a uniform probability measure on $\Omega_{i,{\rm new}}$.
In light of this view, the existence and uniqueness of the solution
to the particle ODEs follows from Theorem \ref{thm:existence ODE}.}. We equip $\tilde{W}\left(t\right)$ with the norms:
\begin{align*}
\interleave\tilde{w}_{i}\interleave_{t} & =\bigg(\frac{1}{n_{i-1}n_{i}}\sum_{j_{i-1}=1}^{n_{i-1}}\sum_{j_{i}=1}^{n_{i}}\sup_{s\leq t}\left|\tilde{w}_{i}\left(s,j_{i-1},j_{i}\right)\right|^{50}\bigg)^{1/50},\\
\interleave\tilde{b}_{i}\interleave_{t} & =\bigg(\frac{1}{n_{i}}\sum_{j_{i}=1}^{n_{i}}\sup_{s\leq t}\left|\tilde{b}_{i}\left(s,j_{i}\right)\right|^{50}\bigg)^{1/50},\qquad i=2,...,L,\\
\interleave\tilde{w}_{1}\interleave_{t} & =\Big(\frac{1}{n_{1}}\sum_{j_{1}=1}^{n_{1}}\sup_{s\leq t}\left|\tilde{w}_{1}\left(s,j_{1}\right)\right|^{50}\Big)^{1/50}.
\end{align*}
as well as
\[
\interleave\tilde{W}\interleave_{t}=\max\left(\max_{1\leq i\leq L}\interleave\tilde{w}_{i}\interleave_{t},\;\max_{2\leq i\leq L}\interleave\tilde{b}_{i}\interleave_{t}\right).
\]
One can also define the measures $\mathscr{D}_{T}\left(W,\tilde{W}\right)$
and $\mathscr{D}_{T}\left(\tilde{W},\mathbf{W}\right)$ similar to
Eq. (\ref{eq:dist_W}): 
\begin{align*}
\mathscr{D}_{T}\left(W,\tilde{W}\right)=\max\Bigg( & \max_{2\leq i\leq L}\bigg(\frac{1}{n_{i-1}n_{i}}\sum_{j_{i-1}=1}^{n_{i-1}}\sum_{j_{i}=1}^{n_{i}}\left|w_{i}\left(t,C_{i-1}\left(j_{i-1}\right),C_{i}\left(j_{i}\right)\right)-\tilde{w}_{i}\left(t,j_{i-1},j_{i}\right)\right|^{2}\bigg)^{1/2},\\
 & \max_{2\leq i\leq L}\bigg(\frac{1}{n_{i}}\sum_{j_{i}=1}^{n_{i}}\sup_{t\le T}\left|b_{i}\left(t,C_{i}\left(j_{i}\right)\right)-\tilde{b}_{i}\left(t,j_{i}\right)\right|^{2}\bigg)^{1/2},\\
 & \bigg(\frac{1}{n_{1}}\sum_{j_{1}=1}^{n_{1}}\sup_{t\le T}\left|w_{1}\left(t,C_{1}\left(j_{1}\right)\right)-\tilde{w}_{1}\left(t,j_{1}\right)\right|^{2}\bigg)^{1/2}\Bigg),\\
\mathscr{D}_{T}\left(\tilde{W},\mathbf{W}\right)=\max\Bigg( & \max_{2\leq i\leq L}\bigg(\frac{1}{n_{i-1}n_{i}}\sum_{j_{i-1}=1}^{n_{i-1}}\sum_{j_{i}=1}^{n_{i}}\left|{\bf w}_{i}\left(\left\lfloor t/\epsilon\right\rfloor ,j_{i-1},j_{i}\right)-\tilde{w}_{i}\left(t,j_{i-1},j_{i}\right)\right|^{2}\bigg)^{1/2},\\
 & \max_{2\leq i\leq L}\bigg(\frac{1}{n_{i}}\sum_{j_{i}=1}^{n_{i}}\sup_{t\le T}\left|{\bf b}_{i}\left(\left\lfloor t/\epsilon\right\rfloor ,j_{i}\right)-\tilde{b}_{i}\left(t,j_{i}\right)\right|^{2}\bigg)^{1/2},\\
 & \bigg(\frac{1}{n_{1}}\sum_{j_{1}=1}^{n_{1}}\sup_{t\le T}\left|{\bf w}_{1}\left(\left\lfloor t/\epsilon\right\rfloor ,j_{1}\right)-\tilde{w}_{1}\left(t,j_{1}\right)\right|^{2}\bigg)^{1/2}\Bigg).
\end{align*}
We have the following results: 
\begin{thm}
\label{thm:particle coupling}Under the same setting as Theorem \ref{thm:gradient descent coupling},
there exist constants $c_{1}\in\left(0,0.5\right)$ and $c_{2}\in\left(0,1/52\right)$,
such that for any $\delta>0$, $L\geq1$ and $T\geq1$, the following
holds. There exists $n^{*}=n^{*}\left(T,L,c_{1},c_{2}\right)\geq1$
such that for any $n_{\min}\geq n^{*}$,
\begin{align*}
\mathbb{P}\bigg(\mathscr{D}_{T}\left(W,\tilde{W}\right) & \geq Kn_{\min}^{-c_{1}}\sqrt{\log\left(\frac{1}{\delta}n_{\max}^{2}+e\right)}\bigg)\leq\delta+KLn_{\max}\exp\left(-Kn_{\min}^{c_{2}}\right).
\end{align*}
Here $n_{\min}=\min_{1\leq j\leq L-1}n_{j}$ and $n_{\max}=\max_{1\leq j\leq L}n_{j}$.
\end{thm}

\begin{thm}
\label{thm:gradient descent}Under the same setting as Theorem \ref{thm:gradient descent coupling},
there exist constants $c_{1}\in\left(0,0.5\right)$ and $c_{2}\in\left(0,1/52\right)$,
such that for any $\delta>0$, $L\geq1$ and $T\geq1$, the following
holds. There exists $\epsilon^{*}=\epsilon^{*}\left(T,L,c_{1},c_{2}\right)\leq1$
such that for any $\epsilon\in\left(0,\epsilon^{*}\right)$,
\begin{align*}
\mathbb{P}\bigg(\mathscr{D}_{T}\left(\tilde{W},\mathbf{W}\right) & \geq K\epsilon^{c_{1}}\sqrt{\log\left(\frac{1}{\delta}n_{\max}^{2}+e\right)}\bigg)\leq\delta+KLn_{\max}\exp\left(-Kn_{\min}^{c_{2}}\right).
\end{align*}
Here $n_{\min}=\min_{1\leq j\leq L-1}n_{j}$ and $n_{\max}=\max_{1\leq j\leq L}n_{j}$.
\end{thm}

\begin{proof}[Proof of Theorem \ref{thm:gradient descent coupling}]
Using the fact 
\[
\mathscr{D}_{T}\left(W,\mathbf{W}\right)\leq\mathscr{D}_{T}\left(W,\tilde{W}\right)+\mathscr{D}_{T}\left(\tilde{W},\mathbf{W}\right),
\]
the thesis is immediate from Theorems \ref{thm:particle coupling}
and \ref{thm:gradient descent}. 
\end{proof}

\subsection{Proof of Theorems \ref{thm:particle coupling} and \ref{thm:gradient descent}}

The proof of Theorem \ref{thm:particle coupling} rests in the following
proposition, which is essentially a version of Theorem \ref{thm:particle coupling}
with an extra boundedness condition at initialization.
\begin{prop}
\label{prop:particle coupling - bounded}Under the same setting as
Theorem \ref{thm:gradient descent coupling}, for a given $B>0$,
further assume that
\begin{align*}
{\rm ess\text{-}sup}\mathsf{max}_{0}^{w}\left(W\right) & ={\rm ess\text{-}sup}\max_{2\leq i\leq L}\left|w_{i}^{0}\left(C_{i-1},C_{i}\right)\right|\leq B,\\
{\rm ess\text{-}sup}\mathsf{max}_{0}^{b}\left(W\right) & ={\rm ess\text{-}sup}\max_{2\leq i\leq L}\left|b_{i}^{0}\left(C_{i}\right)\right|\leq B.
\end{align*}
Then for any $\delta>0$, with probability at least $1-\delta-KLn_{\max}\exp\left(-Kn_{\min}^{1/52}\right)$,
\[
\mathscr{D}_{T}\left(W,\tilde{W}\right)\leq\sqrt{\frac{1}{n_{\min}}\log\left(\frac{2TL}{\delta}n_{\max}^{2}+e\right)}\exp\left(K^{\bar{K}}\left(1+T^{\bar{K}}\right)\left(1+B\right)\right),
\]
in which $n_{\min}=\min_{1\leq j\leq L-1}n_{j}$, $n_{\max}=\max_{1\leq j\leq L}n_{j}$,
and $\bar{K}$ is a constant that depends on $L$ such that $\bar{K}\leq K^{L}$
for some sufficiently large constant $K$.
\end{prop}

Similar to Proposition \ref{prop:particle coupling - bounded}, the
following proposition is essentially a version of Theorem \ref{thm:gradient descent}
with an extra boundedness condition at initialization.
\begin{prop}
\label{prop:gradient descent - bounded}Under the same setting as
Theorem \ref{thm:gradient descent coupling}, for a given $B>0$,
further assume that
\begin{align*}
{\rm ess\text{-}sup}\mathsf{max}_{0}^{w}\left(W\right) & ={\rm ess\text{-}sup}\max_{2\leq i\leq L}\left|w_{i}^{0}\left(C_{i-1},C_{i}\right)\right|\leq B,\\
{\rm ess\text{-}sup}\mathsf{max}_{0}^{b}\left(W\right) & ={\rm ess\text{-}sup}\max_{2\leq i\leq L}\left|b_{i}^{0}\left(C_{i}\right)\right|\leq B.
\end{align*}
Then for any $\delta>0$ and $\epsilon<1$, with probability at least
$1-\delta-KLn_{\max}\exp\left(-Kn_{\min}^{1/52}\right)$,
\[
\mathscr{D}_{T}\left(\tilde{W},\mathbf{W}\right)\leq\sqrt{\epsilon\log\left(\frac{2L}{\delta}n_{\max}^{2}+e\right)}\exp\left(K^{\bar{K}}\left(1+T^{\bar{K}}\right)\left(1+B\right)\right),
\]
in which $n_{\min}=\min_{1\leq j\leq L-1}n_{j}$, $n_{\max}=\max_{1\leq j\leq L}n_{j}$,
and $\bar{K}$ is a constant that depends on $L$ such that $\bar{K}\leq K^{L}$
for some sufficiently large constant $K$.
\end{prop}

The following proposition bridges the last two propositions with their
respective theorems.
\begin{prop}
\label{prop:truncation}Assume the same setting as Theorem \ref{thm:gradient descent coupling}.
Let $\underline{W}\left(t\right)=\left\{ \underline{w}_{1}\left(t,\cdot\right),\underline{w}_{i}\left(t,\cdot,\cdot\right),\underline{b}_{i}\left(t,\cdot\right),\;\;i=2,...,L\right\} $
be the MF ODEs' solution for which its initialization $\underline{W}\left(0\right)$
is a truncated version of $W\left(0\right)$, for a given $B>0$:
\begin{align*}
\underline{w}_{1}\left(0,c_{1}\right) & =w_{1}^{0}\left(c_{1}\right),\\
\underline{w}_{i}\left(0,c_{i-1},c_{i}\right) & ={\rm Trunc}_{B}\left(w_{i}^{0}\left(c_{i-1},c_{i}\right)\right),\\
\underline{b}_{i}\left(0,c_{i}\right) & ={\rm Trunc}_{B}\left(b_{i}^{0}\left(c_{i}\right)\right),
\end{align*}
for $2\leq i\leq L$, where ${\rm Trunc}_{B}\left(u\right)=u\mathbb{I}\left(\left|u\right|\leq B\right)+B{\rm sign}\left(u\right)\mathbb{I}\left(\left|u\right|>B\right)$.
Then:
\[
\left\Vert W-\underline{W}\right\Vert _{T}\leq K\exp\left(-KB^{2}+K^{\bar{K}}\left(1+T^{\bar{K}}\right)\left(1+B\right)\right),
\]
for $\bar{K}$ a constant that depends on $L$ such that $\bar{K}\leq K^{L}$
for some sufficiently large constant $K$. Similarly let $\underline{\tilde{W}}$
and $\underline{\mathbf{W}}$ be the particle ODEs' solution and the
neural network's dynamics with a similarly truncated initialization:
\begin{align*}
\underline{\tilde{w}}_{1}\left(0,j_{1}\right)=\underline{\mathbf{w}}_{1}\left(0,j_{1}\right) & =w_{1}^{0}\left(C_{1}\left(j_{1}\right)\right),\\
\underline{\tilde{w}}_{i}\left(0,j_{i-1},j_{i}\right)=\underline{\mathbf{w}}_{i}\left(0,j_{i-1},j_{i}\right) & ={\rm Trunc}_{B}\left(w_{i}^{0}\left(C_{i-1}\left(j_{i-1}\right),C_{i}\left(j_{i}\right)\right)\right),\\
\underline{\tilde{b}}_{i}\left(0,j_{i}\right)=\underline{\mathbf{b}}_{i}\left(0,j_{i}\right) & ={\rm Trunc}_{B}\left(b_{i}^{0}\left(C_{i}\left(j_{i}\right)\right)\right).
\end{align*}
Then with probability at least $1-KLn_{\max}\exp\left(-Ke^{-KB^{2}}n_{\min}^{1/52}\right)$,
\[
\left\Vert \tilde{W}-\underline{\tilde{W}}\right\Vert _{T},\;\left\Vert \mathbf{W}-\underline{\mathbf{W}}\right\Vert _{T}\leq K\exp\left(-KB^{2}+K^{\bar{K}}\left(1+T^{\bar{K}}\right)\left(1+B\right)\right).
\]
Here $n_{\max}=\max\left(n_{1},...,n_{L}\right)$, $n_{\min}=\min\left(n_{1},...,n_{L-1}\right)$,
\begin{align*}
\left\Vert {\bf W}-\underline{\mathbf{W}}\right\Vert _{t} & =\max\left(\max_{2\leq i\leq L}\left\Vert {\bf w}_{i}-\underline{\mathbf{w}}_{i}\right\Vert _{t},\;\max_{2\leq i\leq L}\left\Vert {\bf b}_{i}-\underline{\mathbf{b}}_{i}\right\Vert _{t},\;\left\Vert {\bf w}_{1}-\underline{\mathbf{w}}_{1}\right\Vert _{t}\right),\\
\left\Vert {\bf w}_{i}-\underline{\mathbf{w}}_{i}\right\Vert _{t} & =\bigg(\frac{1}{n_{i-1}n_{i}}\sum_{j_{i-1}=1}^{n_{i-1}}\sum_{j_{i}=1}^{n_{i}}\sup_{s\leq t}\left|\mathbf{w}_{i}\left(\left\lfloor s/\epsilon\right\rfloor ,j_{i-1},j_{i}\right)-\underline{\mathbf{w}}_{i}\left(\left\lfloor s/\epsilon\right\rfloor ,j_{i-1},j_{i}\right)\right|^{2}\bigg)^{1/2},\\
\left\Vert {\bf b}_{i}-\underline{{\bf b}}_{i}\right\Vert _{t} & =\bigg(\frac{1}{n_{i}}\sum_{j_{i}=1}^{n_{i}}\sup_{s\leq t}\left|{\bf b}_{i}\left(\left\lfloor s/\epsilon\right\rfloor ,j_{i}\right)-\underline{{\bf b}}_{i}\left(\left\lfloor s/\epsilon\right\rfloor ,j_{i}\right)\right|^{2}\bigg)^{1/2},\qquad i=2,...,L,\\
\left\Vert {\bf w}_{1}-\underline{\mathbf{w}}_{1}\right\Vert _{t} & =\bigg(\frac{1}{n_{1}}\sum_{j_{1}=1}^{n_{1}}\sup_{s\leq t}\left|{\bf w}_{1}\left(\left\lfloor s/\epsilon\right\rfloor ,j_{1}\right)-\underline{\mathbf{w}}_{1}\left(\left\lfloor s/\epsilon\right\rfloor ,j_{1}\right)\right|^{2}\bigg)^{1/2},
\end{align*}
and $\left\Vert \tilde{W}-\underline{\tilde{W}}\right\Vert _{t}$
is defined similarly.
\end{prop}

We can now prove Theorems \ref{thm:particle coupling} and \ref{thm:gradient descent}.
\begin{proof}[Proof of Theorem \ref{thm:particle coupling}]
Let $K_{T}=K^{\bar{K}}\left(1+T^{\bar{K}}\right)$. For a given $B>0$,
let $\underline{W}$ and $\underline{\tilde{W}}$ be the initialization-truncated
versions of $W$ and $\tilde{W}$ respectively, as per Proposition
\ref{prop:truncation}. Then Proposition \ref{prop:particle coupling - bounded}
states that for any $\delta>0$, with probability at least $1-\delta-KLn_{\max}\exp\left(-Kn_{\min}^{1/52}\right)$,
\[
\mathscr{D}_{T}\left(\underline{W},\underline{\tilde{W}}\right)\leq\sqrt{\frac{1}{n_{\min}}\log\left(\frac{2TL}{\delta}n_{\max}^{2}+e\right)}e^{K_{T}\left(1+B\right)}.
\]
Proposition \ref{prop:truncation} further gives that with probability
at least $1-KLn_{\max}\exp\left(-Ke^{-KB^{2}}n_{\min}^{1/52}\right)$,
\[
\left\Vert \tilde{W}-\underline{\tilde{W}}\right\Vert _{T},\;\left\Vert W-\underline{W}\right\Vert _{T}\leq Ke^{-KB^{2}+K_{T}\left(1+B\right)}.
\]
Also notice that 
\[
\mathscr{D}_{T}\left(W,\tilde{W}\right)\leq\mathscr{D}_{T}\left(\underline{W},\underline{\tilde{W}}\right)+\left\Vert W-\underline{W}\right\Vert _{T}+\left\Vert \tilde{W}-\underline{\tilde{W}}\right\Vert _{T}.
\]
As such,
\[
\mathscr{D}_{T}\left(W,\tilde{W}\right)\leq\bigg(\sqrt{\frac{1}{n_{\min}}\log\left(\frac{2TL}{\delta}n_{\max}^{2}+e\right)}+e^{-KB^{2}}\bigg)e^{K_{T}\left(1+B\right)},
\]
with probability at least $1-\delta-KLn_{\max}\exp\left(-Ke^{-KB^{2}}n_{\min}^{1/52}\right)$,
for any fixed $B>0$. Then upon choosing $B=c_{0}\sqrt{\log n_{\min}}$
for some suitable constant $c_{0}>0$ independent of $T$, it is easy
to see that there exist constants $c_{1}\in\left(0,0.5\right)$ and
$c_{2}\in\left(0,1/52\right)$ independent of $T$ and some $n^{*}=n^{*}\left(T,L,c_{1},c_{2}\right)\geq1$
such that for any $n_{\min}\geq n^{*}$, we have:

\begin{align*}
\mathbb{P}\bigg(\mathscr{D}_{T}\left(W,\tilde{W}\right) & \geq Kn_{\min}^{-c_{1}}\sqrt{\log\left(\frac{1}{\delta}n_{\max}^{2}+e\right)}\bigg)\leq\delta+KLn_{\max}\exp\left(-Kn_{\min}^{c_{2}}\right).
\end{align*}
\end{proof}
\begin{proof}[Proof of Theorem \ref{thm:gradient descent}]
This comes from Propositions \ref{prop:gradient descent - bounded}
and \ref{prop:truncation}, similar to the proof of Theorem \ref{thm:particle coupling}.
\end{proof}
Let us mention again the correspondence between Theorem \ref{thm:particle coupling}
and Proposition \ref{prop:particle coupling - bounded}, and that
between Theorem \ref{thm:gradient descent} and Proposition \ref{prop:gradient descent - bounded}.
The truncation at initialization allows for technical feasibility
and is then bridged by Proposition \ref{prop:truncation}. The proofs
of Propositions \ref{prop:particle coupling - bounded} and \ref{prop:gradient descent - bounded}
are necessarily lengthy, so let us defer them (as well as missing
proofs of other results) to Appendix \ref{sec:Remaining-proofs-main-MF}.
Let us describe briefly the argument for Proposition \ref{prop:particle coupling - bounded}.
One recalls that $w_{i}\left(0,C_{i-1}\left(j_{i-1}\right),C_{i}\left(j_{i}\right)\right)=\tilde{w}_{i}\left(0,j_{i-1},j_{i}\right)$
at initialization $t=0$, and hence one hopes to prove:
\[
w_{i}\left(t,C_{i-1}\left(j_{i-1}\right),C_{i}\left(j_{i}\right)\right)\approx\tilde{w}_{i}\left(t,j_{i-1},j_{i}\right)
\]
at any finite $t$. In other words, we would like to show
\[
\mathbb{E}_{Z}\left[\Delta_{i}^{w}\left(Z,C_{i-1}\left(j_{i-1}\right),C_{i}\left(j_{i}\right);W\left(t\right)\right)\right]\approx\mathbb{E}_{Z}\left[\Delta_{i}^{\mathbf{w}}\left(Z,j_{i-1},j_{i};\tilde{W}\left(t\right)\right)\right].
\]
Both of these quantities share very similar structures. Roughly speaking,
the left-hand side involves quantities that assume the form of an
expectation $\mathbb{E}_{C_{r}}[g(C_{r})]$ and the right-hand side
correspondingly involves quantities of the form of an empirical average
$(1/n_{r})\cdot\sum_{j_{r}=1}^{n_{r}}g(C_{r}(j_{r}))$, for some function
$g$. An invocation of concentration of measure bounds links the two
sides, and if done correctly over the training horizon (i.e. over
$t\leq T$), the depth of the network (i.e. over index $i\leq L$)
and the width at each layer (i.e. over neuron $j_{i}\leq n_{i}$),
it gives the desired estimation. One also recognizes that the neural
network $\mathbf{W}$ is essentially a time discretization version
of $\tilde{W}$ where the learning rate $\epsilon$ plays the role
of the discretization level. A martingale-type argument then suffices
to prove Proposition \ref{prop:gradient descent - bounded} for small
$\epsilon$.

\section{Simplifications under Independent and Identically Distributed Initialization
\label{sec:iid_init}}

In this section, we prove that the MF limit under an independent and
identically distributed (i.i.d.) initialization degenerates to a simple
structured dynamics. Let us first state the definition of i.i.d. initializations.
\begin{defn}
An initialization law $\rho$ for a neural network of size $\left\{ n_{1},...,n_{L}\right\} $
is called $\left(\rho_{\mathbf{w}}^{1},...,\rho_{\mathbf{w}}^{L},\rho_{\mathbf{b}}^{2},...,\rho_{\mathbf{b}}^{L}\right)$-i.i.d.
initialization (or i.i.d. initialization, for brevity), where $\rho_{\mathbf{w}}^{i}$
is a probability measure over $\mathbb{W}_{i}$ and $\rho_{\mathbf{b}}^{i}$
is a probability measure over $\mathbb{B}_{i}$, if it satisfies the
following: 

\begin{itemize}
\item $\left\{ \mathbf{w}_{1}\left(0,j_{1}\right)\right\} _{j_{1}\in\left[n_{1}\right]}$
are generated i.i.d. according to $\rho_{\mathbf{w}}^{1}$,
\item for each $i=2,...,L$, $\left\{ \mathbf{w}_{i}\left(0,j_{i-1},j_{i}\right)\right\} _{j_{i-1}\in\left[n_{i-1}\right],\;j_{i}\in\left[n_{i}\right]}$
are generated i.i.d. according to $\rho_{\mathbf{w}}^{i}$, and $\left\{ \mathbf{b}_{i}\left(0,j_{i}\right)\right\} _{j_{i}\in\left[n_{i}\right]}$
are generated i.i.d. according to $\rho_{\mathbf{b}}^{i}$,
\item all these generations are independent of each other, and $\rho_{\mathbf{b}}^{L}$
is a single point mass.
\end{itemize}
Observe that given $\left(\rho_{\mathbf{w}}^{1},...,\rho_{\mathbf{w}}^{L},\rho_{\mathbf{b}}^{2},...,\rho_{\mathbf{b}}^{L}\right)$,
one can build a family $\mathsf{Init}$ of i.i.d. initialization laws
that contains any index tuple $\left\{ n_{1},...,n_{L}\right\} $.

In the following, we construct a canonical MF limit under i.i.d. initialization
and show that the MF dynamics can be significantly simplified. Our
plan is as follows:
\end{defn}

\begin{enumerate}
\item We first construct a sequence (in increasing $M$) of neuronal embeddings,
which we call \textit{canonical neuronal embeddings}. In particular,
each of these -- indexed by $M$ -- allows to embed i.i.d.-initialized
neural networks of sizes at most $M$. Each canonical neuronal embedding
is associated with a MF limit, which we call \textit{a canonical MF
limit}.
\item We present a dynamics which is shown to be the infinite-$M$ limit
of the canonical MF limits. This dynamics displays the simplifying
properties that we wish to show. In particular, the dynamics of i.i.d.-initialized
neural networks of large widths are well-approximated by the infinite-$M$
limit, and asymptotically displays the same simplifying properties. 
\end{enumerate}
This plan streamlines our studies of i.i.d.-initialized networks in
the infinite-width limit. As we shall see, the construction of the
canonical neuronal embedding is quite natural due to the cap at finite
$M$. More importantly, on one hand, the fact that the canonical MF
limit tracks closely the neural network of size less than $M$ demonstrates
flexibility of Theorem \ref{thm:gradient descent coupling} from Section
\ref{sec:Main-result}, in that its applicability is not limited to
abstract infinite-width limits. On the other hand, the fact that the
simplifying properties are shown in the infinite-$M$ limit demonstrates
the advantage of working with these abstract infinite-width dynamics:
they reveal properties that are virtually invisible at the finite-width
level.

\subsection{Neuronal embedding construction and main results}

\subsubsection{Canonical neuronal embeddings and canonical MF limits\label{subsec:Canonical-neuronal-embeddings}}

We describe the construction in three steps with a given positive
integer $M$ and a set of measures $\left(\rho_{\mathbf{w}}^{1},...,\rho_{\mathbf{w}}^{L},\rho_{\mathbf{b}}^{2},...,\rho_{\mathbf{b}}^{L}\right)$.

\subparagraph*{Step 1.}

We first give a description of a $\sigma$-finite measure space. Consider
a probability space $\left(\Lambda,P_{0}\right)$ of the random processes
$\mathbb{W}_{1}$-valued $\mathfrak{p}_{1}\left(\theta_{1}\right)$,
$\mathbb{W}_{i}$-valued $\mathfrak{q}_{i}\left(\theta_{i-1},\theta_{i}\right)$
and $\mathbb{B}_{i}$-valued $\mathfrak{p}_{i}\left(\theta_{i}\right)$
for $2\leq i\leq L$. These processes are indexed by $\theta_{i}\in\mathbb{N}_{>0}$
and satisfy the following property. Let $m_{1},...,m_{L-1}$ be $L-1$
arbitrary finite positive integers and, with these integers, let $\left\{ \theta_{i}^{\left(k_{i}\right)}\in\mathbb{N}_{>0}:\;k_{i}\in\left[m_{i}\right],\;i=1,...,L-1\right\} $
be an arbitrary collection. Let $m_{L}=1$ and $\theta_{L}^{\left(1\right)}=1$.
For each $i=1,...,L$, let $S_{i}$ be the set of unique elements
in $\left\{ \theta_{i}^{\left(k_{i}\right)}:\;k_{i}\in\left[m_{i}\right]\right\} $.
Similarly, for each $i=2,...,L$, let $R_{i}$ be the set of unique
pairs in $\left\{ \left(\theta_{i-1}^{\left(k_{i-1}\right)},\theta_{i}^{\left(k_{i}\right)}\right):\;k_{i-1}\in\left[m_{i-1}\right],\;k_{i}\in\left[m_{i}\right]\right\} $.
The space $\left(\Lambda,P_{0}\right)$ satisfies that $\left\{ \mathfrak{p}_{i}\left(\theta_{i}\right):\;\theta_{i}\in S_{i},\;i=1,...,L\right\} $
and $\left\{ \mathfrak{q}_{i}\left(\theta_{i-1},\theta_{i}\right):\;\left(\theta_{i-1},\theta_{i}\right)\in R_{i},\;i=2,...,L\right\} $
are all mutually independent. In addition, we also have 
\[
{\rm Law}\left(\mathfrak{p}_{1}\left(\theta_{1}\right)\right)=\rho_{\mathbf{w}}^{1},\quad{\rm Law}\left(\mathfrak{p}_{i}\left(\theta_{i}\right)\right)=\rho_{\mathbf{b}}^{i},\quad{\rm Law}\left(\mathfrak{q}_{i}\left(\theta_{i-1}',\theta_{i}'\right)\right)=\rho_{\mathbf{w}}^{i}
\]
for any $\theta_{1}\in S_{1}$, $\theta_{i}\in S_{i}$ and $\left(\theta_{i-1}',\theta_{i}'\right)\in R_{i}$,
for $i=2,...,L$. Such a space $\left(\Lambda,P_{0}\right)$ exists
by Kolmogorov's extension theorem.

\subparagraph*{Step 2.}

With this space, given the integer $M$, for each $i\in\left[L-1\right]$,
we define $\Omega_{i}^{M}=\Lambda\times[M]$ equipped with the product
measure $P_{i}^{M}=P_{0}\times{\rm Unif}\left(\left[M\right]\right)$,
where ${\rm Unif}\left(\left[M\right]\right)$ is the uniform measure
over the finite set $[M]$. We also let $\Omega_{L}^{M}=\{1\}$ and
$P_{L}^{M}=\mathbb{I}_{\Omega_{L}^{M}}$. We construct $\Omega^{M}=\prod_{i=1}^{L}\Omega_{i}^{M}$,
equipped with the product measure $P^{M}=\prod_{i=1}^{L}P_{i}^{M}$.
The space $\left(\Omega^{M},P^{M}\right)$ gives a \textit{canonical
neuronal ensemble}.

\subparagraph*{Step 3.}

Let $\Omega_{i}=\Lambda\times\mathbb{N}_{>0}$ and observe $\Omega_{i}^{M}\subset\Omega_{i}$
for any $M$. We define the deterministic functions $w_{1}^{0}:\;\Omega_{1}\to\mathbb{W}_{1}$,
$w_{i}^{0}:\;\Omega_{i-1}\times\Omega_{i}\to\mathbb{W}_{i}$ and $b_{i}^{0}:\;\Omega_{i}\to\mathbb{B}_{i}$,
for $i=2,...,L$: 
\begin{align}
w_{1}^{0}\left(\left(\lambda_{1},\theta_{1}\right)\right) & =\mathfrak{p}_{1}\left(\theta_{1}\right)\left(\lambda_{1}\right),\label{eq:iid_init_embedding_constr1}\\
w_{i}^{0}\left(\left(\lambda_{i-1},\theta_{i-1}\right),\left(\lambda_{i},\theta_{i}\right)\right) & =\mathfrak{q}_{i}\left(\theta_{i-1},\theta_{i}\right)\left(\lambda_{i}\right),\qquad i=2,...,L-1,\\
w_{L}^{0}\left(\left(\lambda_{L-1},\theta_{L-1}\right),1\right) & =\mathfrak{q}_{L}\left(\theta_{L-1},1\right)\left(\lambda_{L-1}\right),\\
b_{i}^{0}\left(\left(\lambda_{i},\theta_{i}\right)\right) & =\mathfrak{p}_{i}\left(\theta_{i}\right)\left(\lambda_{i}\right),\qquad i=2,...,L-1,\\
b_{L}^{0}\left(1\right) & =\mathfrak{p}_{L}\left(1\right).\label{eq:iid_init_embedding_constr2}
\end{align}
These functions, together with $\left(\Omega^{M},P^{M}\right)$ ,
give a \textsl{canonical neuronal embedding}. Per Section \ref{subsec:MF},
given this neuronal embedding, one obtains a \textit{canonical MF
limit} $W^{M}\left(t\right)=\left\{ w_{1}^{M}\left(t,\cdot\right),w_{i}^{M}\left(t,\cdot,\cdot\right),b_{i}^{M}\left(t,\cdot\right),\;\;i=2,...,L\right\} $,
defined on $\left(\Omega^{M},P^{M}\right)$, with initialization $W^{M}\left(0\right)=\left\{ w_{1}^{0},w_{i}^{0},b_{i}^{0}:\;i=2,...,L\right\} $.
With $\left(C_{1},...,C_{L}\right)\sim P^{M}$, one observes that
\[
{\rm Law}\left(w_{1}^{0}\left(C_{1}\right),w_{2}^{0}\left(C_{1},C_{2}\right),b_{2}^{0}\left(C_{2}\right),...,w_{L}^{0}\left(C_{L-1},1\right),b_{L}^{0}\left(1\right)\right)=\rho_{\mathbf{w}}^{1}\times\prod_{i=2}^{L}\rho_{\mathbf{w}}^{i}\times\rho_{\mathbf{b}}^{i}.
\]
We also consider the sampling rule $\overline{P}_{{\bf n}}^{M}$ defined
for each ${\bf n}=(n_{1},\dots,n_{L})$ with $n_{i}\le M$ for $i\in\left[L-1\right]$
and $n_{L}=1$ by independently sampling $\left\{ C_{i}(j_{i})\right\} _{j_{i}\in\left[n_{i}\right]}$
from $\left(P_{i}^{M}\right)^{n_{i}}$ conditioned on that $\left\{ \theta_{i}(j_{i})\right\} _{j_{i}\in\left[n_{i}\right]}$
are all distinct, where $C_{i}(j_{i})=\left(\lambda_{i}(j_{i}),\theta_{i}(j_{i})\right)$,
for each $i\in\left[L\right]$.

The constructed embedding indeed gives a valid neuronal embedding
for neural networks of sizes at most $M$.
\begin{prop}
\label{prop:iid_law_embedding}For ${\bf n}=\left\{ n_{1},...,n_{L}\right\} $
with $n_{i}\leq M$ and $n_{L}=1$, the space $(\Omega^{M},P^{M})$
together with the functions $\left(\left\{ w_{i}^{0}\right\} _{i\in\left[L\right]},\left\{ b_{i}^{0}\right\} _{2\leq i\leq L}\right)$
form a neuronal embedding for the neural network of size ${\bf n}$
under $\left(\rho_{\mathbf{w}}^{1},...,\rho_{\mathbf{w}}^{L},\rho_{\mathbf{b}}^{2},...,\rho_{\mathbf{b}}^{L}\right)$-i.i.d.
initialization, in which the associated sampling rule is $\overline{P}_{{\bf n}}^{M}$.
Furthermore $\overline{P}_{{\bf n}}^{M}$ is $\left(2n_{\max}/M\right)$-independent,
where $n_{\max}=\max\left(n_{1},...,n_{L}\right)$.
\end{prop}

The proof of the proposition is deferred to Appendix \ref{sec:Remaining-proofs-iid-init}.
This result, together with Theorem \ref{thm:gradient descent coupling},
suggests that for large $M$, the canonical MF limit tracks closely
the trajectory of an i.i.d.-initialized neural network, as long as
its (large) size is much smaller than $M$. Equivalently an i.i.d.-initialized
large neural network can be closely tracked by any canonical MF limit
with sufficiently large $M$. This motivates the studies of the canonical
MF limits in the limit $M\to\infty$, which display simplified structures.

\subsubsection{Infinite-$M$ limit of canonical MF limits\label{subsec:Infinite-M-canonical}}

Recall that the space $(\Omega^{M},P^{M})$ depends on $M$ and only
gives an embedding of networks whose widths are at most $M$. More
specifically, while the space $\left(\Lambda,P_{0}\right)$ is independent
of $M$ and $\Omega^{M}$ can be extended to infinite $M$, the measure
$P^{M}$ would become an improper probability measure for infinite
$M$. Nevertheless one can still define a dynamics that is independent
of $M$.

Let $W^{*}\left(t\right)=\left\{ w_{1}^{*}\left(t,\cdot\right),w_{i}^{*}\left(t,\cdot\right),b_{i}^{*}\left(t,\cdot\right),\;\;i=2,...,L\right\} $
be a dynamics to be described shortly, which we shall prove to be
the ``infinite-$M$'' limit of $W^{M}$. The full description is
lengthy and is deferred to Appendix \ref{subsec:Infinite-M-limit-full};
let us give a snapshot description for $L\geq5$ and $i=3,...,L-2$:
\begin{align*}
\frac{\partial}{\partial t}w_{i}^{*}\left(t,u_{i},v_{i-1},v_{i}\right) & =-\xi_{i}^{\mathbf{w}}\left(t\right)\mathbb{E}_{Z}\left[\Delta_{i}^{w*}\left(t,Z,u_{i},v_{i-1},v_{i}\right)\right],\\
\frac{\partial}{\partial t}b_{i}^{*}\left(t,v_{i}\right) & =-\xi_{i}^{\mathbf{b}}\left(t\right)\mathbb{E}_{Z}\left[\Delta_{i}^{b*}\left(t,Z\right)\right],\qquad\forall u_{i}\in{\rm supp}\left(\rho_{\mathbf{w}}^{i}\right),\;v_{i}\in{\rm supp}\left(\rho_{\mathbf{b}}^{i}\right),
\end{align*}
with the initialization $w_{i}^{*}\left(0,u_{i},\cdot,\cdot\right)=u_{i}$
and $b_{i}^{*}\left(0,v_{i}\right)=v_{i}$. Here the quantities are
defined by:
\begin{align*}
H_{i}^{*}\left(t,x,v_{i}\right) & =\int\phi_{i}\left(w_{i}^{*}\left(t,u_{i},v_{i-1},v_{i}\right),b_{i}^{*}\left(t,v_{i}\right),H_{i-1}^{*}\left(t,x,v_{i-1}\right)\right)\rho_{\mathbf{w}}^{i}\left(du_{i}\right)\rho_{\mathbf{b}}^{i-1}\left(dv_{i-1}\right),\\
\Delta_{i}^{w*}\left(t,z,u_{i},v_{i-1},v_{i}\right) & =\sigma_{i}^{\mathbf{w}}\left(\Delta_{i}^{H*}\left(t,z,v_{i}\right),w_{i}^{*}\left(t,u_{i},v_{i-1},v_{i}\right),b_{i}^{*}\left(t,v_{i}\right),H_{i}^{*}\left(t,x,v_{i}\right),H_{i-1}^{*}\left(t,x,v_{i-1}\right)\right),\\
\Delta_{i}^{b*}\left(t,z,v_{i}\right) & =\int\sigma_{i}^{\mathbf{b}}\left(\Delta_{i}^{H*}\left(t,z,v_{i}\right),w_{i}^{*}\left(t,u_{i},v_{i-1},v_{i}\right),b_{i}^{*}\left(t,v_{i}\right),H_{i}^{*}\left(t,x,v_{i}\right),H_{i-1}^{*}\left(t,x,v_{i-1}\right)\right)\\
 & \qquad\times\rho_{\mathbf{w}}^{i}\left(du_{i}\right)\rho_{\mathbf{b}}^{i-1}\left(dv_{i-1}\right),\\
\Delta_{i-1}^{H*}\left(t,z,v_{i-1}\right) & =\int\sigma_{i-1}^{\mathbf{H}}\left(\Delta_{i}^{H*}\left(t,z,v_{i}\right),w_{i}^{*}\left(t,u_{i},v_{i-1},v_{i}\right),b_{i}^{*}\left(t,v_{i}\right),H_{i}^{*}\left(t,x,v_{i}\right),H_{i-1}^{*}\left(t,x,v_{i-1}\right)\right)\\
 & \qquad\times\rho_{\mathbf{w}}^{i}\left(du_{i}\right)\rho_{\mathbf{b}}^{i}\left(dv_{i}\right).
\end{align*}
The existence and uniqueness of such dynamics follow similarly to
the proof of Theorem \ref{thm:existence ODE}. We state the main result
of this section, which shows that the dynamics $W^{*}$ is the infinite-$M$
limit of $W^{M}$. (Again we refer to Appendix \ref{subsec:Infinite-M-limit-full},
specifically Theorem \ref{thm:iid dynamics-full}, for the complete
statement of this theorem.)
\begin{thm}[Snapshot statement]
\label{thm:iid dynamics}Given $\left(\rho_{\mathbf{w}}^{1},...,\rho_{\mathbf{w}}^{L},\rho_{\mathbf{b}}^{2},...,\rho_{\mathbf{b}}^{L}\right)$
and an integer $M$, construct the canonical neuronal ensemble $\left(\Omega^{M},P^{M}\right)$,
the random variables $\left(C_{1},...,C_{L}\right)\sim P^{M}=\prod_{i=1}^{L}P_{i}^{M}$
and the canonical MF limit $W^{M}$ as described in Section \ref{subsec:Canonical-neuronal-embeddings}.
Also construct the dynamics $W^{*}$ described in Section \ref{subsec:Infinite-M-canonical}.
Define the following:
\begin{align*}
w_{i}^{\infty}\left(t,c_{i-1},c_{i}\right) & =w_{i}^{*}\left(t,w_{i}^{0}\left(c_{i-1},c_{i}\right),b_{i-1}^{0}\left(c_{i-1}\right),b_{i}^{0}\left(c_{i}\right)\right),\\
b_{i}^{\infty}\left(t,c_{i}\right) & =b_{i}^{*}\left(t,b_{i}^{0}\left(c_{i}\right)\right),\qquad\forall c_{i}\in\Omega_{i}=\Lambda\times\mathbb{N}_{>0},\quad i=3,...,L-2.
\end{align*}
We also let $W^{\infty}\left(t\right)=\left\{ w_{1}^{\infty}\left(t,\cdot\right),w_{i}^{\infty}\left(t,\cdot,\cdot\right),b_{i}^{\infty}\left(t,\cdot\right),\;\;i=2,...,L\right\} $.
Let us consider:
\begin{align*}
\left\langle W^{M}-W^{\infty}\right\rangle _{t} & =\max\left(\max_{1\leq i\leq L}\left\langle w_{i}^{M}-w_{i}^{\infty}\right\rangle _{t},\;\max_{2\leq i\leq L}\left\langle b_{i}^{M}-b_{i}^{\infty}\right\rangle _{t}\right),\\
\left\langle w_{i}^{M}-w_{i}^{\infty}\right\rangle _{t} & =\mathbb{E}\left[\left|w_{i}^{M}\left(t,C_{i-1},C_{i}\right)-w_{i}^{\infty}\left(t,C_{i-1},C_{i}\right)\right|^{2}\right]^{1/2},\\
\left\langle b_{i}^{M}-b_{i}^{\infty}\right\rangle _{t} & =\mathbb{E}\left[\left|b_{i}^{M}\left(t,C_{i}\right)-b_{i}^{\infty}\left(t,C_{i}\right)\right|^{2}\right]^{1/2}.
\end{align*}
Then under Assumptions \ref{enu:Assump_lrSchedule}-\ref{enu:Assump_backward}
and \ref{assump:init}, for any $T\geq0$ and $L\geq2$,
\[
\sup_{t\leq T}\left\langle W^{M}-W^{\infty}\right\rangle _{t}\le\frac{K_{T,L}}{M^{0.499}},
\]
for sufficiently large $M=M\left(T,L\right)$, where $K_{T,L}$ is
a constant that depends on $T$ and $L$. Furthermore, for $L\geq4$
and $2\leq i\leq L-2$,
\[
\sup_{t\leq T}\mathbb{E}\left[\left|H_{i}\left(X,C_{i};W^{M}(t)\right)-H_{i}^{*}\left(t,X,b_{i}^{0}(C_{i})\right)\right|^{2}\right]^{1/2}\leq\frac{K_{T,L}}{M^{0.499}}.
\]
\end{thm}

We give a sketch of the proof in Section \ref{subsec:Proof-Theorem-iid}.
We now discuss the implications of Theorem \ref{thm:iid dynamics},
and in particular, the simplifying properties induced by i.i.d. initializations.
The complete proofs of this theorem and its corollaries are deferred
to Appendix \ref{sec:Remaining-proofs-iid-init}.

\paragraph*{Tracking i.i.d.-initialized neural nets via $W^{*}$.}

For large $M$, the canonical MF limit $W^{M}$ is well approximated
by $W^{*}$ (and equivalently by $W^{\infty}$ as defined in Theorem
\ref{thm:iid dynamics}), while we recall from Theorem \ref{thm:gradient descent coupling}
that $W^{M}$ tracks closely the trajectory $\mathbf{W}$ of a large-width
i.i.d.-initialized neural network. As such, viewing the bridge through
$W^{M}$ as an intermediate step and taking $M\to\infty$, one can
track $\mathbf{W}$ via $W^{*}$. To be precise, by combining Proposition
\ref{prop:iid_law_embedding} and Corollary \ref{cor:gradient descent quality}
with Theorem \ref{thm:iid dynamics}, we immediately obtain the following
result.
\begin{cor}
\label{cor:iid_tracking}Under Assumptions \ref{enu:Assump_lrSchedule}-\ref{enu:Assump_backward}
and for a set of probability measures $\left(\rho_{\mathbf{w}}^{1},...,\rho_{\mathbf{w}}^{L},\rho_{\mathbf{b}}^{2},...,\rho_{\mathbf{b}}^{L}\right)$
such that
\[
\max_{1\leq i\leq L}\sup_{m\geq1}\frac{1}{\sqrt{m}}\left(\int\left|u\right|^{m}\rho_{\mathbf{w}}^{i}\left(du\right)\right)^{1/m}\leq K,\qquad\max_{2\leq i\leq L}\sup_{m\geq1}\frac{1}{\sqrt{m}}\left(\int\left|v\right|^{m}\rho_{\mathbf{b}}^{i}\left(dv\right)\right)^{1/m}\leq K,
\]
there exist constants $c_{1}\in\left(0,0.5\right)$ and $c_{2}\in\left(0,1/52\right)$
such that the following statements hold.

Consider any positive integer $L\geq2$ and a tuple of positive integers
$\mathbf{n}=\left\{ n_{1},...,n_{L}\right\} $ with $n_{L=1}$. Let
$n_{\min}=\min_{1\leq j\leq L-1}n_{j}$ and $n_{\max}=\max_{1\leq j\leq L}n_{j}$.
Consider a neural network (\ref{eq:NN}) of size $\mathbf{n}$ under
$\left(\rho_{\mathbf{w}}^{1},...,\rho_{\mathbf{w}}^{L},\rho_{\mathbf{b}}^{2},...,\rho_{\mathbf{b}}^{L}\right)$-i.i.d.
initialization, and let $\mathbf{W}$ be its trajectory. Also construct
the dynamics $W^{*}$, as well as the associated quantities, described
in Section \ref{subsec:Infinite-M-canonical}. Then for any $\delta>0$
and $T\in\epsilon\mathbb{N}_{\geq0}$, there exist $n^{*}=n^{*}\left(T,L,c_{1},c_{2}\right)\geq1$
and $\epsilon^{*}=\epsilon^{*}\left(T,L,c_{1},c_{2}\right)\leq1$
such that if $n_{\min}\geq n^{*}$ and the learning rate $\epsilon\in\left(0,\epsilon^{*}\right)$,
for $3\leq i\leq L-2$, for any $K$-Lipschitz and $K$-bounded test
function $\psi:\mathbb{H}_{i}\to\mathbb{S}$ (where $\mathbb{S}$
is a separable Hilbert space), for any $\delta>0$, we have with probability
at least $1-3\delta-KLn_{\max}\exp\left(-Kn_{\min}^{c_{2}}\right)$,
\[
\sup_{t\le T}\left|\frac{1}{n_{i}}\sum_{j_{i}=1}^{n_{i}}\mathbb{E}_{Z}\left[\psi\left({\bf H}_{i}\left(\left\lfloor t/\epsilon\right\rfloor ,X,j_{i}\right)\right)\right]-\mathbb{E}_{Z}\left[\int\psi\left(H_{i}^{*}\left(t,X,v\right)\right)\rho_{\mathbf{b}}^{i}\left(dv\right)\right]\right|=\tilde{O}\left(n_{\min}^{-c_{1}}+\epsilon^{c_{1}}\right),
\]
where $\tilde{O}$ hides the dependency on $T$, $L$ and $\delta$
as well as the logarithmic factors $\log n_{\max}$ and $\log\left(1/\epsilon\right)$.
A similar statement holds for $i=1,2,L-1,L$. In addition, for any
test function $\psi:\mathbb{Y}\times\hat{\mathbb{Y}}\to\mathbb{S}$
which is $K$-Lipschitz in the second variable, uniformly in the first
variable,
\[
\sup_{t\leq T}\left|\mathbb{E}_{Z}\left[\psi\left(Y,\hat{\mathbf{y}}\left(\left\lfloor t/\epsilon\right\rfloor ,X\right)\right)\right]-\mathbb{E}_{Z}\left[\psi\left(Y,\hat{y}^{*}\left(t,X\right)\right)\right]\right|=\tilde{O}\left(n_{\min}^{-c_{1}}+\epsilon^{c_{1}}\right),
\]
with probability at least $1-2\delta-KLn_{\max}\exp\left(-Kn_{\min}^{c_{2}}\right)$.
\end{cor}

\paragraph*{Degeneracy of the dynamics.}

By looking closely at $W^{*}$, we observe a simplifying property.
By Theorem \ref{thm:iid dynamics}, under i.i.d. initialization, for
each intermediate layer $i=3,...,L-2$, the weight $w_{i}^{\infty}\left(t,C_{i-1},C_{i}\right)$
is a function of only the time $t$, its own initialization $w_{i}^{0}\left(C_{i-1},C_{i}\right)$
and the initializations of the adjacent biases $b_{i-1}^{0}\left(C_{i-1}\right)$
and $b_{i}^{0}\left(C_{i}\right)$, and the bias $b_{i}^{\infty}\left(t,C_{i}\right)$
is a function of only the time $t$ and its own initialization $b_{i}^{0}\left(C_{i}\right)$.
When we further assume constant initial biases (i.e. $b_{i}^{0}\left(C_{i}\right)=B_{i}$
a constant almost surely for all $i\geq2$), $w_{i}^{\infty}\left(t,C_{i-1},C_{i}\right)$
is a function of only the time $t$ and its own initialization, and
$b_{i}^{\infty}\left(t,C_{i}\right)$ is almost surely only a function
of time $t$, regardless of $C_{i}$. Consequently, in this scenario,
because the initialization is independent across layers, the weights
of intermediate layers remain mutually independent at all time, for
depth $L\geq5$, in the infinite-width limit.

The theorem in fact further asserts that degeneracy can already be
observed for $L\geq4$. In particular, for $2\leq i\leq L-2$, if
the initial bias $b_{i}^{0}\left(\cdot\right)=B_{i}$ is a constant,
then
\[
\mathbb{E}\left[\left|H_{i}\left(X,C_{i};W^{M}(t)\right)-H_{i}^{*}\left(t,X,B_{i}\right)\right|^{2}\right]^{1/2}\leq\frac{K_{T,L}}{M^{0.499}}.
\]
Note that $H_{i}^{*}\left(t,X,B_{i}\right)$ is independent of $C_{i}$.
This suggests that at any training time $t$, the neurons of each
intermediate layer $i$ compute the same function of the data input
$x\mapsto H_{i}^{*}\left(t,x,B_{i}\right)$ in the infinite-width
limit. This is formalized directly for the neural network $\mathbf{W}$
in the following.
\begin{cor}
\label{cor:iid_same_neurons}Consider the same setting as Corollary
\ref{cor:iid_tracking} with $L\geq4$. For $2\leq i\leq L-2$, supposing
that $b_{i}^{0}\left(C_{i}\right)=B_{i}$ a constant almost surely,
then we have for any $t\leq T$, with probability at least $1-3\delta-KLn_{\max}\exp\left(-Kn_{\min}^{c_{2}}\right)$,
\[
\bigg(\frac{1}{n_{i}}\sum_{j_{i}=1}^{n_{i}}\mathbb{E}_{Z}\left[\left|{\bf H}_{i}\left(\left\lfloor t/\epsilon\right\rfloor ,X,j_{i}\right)-H_{i}^{*}\left(t,X,B_{i}\right)\right|^{2}\right]\bigg)^{1/2}=\tilde{O}\left(n_{\min}^{-c_{1}}+\epsilon^{c_{1}}\right).
\]
\end{cor}

Thus, by Markov's inequality, if one is to pick at random a neuron
$j_{i}\in\left[n_{i}\right]$ at layer $i$ from the neural network
$\mathbf{W}$ at the training step $\left\lfloor t/\epsilon\right\rfloor $,
for $2\leq i\leq L-2$, then with high probability, this neuron would
compute the function $x\mapsto H_{i}^{*}\left(t,x,B_{i}\right)$ which
is independent of the index $j_{i}$.

\paragraph*{Collapse to effectively one parameter per layer.}

Further consideration to standard neural network architectures reveals
a stronger simplifying property. The next consequence of Theorem \ref{thm:iid dynamics}
is that with i.i.d. initialization and constant initial biases, for
each intermediate layer $i=3,...,L-2$, the weight $w_{i}^{\infty}\left(t,c_{i-1},c_{i}\right)$
translates by a quantity that is independent of $c_{i-1}$ and $c_{i}$,
provided that $\sigma_{i}^{\mathbf{w}}$ satisfies a certain condition.
This condition holds for unregularized standard fully-connected or
convolutional neural networks (see Examples \ref{exa:fully-connected}
and \ref{exa:conv}). Therefore, for these networks, in the infinite-width
limit, with i.i.d. initialization and constant initial biases, the
dynamics of the weight at each intermediate layer reduces to a single
deterministic translation parameter.
\begin{cor}
\label{cor:iid_standard-network}Under the same setting as Theorem
\ref{thm:iid dynamics} with $L\geq5$, assume that $b_{i}^{0}\left(C_{i}\right)=B_{i}$
a constant almost surely for all $i\geq2$. Further assume that for
each $i\in\left\{ 3,...,L-2\right\} $, there exists a function $\bar{\sigma}_{i}^{\mathbf{w}}$
that satisfies 
\[
\sigma_{i}^{\mathbf{w}}\left(\Delta,w,b,g,h\right)=\bar{\sigma}_{i}^{\mathbf{w}}\left(\Delta,b,g,h\right),
\]
i.e. $\sigma_{i}^{\mathbf{w}}$ does not depend on the second variable.
Then there are differentiable functions $w_{i}^{\#}\left(t\right)$
such that for $3\le i\le L-2$, almost surely, for any $t\geq0$,
\[
w_{i}^{\infty}\left(t,C_{i-1},C_{i}\right)-w_{i}^{\infty}\left(0,C_{i-1},C_{i}\right)=w_{i}^{\#}\left(t\right).
\]
\end{cor}

\subsection{Proof sketch of Theorem \ref{thm:iid dynamics}\label{subsec:Proof-Theorem-iid}}
\begin{proof}[Sketch of proof for Theorem \ref{thm:iid dynamics}]
We use $K_{T,L}$ to denote a generic constant that depends on $T$
and $L$ and may change from line to line. The main argument exploits
the construction in Section \ref{subsec:Canonical-neuronal-embeddings}
of the canonical neuronal embedding in a suitable way. To illustrate
the idea, consider 
\[
D_{i}\left(t\right)=\mathbb{E}\left[\left|H_{i}\left(X,C_{i};W^{\infty}(t)\right)-H_{i}^{*}\left(t,X,b_{i}^{0}(C_{i})\right)\right|^{2}\right].
\]
We aim to show that for $t\le T$,
\[
D_{i}\left(t\right)\leq\frac{K_{T,L}}{M}.
\]
For brevity, define
\[
g\left(u_{i},v_{i-1},v_{i}\right)=\phi_{i}\left(w_{i}^{*}\left(t,u_{i},v_{i-1},v_{i}\right),\;b_{i}^{*}(t,v_{i}),\;H_{i-1}^{*}\left(t,x,v_{i-1}\right)\right).
\]
We recall $w_{i}^{0}(C_{i-1},C_{i})=\mathfrak{q}_{i}(\theta_{i-1},\theta_{i})(\lambda_{i})$,
$b_{i}^{0}(C_{i})=\mathfrak{p}_{i}(\theta_{i})(\lambda_{i})$ and
$b_{i-1}^{0}(C_{i-1})=\mathfrak{p}_{i-1}(\theta_{i-1})(\lambda_{i-1})$
from the construction (\ref{eq:iid_init_embedding_constr1})-(\ref{eq:iid_init_embedding_constr2}).
To make use of canonical neuronal embedding's construction, we consider
a decomposition of the following squared quantity:
\begin{align*}
 & \mathbb{E}_{C_{i}}\bigg[\left|\mathbb{E}_{C_{i-1}}\left[g\left(w_{i}^{0}(C_{i-1},C_{i}),b_{i-1}^{0}(C_{i-1}),b_{i}^{0}(C_{i})\right)\right]\right|^{2}\bigg]\\
 & \stackrel{\left(a\right)}{=}\mathbb{E}_{\theta_{i-1},\lambda_{i-1},\theta_{i-1}',\lambda_{i-1}',\theta_{i},\lambda_{i}}\bigg[\bigg\langle g\left(\mathfrak{q}_{i}\left(\theta_{i-1},\theta_{i}\right)\left(\lambda_{i}\right),\mathfrak{p}_{i-1}\left(\theta_{i-1}\right)\left(\lambda_{i-1}\right),\mathfrak{p}_{i}\left(\theta_{i}\right)\left(\lambda_{i}\right)\right),\\
 & \qquad g\left(\mathfrak{q}_{i}\left(\theta_{i-1}',\theta_{i}\right)\left(\lambda_{i}\right),\mathfrak{p}_{i-1}\left(\theta_{i-1}'\right)\left(\lambda_{i-1}'\right),\mathfrak{p}_{i}\left(\theta_{i}\right)\left(\lambda_{i}\right)\right)\bigg\rangle\bigg]\\
 & \stackrel{\left(b\right)}{=}\mathbb{E}_{\theta_{i-1},\theta_{i-1}'}\bigg[\mathbb{I}_{\theta_{i-1}=\theta_{i-1}'}\int\left\langle g\left(u_{i},v_{i-1},v_{i}\right),g\left(u_{i},v_{i-1}',v_{i}\right)\right\rangle \rho_{\mathbf{b}}^{i-1}(dv_{i-1})\rho_{\mathbf{b}}^{i-1}(dv_{i-1}')\rho_{\mathbf{b}}^{i}(dv_{i})\rho_{\mathbf{w}}^{i}(du_{i})\\
 & \quad+\mathbb{I}_{\theta_{i-1}\neq\theta_{i-1}'}\int\left\langle g\left(u_{i},v_{i-1},v_{i}\right),g\left(u_{i}',v_{i-1}',v_{i}\right)\right\rangle \rho_{\mathbf{b}}^{i-1}(dv_{i-1})\rho_{\mathbf{b}}^{i-1}(dv_{i-1}')\rho_{\mathbf{b}}^{i}(dv_{i})\rho_{\mathbf{w}}^{i}(du_{i})\rho_{\mathbf{w}}^{i}(du_{i}')\bigg]\\
 & =\frac{1}{M}\int\left\langle g\left(u_{i},v_{i-1},v_{i}\right),g\left(u_{i},v_{i-1}',v_{i}\right)\right\rangle \rho_{\mathbf{b}}^{i-1}(dv_{i-1})\rho_{\mathbf{b}}^{i-1}(dv_{i-1}')\rho_{\mathbf{b}}^{i}(dv_{i})\rho_{\mathbf{w}}^{i}(du_{i})\\
 & \quad+\frac{M-1}{M}\int\left|\int g\left(u_{i},v_{i-1},v_{i}\right)\rho_{\mathbf{w}}^{i}\left(du_{i}\right)\rho_{\mathbf{b}}^{i-1}\left(dv_{i-1}\right)\right|^{2}\rho_{\mathbf{b}}^{i}\left(dv_{i}\right),
\end{align*}
where in step $\left(a\right)$, $\left(\theta_{i-1}',\lambda_{i-1}'\right)\sim{\rm Unif}\left(\left[M\right]\right)\times P_{0}$
is an independent copy of $\left(\theta_{i-1},\lambda_{i-1}\right)$
and is independent of $\left(\theta_{i},\lambda_{i}\right)$, and
step $\left(b\right)$ is by the construction of $\mathfrak{p}_{i-1}$,
$\mathfrak{p}_{i}$ and $\mathfrak{q}_{i}$. We also notice that
\begin{align*}
 & \mathbb{E}_{C_{i}}\bigg[\Big\langle\mathbb{E}_{C_{i-1}}\left[g\left(w_{i}^{0}(C_{i-1},C_{i}),b_{i-1}^{0}(C_{i-1}),b_{i}^{0}(C_{i})\right)\right],\int g\left(u_{i},v_{i-1},b_{i}^{0}(C_{i})\right)\rho_{\mathbf{w}}^{i}\left(du_{i}\right)\rho_{\mathbf{b}}^{i-1}\left(dv_{i-1}\right)\Big\rangle\bigg]\\
 & =\int\left|\int g\left(u_{i},v_{i-1},v_{i}\right)\rho_{\mathbf{w}}^{i}\left(du_{i}\right)\rho_{\mathbf{b}}^{i-1}\left(dv_{i-1}\right)\right|^{2}\rho_{\mathbf{b}}^{i}\left(dv_{i}\right).
\end{align*}
Putting the last two displays together, one easily arrives at the
following:
\begin{align*}
 & \mathbb{E}_{C_{i}}\bigg[\bigg|\mathbb{E}_{C_{i-1}}\left[g\left(w_{i}^{0}(C_{i-1},C_{i}),b_{i-1}^{0}(C_{i-1}),b_{i}^{0}(C_{i})\right)\right]-H_{i}^{*}(t,X,b_{i}^{0}(C_{i}))\bigg|^{2}\bigg]\\
 & =\mathbb{E}_{C_{i}}\bigg[\bigg|\mathbb{E}_{C_{i-1}}\left[g\left(w_{i}^{0}(C_{i-1},C_{i}),b_{i-1}^{0}(C_{i-1}),b_{i}^{0}(C_{i})\right)\right]-\int g\left(u_{i},v_{i-1},b_{i}^{0}(C_{i})\right)\rho_{\mathbf{w}}^{i}\left(du_{i}\right)\rho_{\mathbf{b}}^{i-1}\left(dv_{i-1}\right)\bigg|^{2}\bigg]\\
 & \leq\frac{K}{M}\int\left|g\left(u_{i},v_{i-1},v_{i}\right)\right|^{2}\rho_{\mathbf{w}}^{i}\left(du_{i}\right)\rho_{\mathbf{b}}^{i-1}\left(dv_{i-1}\right)\rho_{\mathbf{b}}^{i}\left(dv_{i}\right)\\
 & \leq\frac{K_{T,L}}{M}.
\end{align*}
This illustrates the main use of the canonical neuronal embedding's
construction. Now from Assumption \ref{enu:Assump_forward}, one can
show that: 
\[
\mathbb{E}\left[\left|\mathbb{E}_{C_{i-1}}\left[g\left(w_{i}^{0}(C_{i-1},C_{i}),b_{i-1}^{0}(C_{i-1}),b_{i}^{0}(C_{i})\right)\right]-H_{i}\left(X,C_{i};W^{\infty}(t)\right)\right|^{2}\right]\le K_{T,L}D_{i-1}\left(t\right).
\]
Therefore,
\[
D_{i}\left(t\right)\le\frac{K_{T,L}}{M}+K_{T,L}D_{i-1}\left(t\right).
\]
One arrives at the claim from this relation.

The rest of the proof involves similar estimates and Gronwall's inequality.
Let us quickly describe the steps for completeness. Similar to the
above argument, for 
\[
D_{i}^{H}\left(t\right)=\mathbb{E}\left[\left|\Delta_{i}^{H}\left(Z,C_{i};W^{\infty}\left(t\right)\right)-\Delta_{i}^{H*}\left(t,Z,b_{i}^{0}(C_{i})\right)\right|^{2}\right],
\]
we can show that for $t\leq T$,
\[
D_{i}^{H}\left(t\right)\leq K_{T,L}\frac{\log^{1/2}M}{M}.
\]
With the previous two claims, one easily shows:
\[
D_{i}^{w}\left(t\right)\leq K_{T,L}\frac{\log^{1/2}M}{M},\qquad D_{i}^{b}\left(t\right)\leq K_{T,L}\frac{\log^{1/2}M}{M},
\]
where we define
\begin{align*}
D_{i}^{w}\left(t\right) & =\mathbb{E}\left[\left|\Delta_{i}^{w}\left(Z,C_{i-1},C_{i};W^{\infty}\left(t\right)\right)-\Delta_{i}^{w*}\left(t,Z,w_{i}^{0}(C_{i-1},C_{i}),b_{i-1}^{0}(C_{i-1}),b_{i}^{0}(C_{i})\right)\right|^{2}\right],\\
D_{i}^{b}\left(t\right) & =\mathbb{E}\left[\left|\Delta_{i}^{b}\left(Z,C_{i};W^{\infty}\left(t\right)\right)-\Delta_{i}^{b*}\left(t,Z,b_{i}^{0}(C_{i})\right)\right|^{2}\right].
\end{align*}
The next step is to show that for $2\leq i\leq L$, any $t\leq T$
and any $B\geq0$,
\[
\mathbb{E}\left[\left|\mathbb{E}_{Z}\left[\Delta_{i}^{w}\left(Z,C_{i-1},C_{i};W^{M}\left(t\right)\right)-\Delta_{i}^{w}\left(Z,C_{i-1},C_{i};W^{\infty}\left(t\right)\right)\right]\right|^{2}\right]^{1/2}\le K_{T,L}\left(\left(1+B\right)\left\langle W^{M}-W^{\infty}\right\rangle _{t}+e^{-KB^{2}}\right).
\]
With this, we then arrive at the following:
\begin{align*}
 & \mathbb{E}\left[\left|\Delta_{i}^{w}(Z,C_{i-1},C_{i};W^{M}(t))-\Delta_{i}^{w*}(t,Z,w_{i}^{0}(C_{i-1},C_{i}),b_{i-1}^{0}(C_{i-1}),b_{i}^{0}(C_{i}))\right|^{2}\right]^{1/2}\\
 & \le\left|D_{i}^{w}\left(t\right)\right|^{1/2}+\mathbb{E}\left[\left|\Delta_{i}^{w}(Z,C_{i-1},C_{i};W^{M}(t))-\Delta_{i}^{w}(Z,C_{i-1},C_{i};W^{\infty}(t))\right|^{2}\right]^{1/2}\\
 & \le K_{T,L}\left(\frac{\log^{1/4}M}{M^{1/2}}+\left(1+B\right)\left\langle W^{M}-W^{\infty}\right\rangle _{t}+e^{-KB^{2}}\right).
\end{align*}
A similar result holds for $\Delta_{i}^{b}$. Hence, we obtain that
for all $t\le T$, 
\[
\left\langle W^{M}-W^{\infty}\right\rangle _{t}\le K_{T,L}\int_{0}^{t}\left(\frac{\log^{1/4}M}{M^{1/2}}+\left(1+B\right)\left\langle W^{M}-W^{\infty}\right\rangle _{t}+e^{-KB^{2}}\right)ds.
\]
Since $\left\langle W^{M}-W^{\infty}\right\rangle _{0}=0$, Gronwall's
inequality implies that 
\[
\sup_{t\leq T}\left\langle W^{M}-W^{\infty}\right\rangle _{t}\le K_{T,L}\inf_{B>0}\left[\left(\frac{\log^{1/4}M}{M^{1/2}}+e^{-KB^{2}}\right)e^{K_{T,L}\left(1+B\right)}\right]\leq K_{T,L}\frac{1}{M^{0.499}},
\]
for sufficiently large $M$. This proves the main statement in Theorem
\ref{thm:iid dynamics}; the other statement follows easily.
\end{proof}

\section{Convergence to Global Optimum: Two-layer and Three-layer Networks
with I.i.d. Initialization\label{sec:global_convergence_iid}}

In this section, we prove several global convergence guarantees for
fully-connected neural networks (without biases) with $L\leq3$ and
i.i.d. initializations. A key element here is a certain universal
approximation property that holds at \textit{any} finite training
time. This is shown using a tool from algebraic topology.

\textbf{\textcolor{red}{}}

\subsection{Warm-up: The case $L=2$\label{subsec:global_convergence_iid_2}}

Our first result is that in the case of two-layer fully-connected
neural networks, the MF limit converges to the global optimum under
some genericity assumptions on the initialization distribution. Before
we proceed, we specify the two-layer network under consideration and
its training:
\begin{equation}
\hat{{\bf y}}\left(k,x\right)=\varphi_{2}\left(\mathbf{H}_{2}\left(k,x,1\right)\right),\qquad\mathbf{H}_{2}\left(k,x,1\right)=\frac{1}{n_{1}}\sum_{j_{1}=1}^{n_{1}}{\bf w}_{2}\left(k,j_{1},1\right)\varphi_{1}\left(\left\langle {\bf w}_{1}\left(k,j_{1}\right),x\right\rangle \right),\label{eq:two-layer-nn}
\end{equation}
in which ${\bf w}_{1}\left(k,j_{1}\right)\in\mathbb{R}^{d}$, $x\in\mathbb{R}^{d}$,
$\varphi_{1}:\;\mathbb{R}\to\mathbb{R}$, ${\bf w}_{2}\left(k,j_{1},1\right)\in\mathbb{R}$
and $\varphi_{2}:\;\mathbb{R}\to\mathbb{R}$. We train the network
with SGD w.r.t. the loss ${\cal L}:\;\mathbb{R}\times\mathbb{R}\to\mathbb{R}_{\geq0}$
and the data $z\left(k\right)=\left(x\left(k\right),y\left(k\right)\right)$
drawn independently at time $k$:
\begin{align*}
{\bf w}_{2}\left(k+1,j_{1},1\right) & ={\bf w}_{2}\left(k,j_{1},1\right)-\epsilon\partial_{2}{\cal L}\left(y\left(k\right),\hat{\mathbf{y}}\left(t,x\left(k\right)\right)\right)\varphi_{2}'\left(\mathbf{H}_{2}\left(t,x\left(k\right),1\right)\right)\varphi_{1}\left({\bf w}_{1}\left(k,j_{1}\right),x\left(k\right)\right),\\
{\bf w}_{1}\left(k+1,j_{1}\right) & ={\bf w}_{1}\left(k,j_{1}\right)-\epsilon\partial_{2}{\cal L}\left(y\left(k\right),\hat{\mathbf{y}}\left(t,x\left(k\right)\right)\right)\varphi_{2}'\left(\mathbf{H}_{2}\left(t,x\left(k\right),1\right)\right){\bf w}_{2}\left(k,j_{1},1\right)\varphi_{1}'\left(\left\langle {\bf w}_{1}\left(k,j_{1}\right),x\left(k\right)\right\rangle \right)x\left(k\right).
\end{align*}
Here $\epsilon\in\mathbb{R}_{>0}$ is the learning rate. The corresponding
MF ODEs are:
\begin{align*}
\frac{\partial}{\partial t}w_{2}\left(t,c_{1},1\right) & =-\mathbb{E}_{Z}\left[\partial_{2}{\cal L}\left(Y,\hat{y}\left(t,X\right)\right)\varphi_{2}'\left(H_{2}\left(t,X,1\right)\right)\varphi_{1}\left(w_{1}\left(t,c_{1}\right),X\right)\right],\\
\frac{\partial}{\partial t}w_{1}\left(t,c_{1}\right) & =-\mathbb{E}_{Z}\left[\partial_{2}{\cal L}\left(Y,\hat{y}\left(t,X\right)\right)\varphi_{2}'\left(H_{2}\left(t,X,1\right)\right)w_{2}\left(t,c_{1},1\right)\varphi_{1}'\left(\left\langle w_{1}\left(t,c_{1}\right),X\right\rangle \right)X\right],
\end{align*}
in which for $f_{1}:\;\Omega_{1}\to\mathbb{R}^{d}$ and $f_{2}:\;\Omega_{1}\to\mathbb{R}$,
we define:
\[
\hat{y}\left(x;f_{1},f_{2}\right)=\varphi_{2}\left(H_{2}\left(x;f_{1},f_{2}\right)\right),\qquad H_{2}\left(x;f_{1},f_{2}\right)=\mathbb{E}_{C_{1}}\left[f_{2}\left(C_{1}\right)\varphi_{1}\left(\left\langle f_{1}\left(C_{1}\right),x\right\rangle \right)\right],
\]
and $\hat{y}\left(t,x\right)$ and $H_{2}\left(t,x,1\right)$ are
short-hands notations when $f_{1}=w_{1}\left(t,\cdot\right)$, $f_{2}=w_{2}\left(t,\cdot,1\right)$.
It is easy to see that this network fits into our framework. In particular,
under the coupling procedure in Section \ref{subsec:Neuronal-Embedding},
our framework allows to study the following initialization scheme:
\[
\left\{ \left({\bf w}_{1}\left(0,j_{1}\right),{\bf w}_{2}\left(0,j_{1},1\right)\right)\right\} _{j_{1}\in\left[n_{1}\right]}\sim\rho^{0}\text{ i.i.d.}
\]
for suitable probability measure $\rho^{0}$ over $\mathbb{R}^{d}\times\mathbb{R}$.
In this case, $\rho^{0}={\rm Law}\left(w_{1}\left(0,C_{1}\right),w_{2}\left(0,C_{1},1\right)\right)$.
To measure the training quality, we consider the population loss:
\[
\mathscr{L}\left(f_{1},f_{2}\right)=\mathbb{E}_{Z}\left[{\cal L}\left(Y,\hat{y}\left(X;f_{1},f_{2}\right)\right)\right].
\]

\begin{assumption}
\label{assump:two-layers}Consider the MF limit corresponding to the
network (\ref{eq:two-layer-nn}), such that they are coupled together
by the coupling procedure in Section \ref{subsec:Neuronal-Embedding}.
We consider the following assumptions:
\begin{enumerate}
\item Initialization: The initialization law $\rho^{0}$ satisfies
\[
\max\left(\sup_{m\geq1}\frac{1}{\sqrt{m}}\mathbb{E}_{C_{1}}\left[\left|w_{1}\left(0,C_{1}\right)\right|^{m}\right]^{1/m},\quad\sup_{m\geq1}\frac{1}{\sqrt{m}}\mathbb{E}_{C_{1}}\left[\left|w_{2}\left(0,C_{1},1\right)\right|^{m}\right]^{1/m}\right)\leq K.
\]
\item Diversity: The support of $\rho^{0}$ contains the graph of a continuous
function $F:\;\mathbb{R}^{d}\to\mathbb{R}$ such that $\left|F\left(u\right)\right|\leq K$
for all $u\in\mathbb{R}^{d}$.
\item Regularity: $\varphi_{1}$ is $K$-bounded, $\varphi_{1}'$ and $\varphi_{2}'$
are $K$-bounded and $K$-Lipschitz, $\varphi_{2}'$ is non-zero everywhere,
$\partial_{2}{\cal L}\left(\cdot,\cdot\right)$ is $K$-Lipschitz
in the second variable and $K$-bounded\footnote{We denote by $\partial_{2}{\cal L}\left(\cdot,\cdot\right)$ the partial
derivative of ${\cal L}$ with respect to the second variable.}, and $\left|X\right|\leq K$ with probability $1$. 
\item Convergence: There exist limits $\bar{w}_{1}$ and $\bar{w}_{2}$
such that as $t\to\infty$, there exists a coupling $\pi_{t}$ of
$P_{1}$ and itself such that 
\[
\mathbb{E}_{\pi_{t}}\left[\left|\bar{w}_{2}(C_{1})\right|\left|w_{1}(t,C_{1}')-\bar{w}_{1}(C_{1})\right|+\left|w_{2}(t,C_{1}',1)-\bar{w}_{2}(C_{1})\right|\right]\to0
\]
for $(C_{1},C_{1}')\sim\pi_{t}$. Furthermore, ${\rm ess\text{-}sup}\left|\frac{\partial}{\partial t}w_{2}\left(t,C_{1},1\right)\right|\to0$.
\item Universal approximation: $\left\{ \varphi_{1}\left(\left\langle u,\cdot\right\rangle \right):\;u\in\mathbb{R}^{d}\right\} $
has dense span in $L^{2}\left({\cal P}_{X}\right)$ (the space of
square integrable functions w.r.t. the measure ${\cal P}_{X}$, which
is the distribution of the input $X$).
\end{enumerate}
\end{assumption}

Note that if $(w_{1}(t,C_{1}),w_{2}(t,C_{1},1))$ converges to $(\bar{w}_{1}(C_{1}),\bar{w}_{2}(C_{1}))$
in the Wasserstein-$2$ distance as $t\to\infty$, then one can prove
the first part of the convergence condition in Assumption \ref{assump:two-layers}
via the initialization and regularity conditions and Lemma \ref{lem:bounds MF a priori}.

We state the main result.
\begin{thm}
\label{thm:global-optimum-2}Consider the MF limit corresponding to
the network (\ref{eq:two-layer-nn}), such that they are coupled together
by the coupling procedure in Section \ref{subsec:Neuronal-Embedding}.
Under Assumption \ref{assump:two-layers}, the following hold:
\begin{itemize}
\item Case 1 (convex loss): If ${\cal L}$ is convex in the second variable,
then:
\[
\lim_{t\to\infty}\mathscr{L}\left(W\left(t\right)\right)=\inf_{f_{1},f_{2}}\mathscr{L}\left(f_{1},f_{2}\right)=\inf_{\tilde{y}}\mathbb{E}_{Z}\left[{\cal L}\left(Y,\tilde{y}\left(X\right)\right)\right].
\]
\item Case 2 (generic non-negative loss): Suppose $\partial_{2}{\cal L}\left(y,\hat{y}\right)=0$
implies ${\cal L}\left(y,\hat{y}\right)=0$. If $y=y(x)$ a function
of $x$, then $\mathscr{L}\left(W\left(t\right)\right)=0$ as $t\to\infty$.
\end{itemize}
\end{thm}

The proof is deferred to Appendix \ref{sec:Remaining-proofs-global-conv-iid}.
We refer the readers to Section \ref{subsec:three-layers-high-level-idea}
where we present a high-level proof plan for the three-layer case,
which is also applicable to the present two-layer case. The following
result is straightforward from Theorem \ref{thm:global-optimum-2}
and Corollary \ref{cor:gradient descent quality}.
\begin{cor}
Consider the neural network (\ref{eq:two-layer-nn}). Under the same
setting as Theorem \ref{thm:global-optimum-2}, in Case 1,
\[
\lim_{t\to\infty}\lim_{n_{1}\to\infty}\lim_{\epsilon\to0}\mathbb{E}_{Z}\left[{\cal L}\left(Y,\hat{{\bf y}}\left(\left\lfloor t/\epsilon\right\rfloor ,X\right)\right)\right]=\inf_{f_{1},f_{2}}\mathscr{L}\left(f_{1},f_{2}\right)=\inf_{\tilde{y}}\mathbb{E}_{Z}\left[{\cal L}\left(Y,\tilde{y}\left(X\right)\right)\right]
\]
in probability, and in Case 2, the same holds with the right-hand
side being $0$.
\end{cor}

Let us make a remark on the setting. Examples of suitable $\varphi_{1}$
include sigmoid/tanh activation, sinusoids and Gaussian pdf, whose
universal approximation is known \cite{cybenko1989approximation,chen1995universal}
(where we assume the convention that the last entry of the data input
$x$ is $1$). Examples of suitable $\varphi_{2}$ include smoothened
leaky-ReLU, sigmoid/tanh and linear activation. Examples of suitable
(and convex) loss ${\cal L}$ include Huber loss and exponential loss.
Importantly ${\cal L}$ needs not be convex. Assumption \ref{assump:two-layers}.4
is technical and does not seem removable. Note that this assumption
specifies the mode of convergence and is not an assumption on the
limits $\bar{w}_{1}$ and $\bar{w}_{2}$. In particular, the first
condition (convergence in moment) of in Assumption \ref{assump:two-layers}.4
is a common assumption in the literature \cite{chizat2018}. See also
Section \ref{sec:Discussion} where we further this discussion in
the context of prior works.

Regarding the uniform convergence condition ${\rm ess\text{-}sup}\left|\frac{\partial}{\partial t}w_{2}\left(t,C_{1},1\right)\right|\to0$
in Assumption \ref{assump:two-layers}.4, there is a converse relation
between global convergence and this condition. Thus this uniform convergence
condition gives a sharp characterization of global convergence.
\begin{prop}
\label{prop:converse-two-layers}Consider the MF limit corresponding
to the network (\ref{eq:two-layer-nn}), such that they are coupled
together by the coupling procedure in Section \ref{subsec:Neuronal-Embedding}.
Suppose that the initialization and regularity assumptions (i.e. the
first and third assumptions) of Assumption \ref{assump:two-layers}
hold, and that ${\cal L}(y,\hat{y})\to\infty$ as $|\hat{y}|\to\infty$
for each $y$. Then the following hold:
\begin{itemize}
\item Case 1 (convex loss): If ${\cal L}$ is convex in the second variable
and $\mathscr{L}\left(W\left(t\right)\right)\to\inf_{F}\mathscr{L}\left(F\right)$
as $t\to\infty$, then it must be that
\[
\sup_{c_{1}\in\Omega_{1}}\left|\frac{\partial}{\partial t}w_{2}\left(t,c_{1},1\right)\right|\to0\quad\text{as }t\to\infty.
\]
\item Case 2 (generic non-negative loss): Suppose that $\partial_{2}{\cal L}\left(y,\hat{y}\right)=0$
implies ${\cal L}\left(y,\hat{y}\right)=0$, and $y=y(x)$ is a function
of $x$. If $\mathscr{L}\left(W\left(t\right)\right)\to0$ as $t\to\infty$,
then the same conclusion also holds.
\end{itemize}
\end{prop}

Such a converse result was shown in the work \cite{wojtowytsch2020convergence}
for two-layer neural networks. It is also a special case of Proposition
\ref{prop:converse-multilayer}, which is a similar converse for multilayer
networks; hence we shall omit the proof of Proposition \ref{prop:converse-two-layers}.

\subsection{The case $L=3$\label{subsec:global_convergence_iid_3}}

We now turn to the case of three-layer networks, $L=3$. Our development
here applies insights already seen in the case $L=2$, most notably
the universal approximation property at the first layer and the topology
argument. Our present case is complicated by the presence of a third
layer, which requires extra conditions to ensure the same proof technique
to be applicable. We again stress that, similar to the case $L=2$,
here we do not rely critically on any convexity property, and the
same proof of global convergence should extend beyond the specific
network architecture to be considered here (the network (\ref{eq:three-layer-nn})
below).

Before we proceed, we specify the three-layer network under consideration
and its training. We also follow the development for i.i.d. initialization
in Section \ref{sec:iid_init}; in particular, we work with the infinite-$M$
MF limit.

\paragraph*{Three-layer network.}

For $\mathbf{n}=\left\{ n_{1},n_{2}\right\} $, we consider the following
neural network:
\begin{align}
\hat{{\bf y}}\left(k,x\right) & =\varphi_{3}\left(\mathbf{H}_{3}\left(k,x,1\right)\right),\label{eq:three-layer-nn}\\
\mathbf{H}_{3}\left(k,x,1\right) & =\frac{1}{n_{2}}\sum_{j_{2}=1}^{n_{2}}{\bf w}_{3}\left(k,j_{2},1\right)\varphi_{2}\left({\bf H}_{2}\left(k,x,j_{2}\right)\right),\nonumber \\
{\bf H}_{2}\left(k,x,j_{2}\right) & =\frac{1}{n_{1}}\sum_{j_{1}=1}^{n_{1}}{\bf w}_{2}\left(k,j_{1},j_{2}\right)\varphi_{1}\left(\left\langle {\bf w}_{1}\left(k,j_{1}\right),x\right\rangle \right),\nonumber 
\end{align}
in which $x\in\mathbb{R}^{d}$, ${\bf w}_{1}\left(k,j_{1}\right)\in\mathbb{R}^{d}$,
${\bf w}_{2}\left(k,j_{1},j_{2}\right)\in\mathbb{R}$, ${\bf w}_{3}\left(k,j_{2},1\right)\in\mathbb{R}$,
$\varphi_{1}:\;\mathbb{R}\to\mathbb{R}$, $\varphi_{2}:\;\mathbb{R}\to\mathbb{R}$,
$\varphi_{3}:\;\mathbb{R}\to\mathbb{R}$ and $k\in\mathbb{N}_{\geq0}$
indicating the discrete time. We train the network with SGD w.r.t.
the loss ${\cal L}:\;\mathbb{R}\times\mathbb{R}\to\mathbb{R}_{\geq0}$
and the data $z\left(k\right)=\left(x\left(k\right),y\left(k\right)\right)$
drawn independently at time $k$:
\begin{align*}
{\bf w}_{3}\left(k+1,j_{2},1\right) & ={\bf w}_{3}\left(k,j_{2},1\right)-\epsilon\xi_{3}\left(k\epsilon\right)\partial_{2}{\cal L}\left(y\left(k\right),\hat{\mathbf{y}}\left(k,x\left(k\right)\right)\right)\varphi_{3}'\left(\mathbf{H}_{3}\left(k,x\left(k\right),1\right)\right)\varphi_{2}\left({\bf H}_{2}\left(k,x\left(k\right),j_{2}\right)\right),\\
{\bf w}_{2}\left(k+1,j_{1},j_{2}\right) & ={\bf w}_{2}\left(k,j_{1},j_{2}\right)-\epsilon\Delta_{2}^{\mathbf{H}}\left(k,z\left(k\right),j_{2}\right)\varphi_{1}\left(\left\langle {\bf w}_{1}\left(k,j_{1}\right),x\right\rangle \right),\\
{\bf w}_{1}\left(k+1,j_{1}\right) & ={\bf w}_{1}\left(k,j_{1}\right)-\epsilon\bigg(\frac{1}{n_{2}}\sum_{j_{2}=1}^{n_{2}}\Delta_{2}^{\mathbf{H}}\left(k,z\left(k\right),j_{2}\right){\bf w}_{2}\left(k,j_{1},j_{2}\right)\bigg)\varphi_{1}'\left(\left\langle {\bf w}_{1}\left(k,j_{1}\right),x\left(k\right)\right\rangle \right)x\left(k\right),
\end{align*}
in which
\[
\Delta_{2}^{\mathbf{H}}\left(k,z,j_{2}\right)=\partial_{2}{\cal L}\left(y,\hat{\mathbf{y}}\left(k,x\right)\right)\varphi_{3}'\left(\mathbf{H}_{3}\left(k,x,1\right)\right){\bf w}_{3}\left(k,j_{2},1\right)\varphi_{2}'\left({\bf H}_{2}\left(k,x,j_{2}\right)\right).
\]
Here $\epsilon\in\mathbb{R}_{>0}$ is the learning rate and $\xi_{3}:\;\mathbb{R}_{\geq0}\mapsto\mathbb{R}_{\geq0}$
is the learning rate schedule for the third layer. Note that here
we only consider non-negative $\xi_{3}$. We consider $\left(\rho^{1},\rho^{2},\rho^{3}\right)$-i.i.d.
initialization:
\[
\left\{ \mathbf{w}_{1}\left(0,j_{1}\right)\right\} _{j_{1}\in\left[n_{1}\right]}\sim\rho^{1}\text{ i.i.d.},\quad\left\{ \mathbf{w}_{2}\left(0,j_{1},j_{2}\right)\right\} _{j_{1}\in\left[n_{1}\right],\;j_{2}\in\left[n_{2}\right]}\sim\rho^{2}\text{ i.i.d.},\quad\left\{ \mathbf{w}_{3}\left(0,j_{2},1\right)\right\} _{j_{2}\in\left[n_{2}\right]}\sim\rho^{3}\text{ i.i.d.}
\]
all independently.

\paragraph*{Infinite-$M$ MF limit.}

Following Section \ref{subsec:Infinite-M-canonical}, in the current
context of the three-layer neural network (\ref{eq:three-layer-nn}),
we define the dynamics $w_{1}^{*}$, $w_{2}^{*}$ and $w_{3}^{*}$
as follows:
\begin{align*}
w_{3}^{*}\left(t,u_{3}\right) & =u_{3}-\int_{0}^{t}\xi_{3}\left(s\right)\mathbb{E}_{Z}\left[\partial_{2}{\cal L}\left(Y,\hat{y}^{*}\left(s,X\right)\right)\varphi_{3}'\left(H_{3}^{*}\left(s,X\right)\right)\varphi_{2}\left(H_{2}^{*}\left(s,X,u_{3}\right)\right)\right]ds,\\
w_{2}^{*}\left(t,u_{1},u_{2},u_{3}\right) & =u_{2}-\int_{0}^{t}\mathbb{E}_{Z}\left[\Delta_{2}^{H*}\left(s,Z,u_{3}\right)\varphi_{1}\left(\left\langle w_{1}^{*}\left(s,u_{1}\right),X\right\rangle \right)\right]ds,\\
w_{1}^{*}\left(t,u_{1}\right) & =u_{1}-\int_{0}^{t}\mathbb{E}_{Z}\left[\left(\int\Delta_{2}^{H*}\left(s,Z,u_{3}\right)w_{2}^{*}\left(s,u_{1},u_{2},u_{3}\right)\rho^{2}\left(du_{2}\right)\rho^{3}\left(du_{3}\right)\right)\varphi_{1}'\left(\left\langle w_{1}^{*}\left(s,u_{1}\right),X\right\rangle \right)X\right]ds,\\
 & \forall u_{i}\in{\rm supp}\left(\rho^{i}\right)\;{\rm for}\;i=1,2,3,
\end{align*}
in which
\begin{align*}
\hat{y}^{*}\left(x;f_{1},f_{2},f_{3}\right) & =\varphi_{3}\left(H_{3}^{*}\left(x;f_{1},f_{2},f_{3}\right)\right),\\
H_{3}^{*}\left(x;f_{1},f_{2},f_{3}\right) & =\int f_{3}\left(u_{3}\right)\varphi_{2}\left(H_{2}^{*}\left(x,u_{3};f_{1},f_{2}\right)\right)\rho^{3}\left(du_{3}\right),\\
H_{2}^{*}\left(x,u_{3};f_{1},f_{2}\right) & =\int f_{2}\left(u_{1},u_{2},u_{3}\right)\varphi_{1}\left(\left\langle f_{1}\left(u_{1}\right),x\right\rangle \right)\rho^{1}\left(du_{1}\right)\rho^{2}\left(du_{2}\right),\\
\Delta_{2}^{H*}\left(z,u_{3};f_{1},f_{2},f_{3}\right) & =\partial_{2}{\cal L}\left(y,\hat{y}^{*}\left(x;f_{1},f_{2},f_{3}\right)\right)\varphi_{3}'\left(H_{3}^{*}\left(x;f_{1},f_{2},f_{3}\right)\right)f_{3}\left(u_{3}\right)\varphi_{2}'\left(H_{2}^{*}\left(x,u_{3};f_{1},f_{2}\right)\right),
\end{align*}
and $\hat{y}^{*}\left(t,x\right)$, $H_{3}^{*}\left(t,x\right)$,
$H_{2}^{*}\left(t,x,u_{3}\right)$ and $\Delta_{2}^{H*}\left(t,z,u_{3}\right)$
are their short-hand notations for when $f_{1}=w_{1}^{*}\left(t,\cdot\right)$,
$f_{2}=w_{2}^{*}\left(t,\cdot,\cdot,\cdot\right)$ and $f_{3}=w_{3}^{*}\left(t,\cdot\right)$.
Let us also define $W^{*}\left(t\right)=\left\{ w_{1}^{*}\left(t,\cdot\right),\;w_{2}^{*}\left(t,\cdot,\cdot,\cdot\right),\;w_{3}^{*}\left(t,\cdot\right)\right\} $.
To measure the training quality for $W^{*}\left(t\right)$, we consider
the population loss $\mathscr{L}\left(W^{*}\left(t\right)\right)$
in which
\[
\mathscr{L}\left(f_{1},f_{2},f_{3}\right)=\mathbb{E}_{Z}\left[{\cal L}\left(Y,\hat{y}^{*}\left(X;f_{1},f_{2},f_{3}\right)\right)\right].
\]
This gives the infinite-$M$ MF limit for the neural network (\ref{eq:three-layer-nn}).
\begin{assumption}
\label{assump:three-layers}Consider the infinite-$M$ MF limit $W^{*}\left(t\right)$
corresponding to the network (\ref{eq:three-layer-nn}). We consider
the following assumptions:
\begin{enumerate}
\item Initialization: The functions $\left\{ w_{i}^{0}\right\} _{i=1,2,3}$
satisfy
\[
\sup_{m\geq1}\frac{1}{\sqrt{m}}\left(\int\left|u_{1}\right|^{m}\rho^{1}\left(du_{1}\right)\right)^{1/m}\leq K,\qquad\sup_{m\geq1}\frac{1}{\sqrt{m}}\left(\int\left|u_{2}\right|^{m}\rho^{2}\left(du_{2}\right)\right)^{1/m}\leq K,
\]
\[
\sup_{m\geq1}\frac{1}{\sqrt{m}}\left(\int\left|u_{3}\right|^{m}\rho^{3}\left(du_{3}\right)\right)^{1/m}\leq K.
\]
\item Diversity: The support of $\rho^{1}$ is $\mathbb{R}^{d}$.
\item Regularity: $\varphi_{1}$ and $\varphi_{2}$ are $K$-bounded, $\varphi_{1}'$,
$\varphi_{2}'$ and $\varphi_{3}'$ are $K$-bounded and $K$-Lipschitz,
$\varphi_{2}'$ and $\varphi_{3}'$ are non-zero everywhere, $\partial_{2}{\cal L}\left(\cdot,\cdot\right)$
is $K$-Lipschitz in the second variable and $K$-bounded, and $\left|X\right|\leq K$
with probability $1$.
\item Convergence: There exist functions $\bar{w}_{1}$, $\bar{w}_{2}$
and $\bar{w}_{3}$ such that as $t\to\infty$, there exists a coupling
$\pi_{t}$ of $\rho^{1}\times\rho^{2}\times\rho^{3}$ and itself such
that 
\begin{align*}
\int\left(1+\left|\bar{w}_{3}(u_{3})\right|\right)\left|\bar{w}_{3}(u_{3})\right|\left|\bar{w}_{2}(u_{1},u_{2},u_{3})\right|\left|w_{1}^{*}(t,u_{1}')-\bar{w}_{1}(u_{1})\right|d\pi_{t}(u_{1},u_{2},u_{3},u_{1}',u_{2}',u_{3}') & \to0,\\
\int\left(1+\left|\bar{w}_{3}(u_{3})\right|\right)\left|\bar{w}_{3}(u_{3})\right|\left|w_{2}^{*}\left(t,u_{1}',u_{2}',u_{3}'\right)-\bar{w}_{2}\left(u_{1},u_{2},u_{3}\right)\right|d\pi_{t}(u_{1},u_{2},u_{3},u_{1}',u_{2}',u_{3}') & \to0,\\
\int\left(1+\left|\bar{w}_{3}(u_{3})\right|\right)\left|w_{3}^{*}\left(t,u_{3}'\right)-\bar{w}_{3}(u_{3})\right|d\pi_{t}(u_{1},u_{2},u_{3},u_{1}',u_{2}',u_{3}') & \to0.
\end{align*}
Furthermore, 
\[
\underset{U_{i}\sim\rho^{i}}{{\rm ess\text{-}sup}}\left|\frac{\partial}{\partial t}w_{2}^{*}\left(t,U_{1},U_{2},U_{3}\right)\right|\to0.
\]
\item Universal approximation: $\left\{ \varphi_{1}\left(\left\langle u,\cdot\right\rangle \right):\;u\in\mathbb{R}^{d}\right\} $
has dense span in $L^{2}\left({\cal P}_{X}\right)$ (the space of
square integrable functions w.r.t. the measure ${\cal P}_{X}$, which
is the distribution of the input $X$).
\end{enumerate}
\end{assumption}

As a remark, the first part of the convergence assumption follows
from the convergence of the tuple $(w_{1}^{*}\left(t,\cdot\right),w_{2}^{*}\left(t,\cdot,\cdot,\cdot\right),w_{3}^{*}\left(t,\cdot\right))$
to $(\bar{w}_{1},\bar{w}_{2},\bar{w}_{3})$ in the Wasserstein-$4$
distance, i.e.
\begin{align*}
 & \inf_{\pi}\int\Big(|w_{1}^{*}(t,u_{1}')-\bar{w}_{1}(u_{1})|^{4}+|w_{2}^{*}(t,u_{1}',u_{2}',u_{3}')-\bar{w}_{2}(u_{1},u_{2},u_{3})|^{4}\\
 & \qquad+|w_{3}^{*}(t,u_{3}')-\bar{w}_{3}(u_{3})|^{4}\Big)d\pi(u_{1},u_{2},u_{3},u_{1}',u_{2}',u_{3}')\to0,
\end{align*}
where the infimum is over all couplings $\pi$ of $\rho^{1}\times\rho^{2}\times\rho^{3}$
and itself. In particular, one can prove so with the initialization
and regularity conditions and Lemma \ref{lem:bounds MF a priori}.
\begin{thm}
\label{thm:global-optimum-3}Consider the infinite-$M$ MF limit $W^{*}\left(t\right)$
corresponding to the network (\ref{eq:three-layer-nn}), under Assumption
\ref{assump:three-layers}. Further assume either:
\begin{itemize}
\item (untrained third layer) $\int\mathbb{I}\left(u_{3}\neq0\right)\rho^{3}\left(du_{3}\right)>0$
and $\xi_{3}\left(\cdot\right)=0$ (and hence $w_{3}^{*}\left(t,u_{3}\right)=u_{3}$
unchanged at all $t\geq0$), or
\item (trained third layer) $\xi_{3}\left(\cdot\right)=1$ and the initialization
satisfies that $\mathscr{L}\left(w_{1}^{0},w_{2}^{0},w_{3}^{0}\right)<\mathbb{E}_{Z}\left[{\cal L}\left(Y,\varphi_{3}\left(0\right)\right)\right]$.
\end{itemize}
Then the following hold:
\begin{itemize}
\item Case 1 (convex loss): If ${\cal L}$ is convex in the second variable,
then:
\[
\lim_{t\to\infty}\mathscr{L}\left(W^{*}\left(t\right)\right)=\inf_{\tilde{y}}\mathbb{E}_{Z}\left[{\cal L}\left(Y,\tilde{y}\left(X\right)\right)\right].
\]
\item Case 2 (generic non-negative loss): Suppose that $\partial_{2}{\cal L}\left(y,\hat{y}\right)=0$
implies ${\cal L}\left(y,\hat{y}\right)=0$. If $y=y(x)$ a function
of $x$, then $\mathscr{L}\left(W^{*}\left(t\right)\right)\to0$ as
$t\to\infty$.
\end{itemize}
\end{thm}

The proof is deferred to Section \ref{subsec:Proof-three-layers}.
While global convergence is proven via the infinite-$M$ MF limit
$W^{*}$, it is easy to adapt the proof to prove the same for the
canonical MF limits that are described in Section \ref{subsec:Canonical-neuronal-embeddings},
giving a statement similar to the two-layer case (Theorem \ref{thm:global-optimum-2}).
By working with the infinite-$M$ limit, specifically combining Theorem
\ref{thm:global-optimum-3} with Corollary \ref{cor:iid_tracking},
we immediately obtain the following result.
\begin{cor}
Consider the neural network (\ref{eq:three-layer-nn}). Under the
same setting as Theorem \ref{thm:global-optimum-3}, in Case 1,
\[
\lim_{t\to\infty}\lim_{n_{1},n_{2}\to\infty}\lim_{\epsilon\to0}\mathbb{E}_{Z}\left[{\cal L}\left(Y,\hat{{\bf y}}\left(\left\lfloor t/\epsilon\right\rfloor ,X\right)\right)\right]=\inf_{\tilde{y}}\mathbb{E}_{Z}\left[{\cal L}\left(Y,\tilde{y}\left(X\right)\right)\right].
\]
in probability, and in Case 2, the same holds with the right-hand
side being $0$. Here $n_{1},n_{2}\to\infty$ in such a way that $n_{\min}\to\infty$
and $n_{\min}^{-c}\log n_{\max}\to0$ for any $c>0$, with $n_{\min}=\min\left\{ n_{1},n_{2}\right\} $
and $n_{\max}=\max\left\{ n_{1},n_{2}\right\} $.
\end{cor}

Similar to Section \ref{subsec:global_convergence_iid_2}, here we
also have a converse relation between global convergence and the essential
supremum condition in Assumption \ref{assump:three-layers}.4.
\begin{prop}
\label{prop:converse-three-layers}Consider the infinite-$M$ MF limit
$W^{*}\left(t\right)$ corresponding to the network (\ref{eq:three-layer-nn}).
Suppose that the initialization and regularity assumptions (i.e. the
first and third assumptions) of Assumption \ref{assump:three-layers}
hold, and that ${\cal L}(y,\hat{y})\to\infty$ as $|\hat{y}|\to\infty$
for each $y$. Further assume that there exists $\bar{w}_{3}$ such
that as $t\to\infty$, there is a coupling $\pi_{t}^{3}$ of $\rho^{3}$
and itself such that 
\[
\int\left|w_{3}^{*}\left(t,u_{3}'\right)-\bar{w}_{3}(u_{3})\right|d\pi_{t}^{3}\left(u_{3},u_{3}'\right)\to0.
\]
Then the following hold:
\begin{itemize}
\item Case 1 (convex loss): If ${\cal L}$ is convex in the second variable
and
\[
\lim_{t\to\infty}\mathscr{L}\left(W^{*}\left(t\right)\right)=\inf_{\tilde{y}}\mathbb{E}_{Z}\left[{\cal L}\left(Y,\tilde{y}\left(X\right)\right)\right],
\]
then it must be that
\[
\sup_{u_{1}\in\mathbb{R}^{d},\;u_{2}\in{\rm supp}\left(\rho^{2}\right)}\mathbb{E}_{U_{3}\sim\rho^{3}}\left[\left|\frac{\partial}{\partial t}w_{2}^{*}\left(t,u_{1},u_{2},U_{3}\right)\right|\right]\to0\quad\text{as }t\to\infty.
\]
\item Case 2 (generic non-negative loss): Suppose that $\partial_{2}{\cal L}\left(y,\hat{y}\right)=0$
implies ${\cal L}\left(y,\hat{y}\right)=0$, and $y=y(x)$ is a function
of $x$. If $\mathscr{L}\left(W^{*}\left(t\right)\right)\to0$ as
$t\to\infty$, then the same conclusion also holds.
\end{itemize}
\end{prop}

We prove Proposition \ref{prop:converse-three-layers} in Appendix
\ref{sec:Remaining-proofs-global-conv-iid}.

\subsubsection{High-level idea of the proof\label{subsec:three-layers-high-level-idea}}

Before we proceed, we give a high-level discussion of the proof of
Theorem \ref{thm:global-optimum-3}. This is meant to provide intuitions
and explain the technical crux, so our discussion may simplify and
deviate from the actual proof. Our first insight is to look at the
second layer's weight $w_{2}^{*}$. Recall that
\[
\frac{\partial}{\partial t}w_{2}^{*}\left(t,u_{1},u_{2},u_{3}\right)=-\mathbb{E}_{Z}\left[\Delta_{2}^{H*}\left(Z,u_{3};W^{*}\left(t\right)\right)\varphi_{1}\left(\left\langle w_{1}^{*}\left(t,u_{1}\right),X\right\rangle \right)\right].
\]
At convergence time $t=\infty$, we expect to have zero movement and
hence, denoting $\bar{W}=\left\{ \bar{w}_{1},\bar{w}_{2},\bar{w}_{3}\right\} $:
\[
\mathbb{E}_{Z}\left[\Delta_{2}^{H*}\left(Z,u_{3};\bar{W}\right)\varphi_{1}\left(\left\langle \bar{w}_{1}\left(u_{1}\right),X\right\rangle \right)\right]=0
\]
for $u_{1}\in{\rm supp}\left(\rho^{1}\right)$, $u_{3}\in{\rm supp}\left(\rho^{3}\right)$.
Suppose for the moment that we are allowed to make an additional (strong)
assumption on the limit $\bar{w}_{1}$: ${\rm supp}\left(\bar{w}_{1}\left(U_{1}\right)\right)=\mathbb{R}^{d}$
for $U_{1}\sim\rho^{1}$. It implies that the universal approximation
property, described in Assumption \ref{assump:three-layers}.5, holds
at $t=\infty$; more specifically, it implies $\left\{ \varphi_{1}\left(\left\langle \bar{w}_{1}\left(u_{1}\right),\cdot\right\rangle \right):\;u_{1}\in{\rm supp}\left(\rho^{1}\right)\right\} $
has dense span in $L^{2}\left({\cal P}_{X}\right)$. This thus yields
\[
\mathbb{E}_{Z}\left[\Delta_{2}^{H*}\left(Z,u_{3};\bar{W}\right)\middle|X=x\right]=0
\]
for ${\cal P}$-almost every $x$. Recalling the definition of $\Delta_{2}^{H*}$,
one can then easily show that
\[
\mathbb{E}_{Z}\left[\partial_{2}{\cal L}\left(Y,\hat{y}^{*}\left(x;\bar{W}\right)\right)\middle|X=x\right]=0.
\]
Global convergence follows immediately; for example, in Case 2 of
Theorem \ref{thm:global-optimum-3}, this is equivalent to that $\partial_{2}{\cal L}\left(y\left(x\right),\hat{y}^{*}\left(x;\bar{W}\right)\right)=0$
and hence ${\cal L}\left(y\left(x\right),\hat{y}^{*}\left(x;\bar{W}\right)\right)=0$
for ${\cal P}$-almost every $x$. In short, the gradient flow structure
of the dynamics of $w_{2}^{*}$ provides a seamless way to obtain
global convergence. Furthermore there is no critical reliance on convexity.

However this plan of attack has a potential flaw in the strong assumption
that ${\rm supp}\left(\bar{w}_{1}\left(U_{1}\right)\right)=\mathbb{R}^{d}$,
i.e. the universal approximation property holds at convergence time.
Indeed there are setups where it is desirable that ${\rm supp}\left(\bar{w}_{1}\left(U_{1}\right)\right)\neq\mathbb{R}^{d}$
\cite{mei2018mean,chizat2019sparse}; for instance, it is the case
where the neural network is to learn some ``sparse and spiky'' solution,
and hence the weight distribution at convergence time, if successfully
trained, cannot have full support. On the other hand, one can entirely
expect that if ${\rm supp}\left(w_{1}^{*}\left(0,U_{1}\right)\right)=\mathbb{R}^{d}$
initially at $t=0$, then ${\rm supp}\left(w_{1}^{*}\left(t,U_{1}\right)\right)=\mathbb{R}^{d}$
at \textsl{any} finite $t\geq0$. The crux of our proof is to show
the latter without assuming ${\rm supp}\left(\bar{w}_{1}\left(U_{1}\right)\right)=\mathbb{R}^{d}$.
This is done via an algebraic topology argument, in which the mapping
$\left(t,u\right)\mapsto M\left(t,u\right)$ that maps from $\left(t,w_{1}^{*}\left(0,u_{1}\right)\right)=\left(t,u_{1}\right)$
to $w_{1}^{*}\left(t,u_{1}\right)$ is shown to preserves a homotopic
structure through time.

\subsection{Proof of Theorem \ref{thm:global-optimum-3}\label{subsec:Proof-three-layers}}

First using an algebraic topology argument, we show that if $w_{1}^{*}\left(0,U_{1}\right)=U_{1}\sim\rho^{1}$
has full support, then so is $w_{1}^{*}\left(t,U_{1}\right)$ at any
time $t$. Note that the following result holds beyond the setting
of Theorem \ref{thm:global-optimum-3}.
\begin{lem}
\label{lem:full-support-3}Assume $L=3$ and $\mathbb{W}_{1}=\mathbb{R}^{d}$
(for some positive integer $d$), along with Assumptions \ref{enu:Assump_lrSchedule}-\ref{enu:Assump_backward}
and \ref{assump:init}. Under a $\left(\rho_{\mathbf{w}}^{1},\rho_{\mathbf{w}}^{2},\rho_{\mathbf{w}}^{3}\right)$-i.i.d.
initialization, consider the infinite-$M$ limit of the canonical
MF limits as described in Section \ref{subsec:Infinite-M-canonical},
and in particular the dynamics $\left\{ w_{i}^{*}\right\} _{i=1,2,3}$.
Here we disregard the biases by considering $\xi_{2}^{\mathbf{b}}\left(\cdot\right)=\xi_{3}^{\mathbf{b}}\left(\cdot\right)=0$
and $b_{2}\left(0,\cdot\right)=b_{3}\left(0,\cdot\right)=0$. Assume
that
\[
\sigma_{2}^{\mathbf{w}}\left(\Delta,w,b,g,h\right)=\bar{\sigma}_{2}^{\mathbf{w}}\left(\Delta,g,h\right),
\]
for some function $\bar{\sigma}_{2}^{\mathbf{w}}$, i.e. $\sigma_{2}^{\mathbf{w}}$
does not depend on the second and third variables. Suppose that the
support of $\rho_{\mathbf{w}}^{1}$ is $\mathbb{W}_{1}$. Then for
all finite time $t$, the support of ${\rm Law}_{U_{_{1}}\sim\rho_{\mathbf{w}}^{1}}\left(w_{1}^{*}\left(t,U_{1}\right)\right)$
is $\mathbb{W}_{1}$.
\end{lem}

\begin{proof}
Specialized to the current setting, $w_{1}^{*}$ and $w_{2}^{*}$
satisfy 
\begin{align*}
\frac{\partial}{\partial t}w_{1}^{*}(t,u_{1}) & =-\xi_{1}^{\mathbf{w}}\left(t\right)\mathbb{E}_{Z}\left[\sigma_{1}^{\mathbf{w}}\left(\int h\left(t,Z,u_{1},u_{2},u_{3}\right)\rho_{\mathbf{w}}^{2}\left(du_{2}\right)\rho_{\mathbf{w}}^{3}\left(du_{3}\right),w_{1}^{*}\left(t,u_{1}\right),X\right)\right],\\
\frac{\partial}{\partial t}w_{2}^{*}(t,u_{1},u_{2},u_{3}) & =-\xi_{2}^{\mathbf{w}}\left(t\right)\mathbb{E}_{Z}\left[g\left(t,Z,u_{1},u_{3}\right)\right],
\end{align*}
for all $u_{1}\in{\rm supp}\left(\rho_{\mathbf{w}}^{1}\right)$, $u_{2}\in{\rm supp}\left(\rho_{\mathbf{w}}^{2}\right)$
and $u_{3}\in{\rm supp}\left(\rho_{\mathbf{w}}^{3}\right)$ where
\begin{align*}
g\left(t,z,u_{1},u_{3}\right) & =\bar{\sigma}_{2}^{\mathbf{w}}\left(\Delta_{2}^{H*}\left(t,z,u_{3}\right),H_{2}^{*}\left(t,x,u_{3}\right),H_{1}^{*}\left(t,x,u_{1}\right)\right),\\
h\left(t,z,u_{1},u_{2},u_{3}\right) & =\sigma_{1}^{\mathbf{H}}\left(\Delta_{2}^{H*}\left(t,z,u_{3}\right),w_{2}^{*}\left(t,u_{1},u_{2},u_{3}\right),0,H_{2}^{*}\left(t,x,u_{3}\right),H_{1}^{*}\left(t,x,u_{1}\right)\right)\\
 & =\sigma_{1}^{\mathbf{H}}\left(\Delta_{2}^{H*}\left(t,z,u_{3}\right),u_{2}-\int_{0}^{t}\xi_{2}^{\mathbf{w}}\left(s\right)\mathbb{E}_{Z}\left[g\left(s,Z,u_{1},u_{3}\right)\right]ds,0,H_{2}^{*}\left(t,x,u_{3}\right),H_{1}^{*}\left(t,x,u_{1}\right)\right).
\end{align*}
Here we have shortened the notations to remove dependency on the biases:
\begin{align*}
w_{2}^{*}(t,u_{1},u_{2},u_{3}) & \equiv w_{2}^{*}(t,u_{1},u_{2},u_{3},0),\\
\Delta_{2}^{H*}\left(t,z,u_{3}\right) & \equiv\Delta_{2}^{H*}\left(t,z,u_{3},0\right),\\
H_{2}^{*}\left(t,x,u_{3}\right) & \equiv H_{2}^{*}\left(t,x,u_{3},0\right).
\end{align*}
We recall the initialization
\[
w_{1}^{*}\left(0,u_{1}\right)=u_{1},\quad w_{2}^{*}\left(0,\cdot,u_{2},\cdot\right)=u_{2},\quad w_{2}^{*}\left(0,u_{3},\cdot\right)=u_{3}.
\]
In the following, we define $K_{t}$ to be a generic constant that
changes with $t$ and is finite with finite $t$. We proceed in several
steps.

\paragraph*{Step 1.}

We study the function $h$. We have from Assumptions \ref{enu:Assump_backward}
and \ref{enu:Assump_lrSchedule}:
\begin{align*}
\left|\Delta_{3}^{w*}\left(t,z,\cdot,\cdot\right)\right| & \leq K\left(1+\left|\Delta_{3}^{H*}\left(t,z\right)\right|\right)\leq K,\\
\left|w_{3}^{*}\left(t,u_{3},\cdot\right)\right| & \leq\left|w_{3}^{*}\left(0,u_{3},\cdot\right)\right|+K_{t}=\left|u_{3}\right|+K_{t},\\
\left|\Delta_{2}^{H*}\left(t,z,u_{3}\right)\right| & \leq K\left(1+\left|\Delta_{3}^{H*}\left(t,z\right)\right|\right)\left(1+\left|w_{3}^{*}\left(t,u_{3},\cdot\right)\right|\right)\leq K_{t}\left(1+\left|u_{3}\right|\right).
\end{align*}
Consequently by Assumption \ref{enu:Assump_backward},
\begin{align*}
\left|g\left(t,z,u_{1},u_{3}\right)\right| & \leq K\left(1+\left|\Delta_{2}^{H*}\left(t,z,u_{3}\right)\right|\right)\\
 & \leq K_{t}\left(1+\left|u_{3}\right|\right),\\
\left|g\left(t,z,u_{1},u_{3}\right)-g\left(t,z,u_{1}',u_{3}\right)\right| & \leq K\left(1+\left|\Delta_{2}^{H*}\left(t,z,u_{3}\right)\right|\right)\left|H_{1}^{*}\left(t,x,u_{1}\right)-H_{1}^{*}\left(t,x,u_{1}'\right)\right|\\
 & \leq K_{t}\left(1+\left|u_{3}\right|\right)\left|H_{1}^{*}\left(t,x,u_{1}\right)-H_{1}^{*}\left(t,x,u_{1}'\right)\right|,
\end{align*}
for all $u_{1}\in{\rm supp}\left(\rho_{\mathbf{w}}^{1}\right)$ and
$u_{3}\in\textrm{supp}(\rho_{{\bf w}}^{3})$. Using these bounds and
Assumption \ref{enu:Assump_backward}, we obtain:
\begin{align*}
 & \left|h\left(t,z,u_{1},u_{2},u_{3}\right)-h\left(t,z,u_{1}',u_{2},u_{3}\right)\right|\\
 & \leq K\left(1+\left|\Delta_{2}^{H*}\left(t,z,u_{3}\right)\right|\right)\int_{0}^{t}\mathbb{E}_{Z}\left[\left|g\left(s,Z,u_{1},u_{3}\right)-g\left(s,Z,u_{1}',u_{3}\right)\right|\right]ds\\
 & \quad+K\left(1+\left|\Delta_{2}^{H*}\left(t,z,u_{3}\right)\right|\right)\left(1+\left|u_{2}\right|+\int_{0}^{t}\mathbb{E}_{Z}\left[\left|g\left(s,Z,u_{1},u_{3}\right)\right|\right]ds\right)\left|H_{1}^{*}\left(t,x,u_{1}\right)-H_{1}^{*}\left(t,x,u_{1}'\right)\right|\\
 & \leq K_{t}\left(1+\left|u_{2}\right|^{2}+\left|u_{3}\right|^{2}\right)\sup_{s\leq t}\left|H_{1}^{*}\left(s,x,u_{1}\right)-H_{1}^{*}\left(s,x,u_{1}'\right)\right|,
\end{align*}
as well as that
\begin{align*}
\left|h\left(t,z,u_{1},u_{2},u_{3}\right)\right| & \leq K\left(1+\left|\Delta_{2}^{H*}\left(t,z,u_{3}\right)\right|\right)\left(1+\left|u_{2}\right|+\int_{0}^{t}\mathbb{E}_{Z}\left[\left|g\left(s,Z,u_{1},u_{3}\right)\right|\right]ds\right)\\
 & \leq K_{t}\left(1+\left|u_{2}\right|^{2}+\left|u_{3}\right|^{2}\right).
\end{align*}
These are the desired bounds for $h$.

\paragraph{Step 2.}

We show that for an arbitrary $T\geq0$, $w_{1}^{*}:\;\left[0,T\right]\times\mathbb{W}_{1}\to\mathbb{W}_{1}$
is continuous. Using the bound for the function $h$ in Step 1 and
Assumptions \ref{enu:Assump_forward}-\ref{enu:Assump_backward},
for $u_{1},u_{1}'\in\mathbb{W}_{1}$,
\begin{align*}
 & \left|\frac{d}{dt}\left(w_{1}^{*}\left(t,u_{1}\right)-w_{1}^{*}\left(t,u_{1}'\right)\right)\right|\\
 & =\bigg|\xi_{1}^{\mathbf{w}}\left(t\right)\mathbb{E}_{Z}\bigg[\sigma_{1}^{\mathbf{w}}\left(\int h\left(t,Z,u_{1}',u_{2},u_{3}\right)\rho_{\mathbf{w}}^{2}\left(du_{2}\right)\rho_{\mathbf{w}}^{3}\left(du_{3}\right),w_{1}^{*}\left(t,u_{1}'\right),X\right)\\
 & \qquad\qquad-\sigma_{1}^{\mathbf{w}}\left(\int h\left(t,Z,u_{1},u_{2},u_{3}\right)\rho_{\mathbf{w}}^{2}\left(du_{2}\right)\rho_{\mathbf{w}}^{3}\left(du_{3}\right),w_{1}^{*}\left(t,u_{1}\right),X\right)\bigg]\bigg|\\
 & \leq K\left|w_{1}^{*}\left(t,u_{1}\right)-w_{1}^{*}\left(t,u_{1}'\right)\right|\\
 & \quad+K_{t}\left(1+\int\left|u_{2}\right|^{2}\rho_{\mathbf{w}}^{2}\left(du_{2}\right)+\int\left|u_{3}\right|^{2}\rho_{\mathbf{w}}^{3}\left(du_{3}\right)\right)\mathbb{E}_{Z}\left[\sup_{s\leq t}\left|H_{1}^{*}\left(s,x,u_{1}\right)-H_{1}^{*}\left(s,x,u_{1}'\right)\right|\right]\\
 & \leq K\left|w_{1}^{*}\left(t,u_{1}\right)-w_{1}^{*}\left(t,u_{1}'\right)\right|+K_{t}\mathbb{E}_{Z}\left[\sup_{s\leq t}\left|H_{1}^{*}\left(s,x,u_{1}\right)-H_{1}^{*}\left(s,x,u_{1}'\right)\right|\right]\\
 & =K\left|w_{1}^{*}\left(t,u_{1}\right)-w_{1}^{*}\left(t,u_{1}'\right)\right|+K_{t}\mathbb{E}_{Z}\left[\sup_{s\leq t}\left|\phi_{1}(w_{1}^{*}(s,u_{1}),X)-\phi_{1}(w_{1}^{*}(s,u_{1}'),X)\right|\right]\\
 & \le K_{t}\sup_{s\leq t}\left|w_{1}^{*}\left(s,u_{1}\right)-w_{1}^{*}\left(s,u_{1}'\right)\right|.
\end{align*}
By Gronwall's inequality,
\[
\sup_{s\leq t}\left|w_{1}^{*}\left(s,u_{1}\right)-w_{1}^{*}\left(s,u_{1}'\right)\right|\leq e^{K_{t}t}\left|w_{1}^{*}\left(0,u_{1}\right)-w_{1}^{*}\left(0,u_{1}'\right)\right|=e^{K_{t}t}\left|u_{1}-u_{1}'\right|.
\]
Furthermore, by Assumptions \ref{enu:Assump_lrSchedule} and \ref{enu:Assump_backward},
for $t'\leq t$,
\begin{align*}
\left|w_{1}^{*}\left(t,u_{1}\right)-w_{1}^{*}\left(t',u_{1}\right)\right| & \leq\int_{t'}^{t}\left|\xi_{1}^{\mathbf{w}}\left(s\right)\mathbb{E}_{Z}\left[\sigma_{1}^{\mathbf{w}}\left(\int h\left(s,Z,u_{1},u_{2},u_{3}\right)\rho_{\mathbf{w}}^{2}\left(du_{2}\right)\rho_{\mathbf{w}}^{3}\left(du_{3}\right),w_{1}^{*}\left(s,u_{1}\right),X\right)\right]\right|ds\\
 & \leq K\int_{t'}^{t}\mathbb{E}_{Z}\left[1+\int\left|h\left(s,Z,u_{1},u_{2},u_{3}\right)\right|\rho_{\mathbf{w}}^{2}\left(du_{2}\right)\rho_{\mathbf{w}}^{3}\left(du_{3}\right)\right]ds\\
 & \leq K_{t}\left(1+\int\left|u_{2}\right|^{2}\rho_{\mathbf{w}}^{2}\left(du_{2}\right)+\int\left|u_{3}\right|^{2}\rho_{\mathbf{w}}^{3}\left(du_{3}\right)\right)\left|t-t'\right|\\
 & \leq K_{t}\left|t-t'\right|.
\end{align*}
This shows that $w_{1}^{*}$ defines a continuous function $w_{1}^{*}:[0,T]\times\mathbb{W}_{1}\to\mathbb{W}_{1}$.

\paragraph{Step 3.}

Consider the sphere $\mathbb{S}^{d}$ which is a compactification
of $\mathbb{R}^{d}$. We can extend $w_{1}^{*}$ to a function $M:\;\left[0,T\right]\times\mathbb{S}^{d}\to\mathbb{S}^{d}$
fixing the point at infinity, which remains a continuous map since
$\left|M\left(t,u_{1}\right)-u_{1}\right|=\left|M\left(t,u_{1}\right)-M\left(0,u_{1}\right)\right|\leq K_{T}t$.
Let $M_{t}:\;\mathbb{R}^{d}\to\mathbb{R}^{d}$ be defined by $M_{t}\left(u_{1}\right)=M\left(t,u_{1}\right)$.
We claim that $M_{t}$ is surjective for all finite $t$. Indeed,
if $M_{t}$ fails to be surjective for some $t$, then for some $p\in\mathbb{S}^{d}$,
$M_{t}:\;\mathbb{S}^{d}\to\mathbb{S}^{d}\backslash\left\{ p\right\} \to\mathbb{S}^{d}$
is homotopic to the constant map, but $M$ then gives a homotopy from
the identity map $M_{0}$ on the sphere to a constant map, which is
a contradiction as the sphere $\mathbb{S}^{d}$ is not contractible.
Hence $w_{1}^{*}\left(t,\cdot\right)$ is surjective for all finite
$t$. Now let $U_{1}\sim\rho_{\mathbf{w}}^{1}$, which has full support,
and consider $w_{1}^{*}\left(t,U_{1}\right)$. Let us assume that
$w_{1}^{*}\left(t,U_{1}\right)$ does not have full support at some
time $t$, which implies there is an open ball $B$ in $\mathbb{R}^{d}$
for which $\mathbb{P}\left(w_{1}^{*}\left(t,U_{1}\right)\in B\right)=0$.
Due to surjectivity and continuity of $u_{1}\mapsto w_{1}^{*}\left(t,u_{1}\right)$,
there is an open set $U$ such that $w_{1}^{*}\left(t,u_{1}\right)\in B$
for all $u_{1}\in U$. Then $\mathbb{P}\left(U_{1}\in U\right)=0$,
contradicting the assumption that $U_{1}$ has full support. Therefore
$w_{1}^{*}\left(t,U_{1}\right)$ must have full support at all $t\geq0$.
\end{proof}
With this lemma, we are ready to prove Theorem \ref{thm:global-optimum-3}.
We recall the setting of Theorem \ref{thm:global-optimum-3}, and
in particular, the neural network (\ref{eq:three-layer-nn}).
\begin{proof}[Proof of Theorem \ref{thm:global-optimum-3}]
Let $U_{i}\sim\rho^{i}$, $i=1,2,3$ independently. It is easy to
check that Assumptions \ref{enu:Assump_lrSchedule}-\ref{enu:Assump_backward},
as well as the conditions of Lemma \ref{lem:full-support-3}, hold.
Therefore, by Lemma \ref{lem:full-support-3}, the support of ${\rm Law}\left(w_{1}^{*}\left(t,U_{1}\right)\right)$
is $\mathbb{R}^{d}$ at all $t$. We recall from the convergence assumption
the limits $\bar{w}_{1}$, $\bar{w}_{2}$ and $\bar{w}_{3}$, and
we shall first prove $\left(\bar{w}_{1},\bar{w}_{2},\bar{w}_{3}\right)$
is a global minimizer of $\mathscr{L}$ in Case 1 and $\mathscr{L}\left(\bar{w}_{1},\bar{w}_{2},\bar{w}_{3}\right)=0$
in Case 2.

By the convergence assumption, we have that for any $\epsilon>0$,
there exists $T\left(\epsilon\right)$ such that for all $t\geq T\left(\epsilon\right)$
and almost surely:
\[
\epsilon\geq\left|\mathbb{E}_{Z}\left[\Delta_{2}^{H*}\left(t,Z,U_{3}\right)\varphi_{1}\left(\left\langle w_{1}^{*}\left(t,U_{1}\right),X\right\rangle \right)\right]\right|=\left|\left\langle \mathbb{E}_{Z}\left[\Delta_{2}^{H*}\left(t,Z,U_{3}\right)|X=x\right],\varphi_{1}\left(\left\langle w_{1}^{*}\left(t,U_{1}\right),x\right\rangle \right)\right\rangle _{L^{2}\left({\cal P}_{X}\right)}\right|.
\]
Since ${\rm Law}\left(w_{1}^{*}\left(t,U_{1}\right)\right)$ has full
support, we obtain that for $u$ in a dense subset of $\mathbb{R}^{d}$,
\[
{\rm ess\text{-}sup}\left|\left\langle \mathbb{E}_{Z}\left[\Delta_{2}^{H*}\left(t,Z,U_{3}\right)|X=x\right],\varphi_{1}\left(\left\langle u,x\right\rangle \right)\right\rangle _{L^{2}\left({\cal P}_{X}\right)}\right|\leq\epsilon.
\]
By continuity of $u\mapsto\varphi_{1}(\left\langle u,\cdot\right\rangle )$
in $L^{2}({\cal P}_{X})$, we extend the above to all $u\in\mathbb{R}^{d}$.
Recall the couplings $\pi_{t}$ in Assumption \ref{assump:three-layers}.4,
since $\varphi_{1}$ is bounded,
\begin{align*}
 & \mathbb{E}_{(U_{3},U_{3}')\sim\pi_{t}}\left[\left|\left\langle \mathbb{E}_{Z}\left[\Delta_{2}^{H*}\left(t,Z,U_{3}\right)-\Delta_{2}^{H*}\left(Z,U_{3}';\bar{w}_{1},\bar{w}_{2},\bar{w}_{3}\right)\middle||X=x\right],\varphi_{1}\left(\left\langle u,x\right\rangle \right)\right\rangle _{L^{2}\left({\cal P}_{X}\right)}\right|\right]\\
 & \le K\mathbb{E}_{\pi_{t}}\left[\left|\Delta_{2}^{H*}\left(t,Z,U_{3}\right)-\Delta_{2}^{H*}\left(Z,U_{3}';\bar{w}_{1},\bar{w}_{2},\bar{w}_{3}\right)\right|\right]\\
 & \le K\mathbb{E}_{\pi_{t}}\Big[\left(1+\left|\bar{w}_{3}(U_{3})\right|\right)\Big(\left|w_{3}^{*}(t,U_{3}')-\bar{w}_{3}(U_{3})\right|+\left|\bar{w}_{3}(U_{3})\right|\left|w_{2}^{*}(t,U_{1}',U_{2}',U_{3}')-\bar{w}_{2}(U_{1},U_{2},U_{3})\right|\\
 & \qquad+\left|\bar{w}_{3}(U_{2})\right|\left|\bar{w}_{2}(U_{1},U_{2},U_{3})\right|\left|w_{1}^{*}(t,U_{1}')-\bar{w}_{1}(U_{1})\right|\Big)\Big],
\end{align*}
where the last step is by the regularity assumption, similar to the
calculation in the proof of Theorem \ref{thm:global-optimum-2}. Recall
that the right-hand side converges to $0$ as $t\to\infty$. We thus
obtain that for all $u\in\mathbb{R}^{d}$,
\[
\mathbb{E}_{U_{3}}\left[\left|\left\langle \mathbb{E}_{Z}\left[\Delta_{2}^{H*}\left(Z,U_{3};\bar{w}_{1},\bar{w}_{2},\bar{w}_{3}\right)|X=x\right],\varphi_{1}\left(\left\langle u,x\right\rangle \right)\right\rangle _{L^{2}\left({\cal P}_{X}\right)}\right|\right]=0,
\]
which yields that for all $u\in\mathbb{R}^{d}$ and almost surely,
\[
\left|\left\langle \mathbb{E}_{Z}\left[\Delta_{2}^{H*}\left(Z,U_{3};\bar{w}_{1},\bar{w}_{2},\bar{w}_{3}\right)|X=x\right],\varphi_{1}\left(\left\langle u,x\right\rangle \right)\right\rangle _{L^{2}\left({\cal P}_{X}\right)}\right|=0.
\]
Here we note that by the regularity assumption that 
\[
\left|\mathbb{E}_{Z}\left[\Delta_{2}^{H*}\left(Z,U_{3};\bar{w}_{1},\bar{w}_{2},\bar{w}_{3}\right)|X=x\right]\right|\leq K\left|\bar{w}_{3}\left(U_{3}\right)\right|,
\]
and so $\mathbb{E}_{Z}\left[\Delta_{2}^{H*}\left(Z,u_{3};\bar{w}_{1},\bar{w}_{2},\bar{w}_{3}\right)|X=x\right]$
is in $L^{2}\left({\cal P}_{X}\right)$ for almost every $u_{3}$.
Since $\left\{ \varphi_{1}\left(\left\langle u,\cdot\right\rangle \right):\;u\in\mathbb{R}^{d}\right\} $
has dense span in $L^{2}\left({\cal P}_{X}\right)$, we have $\mathbb{E}_{Z}\left[\Delta_{2}^{H*}\left(Z,u_{3};\bar{w}_{1},\bar{w}_{2},\bar{w}_{3}\right)|X=x\right]=0$
for ${\cal P}_{X}$-almost every $x$ and almost every $u_{3}$, and
hence
\[
\mathbb{E}_{Z}\left[\partial_{2}{\cal L}\left(Y,\hat{y}^{*}\left(X;\bar{w}_{1},\bar{w}_{2},\bar{w}_{3}\right)\right)\middle|X=x\right]\varphi_{3}'\left(H_{3}^{*}\left(x;\bar{w}_{1},\bar{w}_{2},\bar{w}_{3}\right)\right)\bar{w}_{3}\left(u_{3}\right)\varphi_{2}'\left(H_{2}^{*}\left(x,u_{3};\bar{w}_{1},\bar{w}_{2}\right)\right)=0.
\]
We note that our assumptions guarantee that $\mathbb{P}\left(\bar{w}_{3}\left(U_{3}\right)\ne0\right)$
is positive. Indeed:
\begin{itemize}
\item In the case $\int\mathbb{I}\left(u_{3}\neq0\right)\rho^{3}\left(du_{3}\right)>0$
and $\xi_{3}\left(\cdot\right)=0$, it is obvious that $\mathbb{P}\left(\bar{w}_{3}\left(U_{3}\right)\ne0\right)>0$.
\item In the case $\mathscr{L}\left(w_{1}^{0},w_{2}^{0},w_{3}^{0}\right)<\mathbb{E}_{Z}\left[{\cal L}\left(Y,\varphi_{3}\left(0\right)\right)\right]$,
it can be easily checked that
\[
\mathscr{L}\left(w_{1}^{*}\left(t,\cdot\right),w_{2}^{*}\left(t,\cdot,\cdot,\cdot\right),w_{3}^{*}\left(t,\cdot\right)\right)\leq\mathscr{L}\left(w_{1}^{*}\left(t',\cdot\right),w_{2}^{*}\left(t',\cdot,\cdot,\cdot\right),w_{3}^{*}\left(t',\cdot\right)\right),
\]
for $t\geq t'$. This is in fact a standard property of gradient flows.
In particular, setting $t'=0$ and taking $t\to\infty$, it is easy
to see that
\[
\mathscr{L}\left(\bar{w}_{1},\bar{w}_{2},\bar{w}_{3}\right)\leq\mathscr{L}\left(w_{1}^{0},w_{2}^{0},w_{3}^{0}\right)<\mathbb{E}_{Z}\left[{\cal L}\left(Y,\varphi_{3}\left(0\right)\right)\right].
\]
If $\mathbb{P}\left(\bar{w}_{3}\left(U_{3}\right)=0\right)=1$ then
$\mathscr{L}\left(\bar{w}_{1},\bar{w}_{2},\bar{w}_{3}\right)=\mathbb{E}_{Z}\left[{\cal L}\left(Y,\varphi_{3}\left(0\right)\right)\right]$,
a contradiction.
\end{itemize}
Then since $\varphi_{2}'$ and $\varphi_{3}'$ are strictly non-zero,
we have $\mathbb{E}_{Z}\left[\partial_{2}{\cal L}\left(Y,\hat{y}^{*}\left(X;\bar{w}_{1},\bar{w}_{2},\bar{w}_{3}\right)\right)\middle|X=x\right]=0$
for ${\cal P}_{X}$-almost every $x$.

In Case 1, since ${\cal L}$ convex in the second variable, for any
measurable function $\tilde{y}(x)$, 
\[
{\cal L}\left(y,\tilde{y}\left(x\right)\right)-{\cal L}\left(y,\hat{y}\left(x;\bar{w}_{1},\bar{w}_{2},\bar{w}_{3}\right)\right)\ge\partial_{2}{\cal L}\left(y,\hat{y}\left(x;\bar{w}_{1},\bar{w}_{2},\bar{w}_{3}\right)\right)\left(\tilde{y}\left(x\right)-\hat{y}\left(x;\bar{w}_{1},\bar{w}_{2},\bar{w}_{3}\right)\right).
\]
Taking expectation, we get $\mathbb{E}_{Z}\left[{\cal L}\left(Y,\tilde{y}\left(X\right)\right)\right]\geq\mathscr{L}\left(\bar{w}_{1},\bar{w}_{2},\bar{w}_{3}\right)$,
i.e. $\left(\bar{w}_{1},\bar{w}_{2},\bar{w}_{3}\right)$ is a global
minimizer of $\mathscr{L}$.

In Case 2, since $y$ is a function of $x$, we obtain $\partial_{2}{\cal L}\left(y,\hat{y}\left(x;\bar{w}_{1},\bar{w}_{2},\bar{w}_{3}\right)\right)=0$
and hence ${\cal L}\left(y,\hat{y}\left(x;\bar{w}_{1},\bar{w}_{2},\bar{w}_{3}\right)\right)=0$
for ${\cal P}_{X}$-almost every $x$.

Finally to connect $\mathscr{L}\left(\bar{w}_{1},\bar{w}_{2},\bar{w}_{3}\right)$
with $\mathscr{L}\left(W^{*}\left(t\right)\right)$ in the limit $t\to\infty$,
we have:
\begin{align*}
\left|\mathscr{L}\left(W^{*}\left(t\right)\right)-\mathscr{L}\left(\bar{w}_{1},\bar{w}_{2},\bar{w}_{3}\right)\right| & =\left|\mathbb{E}_{Z}\left[{\cal L}\left(Y,\hat{y}^{*}\left(t,X\right)\right)-{\cal L}\left(Y,\hat{y}^{*}\left(X;\bar{w}_{1},\bar{w}_{2},\bar{w}_{3}\right)\right)\right]\right|\\
 & \leq K\mathbb{E}_{Z}\left[\left|\hat{y}^{*}\left(t,X\right)-\hat{y}^{*}\left(X;\bar{w}_{1},\bar{w}_{2},\bar{w}_{3}\right)\right|\right]\\
 & \leq K\mathbb{E}_{\pi_{t}}\Big[\left|w_{3}^{*}\left(t,U_{3}'\right)-\bar{w}_{3}\left(U_{3}\right)\right|+\left|\bar{w}_{3}\left(U_{3}\right)\right|\left|w_{2}^{*}\left(t,U_{1}',U_{2}',U_{3}'\right)-\bar{w}_{2}\left(U_{1},U_{2},U_{3}\right)\right|\\
 & \qquad+\left|\bar{w}_{3}\left(U_{3}\right)\right|\left|\bar{w}_{2}\left(U_{1},U_{2},U_{3}\right)\right|\left|w_{1}^{*}\left(t,U_{1}'\right)-\bar{w}_{1}\left(U_{1}\right)\right|\Big]
\end{align*}
which tends to $0$ as $t\to\infty$. This completes the proof.
\end{proof}

\section{Convergence to Global Optimum: Multilayer Networks with Correlated
Initializations\label{sec:global_convergence_general}}

In Section \ref{sec:global_convergence_iid}, we prove global convergence
guarantees for networks with $L\leq3$ and i.i.d. initializations.
Underlying these results is a universal approximation property that
holds throughout the course of training, and this is shown for quite
general data distributions. Recall from Section \ref{sec:iid_init}
that i.i.d. initializations cause a certain degenerate behavior in
the network with $L\geq4$. In particular, by Corollary \ref{cor:iid_same_neurons},
neurons at intermediate layers collapse to the same function of the
input and therefore are not expected to span the space of functions
of the input $L^{2}\left({\cal P}_{X}\right)$. In other words, these
intermediate layers become a bottleneck that hinders universal approximation
in the context of more than three layers and general data distributions.

To attain meaningful training, this suggests a departure from i.i.d.
initializations. In particular, we propose a correlated initialization
scheme that resolves the aforementioned bottleneck problem. To be
precise, the key idea lies in the new concept of bidirectional diversity.
A similar concept has been encountered in Section \ref{sec:global_convergence_iid};
for instance, diversity in the two-layer case in Section \ref{subsec:global_convergence_iid_2}
refers to the full support condition of the first layer's weight distribution
in the Euclidean space, implied at initialization $t=0$ by Assumption
\ref{assump:two-layers}.2 and shown to hold at any finite time $t$
by Lemma \ref{lem:full-support-2}. Here bidirectional diversity furthers
this idea to the multilayer case with arbitrary depths. Firstly, it
is realized in function spaces that are naturally described by our
neuronal embedding framework. Secondly, it is bidirectional: roughly
speaking, for intermediate layers, diversity holds in both the forward
and backward passes. The effect of bidirectional diversity is that
a certain universal approximation property, at any finite training
time $t$, is propagated from the first layer to the second last one.
Importantly the proposed correlated initialization only ensures bidirectional
diversity at initialization $t=0$, but it is the learning dynamics
that automatically maintains bidirectional diversity at any finite
$t$. This fact is again shown by a topological invariance argument.

In the following, we first describe the multilayer fully-connected
neural network under consideration and its corresponding MF limit.
We then describe the proposed correlated initialization, and state
and prove the global convergence guarantee.

\textbf{\textcolor{red}{}}

\subsection{Multilayer fully-connected neural network}

We consider the following $L$-layer fully-connected network:
\begin{align}
\hat{{\bf y}}\left(x;\mathbf{W}\left(k\right)\right) & =\varphi_{L}\left(\mathbf{H}_{L}\left(x,1;\mathbf{W}\left(k\right)\right)\right),\label{eq:nn_FC_multilayer}\\
\mathbf{H}_{i}\left(x,j_{i};\mathbf{W}\left(k\right)\right) & =\frac{1}{n_{i-1}}\sum_{j_{i-1}=1}^{n_{i-1}}{\bf w}_{i}\left(k,j_{i-1},j_{i}\right)\varphi_{i-1}\left({\bf H}_{i-1}\left(x,j_{i-1};\mathbf{W}\left(k\right)\right)\right),\qquad i=L,...,2,\nonumber \\
{\bf H}_{1}\left(x,j_{1};\mathbf{W}\left(k\right)\right) & =\left\langle {\bf w}_{1}\left(k,j_{1}\right),x\right\rangle ,\nonumber 
\end{align}
in which $x\in\mathbb{R}^{d}$ is the input, $\mathbf{W}\left(k\right)=\left\{ {\bf w}_{1}\left(k,\cdot\right),{\bf w}_{i}\left(k,\cdot,\cdot\right):\;i=2,...,L\right\} $
is the weight with ${\bf w}_{1}\left(k,j_{1}\right)\in\mathbb{R}^{d}$,
${\bf w}_{i}\left(k,j_{i-1},j_{i}\right)\in\mathbb{R}$, $\varphi_{i}:\;\mathbb{R}\to\mathbb{R}$
is the activation. Here the network has widths $\left\{ n_{i}\right\} _{i\leq L}$
with $n_{L}=1$. We train the network with stochastic gradient descent
(SGD) w.r.t. the loss ${\cal L}:\;\mathbb{R}\times\mathbb{R}\to\mathbb{R}_{\geq0}$
and the data $z\left(k\right)=\left(x\left(k\right),y\left(k\right)\right)\in\mathbb{R}^{d}\times\mathbb{R}$
drawn independently at time $k$ from a training distribution ${\cal P}$.
Given an initialization $\mathbf{W}\left(0\right)$, we update $\mathbf{W}\left(k\right)$
according to
\begin{align*}
{\bf w}_{i}\left(k+1,j_{i-1},j_{i}\right) & ={\bf w}_{i}\left(k,j_{i-1},j_{i}\right)-\epsilon\ensuremath{\xi}_{i}\left(t\epsilon\right)\Delta_{i}^{\mathbf{w}}\left(z\left(k\right),j_{i-1},j_{i};\mathbf{W}\left(k\right)\right),\qquad i=2,...,L,\\
{\bf w}_{1}\left(k+1,j_{1}\right) & ={\bf w}_{1}\left(k,j_{1}\right)-\epsilon\ensuremath{\xi}_{1}\left(t\epsilon\right)\Delta_{1}^{\mathbf{w}}\left(z\left(k\right),j_{1};\mathbf{W}\left(k\right)\right),
\end{align*}
in which $j_{i}\in\left[n_{i}\right]$, $\epsilon\in\mathbb{R}_{>0}$
is the learning rate, $\xi_{i}:\;\mathbb{R}_{\geq0}\mapsto\mathbb{R}_{\geq0}$
is the learning rate schedule for $\mathbf{w}_{i}$, and for $z=\left(x,y\right)$,
we define
\begin{align*}
\Delta_{L}^{\mathbf{H}}\left(z,1;\mathbf{W}\left(k\right)\right) & =\partial_{2}{\cal L}\left(y,\hat{\mathbf{y}}\left(x;\mathbf{W}\left(k\right)\right)\right)\varphi_{L}'\left(\mathbf{H}_{L}\left(x,1;\mathbf{W}\left(k\right)\right)\right),\\
\Delta_{i-1}^{\mathbf{H}}\left(z,j_{i-1};\mathbf{W}\left(k\right)\right) & =\frac{1}{n_{i}}\sum_{j_{i}=1}^{n_{i}}\Delta_{i}^{\mathbf{H}}\left(z,j_{i};\mathbf{W}\left(k\right)\right){\bf w}_{i}\left(k,j_{i-1},j_{i}\right)\varphi_{i-1}'\left({\bf H}_{i-1}\left(x,j_{i-1};\mathbf{W}\left(k\right)\right)\right),\qquad i=L,...,2,\\
\Delta_{i}^{\mathbf{w}}\left(z,j_{i-1},j_{i};\mathbf{W}\left(k\right)\right) & =\Delta_{i}^{\mathbf{H}}\left(z,j_{i};\mathbf{W}\left(k\right)\right)\varphi_{i-1}\left({\bf H}_{i-1}\left(x,j_{i-1};\mathbf{W}\left(k\right)\right)\right),\qquad i=L,...,2,\\
\Delta_{1}^{\mathbf{w}}\left(z,j_{1};\mathbf{W}\left(k\right)\right) & =\Delta_{1}^{\mathbf{H}}\left(z,j_{1};\mathbf{W}\left(k\right)\right)x.
\end{align*}
In short, for an initialization $\mathbf{W}\left(0\right)$, we obtain
a SGD trajectory $\mathbf{W}\left(k\right)$ of an $L$-layer network
with size $\left\{ n_{i}\right\} _{i\leq L}$. We also note that this
neural network fits into the framework in Section \ref{sec:Framework}.

\subsection{Mean field limit}

Given a neuronal ensemble $\left(\Omega,P\right)=\prod_{i=1}^{L}\left(\Omega_{i},P_{i}\right)$
(in which $\Omega_{L}=\left\{ 1\right\} $), the MF limit that is
associated with the network (\ref{eq:nn_FC_multilayer}) is described
by the continuous-time evolution of $W\left(t\right)=\left\{ w_{1}\left(t,\cdot\right),w_{i}\left(t,\cdot,\cdot\right):\;i=2,...,L\right\} $,
given by the following MF ODEs:
\begin{align*}
\frac{\partial}{\partial t}w_{i}\left(t,c_{i-1},c_{i}\right) & =-\xi_{i}\left(t\right)\mathbb{E}_{Z}\left[\Delta_{i}^{w}\left(Z,c_{i-1},c_{i};W\left(t\right)\right)\right],\qquad i=2,...,L,\\
\frac{\partial}{\partial t}w_{1}\left(t,c_{1}\right) & =-\xi_{1}\left(t\right)\mathbb{E}_{Z}\left[\Delta_{1}^{w}\left(Z,c_{1};W\left(t\right)\right)\right],
\end{align*}
where $w_{1}:\,\mathbb{R}_{\geq0}\times\Omega_{1}\to\mathbb{R}^{d}$
and $w_{i}:\,\mathbb{R}_{\geq0}\times\Omega_{i-1}\times\Omega_{i}\to\mathbb{R}$.
Here we define the forward quantities:
\begin{align*}
\hat{y}\left(x;W\left(t\right)\right) & =\varphi_{L}\left(H_{L}\left(x,1;W\left(t\right)\right)\right),\\
H_{i}\left(x,c_{i};W\left(t\right)\right) & =\mathbb{E}_{C_{i-1}}\left[w_{i}\left(t,C_{i-1},c_{i}\right)\varphi_{i-1}\left(H_{i-1}\left(x,C_{i-1};W\left(t\right)\right)\right)\right],\qquad i=L,...,2,\\
H_{1}\left(x,c_{1};W\left(t\right)\right) & =\left\langle w_{1}\left(t,c_{1}\right),x\right\rangle ,
\end{align*}
and the backward quantities:
\begin{align*}
\Delta_{L}^{H}\left(z,1;W\left(t\right)\right) & =\partial_{2}{\cal L}\left(y,\hat{y}\left(x;W\left(t\right)\right)\right)\varphi_{L}'\left(H_{L}\left(x,1;W\left(t\right)\right)\right),\\
\Delta_{i-1}^{H}\left(z,c_{i-1};W\left(t\right)\right) & =\mathbb{E}_{C_{i}}\left[\Delta_{i}^{H}\left(z,C_{i};W\left(t\right)\right)w_{i}\left(t,c_{i-1},C_{i}\right)\varphi_{i-1}'\left(H_{i-1}\left(x,c_{i-1};W\left(t\right)\right)\right)\right],\qquad i=L,...,2,\\
\Delta_{i}^{w}\left(z,c_{i-1},c_{i};W\left(t\right)\right) & =\Delta_{i}^{H}\left(z,c_{i};W\left(t\right)\right)\varphi_{i-1}\left(H_{i-1}\left(x,c_{i-1};W\left(t\right)\right)\right),\qquad i=L,...,2,\\
\Delta_{1}^{w}\left(z,c_{1};W\left(t\right)\right) & =\Delta_{1}^{H}\left(z,c_{1};W\left(t\right)\right)x.
\end{align*}
As a reminder, the data $Z=\left(X,Y\right)\sim{\cal P}$ and $C_{i}\sim P_{i}$.
To recap, given a neuronal ensemble $\left(\Omega,P\right)$, for
each initialization $W\left(0\right)$, we have defined a MF limit
$W\left(t\right)$.

\subsection{Global convergence and bidirectional diversity\label{subsec:Global-convergence-and-bidirectional-diversity}}

We begin the study of global convergence of the network (\ref{eq:nn_FC_multilayer})
with an analysis of its MF limit, which is the focus of this section.
To measure the learning quality, we consider the loss averaged over
the data $Z\sim{\cal P}$:
\[
\mathscr{L}\left(F\right)=\mathbb{E}_{Z}\left[{\cal L}\left(Y,\hat{y}\left(X;F\right)\right)\right],
\]
where $F=\left\{ f_{i}:\;i=1,...,L\right\} $ a set of measurable
functions $f_{1}:\;\Omega_{1}\to\mathbb{R}^{d}$, $f_{i}:\;\Omega_{i-1}\times\Omega_{i}\to\mathbb{R}$
for $i=2,...,L$.

Recall that in our framework, the finite-sized neural network is formally
connected with its MF limit via a neuronal embedding. Here without
making explicit this connection, one can study the MF limit that is
defined on the basis of a given neuronal embedding $\left(\Omega,P,\left\{ w_{i}^{0}\right\} _{i\leq L}\right)$,
where $w_{1}^{0}:\;\Omega_{1}\to\mathbb{R}^{d}$, $w_{i}^{0}:\;\Omega_{i-1}\times\Omega_{i}\to\mathbb{R}$
for $i=2,...,L$. In particular, we make the following assumptions:
\begin{assumption}
\label{assump:multilayer}Consider a neuronal embedding $\left(\Omega,P,\left\{ w_{i}^{0}\right\} _{i\leq L}\right)$,
recalling $\Omega=\prod_{i=1}^{L}\Omega_{i}$ and $P=\prod_{i=1}^{L}P_{i}$
with $\Omega_{L}=\left\{ 1\right\} $. Consider the MF limit associated
with the neuronal ensemble $\left(\Omega,P\right)$ with initialization
$W\left(0\right)$ such that $w_{1}\left(0,\cdot\right)=w_{1}^{0}\left(\cdot\right)$
and $w_{i}\left(0,\cdot,\cdot\right)=w_{i}^{0}\left(\cdot,\cdot\right)$.
We make the following assumptions:
\begin{enumerate}
\item Initialization: The functions $\left\{ w_{i}^{0}\right\} _{i\leq L}$
satisfy:
\[
\sup_{m\geq1}\frac{1}{\sqrt{m}}\mathbb{E}\left[\left|w_{1}^{0}\left(C_{1}\right)\right|^{m}\right]^{1/m}\leq K,\quad\sup_{m\geq1}\frac{1}{\sqrt{m}}\mathbb{E}\left[\left|w_{i}^{0}\left(C_{i-1},C_{i}\right)\right|^{m}\right]^{1/m}\leq K,\quad i=2,...,L.
\]
\item Diversity: The functions $\left\{ w_{i}^{0}\right\} _{i\leq L}$ satisfy:
\begin{itemize}
\item ${\rm supp}\left(w_{1}^{0}\left(C_{1}\right),w_{2}^{0}\left(C_{1},\cdot\right)\right)=\mathbb{R}^{d}\times L^{2}\left(P_{2}\right)$,
\item ${\rm supp}\left(w_{i}^{0}\left(\cdot,C_{i}\right),w_{i+1}^{0}\left(C_{i},\cdot\right)\right)=L^{2}\left(P_{i-1}\right)\times L^{2}\left(P_{i+1}\right)$
for $i=2,...,L-1$.
\end{itemize}
(Remark: we write $w_{i}^{0}\left(\cdot,C_{i}\right)$ to denote the
random mapping $c\mapsto w_{i}^{0}\left(c,C_{i}\right)$, and similar
for $w_{i+1}^{0}\left(C_{i},\cdot\right)$.)
\item Regularity: We assume that:
\begin{itemize}
\item $\varphi_{i}$ is $K$-bounded for $1\leq i\leq L-1$, $\varphi_{i}'$
is $K$-bounded and $K$-Lipschitz for $1\leq i\leq L$, and $\varphi_{L}'$
is non-zero everywhere,
\item $\partial_{2}{\cal L}\left(\cdot,\cdot\right)$ is $K$-Lipschitz
in the second variable and $K$-bounded,
\item $\left|X\right|\leq K$ with probability $1$,
\item the learning rate schedule $\xi_{i}$ is $K$-bounded and $K$-Lipschitz
for $1\leq i\leq L$.
\end{itemize}
\item Convergence: There exist a coupling $\pi_{t}$ of $\prod_{i=1}^{L}P_{i}$
and itself such that 
\begin{align*}
\mathbb{E}_{\pi_{t}}\left[\left|w_{i}\left(t,C_{i-1}',C_{i}'\right)-\bar{w}_{i}\left(C_{i-1},C_{i}\right)\right|\sideset{}{_{j=i+1}^{L}}\prod\left|\bar{w}_{j}\left(C_{j-1},C_{j}\right)\right|\right] & \to0,\quad i=2,...,L,\\
\mathbb{E}_{\pi_{t}}\left[\left|w_{1}\left(t,C_{1}'\right)-\bar{w}_{1}\left(C_{1}\right)\right|\sideset{}{_{j=2}^{L}}\prod\left|\bar{w}_{j}\left(C_{j-1},C_{j}\right)\right|\right] & \to0,
\end{align*}
as $t\to\infty$, where $(C_{1},\dots,C_{L},C_{1}',\dots,C_{L}')\sim\pi_{t}$.
Furthermore, 
\begin{align*}
{\rm ess\text{-}sup}\left|\frac{\partial}{\partial t}w_{L}\left(t,C_{L-1},1\right)\right| & \to0.
\end{align*}
(Here we take $\prod_{j=i+1}^{L}=1$ for $i=L$.)
\item Universal approximation: The set $\left\{ \varphi_{1}\left(\left\langle u,\cdot\right\rangle \right):\;u\in\mathbb{R}^{d}\right\} $
has dense span in $L^{2}\left({\cal P}_{X}\right)$ (the space of
square integrable functions w.r.t. the measure ${\cal P}_{X}$, which
is the distribution of the input $X$). Furthermore, for each $i=2,...,L-1$,
$\varphi_{i}$ is non-obstructive in the sense that the set $\left\{ \varphi_{i}\circ f:\;f\in L^{2}\left({\cal P}_{X}\right)\right\} $
has dense span in $L^{2}\left({\cal P}_{X}\right)$.
\end{enumerate}
\end{assumption}

It is easy to see that this set of assumptions satisfies Assumptions
\ref{enu:Assump_lrSchedule}-\ref{enu:Assump_backward} and \ref{assump:init}.
As a consequence, by Theorem \ref{thm:existence ODE}, there exists
a unique solution $W$ to the MF ODEs on $t\in[0,\infty)$.
\begin{thm}
\label{thm:global-optimum-multilayer}Consider a neuronal embedding
$\left(\Omega,P,\left\{ w_{i}^{0}\right\} _{i\leq L}\right)$ and
the MF limit as in Assumption \ref{assump:multilayer}. Assume $\xi_{L}\left(\cdot\right)=1$.
Then:
\begin{itemize}
\item Case 1 (convex loss): If ${\cal L}$ is convex in the second variable,
then:
\[
\lim_{t\to\infty}\mathscr{L}\left(W\left(t\right)\right)=\inf_{F}\mathscr{L}\left(F\right)=\inf_{\tilde{y}:\;\mathbb{R}^{d}\to\mathbb{R}}\mathbb{E}_{Z}\left[{\cal L}\left(Y,\tilde{y}\left(X\right)\right)\right].
\]
\item Case 2 (generic non-negative loss): Suppose that $\partial_{2}{\cal L}\left(y,\hat{y}\right)=0$
implies ${\cal L}\left(y,\hat{y}\right)=0$. If $y=y(x)$ is a function
of $x$, then $\mathscr{L}\left(W\left(t\right)\right)\to0$ as $t\to\infty$.
\end{itemize}
\end{thm}

The assumptions here are similar to those made in Theorems \ref{thm:global-optimum-2}
and \ref{thm:global-optimum-3} of Section \ref{sec:global_convergence_iid}.
Similar to the settings of Section \ref{sec:global_convergence_iid},
the regularity assumption can be satisfied for several common setups
and loss functions; for example, this holds when $\varphi_{i}$ is
sigmoid or tanh for $i\leq L-1$, $\varphi_{L}$ is the identity,
and ${\cal L}$ is the Huber loss. The convergence assumption here
is also similar to the convergence assumption in Assumption \ref{assump:two-layers}
or Assumption \ref{assump:three-layers}. In particular, the first
part of the convergence assumption is essentially a Wasserstein-type
convergence; it follows from the convergence of $(w_{i}(t))_{i=1}^{L}$
to $(\bar{w}_{i})_{i=1}^{L}$ in an appropriate Wasserstein distance.
The fifth assumption is again natural and can be satisfied by common
activations. For example, $\varphi_{i}$ can be $\tanh$ for $i=1,...,L-1$.
Indeed, whenever $\left\{ \varphi_{i}\left(\left\langle u,\cdot\right\rangle \right):\;u\in\mathbb{R}^{d}\right\} $
has dense span in $L^{2}\left({\cal P}_{X}\right)$, $\varphi_{i}$
is non-obstructive since
\[
{\rm span}\left(\left\{ \varphi_{i}\left(\left\langle u,\cdot\right\rangle \right):\;u\in\mathbb{R}^{d}\right\} \right)\subseteq{\rm span}\left(\left\{ \varphi_{i}\circ f:\;f\in L^{2}\left({\cal P}_{X}\right)\right\} \right).
\]
The diversity assumption is new: it refers to an initialization scheme
that introduces correlation among the weights. In particular, i.i.d.
initializations do not satisfy this assumption for $L\geq3$.

The second assumption is the counterpart of the diversity assumption
made in Theorems \ref{thm:global-optimum-2} and \ref{thm:global-optimum-3},
but there is a special difference. In Section \ref{sec:global_convergence_iid},
the diversity assumption refers to a full support condition of only
the first layer's initial weight, which is in the Euclidean space.
Here our diversity assumption refers to a particular full support
condition for all layers. At a closer look, the condition is in the
function space and reflects certain \textsl{bidirectional diversity}.
In particular, this assumption implies both $w_{i}^{0}\left(\cdot,C_{i}\right)$
and $w_{i}^{0}\left(C_{i-1},\cdot\right)$ have full supports in $L^{2}\left(P_{i-1}\right)$
and $L^{2}\left(P_{i}\right)$ respectively (which we shall refer
to as \textit{forward diversity} and \textit{backward diversity},
respectively).

\paragraph*{High-level idea of the proof.}

The proof proceeds with several insights that have already appeared
in Section \ref{sec:global_convergence_iid}. The novelty of our present
analysis lies in the use of the aforementioned bidirectional diversity.
To clarify the point, let us give a brief high-level idea of the proof.
At time $t$ sufficiently large, we expect to have:
\[
\left|\frac{\partial}{\partial t}w_{L}\left(t,c_{L-1},1\right)\right|=\left|\mathbb{E}_{Z}\left[\partial_{2}{\cal L}\left(Y,\hat{y}\left(X;W\left(t\right)\right)\right)\varphi_{L}'\left(H_{L}\left(X,1;W\left(t\right)\right)\right)\varphi_{L-1}\left(H_{L-1}\left(X,c_{L-1};W\left(t\right)\right)\right)\right]\right|\approx0
\]
for $P_{L-1}$-almost every $c_{L-1}$. If the set of mappings $x\mapsto H_{L-1}\left(x,c_{L-1};W\left(t\right)\right)$,
indexed by $c_{L-1}$, is diverse in the sense that ${\rm supp}\left(H_{L-1}\left(\cdot,C_{L-1};W\left(t\right)\right)\right)=L^{2}\left({\cal P}_{X}\right)$,
then since $\varphi_{L-1}$ is non-obstructive, we obtain
\[
\mathbb{E}_{Z}\left[\partial_{2}{\cal L}\left(Y,\hat{y}\left(X;W\left(t\right)\right)\right)\middle|X=x\right]\varphi_{L}'\left(H_{L}\left(x,1;W\left(t\right)\right)\right)\approx0
\]
and consequently
\[
\mathbb{E}_{Z}\left[\partial_{2}{\cal L}\left(Y,\hat{y}\left(X;W\left(t\right)\right)\right)\middle|X=x\right]\approx0
\]
for ${\cal P}_{X}$-almost every $x$. The desired conclusion then
follows.

Hence the crux of the proof is to show that ${\rm supp}\left(H_{L-1}\left(\cdot,C_{L-1};W\left(t\right)\right)\right)=L^{2}\left({\cal P}_{X}\right)$.
In fact, we show that this holds for any finite time $t\geq0$. This
follows if we can prove the forward diversity property of the weights,
in which $w_{i}\left(t,\cdot,C_{i}\right)$ has full support in $L^{2}\left(P_{i-1}\right)$
for any $t\geq0$ and $2\leq i\leq L-1$, and a similar property for
$w_{1}\left(t,C_{1}\right)$. Interestingly to that end, we actually
show that bidirectional diversity, and hence both forward diversity
and backward diversity, hold at any time $t\geq0$, even though we
only need forward diversity for our purpose. The full proof is deferred
to Section \ref{subsec:Proof-multilayer}.

\paragraph*{A converse for global convergence.}

Similar to Section \ref{sec:global_convergence_iid}, we also have
a converse relation between global convergence and the essential supremum
condition in Assumption \ref{assump:multilayer}.4. The proof is presented
in Appendix \ref{sec:Remaining-proofs-global-conv-general}.
\begin{prop}
\label{prop:converse-multilayer}Consider the MF limit corresponding
to the network (\ref{eq:nn_FC_multilayer}), such that they are coupled
together by the coupling procedure in Section \ref{subsec:Neuronal-Embedding}
with a neuronal embedding $\left(\Omega,P,\left\{ w_{i}^{0}\right\} _{i\leq L}\right)$.
Suppose that the initialization and regularity assumptions (i.e. the
first and third assumptions) of Assumption \ref{assump:multilayer}
hold, and that ${\cal L}(y,\hat{y})\to\infty$ as $|\hat{y}|\to\infty$
for each $y$. Further assume $\xi_{L}\left(\cdot\right)=1$. Then
the following hold:
\begin{itemize}
\item Case 1 (convex loss): If ${\cal L}$ is convex in the second variable
and $\mathscr{L}\left(W\left(t\right)\right)\to\inf_{F}\mathscr{L}\left(F\right)$
as $t\to\infty$, then it must be that
\[
\sup_{c_{L-1}\in\Omega_{L-1}}\left|\frac{\partial}{\partial t}w_{L}\left(t,C_{L-1},1\right)\right|\to0\quad\text{as }t\to\infty.
\]
\item Case 2 (generic non-negative loss): Suppose that $\partial_{2}{\cal L}\left(y,\hat{y}\right)=0$
implies ${\cal L}\left(y,\hat{y}\right)=0$, and $y=y(x)$ is a function
of $x$. If $\mathscr{L}\left(W\left(t\right)\right)\to0$ as $t\to\infty$,
then the same conclusion also holds.
\end{itemize}
\end{prop}

\subsection{Connection to the network (\ref{eq:nn_FC_multilayer})}

Theorem \ref{thm:global-optimum-multilayer} concerns with the global
convergence of the MF limit. To make the connection with a finite-width
neural network (\ref{eq:nn_FC_multilayer}), we recall the neuronal
embedding $\left(\Omega,P,\left\{ w_{i}^{0}\right\} _{i\leq L}\right)$,
as well as the coupling procedure in Section \ref{subsec:Neuronal-Embedding}.
We however present a twist to the procedure. We first choose the neuronal
embedding $\left(\Omega,P,\left\{ w_{i}^{0}\right\} _{i\leq L}\right)$
and then perform the following two steps:
\begin{enumerate}
\item We form the MF limit $W\left(t\right)$ (for $t\in\mathbb{R}_{\geq0}$)
associated with the neuronal ensemble $\left(\Omega,P\right)$ by
setting the initialization $W\left(0\right)$ to $w_{1}\left(0,\cdot\right)=w_{1}^{0}\left(\cdot\right)$,
$w_{i}\left(0,\cdot,\cdot\right)=w_{i}^{0}\left(\cdot,\cdot\right)$
and running the MF ODEs.
\item We independently sample $C_{i}\left(j_{i}\right)\sim P_{i}$ for $i=1,...,L$
and $j_{i}=1,...,n_{i}$. We then form the neural network initialization
$\mathbf{W}\left(0\right)$ with $\mathbf{w}_{1}\left(0,j_{1}\right)=w_{1}^{0}\left(C_{1}\left(j_{1}\right)\right)$
and $\mathbf{w}_{i}\left(0,j_{i-1},j_{i}\right)=w_{i}^{0}\left(C_{i-1}\left(j_{i-1}\right),C_{i}\left(j_{i}\right)\right)$
for $j_{i}\in\left[n_{i}\right]$. We obtain the network's trajectory
$\mathbf{W}\left(k\right)$ for $k\in\mathbb{N}_{\geq0}$ for the
network (\ref{eq:nn_FC_multilayer}), with the data $z\left(k\right)$
generated independently of $\left\{ C_{i}\left(j_{i}\right)\right\} _{i\leq L}$
and hence $\mathbf{W}\left(0\right)$.
\end{enumerate}
That is, instead of starting with a given initialization law of $\mathbf{W}\left(0\right)$
as done in Section \ref{subsec:Neuronal-Embedding}, here we first
start with a chosen neuronal embedding. We then form the MF limit
$W\left(t\right)$ and the neural network initialization $\mathbf{W}\left(0\right)$,
and hence the dynamics $\mathbf{W}\left(k\right)$, based on this
neuronal embedding. In other words, the initialization law of $\mathbf{W}\left(0\right)$
is deduced from the chosen neuronal embedding. Obviously this procedure
ensures that $\bar{\eta}$-independence is satisfied (Assumption \ref{assump:neuronal-embedding}).

In summary, in the present context, the neuronal embedding forms the
basis on which the finite-width neural network is realized. Furthermore
the neural network and its MF limit are coupled. Then using Theorem
\ref{thm:global-optimum-multilayer} and Corollary \ref{cor:gradient descent quality},
one can obtain the following result on the optimization efficiency
of the neural network with SGD:
\begin{cor}
Consider the neural network (\ref{eq:three-layer-nn}) as described
by the coupling procedure with the aforementioned twist. Under the
same setting as Theorem \ref{thm:global-optimum-multilayer}, in Case
1,
\[
\lim_{t\to\infty}\lim_{\left\{ n_{i}\right\} _{i\leq L}}\lim_{\epsilon\to0}\mathbb{E}_{Z}\left[{\cal L}\left(Y,\hat{{\bf y}}\left(X;\mathbf{W}\left(\left\lfloor t/\epsilon\right\rfloor \right)\right)\right)\right]=\inf_{F}\mathscr{L}\left(F\right)=\inf_{\tilde{y}}\mathbb{E}_{Z}\left[{\cal L}\left(Y,\tilde{y}\left(X\right)\right)\right]
\]
in probability, where the limit of the widths is such that $n_{\min}\to\infty$
and $n_{\min}^{-c}\log n_{\max}\to0$ for any $c>0$, with $n_{\min}=\min\left\{ n_{i}:\;1\leq i\le L-1\right\} $
and $n_{\max}=\max\left\{ n_{i}:\;1\leq i\le L-1\right\} $. In Case
2, the same holds with the right-hand side being $0$.
\end{cor}

\subsection{Proof of Theorem \ref{thm:global-optimum-multilayer}\label{subsec:Proof-multilayer}}
\begin{proof}[Proof of Theorem \ref{thm:global-optimum-multilayer}]
We divide the proof into several steps.

\paragraph*{Step 1: Diversity of the weights.}

We show that ${\rm supp}\left(w_{1}\left(t,C_{1}\right)\right)=\mathbb{R}^{d}$
and ${\rm supp}\left(w_{i}\left(t,\cdot,C_{i}\right)\right)=L^{2}\left(P_{i-1}\right)$
for $i=2,...,L-1$, for any $t\geq0$. We do so by showing a stronger
statement, that the following bidirectional diversity condition holds
at any finite training time: 
\begin{align*}
{\rm supp}\left(w_{1}\left(t,C_{1}\right),w_{2}\left(t,C_{1},\cdot\right)\right) & =\mathbb{R}^{d}\times L^{2}\left(P_{2}\right),\\
{\rm supp}\left(w_{i}\left(t,\cdot,C_{i}\right),w_{i+1}\left(t,C_{i},\cdot\right)\right) & =L^{2}\left(P_{i-1}\right)\times L^{2}\left(P_{i+1}\right),\qquad i=2,...,L-1,
\end{align*}
for any $t\geq0$.

We prove the first statement. Given a MF trajectory $\left(W\left(t\right)\right)_{t\ge0}$
and $u_{1}\in\mathbb{R}^{d}$, $u_{2}\in L^{2}\left(P_{2}\right)$,
we consider the following flow on $\mathbb{R}^{d}\times L^{2}\left(P_{2}\right)$:
\begin{align}
\frac{\partial}{\partial t}a_{2}^{+}\left(t,c_{2};u\right) & =-\xi_{2}(t)\mathbb{E}_{Z}\left[\Delta_{2}^{H}\left(Z,c_{2};W(t)\right)\varphi_{1}\left(\left\langle a_{1}^{+}\left(t;u\right),X\right\rangle \right)\right],\nonumber \\
\frac{\partial}{\partial t}a_{1}^{+}\left(t;u\right) & =-\xi_{1}(t)\mathbb{E}_{Z}\left[\mathbb{E}_{C_{2}}\left[\Delta_{2}^{H}\left(Z,C_{2};W(t)\right)a_{2}^{+}\left(t,C_{2};u\right)\right]\varphi_{1}'\left(\left\langle a_{1}^{+}\left(t;u\right),X\right\rangle \right)X\right],\label{eq:ODE-w}
\end{align}
for $u=\left(u_{1},u_{2}\right)$, with the initialization $a_{1}^{+}\left(0;u\right)=u_{1}$
and $a_{2}^{+}\left(0,c_{2};u\right)=u_{2}\left(c_{2}\right)$. Existence
and uniqueness of $\left(a_{1}^{+},a_{2}^{+}\right)$ follows similarly
to Theorem \ref{thm:existence ODE}. We next prove for all finite
$T>0$ and $u^{+}=\left(u_{1}^{+},u_{2}^{+}\right)\in\mathbb{R}^{d}\times L^{2}\left(P_{2}\right)$,
there exists $u^{-}=\left(u_{1}^{-},u_{2}^{-}\right)\in\mathbb{R}^{d}\times L^{2}\left(P_{2}\right)$
such that 
\[
a_{1}^{+}\left(T;u^{-}\right)=u_{1}^{+},\qquad a_{2}^{+}\left(T,\cdot;u^{-}\right)=u_{2}^{+}.
\]
We consider the following auxiliary dynamics on $\mathbb{R}^{d}\times L^{2}\left(P_{2}\right)$:
\begin{align}
\frac{\partial}{\partial t}a_{2}^{-}\left(t,c_{2};u\right) & =\xi_{2}(T-t)\mathbb{E}_{Z}\left[\Delta_{2}^{H}\left(Z,c_{2};W(T-t)\right)\varphi_{1}\left(\left\langle a_{1}^{-}\left(t;u\right),X\right\rangle \right)\right],\nonumber \\
\frac{\partial}{\partial t}a_{1}^{-}\left(t;u\right) & =\xi_{1}(T-t)\mathbb{E}_{Z}\left[\mathbb{E}_{C_{2}}\left[\Delta_{2}^{H}\left(Z,C_{2};W(T-t)\right)a_{2}^{-}\left(t,C_{2};u\right)\right]\varphi_{1}'\left(\left\langle a_{1}^{-}\left(t;u\right),X\right\rangle \right)X\right],\label{eq:ODE-a}
\end{align}
initialized at $a_{1}^{-}\left(0;u\right)=u_{1}$ and $a_{2}^{-}\left(0,c_{2};u\right)=u_{2}\left(c_{2}\right)$,
for $u=\left(u_{1},u_{2}\right)\in\mathbb{R}^{d}\times L^{2}\left(P_{2}\right)$.
Existence and uniqueness of $(a_{1}^{-},a_{2}^{-})$ follow similarly
to Theorem \ref{thm:existence ODE}. Observe that the pair
\[
\tilde{a}_{1}^{-}\left(t\right)=a_{1}^{-}\left(T-t;u^{+}\right),\qquad\tilde{a}_{2}^{-}\left(t,c_{2}\right)=a_{2}^{-}\left(T-t,c_{2};u^{+}\right)
\]
solves the system 
\begin{align*}
\frac{\partial}{\partial t}\tilde{a}_{2}^{-}\left(t,c_{2}\right) & =-\frac{\partial}{\partial t}a_{2}^{-}\left(T-t,c_{2};u^{+}\right)=-\xi_{2}(t)\mathbb{E}_{Z}\left[\Delta_{2}^{H}\left(Z,c_{2};W(t)\right)\varphi_{1}\left(\left\langle \tilde{a}_{1}^{-}\left(t\right),X\right\rangle \right)\right],\\
\frac{\partial}{\partial t}\tilde{a}_{1}^{-}\left(t\right) & =-\frac{\partial}{\partial t}a_{1}^{-}\left(T-t;u^{+}\right)=-\xi_{1}(t)\mathbb{E}_{Z}\left[\mathbb{E}_{C_{2}}\left[\Delta_{2}^{H}\left(Z,C_{2};W(t)\right)\tilde{a}_{2}^{-}\left(t,C_{2}\right)\right]\varphi_{1}'\left(\left\langle \tilde{a}_{1}^{-}\left(t\right),X\right\rangle \right)X\right],
\end{align*}
initialized at $\tilde{a}_{2}^{-}\left(0,c_{2}\right)=a_{2}^{-}\left(T,c_{2};u^{+}\right)$
and $\tilde{a}_{1}^{-}(0)=a_{1}^{-}\left(T;u^{+}\right)$. Thus, by
uniqueness of the solution to the ODE (\ref{eq:ODE-w}), $(\tilde{a}_{1}^{-},\tilde{a}_{2}^{-})$
forms a solution of the ODE (\ref{eq:ODE-w}) initialized at 
\[
\tilde{a}_{1}^{-}(0)=a_{1}^{-}\left(T;u^{+}\right),\qquad\tilde{a}_{2}^{-}\left(0,c_{2}\right)=a_{2}^{-}\left(T,c_{2};u^{+}\right).
\]
In particular, the solution $(\tilde{a}_{1}^{-},\tilde{a}_{2}^{-})$
of the ODE (\ref{eq:ODE-w}) with this initialization satisfies 
\[
\tilde{a}_{1}^{-}\left(T\right)=a_{1}^{-}\left(0;u^{+}\right)=u_{1}^{+},\qquad\tilde{a}_{2}^{-}\left(T,\cdot\right)=a_{2}^{-}\left(0,\cdot;u^{+}\right)=u_{2}^{+}.
\]
Let $u_{1}^{-}=a_{1}^{-}\left(T;u^{+}\right)$ and $u_{2}^{-}=a_{2}^{-}\left(T,\cdot;u^{+}\right)$.
Then we have $a_{1}^{+}\left(T;u^{-}\right)=u_{1}^{+}$ and $a_{2}^{+}\left(T,\cdot;u^{-}\right)=u_{2}^{+}$
as desired.

Using this, by continuity of the map $u\mapsto\left(a_{1}^{+}\left(T;u\right),a_{2}^{+}\left(T,\cdot;u\right)\right)$,
for every $\epsilon>0$, there exists a neighborhood $U$ of $u^{-}$
such that for any $u\in U$, $\left|\left(a_{1}^{+}\left(T;u\right),a_{2}^{+}\left(T,\cdot;u\right)\right)-u^{+}\right|\le\epsilon$.
Notice that the MF trajectory $W\left(t\right)$ satisfies 
\[
w_{1}\left(t,c_{1}\right)=a_{1}^{+}\left(t;w_{1}\left(0,c_{1}\right),w_{2}\left(0,c_{1},\cdot\right)\right),\qquad w_{2}\left(t,c_{1},\cdot\right)=a_{2}^{+}\left(t,\cdot;w_{1}\left(0,c_{1}\right),w_{2}\left(0,c_{1},\cdot\right)\right).
\]
Then since $\left(w_{1}\left(0,C_{1}\right),w_{2}\left(0,C_{1},\cdot\right)\right)$
has full support in $\mathbb{R}^{d}\times L^{2}\left(P_{2}\right)$,
for any finite $T>0$, we have $\left(w_{1}\left(T,C_{1}\right),w_{2}\left(T,C_{1},\cdot\right)\right)$
has full support in $\mathbb{R}^{d}\times L^{2}\left(P_{2}\right)$,
proving the first statement.

The other statements can be proven similarly by considering the following
pairs of flows on $L^{2}(P_{i-1})\times L^{2}(P_{i+1})$, for $u=\left(u_{1},u_{2}\right)\in L^{2}(P_{i-1})\times L^{2}(P_{i+1})$:
\begin{align*}
\frac{\partial}{\partial t}a_{i}^{+}\left(t,c_{i-1};u\right) & =-\xi_{i}\left(t\right)\mathbb{E}_{Z}\left[\Delta_{i}^{a}\left(Z,a_{i}^{+}\left(t,\cdot;u\right),a_{i+1}^{+}\left(t,\cdot;u\right);W\left(t\right)\right)\varphi_{i-1}\left(H_{i-1}\left(X,c_{i-1};W\left(t\right)\right)\right)\right],\\
\frac{\partial}{\partial t}a_{i+1}^{+}\left(t,c_{i+1};u\right) & =-\xi_{i+1}\left(t\right)\mathbb{E}_{Z}\left[\Delta_{i+1}^{H}\left(Z,c_{i+1};W\left(t\right)\right)\varphi_{i}\left(H_{i}^{a}\left(Z,a_{i}^{+}\left(t,\cdot;u\right);W\left(t\right)\right)\right)\right],
\end{align*}
initialized at $a_{i}^{+}\left(0,c_{i-1};u\right)=u_{1}\left(c_{i-1}\right)$
and $a_{i+1}^{+}\left(0,c_{i+1};u\right)=u_{2}\left(c_{i+1}\right)$,
and
\begin{align*}
\frac{\partial}{\partial t}a_{i}^{-}\left(t,c_{i-1};u\right) & =\xi_{i}\left(T-t\right)\mathbb{E}_{Z}\left[\Delta_{i}^{a}\left(Z,a_{i}^{-}\left(t,\cdot;u\right),a_{i+1}^{-}\left(t,\cdot;u\right);W\left(T-t\right)\right)\varphi_{i-1}\left(H_{i-1}\left(X,c_{i-1};W\left(T-t\right)\right)\right)\right],\\
\frac{\partial}{\partial t}a_{i+1}^{-}\left(t,c_{i+1};u\right) & =\xi_{i+1}\left(T-t\right)\mathbb{E}_{Z}\left[\Delta_{i+1}^{H}\left(Z,c_{i+1};W(T-t)\right)\varphi_{i}\left(H_{i}^{a}\left(Z,a_{i}^{-}\left(t,\cdot;u\right);W\left(T-t\right)\right)\right)\right],
\end{align*}
initialized at $a_{i}^{-}\left(0,c_{i-1};u\right)=u_{1}\left(c_{i-1}\right)$
and $a_{i+1}^{-}\left(0,c_{i+1};u\right)=u_{2}\left(c_{i+1}\right)$,
in which we define:
\begin{align*}
\Delta_{i}^{a}\left(z,f,g;W\left(t\right)\right) & =\mathbb{E}_{C_{i+1}}\left[\Delta_{i+1}^{H}\left(z,C_{i+1};W\left(t\right)\right)g\left(C_{i+1}\right)\varphi_{i}'\left(H_{i}^{a}\left(z,f;W\left(t\right)\right)\right)\right],\\
H_{i}^{a}\left(z,f;W\left(t\right)\right) & =\mathbb{E}_{C_{i-1}}\left[f\left(C_{i-1}\right)\varphi_{i-1}\left(H_{i-1}\left(x,C_{i-1};W\left(t\right)\right)\right)\right],
\end{align*}
for $f\in L^{2}\left(P_{i-1}\right)$ and $g\in L^{2}\left(P_{i+1}\right)$.

\paragraph*{Step 2: Diversity of the pre-activations.}

We show that ${\rm supp}\left(H_{i}\left(\cdot,C_{i};W\left(t\right)\right)\right)=L^{2}\left({\cal P}_{X}\right)$
for any $t\geq0$, for $i=2,...,L-1$ by induction.

Firstly consider the base case $i=2$. Recall that
\[
H_{2}\left(x,c_{2};W\left(t\right)\right)=\mathbb{E}_{C_{1}}\left[w_{2}\left(t,C_{1},c_{2}\right)\varphi_{1}\left(\left\langle w_{1}\left(t,C_{1}\right),x\right\rangle \right)\right]\equiv{\cal H}_{2}\left(t,x,w_{2}\left(t,\cdot,c_{2}\right)\right).
\]
Observe that the set ${\rm cl}\left(\left\{ {\cal H}_{2}\left(t,\cdot,f\right):\;f\in L^{2}\left(P_{1}\right)\right\} \right)$
is a closed linear subspace of $L^{2}\left({\cal P}_{X}\right)$.
Hence this set is equal to $L^{2}\left({\cal P}_{X}\right)$ if it
has dense span in $L^{2}\left({\cal P}_{X}\right)$, which we show
now. Indeed, suppose that for some $g\in L^{2}\left({\cal P}_{X}\right)$
such that $\left|g\right|\neq0$, we have $\mathbb{E}_{Z}\left[g\left(X\right){\cal H}_{2}\left(t,X,f\right)\right]=0$
for all $f\in L^{2}\left(P_{1}\right)$. Equivalently,
\[
\mathbb{E}_{C_{1}}\left[f\left(C_{1}\right)\mathbb{E}_{Z}\left[g\left(X\right)\varphi_{1}\left(\left\langle w_{1}\left(t,C_{1}\right),X\right\rangle \right)\right]\right]=0,
\]
for all $f\in L^{2}\left(P_{1}\right)$. As such, for $P_{1}$-almost
every $c_{1}$,
\[
\mathbb{E}_{Z}\left[g\left(X\right)\varphi_{1}\left(\left\langle w_{1}\left(t,c_{1}\right),X\right\rangle \right)\right]=0.
\]
Since ${\rm supp}\left(w_{1}\left(t,C_{1}\right)\right)=\mathbb{R}^{d}$
and that the mapping $u\mapsto\varphi_{1}\left(\left\langle u,x\right\rangle \right)$
is continuous, by the universal approximation assumption for $\varphi_{1}$,
we then obtain $g(x)=0$ for $P_{X}$-almost every $x$, which is
a contradiction. We have thus proved that ${\rm cl}\left(\left\{ {\cal H}_{2}\left(t,\cdot,f\right):\;f\in L^{2}\left(P_{1}\right)\right\} \right)=L^{2}\left({\cal P}_{X}\right)$.
Note that $f\mapsto{\cal H}_{2}\left(t,x,f\right)$ is continuous,
and ${\rm supp}\left(w_{2}\left(t,\cdot,C_{2}\right)\right)=L^{2}\left(P_{1}\right)$,
we then have ${\rm supp}\left(H_{2}\left(\cdot,C_{2};W\left(t\right)\right)\right)=L^{2}\left({\cal P}_{X}\right)$
as desired.

Now let us assume that ${\rm supp}\left(H_{i-1}\left(\cdot,C_{i-1};W\left(t\right)\right)\right)=L^{2}\left({\cal P}_{X}\right)$
for some $i\geq3$ (the induction hypothesis). We would like to show
${\rm supp}\left(H_{i}\left(\cdot,C_{i};W\left(t\right)\right)\right)=L^{2}\left({\cal P}_{X}\right)$.
This is similar to the base case. In particular, recall that
\[
H_{i}\left(x,c_{i};W\left(t\right)\right)=\mathbb{E}_{C_{i-1}}\left[w_{i}\left(t,C_{i-1},c_{i}\right)\varphi_{i-1}\left(H_{i-1}\left(x,C_{i-1};W\left(t\right)\right)\right)\right]\equiv{\cal H}_{i}\left(t,x,w_{i}\left(t,\cdot,c_{i}\right)\right).
\]
Now suppose that for some $g\in L^{2}\left({\cal P}_{X}\right)$ such
that $\left|g\right|\neq0$, we have $\mathbb{E}_{Z}\left[g\left(X\right){\cal H}_{i}\left(t,X,f\right)\right]=0$
for all $f\in L^{2}\left(P_{i-1}\right)$. Then, for $P_{i-1}$-almost
every $c_{i-1}$,
\[
\mathbb{E}_{Z}\left[g\left(X\right)\varphi_{i-1}\left(H_{i-1}\left(X,c_{i-1};W\left(t\right)\right)\right)\right]=0.
\]
Recall the induction hypothesis ${\rm supp}\left(H_{i-1}\left(\cdot,C_{i-1};W\left(t\right)\right)\right)=L^{2}\left({\cal P}_{X}\right)$.
Since $\varphi_{i-1}$ is non-obstructive and continuous, we obtain
$g(x)=0$ for $P_{X}$-almost every $x$, which is a contradiction.
Therefore the set ${\rm cl}\left(\left\{ {\cal H}_{i}\left(t,\cdot,f\right):\;f\in L^{2}\left(P_{i-1}\right)\right\} \right)$
has dense span in $L^{2}\left({\cal P}_{X}\right)$, and again, this
implies it is equal to $L^{2}\left({\cal P}_{X}\right)$. Since $f\mapsto{\cal H}_{i}\left(t,x,f\right)$
is continuous and ${\rm supp}\left(w_{i}\left(t,\cdot,C_{i}\right)\right)=L^{2}\left(P_{i-1}\right)$,
we have ${\rm supp}\left(H_{i}\left(\cdot,C_{i};W\left(t\right)\right)\right)=L^{2}\left({\cal P}_{X}\right)$.

\paragraph*{Step 3: Concluding.}

Let $\mathbb{E}_{Z}\left[\partial_{2}{\cal L}\left(Y,\hat{y}\left(X;W\left(t\right)\right)\right)\middle|X=x\right]\varphi_{L}'\left(H_{L}\left(x,1;W\left(t\right)\right)\right)={\cal H}\left(x,W\left(t\right)\right)$.
From the last step, we have ${\rm supp}\left(H_{L-1}\left(\cdot,C_{L-1};W\left(t\right)\right)\right)=L^{2}\left({\cal P}_{X}\right)$
for any $t\geq0$. Recall that
\[
\frac{\partial}{\partial t}w_{L}\left(t,c_{L-1},1\right)=-\mathbb{E}_{Z}\left[{\cal H}\left(X,W\left(t\right)\right)\varphi_{L-1}\left(H_{L-1}\left(X,c_{L-1};W\left(t\right)\right)\right)\right].
\]
By the convergence assumption, for any $\epsilon>0$, there exists
$T\left(\epsilon\right)>0$ such that for any $t\geq T\left(\epsilon\right)$,
for $P_{L-1}$-almost every $c_{L-1}$,
\[
\left|\mathbb{E}_{Z}\left[{\cal H}\left(X,W\left(t\right)\right)\varphi_{L-1}\left(H_{L-1}\left(X,c_{L-1};W\left(t\right)\right)\right)\right]\right|\leq\epsilon.
\]
We claim that ${\cal H}\left(x,W\left(t\right)\right)\to{\cal H}\left(x,\left\{ \bar{w}_{i}\right\} _{i\leq L}\right)$
in $L^{1}\left({\cal P}_{X}\right)$ as $t\to\infty$. Assuming this
claim and recalling that $\varphi_{L-1}$ is $K$-bounded by the regularity
assumption, we then have that for some $T'\left(\epsilon\right)\geq T\left(\epsilon\right)$,
for any $t\geq T'\left(\epsilon\right)$,
\begin{align*}
 & {\rm ess\text{-}sup}\left|\mathbb{E}_{Z}\left[{\cal H}\left(X,\left\{ \bar{w}_{i}\right\} _{i\leq L}\right)\varphi_{L-1}\left(H_{L-1}\left(X,C_{L-1};W\left(t\right)\right)\right)\right]\right|\\
 & \leq K\mathbb{E}_{Z}\left[\left|{\cal H}\left(X,\left\{ \bar{w}_{i}\right\} _{i\leq L}\right)-{\cal H}\left(X,W\left(t\right)\right)\right|\right]+{\rm ess\text{-}sup}\left|\mathbb{E}_{Z}\left[{\cal H}\left(X,W\left(t\right)\right)\varphi_{L-1}\left(H_{L-1}\left(X,C_{L-1};W\left(t\right)\right)\right)\right]\right|\\
 & \leq K\epsilon.
\end{align*}
Since ${\rm supp}\left(H_{L-1}\left(\cdot,C_{L-1};W\left(t\right)\right)\right)=L^{2}\left({\cal P}_{X}\right)$
and $\varphi_{L-1}$ is continuous,
\[
\left|\mathbb{E}_{Z}\left[{\cal H}\left(X,\left\{ \bar{w}_{i}\right\} _{i\leq L}\right)f\left(X\right)\right]\right|\leq K\epsilon\qquad\forall f\in S,
\]
for $S=\left\{ \varphi_{L-1}\circ g:\;g\in L^{2}\left({\cal P}_{X}\right)\right\} $.
Since $\epsilon>0$ is arbitrary,
\[
\left|\mathbb{E}_{Z}\left[{\cal H}\left(X,\left\{ \bar{w}_{i}\right\} _{i\leq L}\right)f\left(X\right)\right]\right|=0\qquad\forall f\in S.
\]
Furthermore, since $\varphi_{L-1}$ is non-obstructive, $S$ has dense
span in $L^{2}\left({\cal P}_{X}\right)$. Therefore ${\cal H}\left(x,\left\{ \bar{w}_{i}\right\} _{i\leq L}\right)=0$
for ${\cal P}_{X}$-almost every $x$. Since $\varphi_{L}'$ is non-zero
everywhere, 
\[
\mathbb{E}_{Z}\left[\partial_{2}{\cal L}\left(Y,\hat{y}\left(X;\left\{ \bar{w}_{i}\right\} _{i\leq L}\right)\right)\middle|X=x\right]=0
\]
 for ${\cal P}_{X}$-almost every $x$.

In Case 1, due to convexity of ${\cal L}$, for any measurable function
$\tilde{y}$:
\[
{\cal L}\left(y,\tilde{y}\left(x\right)\right)-{\cal L}\left(y,\hat{y}\left(x;\left\{ \bar{w}_{i}\right\} _{i\leq L}\right)\right)\geq\partial_{2}{\cal L}\left(y,\hat{y}\left(x,\left\{ \bar{w}_{i}\right\} _{i\leq L}\right)\right)\left(\tilde{y}\left(x\right)-\hat{y}\left(x,\left\{ \bar{w}_{i}\right\} _{i\leq L}\right)\right).
\]
Taking expectation, we get $\mathbb{E}_{Z}\left[{\cal L}\left(Y,\tilde{y}\left(X\right)\right)\right]\geq\mathscr{L}\left(\left\{ \bar{w}_{i}\right\} _{i\leq L}\right)$.

In Case 2, we have $\partial_{2}{\cal L}\left(y\left(x\right),\hat{y}\left(x;\left\{ \bar{w}_{i}\right\} _{i\leq L}\right)\right)=0$,
and hence ${\cal L}\left(y\left(x\right),\hat{y}\left(x;\left\{ \bar{w}_{i}\right\} _{i\leq L}\right)\right)=0$,
for ${\cal P}_{X}$-almost every $x$, since $y$ is a function of
$x$.

This hence gives a result on $\mathscr{L}\left(\left\{ \bar{w}_{i}\right\} _{i\leq L}\right)$,
conditional on the claim that ${\cal H}(x,W(t))\to{\cal H}\left(x,\{\bar{w}_{i}\}_{i\le L}\right)$
in $L^{1}({\cal P}_{X})$ as $t\to\infty$. We now prove the claim.
Recall the coupling $\pi_{t}$ in Assumption \ref{assump:multilayer}.4.
In the following, we let $(C_{1},\dots,C_{L},C_{1}',\dots,C_{L}')\sim\pi_{t}$.
For brevity, we denote
\[
\delta_{i}\left(t,x,c_{i},c_{i}'\right)=\left|H_{i}\left(x,c_{i}';W\left(t\right)\right)-H_{i}\left(x,c_{i};\left\{ \bar{w}_{i}\right\} _{i\leq L}\right)\right|.
\]
First observe that by the regularity assumption, for $2\leq i\leq L$:
\begin{align*}
\delta_{i}\left(t,x,c_{i},c_{i}'\right) & \leq K\mathbb{E}_{C_{i-1},C_{i-1}'}\left[\left|w_{i}\left(t,C_{i-1}',c_{i}'\right)-\bar{w}_{i}\left(C_{i-1},c_{i}\right)\right|+\left|\bar{w}_{i}\left(C_{i-1},c_{i}\right)\right|\delta_{i-1}\left(t,x,C_{i-1},C_{i-1}'\right)\right],\\
\delta_{1}\left(t,x,c_{1},c_{1}'\right) & \leq K\left|w_{1}\left(t,c_{1}'\right)-\bar{w}_{1}\left(c_{1}\right)\right|.
\end{align*}
This thus gives:
\begin{align*}
 & \mathbb{E}_{Z}\left[\left|{\cal H}\left(X,W\left(t\right)\right)-{\cal H}\left(X,\left\{ \bar{w}_{i}\right\} _{i\leq L}\right)\right|\right]\\
 & \leq K\mathbb{E}_{Z}\left[\delta_{L}\left(t,X,1,1\right)\right]\\
 & \leq K^{L}\sum_{i=2}^{L}\mathbb{E}\left[\left|w_{i}\left(t,C_{i-1}',C_{i}'\right)-\bar{w}_{i}\left(C_{i-1},C_{i}\right)\right|\sideset{}{_{j=i+1}^{L}}\prod\left|\bar{w}_{j}\left(C_{j-1},C_{j}\right)\right|\right]\\
 & \quad+K^{L}\mathbb{E}\left[\left|w_{1}\left(t,C_{1}'\right)-\bar{w}_{1}\left(C_{1}\right)\right|\sideset{}{_{j=2}^{L}}\prod\left|\bar{w}_{j}\left(C_{j-1},C_{j}\right)\right|\right].
\end{align*}
By the convergence assumption, the right-hand side tends to $0$ as
$t\to\infty$. This proves the claim.

Finally let us connect $\mathscr{L}\left(W\left(t\right)\right)$
with $\mathscr{L}\left(\left\{ \bar{w}_{i}\right\} _{i\leq L}\right)$:
\begin{align*}
\left|\mathscr{L}\left(W\left(t\right)\right)-\mathscr{L}\left(\left\{ \bar{w}_{i}\right\} _{i\leq L}\right)\right| & =\left|\mathbb{E}_{Z}\left[{\cal L}\left(Y,\hat{y}\left(X;W\left(t\right)\right)\right)-{\cal L}\left(Y,\hat{y}\left(X;\left\{ \bar{w}_{i}\right\} _{i\leq L}\right)\right)\right]\right|\\
 & \leq K\mathbb{E}_{Z}\left[\left|\hat{y}\left(X;W\left(t\right)\right)-\hat{y}\left(X;\left\{ \bar{w}_{i}\right\} _{i\leq L}\right)\right|\right]\\
 & \leq K\mathbb{E}_{Z}\left[\delta_{L}\left(t,X,1,1\right)\right],
\end{align*}
which again tends to $0$ as $t\to\infty$. This concludes the proof.
\end{proof}

\section{Convergence to Global Optimum under Morse-Sard Assumptions\label{sec:Global-convergence-ms}}

\textbf{\textcolor{red}{}}

In this section, we show that global convergence is guaranteed under
a different set of convergence assumptions, namely convergence in
moments of the MF limit and certain Morse-Sard assumptions. This generalizes
the global convergence mechanism of \cite{chizat2018} for two-layer
networks to settings where the loss ${\cal L}$ is not necessarily
convex and the depth $L\geq2$.

\subsection{The two-layer case\label{subsec:global-convergence-ms-2}}

Consider the two-layer setting of Section \ref{subsec:global_convergence_iid_2}.
We make the following assumption.
\begin{assumption}
\label{assump:Morse-Sard-2}There exist limits $\overline{w}_{1}$
and $\overline{w}_{2}$ such that the following hold:
\begin{enumerate}
\item (Wasserstein-type convergence.) There exists a coupling $\pi_{t}$
of $P_{1}$ and itself such that 
\[
\mathbb{E}_{\pi_{t}}\left[\left|\overline{w}_{2}(C_{1})\right|\left|w_{1}(t,C_{1}')-\overline{w}_{1}(C_{1})\right|+\left|w_{2}(t,C_{1}',1)-\overline{w}_{2}(C_{1})\right|\right]\to0
\]
 as $t\to\infty$, where $(C_{1},C_{1}')\sim\pi_{t}$.
\item (Morse-Sard in the limit.) With $\overline{W}=\left\{ \overline{w}_{1},\overline{w}_{2}\right\} $,
the mapping 
\[
u_{1}\mapsto\overline{{\cal F}}(u_{1}):=\left|\mathbb{E}_{Z}\left[\partial_{2}{\cal L}(Y,\hat{y}(X;\overline{W}))\varphi_{2}'(H_{2}(X,1;\overline{W}))\varphi_{1}(\langle u_{1},X\rangle)\right]\right|^{2}
\]
satisfies the following property. As $r\to\infty$, $\tilde{u}_{1}\mapsto\overline{{\cal F}}(r\tilde{u}_{1})$
converges uniformly in $C^{1}(\mathbb{S}^{d-1})$ to a function $\overline{{\cal F}}^{\infty}:\;\mathbb{S}^{d-1}\to\mathbb{R}$,
where $\mathbb{S}^{d-1}=\{\tilde{u}_{1}\in\mathbb{R}^{d}:\;|\tilde{u}_{1}|=1\}$.
Furthermore, for any stationary point $u_{1}^{*}$ of $\overline{{\cal F}}$
with $\overline{{\cal F}}(u_{1}^{*})>0$, and for any $\delta_{0}>0$,
there exists $\delta\in\left(0,\delta_{0}\right)$ so that for $S_{\delta}$
the connected component of the set $\{u:\;\overline{{\cal F}}(u)>\overline{{\cal F}}(u_{1}^{*})-\delta\}$
that contains $u_{1}^{*}$, there is $\xi>0$ such that $\left|\nabla\overline{{\cal F}}(u)\right|>\xi$
for all $u\in\partial{\rm cl}\left(S_{\delta}\right)$ the boundary
of the closure of $S_{\delta}$. Similarly, for any stationary point
$\tilde{u}_{1}^{*}$ of $\overline{{\cal F}}^{\infty}$ with $\overline{{\cal F}}^{\infty}(\tilde{u}_{1}^{*})>0$
and for any $\delta_{0}>0$, there exists $\delta\in\left(0,\delta_{0}\right)$
so that for $\tilde{S}_{\delta}$ the connected component of the set
$\{\tilde{u}\in\mathbb{S}^{d-1}:\;\overline{{\cal F}}^{\infty}(\tilde{u})>\overline{{\cal F}}^{\infty}(\tilde{u}_{1}^{*})-\delta\}$
that contains $\tilde{u}_{1}^{*}$, there is $\xi>0$ such that $\left|\nabla{\cal \overline{{\cal F}}}^{\infty}(\tilde{u})\right|>\xi$
for all $\tilde{u}\in\partial{\rm cl}\left(\tilde{S}_{\delta}\right)$.
\end{enumerate}
\end{assumption}

The convergence condition in the above assumption is actually the
first part of Assumption \ref{assump:two-layers}.4, and hence the
same remark for Assumption \ref{assump:two-layers}.4 applies; i.e.
one can deduce this condition from the convergence of $(w_{1}(t,\cdot),w_{2}(t,\cdot,1))$
to $(\overline{w}_{1},\overline{w}_{2})$ as $t\to\infty$ in the
Wasserstein-$2$ distance.
\begin{thm}
\label{thm:global-optimum-2-ms}Consider the MF limit corresponding
to the network (\ref{eq:two-layer-nn}), such that they are coupled
together by the coupling procedure in Section \ref{subsec:Neuronal-Embedding}.
Under Assumptions \ref{assump:two-layers}.1, \ref{assump:two-layers}.2,
\ref{assump:two-layers}.3, \ref{assump:two-layers}.5 and \ref{assump:Morse-Sard-2},
the following hold:
\begin{itemize}
\item Case 1 (convex loss): If ${\cal L}$ is convex in the second variable,
then:
\[
\lim_{t\to\infty}\mathscr{L}\left(W\left(t\right)\right)=\inf_{f_{1},f_{2}}\mathscr{L}\left(f_{1},f_{2}\right)=\inf_{\tilde{y}}\mathbb{E}_{Z}\left[{\cal L}\left(Y,\tilde{y}\left(X\right)\right)\right].
\]
\item Case 2 (generic non-negative loss): Suppose $\partial_{2}{\cal L}\left(y,\hat{y}\right)=0$
implies ${\cal L}\left(y,\hat{y}\right)=0$. If $y=y(x)$ a function
of $x$, then $\mathscr{L}\left(W\left(t\right)\right)=0$ as $t\to\infty$.
\end{itemize}
\end{thm}

Let us make a comparison with the two-layer setting in Section \ref{subsec:global_convergence_iid_2},
and in particular, Assumption \ref{assump:two-layers}. We see that
the convergence assumption in Assumption \ref{assump:two-layers}
is replaced by Assumption \ref{assump:Morse-Sard-2}. More specifically
the uniform convergence condition of $\frac{\partial}{\partial t}w_{2}\left(t,C_{1},1\right)$
in Assumption \ref{assump:two-layers} is replaced by the Morse-Sard
condition of Assumption \ref{assump:Morse-Sard-2}.

Similar to the proof of Theorem \ref{thm:global-optimum-2} in Section
\ref{subsec:global_convergence_iid_2}, the role of the Morse-Sard
condition is -- together with the full support property ${\rm supp}\left({\rm Law}\left(w_{1}\left(t,C_{1}\right)\right)\right)=\mathbb{R}^{d}$
by Lemma \ref{lem:full-support-2} -- to affirm that
\[
\overline{{\cal F}}\left(u_{1}\right)=\left|\mathbb{E}_{Z}\left[\partial_{2}{\cal L}(Y,\hat{y}(X;\overline{W}))\varphi_{2}'(H_{2}(X,1;\overline{W}))\varphi_{1}(\langle u_{1},X\rangle)\right]\right|^{2}=0\quad\forall u_{1}\in\mathbb{R}^{d},
\]
after which universal approximation is invoked to yield the desired
global convergence. The main idea is the following: should the above
not hold, there exists a region of $u_{1}$ where $\overline{{\cal F}}\left(u_{1}\right)>0$.
Since ${\rm supp}\left({\rm Law}\left(w_{1}\left(t,C_{1}\right)\right)\right)=\mathbb{R}^{d}$,
at any time $t$, for a non-negligible mass on $C_{1}$, $w_{1}\left(t,C_{1}\right)$
fully occupies the region. The Morse-Sard condition ensures that the
interaction over time between the two layers $w_{1}\left(t,C_{1}\right)$
and $w_{2}\left(t,C_{1},1\right)$ in this region would however force
the dynamics to diverge so long as $\overline{{\cal F}}\left(u_{1}\right)>0$.

The proof of Theorem \ref{thm:global-optimum-2-ms} is deferred to
Appendix \ref{sec:Remaining-proofs-global-conv-ms}. We also refer
to Section \ref{sec:Discussion} for further discussions. In the following,
we extend this result to the multilayer case and present its proof.

\subsection{The multilayer case}

One can obtain a multilayer analog of Theorem \ref{thm:global-optimum-2-ms}.
The key idea behind the Morse-Sard condition is similar to the two-layer
case. Here the advantage of our framework becomes clearer since it
easily accommodates the idea in the multilayer setup.

Recall the setting of Section \ref{sec:global_convergence_general}.
We make the following assumption which is a direct analogue of Assumption
\ref{assump:Morse-Sard-2} in the two-layer case.
\begin{assumption}
\label{assump:Morse-Sard-multi}There exist limits $\left\{ \overline{w}_{i}\right\} _{i\leq L}$
such that the following hold:
\begin{enumerate}
\item (Wasserstein-type convergence.) There exist couplings $\pi_{t}$ of
$\prod_{i=1}^{L}P_{i}$ and itself such that 
\begin{align*}
\mathbb{E}_{\pi_{t}}\left[\left|w_{i}\left(t,C_{i-1}',C_{i}'\right)-\overline{w}_{i}\left(C_{i-1},C_{i}\right)\right|^{2}\sideset{}{_{j=i+1}^{L}}\prod\left|\overline{w}_{j}\left(C_{j-1},C_{j}\right)\right|^{2}\right] & \to0,\quad i=2,...,L,\\
\mathbb{E}_{\pi_{t}}\left[\left|w_{1}\left(t,C_{1}'\right)-\overline{w}_{1}\left(C_{1}\right)\right|^{2}\sideset{}{_{j=2}^{L}}\prod\left|\overline{w}_{j}\left(C_{j-1},C_{j}\right)\right|^{2}\right] & \to0,
\end{align*}
as $t\to\infty$, where $(C_{1},\dots,C_{L},C_{1}',\dots,C_{L}')\sim\pi_{t}$.
\item (Morse-Sard in the limit.) With $\overline{W}=\left\{ \overline{w}_{i}\right\} _{i\leq L}$,
the mapping $u_{L-1}\in L^{2}(P_{L-2})\mapsto\overline{{\cal F}}(u_{L-1})$
defined by
\[
\overline{{\cal F}}(u_{L-1}):=\left|\mathbb{E}_{Z}\left[\partial_{2}{\cal L}(Y,y(X;\overline{W}))\varphi_{L}'(H_{L}(X,1;\overline{W}))\varphi_{L-1}\left(\left\langle u_{L-1},H_{L-2}(X,\cdot;\overline{W})\right\rangle _{L^{2}\left(P_{L-2}\right)}\right)\right]\right|^{2}
\]
satisfies the following property. As $r\to\infty$, $\tilde{u}_{L-1}\mapsto\overline{{\cal F}}(r\tilde{u}_{L-1})$
converges uniformly in $C^{1}\left(\mathbb{S}\left(L^{2}\left(P_{L-2}\right)\right)\right)$
to a function $\overline{{\cal F}}^{\infty}:\;\mathbb{S}\left(L^{2}\left(P_{L-2}\right)\right)\to\mathbb{R}$,
where $\mathbb{S}\left(L^{2}\left(P_{L-2}\right)\right)=\{u\in L^{2}\left(P_{L-2}\right):\;|u|_{L^{2}\left(P_{L-2}\right)}=1\}$.
Furthermore, for any stationary point $u_{L-1}^{*}$ of $\overline{{\cal F}}$
and for any $\delta_{0}>0$, there exists $\delta\in\left(0,\delta_{0}\right)$
so that for $S_{\delta}$ the connected component of the set $\{u:\;\overline{{\cal F}}(u)>\overline{{\cal F}}(u_{L-1}^{*})-\delta\}$
that contains $u_{L-1}^{*}$, there is $\xi>0$ such that $\left|\nabla\overline{{\cal F}}(u)\right|>\xi$
for all $u\in\partial{\rm cl}\left(S_{\delta}\right)$ the boundary
of the closure of $S_{\delta}$. Similarly, for any stationary point
$\tilde{u}_{L-1}^{*}$ of $\overline{{\cal F}}^{\infty}$ and for
any $\delta_{0}>0$, there exists $\delta\in\left(0,\delta_{0}\right)$
so that for $S_{\delta}$ the connected component of the set $\{u:\;\overline{{\cal F}}(u)>\overline{{\cal F}}^{\infty}(\tilde{u}_{L-1}^{*})-\delta\}$
which contains $r\tilde{u}_{L-1}^{*}$ for all $r$ sufficiently large,
there is $\xi>0$ such that $\left|\nabla\overline{{\cal F}}(u)\right|>\xi$
for all $u\in\partial{\rm cl}\left(S_{\delta}\right)$.
\end{enumerate}
\end{assumption}

The convergence condition in the above assumption is actually the
first part of Assumption \ref{assump:multilayer}.4, and hence the
same remark for Assumption \ref{assump:multilayer}.4 applies; i.e.
one can deduce this condition from the convergence of $(w_{i}(t))_{i=1}^{L}$
to $(\bar{w}_{i})_{i=1}^{L}$ in an appropriate Wasserstein distance.
We now state the theorem. The proof is deferred to Section \ref{subsec:Proof-multilayer-ms}.
\begin{thm}
\label{thm:global-optimum-multilayer-ms}Consider a neuronal embedding
$\left(\Omega,P,\left\{ w_{i}^{0}\right\} _{i\leq L}\right)$ and
the MF limit described in Section \ref{sec:global_convergence_general},
and in particular, under Assumptions Assumptions \ref{assump:multilayer}.1,
\ref{assump:multilayer}.2, \ref{assump:multilayer}.3, \ref{assump:multilayer}.5
and \ref{assump:Morse-Sard-multi}. Assume $\xi_{L}\left(\cdot\right)=\xi_{L-1}\left(\cdot\right)=1$.
Then:
\begin{itemize}
\item Case 1 (convex loss): If ${\cal L}$ is convex in the second variable,
then:
\[
\lim_{t\to\infty}\mathscr{L}\left(W\left(t\right)\right)=\inf_{F}\mathscr{L}\left(F\right)=\inf_{\tilde{y}:\;\mathbb{R}^{d}\to\mathbb{R}}\mathbb{E}_{Z}\left[{\cal L}\left(Y,\tilde{y}\left(X\right)\right)\right].
\]
\item Case 2 (generic non-negative loss): Suppose that $\partial_{2}{\cal L}\left(y,\hat{y}\right)=0$
implies ${\cal L}\left(y,\hat{y}\right)=0$. If $y=y(x)$ is a function
of $x$, then $\mathscr{L}\left(W\left(t\right)\right)\to0$ as $t\to\infty$.
\end{itemize}
\end{thm}

\subsection{Proof of Theorem \ref{thm:global-optimum-multilayer-ms}\label{subsec:Proof-multilayer-ms}}
\begin{proof}[Proof of Theorem \ref{thm:global-optimum-multilayer-ms}]
In the following, for $f,g\in L^{2}\left(P_{L-2}\right)$, let us
write $\left\langle f,g\right\rangle $ in place of $\left\langle f,g\right\rangle _{L^{2}\left(P_{L-2}\right)}$
for brevity. Let us define 
\[
\overline{H}_{i}(x,C_{i})=H_{i}(x,C_{i};\overline{W}),\qquad\overline{H}_{L}(x)=\overline{H}_{L}(x,1),\qquad\bar{y}(x)=\hat{y}(x;\overline{W}),
\]
and for $u_{L-1}\in L^{2}(P_{L-2})$, 
\begin{align*}
\overline{G}_{L}(u_{L-1}) & =\mathbb{E}_{Z}\left[\partial_{2}{\cal L}(Y,\overline{y}(X))\varphi_{L}'(\overline{H}_{L}(X))\varphi_{L-1}(\langle u_{L-1},\overline{H}_{L-2}(X,\cdot)\rangle)\right],\\
\overline{G}_{L-1}(u_{L-1}) & =\mathbb{E}_{Z}\left[\partial_{2}{\cal L}(Y,\overline{y}(X))\varphi_{L}'(\overline{H}_{L}(X))\varphi_{L-1}'(\langle u_{L-1},\overline{H}_{L-2}(X,\cdot)\rangle)\overline{H}_{L-2}(X,\cdot)\right],\\
G_{L}(t,u_{L-1}) & =\mathbb{E}_{Z}\left[\partial_{2}{\cal L}(Y,\hat{y}(t,X))\varphi_{L}'(H_{L}(t,X,1))\varphi_{L-1}(\langle u_{L-1},H_{L-2}(t,X,\cdot)\rangle\right],\\
G_{L-1}(t,u_{L-1}) & =\mathbb{E}_{Z}\left[\partial_{2}{\cal L}(Y,\hat{y}(t,X))\varphi_{L}'(H_{L}(t,X,1))\varphi_{L-1}'(\langle u_{L-1},H_{L-2}(t,X,\cdot)\rangle)H_{L-2}(t,X,\cdot)\right].
\end{align*}
Note that $\overline{G}_{L-1}(u_{L-1}),G_{L-1}(t,u_{L-1})\in L^{2}\left(P_{L-2}\right)$
which follows easily from Assumptions \ref{assump:multilayer}.3 and
\ref{assump:Morse-Sard-multi}.1 and Lemma \ref{lem:bounds MF a priori}.
Similar to the proof of Theorem \ref{thm:global-optimum-2-ms}, with
Assumptions \ref{assump:multilayer}.3 and \ref{assump:Morse-Sard-multi}.1,
we obtain that as $t\to\infty$,
\[
\mathbb{E}\left[\left|\overline{H}_{L}(X)-H_{L}(t,X,1)\right|^{2}\right]\to0,\qquad\mathbb{E}\left[\left|\bar{y}(X)-\hat{y}(t,X)\right|^{2}\right]\to0,
\]
\[
\mathbb{E}\left[\left|\partial_{2}{\cal L}(Y,\hat{y}(t,X))-\partial_{2}{\cal L}(Y,\overline{y}(X))\right|^{2}\right]\to0,
\]
and uniformly in $u_{L-1}$, 
\[
\left|G_{L-1}(t,u_{L-1})-\overline{G}_{L-1}(u_{L-1})\right|^{2}\to0,\qquad\left|G_{L}(t,u_{L-1})-\overline{G}_{L}(u_{L-1})\right|^{2}\to0.
\]

Consider the limit potential given by 
\[
\overline{{\cal F}}(u_{L-1})=\frac{1}{2}\left|\overline{G}_{L}(u_{L-1})\right|^{2}.
\]
By Assumption \ref{assump:multilayer}.3, $u_{L-1}\mapsto\overline{{\cal F}}\left(u_{L-1}\right)$
is continuous. Notice that 
\[
\nabla\overline{{\cal F}}(u_{L-1})=\frac{1}{2}\cdot2\overline{G}_{L}(u_{L-1})\nabla\left(\overline{G}_{L}(u_{L-1})\right)=\overline{G}_{L}(u_{L-1})\overline{G}_{L-1}(u_{L-1}).
\]
Let $\overline{{\cal F}}^{\infty}:\;\mathbb{S}\left(L^{2}\left(P_{L-2}\right)\right)\to\mathbb{R}$
be defined by $\overline{{\cal F}}^{\infty}(\tilde{u}_{L-1})=\lim_{r\to\infty}\overline{{\cal F}}(r\tilde{u}_{L-1})$,
which exists by Assumption \ref{assump:Morse-Sard-multi}. We shall
argue that $\overline{{\cal F}}(u_{L-1})=0$ for all $u_{L-1}\in L^{2}\left(P_{L-2}\right)$,
by contradiction. To that end, let us assume that $\overline{{\cal F}}(u_{L-1})\ne0$
for some $u_{L-1}$. Note that $\overline{{\cal F}}$ is bounded by
a constant by Assumption \ref{assump:multilayer}.3. Thus, either
there is a local maximizer $u_{L-1}^{*}$ of $\overline{{\cal F}}$
with $\overline{{\cal F}}(u_{L-1}^{*})>0$ or there is a local maximizer
$\tilde{u}_{L-1}^{*}$ of $\overline{{\cal F}}^{\infty}$ with $\overline{{\cal F}}^{\infty}(\tilde{u}_{L-1}^{*})>0$.

First consider the case that $\overline{{\cal F}}$ has a local maximizer
$u_{L-1}^{*}$ with $\overline{{\cal F}}(u_{L-1}^{*})>0$. Under Assumption
\ref{assump:Morse-Sard-multi}, there exists $\delta\in\left(0,\overline{{\cal F}}\left(u_{L-1}^{*}\right)\right)$
arbitrarily small so that for $S_{\delta}$ the connected component
of the set $\{u:\;\overline{{\cal F}}(u)>\overline{{\cal F}}(u_{L-1}^{*})-\delta\}$
that contains $u_{L-1}^{*}$, there is $\xi>0$ such that $\left|\nabla\overline{{\cal F}}(u_{L-1})\right|>\xi$
for all $u_{L-1}\in\partial{\rm cl}\left(S_{\delta}\right)$. Let
$T_{0}$ be sufficiently large so that for $t\ge T_{0}$, we have
if $u_{L-1}\in\partial{\rm cl}\left(S_{\delta}\right)$, $\left|\overline{G}_{L-1}(u_{L-1})-G_{L-1}(t,u_{L-1})\right|\le\xi/\sqrt{8\overline{{\cal F}}(u_{L-1}^{*})}$,
which -- similar to the proof of Theorem \ref{thm:global-optimum-2-ms}
-- implies
\begin{equation}
\left\langle \overline{G}_{L-1}(u_{L-1}),G_{L-1}(t,u_{L-1})\right\rangle >\frac{\xi^{2}}{4\overline{{\cal F}}(u_{L-1}^{*})}.\label{eq:proof-global-opt-multims-1}
\end{equation}
Also, we further enlarge $T_{0}$ so that for $t\ge T_{0}$ and any
$u_{L-1}\in{\rm cl}\left(S_{\delta}\right)$, $\left|\overline{G}_{L}(u_{L-1})-G_{L}(t,u_{L-1})\right|\le\frac{1}{2}\sqrt{\overline{{\cal F}}(u_{L-1}^{*})-\delta}$
and hence 
\begin{align}
G_{L}(t,u_{L-1}) & \geq\overline{G}_{L}(u_{L-1})-\frac{1}{2}\sqrt{\overline{{\cal F}}(u_{L-1}^{*})-\delta}>\overline{G}_{L}(u_{L-1})-\frac{1}{2}\sqrt{\overline{{\cal F}}(u_{L-1})}=\overline{G}_{L}(u_{L-1})-\frac{1}{2}\left|\overline{G}_{L}(u_{L-1})\right|,\label{eq:proof-global-opt-multims-2.1}\\
G_{L}(t,u_{L-1}) & \leq\overline{G}_{L}(u_{L-1})+\frac{1}{2}\sqrt{\overline{{\cal F}}(u_{L-1}^{*})-\delta}<\overline{G}_{L}(u_{L-1})+\frac{1}{2}\sqrt{\overline{{\cal F}}(u_{L-1})}=\overline{G}_{L}(u_{L-1})+\frac{1}{2}\left|\overline{G}_{L}(u_{L-1})\right|.\label{eq:proof-global-opt-multims-2.2}
\end{align}
Furthermore notice that
\begin{equation}
\frac{\partial}{\partial t}\overline{G}_{L}(w_{L-1}(t,\cdot,C_{L-1}))=-w_{L}(t,C_{L-1},1)\left\langle \overline{G}_{L-1}(w_{L-1}(t,\cdot,C_{L-1})),G_{L-1}(t,w_{L-1}(t,\cdot,C_{L-1}))\right\rangle .\label{eq:proof-global-opt-multims-3}
\end{equation}
Let $\tilde{\Omega}_{L-1}$ be the subset of $\Omega_{L-1}$ consisting
of $c_{L-1}$ where $\left|w_{L}(0,c_{L},1)\right|<1$. As shown in
Step 1 of the proof of Theorem \ref{thm:global-optimum-multilayer},
for any $t\geq0$, we have 
\[
{\rm supp}\left(w_{L-1}\left(t,\cdot,C_{L-1}\right),C_{L-1}\in\tilde{\Omega}_{L-1}\right)=L^{2}\left(P_{L-2}\right),
\]
and therefore for any open subset $B$ of $L^{2}\left(P_{L-2}\right)$,
there exists a positive mass of $C_{L-1}\in\tilde{\Omega}_{L-1}$
such that $w_{L-1}\left(t,\cdot,C_{L-1}\right)\in B$. In the following,
we consider $C_{L-1}\in\tilde{\Omega}_{L-1}$. We further divide the
argument into two cases: $\overline{G}_{L}(u_{L-1}^{*})>0$ and $\overline{G}_{L}(u_{L-1}^{*})<0$.

Let us consider the case that $\overline{G}_{L}(u_{L-1}^{*})>0$.
Then we can choose sufficiently small $\delta$ such that $\overline{G}_{L}(u_{L-1})>0$
for all $u_{L-1}\in{\rm cl}\left(S_{\delta}\right)$. Furthermore
consider the scenario that there exists $T\ge T_{0}$ such that a
positive mass of $\left(w_{L-1}\left(T,\cdot,C_{L-1}\right),w_{L}\left(T,C_{L-1},1\right)\right)$
on $C_{L-1}\in\tilde{\Omega}_{L-1}$ has $w_{L-1}\left(T,\cdot,C_{L-1}\right)\in S_{\delta}$
and $w_{L}\left(T,C_{L-1},1\right)<0$. Note that if $w_{L-1}\left(t,\cdot,C_{L-1}\right)\in S_{\delta}$,
\[
\frac{\partial}{\partial t}w_{L}\left(t,C_{L-1},1\right)=-G_{L}\left(t,w_{L-1}\left(T,\cdot,C_{L-1}\right)\right)\leq-\left(\overline{G}_{L}\left(w_{L-1}\left(T,\cdot,C_{L-1}\right)\right)-\frac{1}{2}\left|\overline{G}_{L}\left(w_{L-1}\left(T,\cdot,C_{L-1}\right)\right)\right|\right)<0
\]
by Eq. (\ref{eq:proof-global-opt-multims-2.1}). Define $T_{1}=\inf\left\{ t\geq T:\;w_{L-1}\left(t,\cdot,C_{L-1}\right)\notin S_{\delta}\right\} $.
Then $t\mapsto w_{L}\left(t,C_{L-1},1\right)$ is decreasing on $t\in[T,T_{1})$.
Let us argue that $T_{1}=\infty$. Indeed, suppose $T_{1}$ is finite.
We then have, by continuity, $w_{L-1}\left(T_{1},\cdot,C_{L-1}\right)\in\partial{\rm cl}\left(S_{\delta}\right)$
and $w_{L}\left(T_{1},C_{L-1},1\right)\le w_{L}\left(T,C_{L-1},1\right))<0$.
As such, $\frac{\partial}{\partial t}\overline{G}_{L}\left(w_{L-1}\left(T_{1},\cdot,C_{L-1}\right)\right)>0$
by Eq. (\ref{eq:proof-global-opt-multims-1}) and (\ref{eq:proof-global-opt-multims-3}).
By continuity, for some $\gamma>0$, $\frac{\partial}{\partial t}\overline{G}_{L}\left(w_{L-1}\left(T_{1}+t,\cdot,C_{L-1}\right)\right)>0$
for all $t\in\left[0,\gamma\right]$. But then $\overline{G}_{L}\left(w_{L-1}\left(T_{1}+t,\cdot,C_{L-1}\right)\right)\ge\overline{G}_{L}\left(w_{L-1}\left(T_{1},\cdot,C_{L-1}\right)\right)\ge\sqrt{2(\overline{{\cal F}}(u_{L-1}^{*})-\delta)}$,
and hence $w_{L-1}\left(T_{1}+t,\cdot,C_{L-1}\right)\in S_{\delta}$
for all $t\leq\gamma$, contradicting the definition of $T_{1}$.
Therefore $T_{1}=\infty$, i.e. for $t\ge T$ and $C_{1}\in\tilde{\Omega}_{1}$
with $w_{L-1}\left(T,\cdot,C_{L-1}\right)\in S_{\delta}$ and $w_{L}(T,C_{L-1},1)<0$,
we have $w_{L-1}\left(t,\cdot,C_{L-1}\right)\in S_{\delta}$, which
implies 
\[
G_{L}(t,w_{L-1}\left(t,\cdot,C_{L-1}\right))\stackrel{\left(a\right)}{\ge}\frac{1}{2}\overline{G}_{L}\left(w_{L-1}\left(t,\cdot,C_{L-1}\right)\right)=\sqrt{\frac{1}{2}\overline{{\cal F}}\left(w_{L-1}\left(t,\cdot,C_{L-1}\right)\right)}\geq\sqrt{\frac{1}{2}(\overline{{\cal F}}(u_{L-1}^{*})-\delta)},
\]
where $\left(a\right)$ is by Eq. (\ref{eq:proof-global-opt-multims-2.1})
and the fact $\overline{G}_{L}(u_{L-1})>0$ for all $u_{L-1}\in{\rm cl}\left(S_{\delta}\right)$.
In particular, there is a positive mass of $\left(w_{L-1}\left(T,\cdot,C_{L-1}\right),w_{L}\left(T,C_{L-1},1\right)\right)$
with $G_{L}(t,w_{L-1}\left(t,\cdot,C_{L-1}\right))\ge\sqrt{(\overline{{\cal F}}(u_{L-1}^{*})-\delta)/2}$
for all $t\ge T$. Noting that 
\[
\frac{d}{dt}\mathbb{E}\left[{\cal L}(Y,\hat{y}(t,X))\right]\leq-\mathbb{E}\left[\left|G_{L}(t,w_{L-1}\left(t,\cdot,C_{L-1}\right))\right|^{2}\right],
\]
we obtain $\frac{d}{dt}\mathbb{E}\left[{\cal L}(Y,\hat{y}(t,X))\right]$
being bounded above by a strictly negative constant for all $t\ge T$,
which is a contradiction since ${\cal L}$ is bounded below.

Next consider the scenario that for all $t\ge T_{0}$, the probability
that $w_{L-1}\left(t,\cdot,C_{L-1}\right)\in S_{\delta}$ and $w_{L}\left(t,C_{L-1},1\right)<0$
on $C_{L-1}\in\tilde{\Omega}_{L-1}$ is zero. Let us argue that for
any $t\ge T_{0}$ and for a.e. $C_{L-1}\in\tilde{\Omega}_{L-1}$ with
$w_{L-1}\left(t,\cdot,C_{L-1}\right)\in S_{\delta}$, we have $w_{L-1}\left(s,\cdot,C_{L-1}\right)\in S_{\delta}$
for all $s\in[T_{0},t]$. Indeed, consider $t$ and $C_{L-1}\in\tilde{\Omega}_{L-1}$
such that $w_{L-1}\left(t,\cdot,C_{L-1}\right)\in S_{\delta}$ and
$w_{L-1}\left(T',\cdot,C_{L-1}\right)\notin S_{\delta}$ for some
$T'\in[T_{0},t)$. Let $t'=\sup\{s\in\left[T',t\right]:\;w_{L-1}\left(s,\cdot,C_{L-1}\right)\notin S_{\delta}\}<t$.
By continuity, $w_{L-1}\left(t',\cdot,C_{L-1}\right)\in\partial{\rm cl}\left(S_{\delta}\right)$
and so by Eq. (\ref{eq:proof-global-opt-multims-1}),
\[
\left\langle \overline{G}_{L-1}(w_{L-1}\left(t',\cdot,C_{L-1}\right)),G_{L-1}(t',w_{L-1}\left(t',\cdot,C_{L-1}\right))\right\rangle >\frac{\xi^{2}}{4\overline{{\cal F}}(u_{L-1}^{*})}.
\]
By continuity, there exists $t''\in\left(t',t\right)$ such that for
all $s\in[t',t'']$,
\[
\left\langle \overline{G}_{L-1}(w_{L-1}\left(s,\cdot,C_{L-1}\right)),G_{L-1}(s,w_{L-1}\left(s,\cdot,C_{L-1}\right))\right\rangle \geq\frac{\xi^{2}}{100\overline{{\cal F}}(u_{L-1}^{*})}.
\]
By definition of $t'$, we also have $w_{L-1}\left(s,\cdot,C_{L-1}\right)\in S_{\delta}$
and therefore $w_{L}(s,C_{L-1},1)\ge0$ for any $s\in(t',t]$. Then
by Eq. (\ref{eq:proof-global-opt-2ms-3}), $\frac{\partial}{\partial t}\overline{G}_{L}(w_{L-1}\left(s,\cdot,C_{L-1}\right))\leq0$
for all $s\in(t',t'']$ and therefore
\[
\overline{G}_{L}(w_{L-1}\left(t'',\cdot,C_{L-1}\right))\leq\overline{G}_{L}(w_{L-1}\left(t',\cdot,C_{L-1}\right))=\sqrt{2(\overline{{\cal F}}(u_{L-1}^{*})-\delta)},
\]
where the equality follows from $w_{L-1}\left(t',\cdot,C_{L-1}\right)\in\partial{\rm cl}\left(S_{\delta}\right)$.
However this contradicts with $w_{L-1}\left(t'',\cdot,C_{L-1}\right)\in S_{\delta}$.
Therefore it holds that for any $t\ge T_{0}$ , for a.e. $C_{L-1}\in\tilde{\Omega}_{L-1}$
with $w_{L-1}\left(t,\cdot,C_{L-1}\right)\in S_{\delta}$, we have
$w_{L-1}\left(s,\cdot,C_{L-1}\right)\in S_{\delta}$ and therefore
$w_{L}\left(s,C_{L-1},1\right)\ge0$ for all $s\in[T_{0},t]$. Since
$w_{L-1}\left(t,\cdot,C_{L-1}\right)$ on $C_{L-1}\in\tilde{\Omega}_{L-1}$
has full support at any $t\geq0$, we have for any $t_{0}\geq T_{0}$,
there is a positive mass on $C_{L-1}\in\tilde{\Omega}_{L-1}$ such
that $w_{L-1}\left(t_{0},\cdot,C_{L-1}\right)\in S_{\delta}$ and
hence, as shown, $w_{L-1}\left(s,\cdot,C_{L-1}\right)\in S_{\delta}$
and $w_{L}\left(s,C_{L-1},1\right)\ge0$ for all $s\in\left[T_{0},t_{0}\right]$.
Note that we have $w_{L}\left(T_{0},C_{L-1},1\right)\le M(T_{0})$
for some finite $M(T_{0})>0$ for $C_{L-1}\in\tilde{\Omega}_{L-1}$
(which follows from the fact $\left|\frac{\partial}{\partial t}w_{L}\left(t,\cdot,1\right)\right|\leq K$
by Assumption \ref{assump:two-layers}.3 and that $w_{L}\left(0,C_{L-1},1\right)<1$).
Also note that for $w_{L-1}\left(s,\cdot,C_{L-1}\right)\in S_{\delta}$
and $s\geq T_{0}$,
\begin{align*}
 & \frac{\partial}{\partial t}w_{L}\left(s,C_{L-1},1\right)=-G_{L}(s,w_{L-1}\left(s,\cdot,C_{L-1}\right))\stackrel{\left(a\right)}{\leq}-\frac{1}{2}\overline{G}_{L}(w_{L-1}\left(s,\cdot,C_{L-1}\right))\\
 & \qquad=-\sqrt{\frac{1}{2}\overline{{\cal F}}\left(w_{L-1}\left(s,\cdot,C_{L-1}\right)\right)}\leq-\sqrt{\frac{1}{2}(\overline{{\cal F}}(u_{L-1}^{*})-\delta)}
\end{align*}
a strictly negative constant, where $\left(a\right)$ is by Eq. (\ref{eq:proof-global-opt-multims-2.1})
and the fact $\overline{G}_{L}(u_{L-1})>0$ for all $u_{L-1}\in{\rm cl}\left(S_{\delta}\right)$.
As such, for any $t_{0}\geq T_{0}$ such that
\[
M\left(T_{0}\right)-\left(t_{0}-T_{0}\right)\sqrt{\frac{1}{2}(\overline{{\cal F}}(u_{L-1}^{*})-\delta)}<0,
\]
there is a positive mass on $C_{L-1}\in\tilde{\Omega}_{L-1}$ such
that firstly $w_{L}\left(s,C_{L-1},1\right)\ge0$ for all $s\in\left[T_{0},t_{0}\right]$
and secondly there exists $t\in\left[T_{0},t_{0}\right]$ in which
\[
w_{L}\left(t,C_{L-1},1\right)\leq M\left(T_{0}\right)-\left(t-T_{0}\right)\sqrt{\frac{1}{2}(\overline{{\cal F}}(u_{L-1}^{*})-\delta)}<0.
\]
We again obtain a contradiction.

The case $\overline{G}_{L}(u_{L-1}^{*})<0$ can be treated similarly,
with the use of Eq. (\ref{eq:proof-global-opt-multims-2.1}) replaced
by Eq. (\ref{eq:proof-global-opt-multims-2.2}). Both cases lead to
a contradiction, ruling out the possibility that there is a local
maximizer $u_{L-1}^{*}$ of $\overline{{\cal F}}$ with $\overline{{\cal F}}(u_{L-1}^{*})>0$.

Next consider the case $\overline{{\cal F}}$ does not have any local
maximizer in $L^{2}(P_{L-2})$ but $\overline{{\cal F}}^{\infty}$
has a local maximizer $\tilde{u}_{L-1}^{*}$ with $\overline{{\cal F}}^{\infty}(\tilde{u}_{L-1}^{*})>0$.
Under Assumption \ref{assump:Morse-Sard-multi}, there exists $\delta\in\left(0,\overline{{\cal F}}^{\infty}\left(\tilde{u}_{L-1}^{*}\right)\right)$
arbitrarily small so that for $S_{\delta}$ the connected component
of the set $\{u:\;\overline{{\cal F}}(u)>\overline{{\cal F}}^{\infty}(\tilde{u}_{L-1}^{*})-\delta\}$
which contains $r\tilde{u}_{L-1}^{*}$ for all $r$ sufficiently large,
there is $\xi>0$ such that $\left|\nabla\overline{{\cal F}}(u)\right|>\xi$
for all $u\in\partial{\rm cl}\left(S_{\delta}\right)$. The rest of
the argument can be repeated as before to yield a contradiction.

In short, we have shown that $\overline{{\cal F}}(u_{L-1})=\frac{1}{2}\left|\overline{G}_{L}(u_{L-1})\right|^{2}=0$,
and equivalently,
\[
\mathbb{E}_{Z}\left[\partial_{2}{\cal L}(Y,\overline{y}(X))\varphi_{L}'(\overline{H}_{L}(X))\varphi_{L-1}(\langle u_{L-1},\overline{H}_{L-2}(X,\cdot)\rangle)\right]=0
\]
for all $u_{L-1}\in L^{2}(P_{L-2})$. The proof can now be completed
similar to the proof of Theorem \ref{thm:global-optimum-multilayer}.
\end{proof}

\section{Further discussions\label{sec:Discussion}}

Having presented our neuronal embedding framework for multilayer MF
neural networks and proven several results concerning i.i.d. initializations
and global convergence under various settings, we now place the discussion
of our work in the context of related works.

\subsection{Two-layer neural networks}

The MF view on the training dynamics of neural networks has gathered
significant interests in the recent literature, starting with the
two-layer case \cite{nitanda2017stochastic,mei2018mean,chizat2018,rotskoff2018neural,sirignano2018mean}.
In this case, it is known that convergence to global optimum is possible
for gradient descent or SGD \cite{mei2018mean,chizat2018,rotskoff2018neural},
with a potentially exponential rate \cite{javanmard2019analysis}
and a dimension-independent width \cite{mei2019mean}. This line of
works has also inspired research into new training algorithms \cite{wei2018margin,rotskoff2019global,nitanda2020particle},
stability properties of the trained networks \cite{alex2019landscape},
other architectures which are compositions of multiple MF neural networks
\cite{weinan2020machine,lu2020mean} and MF neural networks in other
machine learning contexts \cite{agazzi2020global,nguyen2021analysis}.
Most works focus on fully-connected networks on the Euclidean space
and utilize certain convexity properties to study optimization efficiency.
The MF formulation of the two-layer case in these works enjoys the
wealth of the mathematics of optimal transport and gradient flows
in measure spaces \cite{ambrosio2008gradient}.

Our work, on the other hand, considers general Hilbert spaces which
can be infinite-dimensional (Section \ref{sec:Framework}) and does
not rely critically on convexity (Theorems \ref{thm:global-optimum-2},
\ref{thm:global-optimum-3}, \ref{thm:global-optimum-multilayer},
\ref{thm:global-optimum-2-ms} and \ref{thm:global-optimum-multilayer-ms}).
Our framework departs from the Wasserstein gradient flow viewpoint,
and while being in the early stage of technical foundations, it is
demonstrated to give useful results including and beyond the two-layer
case.

\subsection{Multilayer neural networks}

As mentioned in the introduction, the multilayer case poses a major
conceptual challenge. Prior to our work, several ideas have been proposed
independently. \cite{nguyen2019mean} puts forth the idea that a neuron
is represented by a stochastic (Markov) kernel and gives a heuristic
derivation, where the MF limit is described by a certain evolution
of measures over the space of stochastic kernels. \cite{araujo2019mean}
rigorously derives the MF limit as an evolution of a measure on paths
through layers. In \cite{sirignano2019mean}, the network is viewed
as a time-dependent function of its initialization and this function
simplifies upon concentrations over the randomness of the initialization.
All three works employ scalings with respect to the widths, in which
normalizations are applied at every layers, not just the last layer,
together with compensating learning rates. This thereby ensures nonlinear
evolution at all layers.

Working under the same scalings, our framework gives a new perspective
via a central question: how does one describe an ensemble of an arbitrary
number of neurons? Answering this question, our idea of a neuronal
embedding allows one to describe the MF limit in a clean and rigorous
manner. In particular, it avoids extra assumptions made in \cite{araujo2019mean,sirignano2019mean}:
unlike our work, \cite{araujo2019mean} assumes untrained first and
last layers and requires non-trivial technical tools; \cite{sirignano2019mean}
takes an unnatural sequential limit of the widths and proves a non-quantitative
result, whereas we prove a quantitative bound that essentially requires
only the minimum of the widths to be large. An advantage of our framework
comes from the fact that while MF formulations in \cite{araujo2019mean,sirignano2019mean}
are specific to and exploit i.i.d. initializations, our formulation
does not and thereby allows to study i.i.d. initializations as well
as interesting non-i.i.d. initialization schemes. Compared to \cite{nguyen2019mean},
while a certain step of our analysis takes an inspiration from the
idea of stochastic kernels in \cite{nguyen2019mean}, our framework
circumvents its technical cumbersomeness and gives a rigorous and
clean mathematical treatment.

After our first preprint, the work \cite{fang2020modeling} takes
another view on this challenge. In particular, considering a finite
set of training data, \cite{fang2020modeling} encodes each neuron
by its pre-activation values, computed over the entire training data,
at initialization. As a specific interpretation by \cite{fang2020modeling},
the pre-activation values at initialization capture a certain sense
of ``features'' seen by the neurons. Meanwhile our framework identifies
neuron $j_{i}$ at layer $i$ via the sample $C_{i}\left(j_{i}\right)$
drawn from the space $\left(\Omega_{i},P_{i}\right)$ (Section \ref{subsec:Neuronal-Embedding})
and remains general about this space. One may observe the following
connection: a specific choice of the neuronal embedding can be built
over random variables that are defined by the pre-activation values
at initialization. The generality of our framework maintains freedom
over choices of the neuronal embedding, including this specific choice.
For example, when the training data size is infinite -- an idealized
situation commonly assumed in theoretical studies, then if one follows
\cite{fang2020modeling}, each pre-activation becomes a function over
an infinite domain, instead of a finite-dimensional vector. This potentially
poses technical complications, which can be avoided simply by a different
choice of the neuronal embedding in our framework.

\subsection{Degeneracy with i.i.d. initializations}

As shown in Section \ref{sec:iid_init}, i.i.d. initializations cause
strong degeneracy for a network depth at least four. The work \cite{araujo2019mean}
is the first to realize and take advantage of this phenomenon to formulate
the MF limit; in particular, the measure on the paths in \cite{araujo2019mean}
admits a product structure, signifying the mutually independent nature
of the evolutions of weights at different layers in the infinite width
limit. Note however that \cite{araujo2019mean} explicitly exploits
this degeneracy phenomenon to formulate the MF limit. In contrast,
our framework is general and upon specializing to the case of i.i.d.
initializations, it allows to derive this phenomenon in greater details
and simultaneously remove certain technical assumptions in \cite{araujo2019mean}.
In particular, we remove the technical conditions of random input
and output features and no biases of \cite{araujo2019mean}. In addition,
one can use Corollary \ref{cor:iid_tracking} to immediately verify
that in the setting of no biases, $L\ge5$ and untrained first and
last layers ($\xi_{1}^{\mathbf{w}}\left(\cdot\right)=\xi_{L}^{\mathbf{w}}\left(\cdot\right)=0$),
the weights and activations in the limit satisfy the McKean-Vlasov
equation in \cite{araujo2019mean}.

Such degeneracy is generally undesirable. The fact that our framework
is not specific to i.i.d. initializations allows for an escape from
this situation. In this aspect, our framework follows closely the
spirit of the work \cite{nguyen2019mean}, whose MF formulation is
also not specific to i.i.d. initializations. Through the language
of stochastic kernels, \cite{nguyen2019mean} envisions a scenario
in which evolutions of the weights at different layers are stochastically
coupled. The usefulness of such scenario is realized by our global
convergence guarantee for multilayer networks with arbitrary depths
in Sections \ref{sec:global_convergence_general} and \ref{sec:Global-convergence-ms}
(Theorems \ref{thm:global-optimum-multilayer} and \ref{thm:global-optimum-multilayer-ms}),
with the novel idea of bidirectional diversity for non-i.i.d. initialization.

\subsection{Global convergence}

Optimization efficacy has been one major question that sets the MF
literature apart from other theoretical studies of neural networks,
where one witnesses new involvement of sophisticated mathematical
tools and insights. As mentioned, the two-layer case has enjoyed numerous
efforts to establish global convergence (see e.g. \cite{mei2018mean,chizat2018,rotskoff2018neural,javanmard2019analysis,rotskoff2019global,wei2018margin,chizat2019sparse,wojtowytsch2020convergence}).
Our work is the first to obtain global convergence guarantees in the
MF regime for the multilayer case.

\paragraph*{Two-layer networks: comparison to \cite{chizat2018}.}

Closely relevant to our thread of results is the work \cite{chizat2018}.
This work treats the two-layer case under certain convergence and
Morse-Sard assumptions and convex losses. To make a direct comparison
with \cite{chizat2018}, let us first focus on the two-layer case,
and in particular, Theorem \ref{thm:global-optimum-2} together with
its accompanying Assumption \ref{assump:two-layers}. Several elements
in our analysis are inspired by this work; we also differ in crucial
ways. Similar to \cite{chizat2018}, our proof also hinges on the
insight that a certain diversity property is held throughout the course
of training. We assume a universal approximation property (Assumption
\ref{assump:two-layers}.5), which is natural in neural network learning,
and dispense with convexity of the loss, whereas \cite{chizat2018}
does not utilize universal approximation and requires convex losses.
In our convergence assumption -- Assumption \ref{assump:two-layers}.4
-- the moment convergence condition is similar to the convergence
assumption in \cite{chizat2018}. We differ from \cite{chizat2018}
fundamentally in the uniform convergence condition ${\rm ess\text{-}sup}\left|\frac{\partial}{\partial t}w_{2}\left(t,C_{1},1\right)\right|\to0$
of the second layer's weight. On one hand, this condition replaces
the Morse-Sard condition in \cite{chizat2018}, which is difficult
to verify in general. On the other hand, it is a natural assumption
to make: as shown in Proposition \ref{prop:converse-two-layers},
if this uniform convergence condition fails, global convergence cannot
be attained. In shorts, using the insight on diversity, together with
universal approximation, we uncover a new mechanism for global convergence
without the need for convex losses.

\paragraph*{Multilayer networks.}

While \cite{chizat2018} is specific to two-layer networks, we further
the insight on diversity to the multilayer case, where we introduce
the new notion of bidirectional diversity. In the context of two-layer
networks, diversity refers to that the first layer’s weight distribution
has full support in the Euclidean space. In the multilayer case, this
notion no longer resides in the Euclidean space, but is realized in
function spaces that are naturally described by the neuronal embedding
framework. Moreover, as noted in Section \ref{subsec:Global-convergence-and-bidirectional-diversity},
it highlights an interesting dynamical mechanism, in which adjacent
layers interact with each other over time in such a way that diversity
is preserved through the depth of the network and at any time, roughly
speaking.

Similar to the two-layer case, in place of the Morse-Sard assumption
in \cite{chizat2018}, we show global convergence under uniform convergence
of the gradient update at a certain layer (Theorems \ref{thm:global-optimum-3}
and \ref{thm:global-optimum-multilayer}). Again we note per Propositions
\ref{prop:converse-three-layers} and \ref{prop:converse-multilayer},
there is a converse relation between this uniform convergence and
global convergence; if the former fails, so does the latter.

Several of these insights are utilized in the recent work \cite{fang2020modeling}
that proves a global convergence guarantee for a residual MF neural
architecture under the uniform convergence assumption of the gradient
update. In this architecture, a skip connection is introduced to route
the first layer directly to the second last one. Thanks to this skip
connection, diversity is essentially transferred directly from the
first layer to the second last layer. In short, in \cite{fang2020modeling},
diversity is maintained with the help of architectural imposition.
In contrast, in our global convergence result for the multilayer case,
diversity is maintained automatically by the training dynamics.

The work \cite{lu2020mean}, which studies a type of composition of
many two-layer MF networks, and a recent update of \cite{sirignano2019mean},
which studies the three-layer case, establish conditions of stationary
points to be global optima with certain overlapping ideas. However
they require essentially a certain diversity assumption on the limit
point (i.e. at convergence $t=\infty$). We do not need to make this
assumption: the remark in Section \ref{subsec:three-layers-high-level-idea}
highlights the dynamical nature of the proof where diversity is assumed
at initialization $t=0$ only and proven to hold at any finite training
time $t<\infty$. As explained in Section \ref{subsec:three-layers-high-level-idea},
diversity may not hold at $t=\infty$ and global convergence can still
be attained regardless.

Let us mention again that global convergence results in those works
are proven under the convex loss assumption. On the other hand, our
results allow for removal of this assumption and our proofs do not
make use of convexity in any crucial way.

\paragraph*{Convergence under Morse-Sard assumptions.}

Our framework is able to give a self-contained proof of global convergence
under the Morse-Sard assumption, without the aforementioned uniform
convergence assumption (Theorems \ref{thm:global-optimum-2-ms} and
\ref{thm:global-optimum-multilayer-ms}).

Let us place this discussion in the two-layer context, particularly
Theorem \ref{thm:global-optimum-2-ms} and its accompanying Assumption
\ref{assump:Morse-Sard-2}. Observe that Assumption \ref{assump:Morse-Sard-2}.2
follows immediately if $\overline{{\cal F}}$ and $\overline{{\cal F}}^{\infty}$
satisfy Morse-Sard type regularity, i.e., the sets of regular values
of $\overline{{\cal F}}$ and $\overline{{\cal F}}^{\infty}$ are
dense (hence the name ``Morse-Sard''). Indeed, assume that $\overline{{\cal F}}$
and $\overline{{\cal F}}^{\infty}$ satisfy Morse-Sard type regularity.
Let $\widehat{S}_{\delta}=\{u:\;\overline{{\cal F}}(u)>\overline{{\cal F}}(u_{1}^{*})-\delta\}$.
In that case, for any stationary point $u_{1}^{*}$ of $\overline{{\cal F}}$
with $\overline{{\cal F}}(u_{1}^{*})>0$, and for any $\delta_{0}>0$,
there exists $\delta\in\left(0,\delta_{0}\right)$ so that any $u\in\partial\textrm{cl}(\widehat{S}_{\delta})$
satisfies $\nabla\overline{{\cal F}}(u)\ne0$. Over a bounded connected
component $S_{\delta}$ of $\widehat{S}_{\delta}$, this immediately
implies the existence of $\xi>0$ such that $\left|\nabla\overline{{\cal F}}(u)\right|>\xi$
for all $u\in\partial\textrm{cl}(S_{\delta})$. Over an unbounded
connected component $S_{\delta}$ of $\widehat{S}_{\delta}$, whenever
$\overline{{\cal F}}(u_{1}^{*})-\delta$ is a regular value of $\overline{{\cal F}}^{\infty}$,
there is $\xi>0$ such that $|\nabla\overline{{\cal F}}(u)|>\xi$
for $u\in\partial\textrm{cl}(S_{\delta})\setminus\mathbb{B}(0,r)$
for some $r$ sufficiently large, where $\mathbb{B}(0,r)$ is the
ball around $0$ with radius $r$. Since $\overline{{\cal F}}(u_{1}^{*})-\delta$
is a regular value of $\overline{{\cal F}}$, by making $\xi>0$ smaller
if needed, we can guarantee that $|\nabla\overline{{\cal F}}(u)|>\xi$
for $u\in\partial\textrm{cl}(S_{\delta})\cap\mathbb{B}(0,r)$, and
hence $|\nabla\overline{{\cal F}}(u)|>\xi$ for all $u\in\partial\textrm{cl}(S_{\delta})$.

Thus our Morse-Sard condition is similar to (and slightly weaker than)
the Morse-Sard assumption of \cite{chizat2018}. As stated, it is
sufficient for this condition to hold w.r.t. the limit $\overline{W}=\left\{ \overline{w}_{1},\overline{w}_{2}\right\} $.
A counterpart statement of the Morse-Sard assumption of \cite{chizat2018}
would impose the condition on a generic class of pairs of functions
$\tilde{W}=\left\{ \tilde{w}_{1},\tilde{w}_{2}\right\} $ that contains
$\overline{W}$ and as such trivially imply our assumption.

As explained in Section \ref{subsec:global-convergence-ms-2}, the
Morse-Sard condition forces the interaction over time between the
weights of the two layers in a specific way that guarantees global
convergence. This idea was realized by \cite{chizat2018} in the language
of Wasserstein gradient flows for a convex loss function $\mathscr{L}$
and two-layer neural networks. Here in the two-layer case, firstly
Theorem \ref{thm:global-optimum-2-ms} extends the result to generic
losses; secondly and more importantly, it demonstrates that the same
idea could be naturally executed in our framework without the use
of Wasserstein gradient flows.

Theorem \ref{thm:global-optimum-multilayer-ms} demonstrates further
the applicability of our argument to the multilayer case, which the
Wasserstein gradient flow formulation has difficulty with.

\subsection{Empirical findings and other infinite-width scalings}

Mathematical ideas aside, one important aspect is how well one can
observe the MF limiting behavior in multilayer networks with finite
but large widths, normalized under the MF scaling. This has been demonstrated
positively in the work \cite{nguyen2019mean}. In particular, \cite{nguyen2019mean}
performs experiments on several real-life machine learning tasks and
finds that the evolution curves of certain performance metrics, such
as the training loss and the classification accuracy, are almost insensitive
to the widths -- provided sufficiently large -- and hence they exhibit
a limiting behavior. As \cite{nguyen2019mean} shows, this occurs
as soon as the widths are on the order of just a few hundreds, which
is common in practice.

The MF scaling is not the only infinite-width scaling with interesting
properties. Another popular scaling regime is the neural tangent kernel
(NTK) scaling \cite{jacot2018neural,chizat2018note,du2018gradient,zou2018stochastic,allen2018convergence,lee2019wide}.
In the NTK scaling, the weights do not move and the learning dynamics
becomes linearized, although several interesting properties such as
convergence to the global optimum are attainable. For this reason,
it is often said that the NTK-scaled infinite-width neural networks
do not perform feature learning. This NTK-like behavior is not what
is observed in practical neural networks with finite but large widths.
In contrast, the MF-scaled networks have nonlinear dynamics and weights
moving away from initialization, and are thus said to perform feature
learning in the literature.

The MF scaling is not necessarily the only scaling with feature learning
(see e.g. \cite{golikov2020dynamically,yang2021feature,Hajjar:aa}).
It is known that in the standard scaling that matches with the usual
practice, the networks are NTK-like in the infinite-width limit \cite{mei2019mean,yang2021feature}.
Consequently \textit{any} infinite-width scalings with feature learning
are only proxies of practical finite-width neural networks. Despite
this fact, we note that \cite{nguyen2019mean} demonstrates on several
real-life machine learning tasks that a MF multilayer network, without
heavy hyperparameter tuning, can achieve realistic performances, comparable
to practical neural networks that are similar in architectural designs
and training procedures; \cite{lu2020mean} demonstrates an improved
performance over strong and well-tuned practical neural networks by
using the MF scaling. In other words, the MF scaling offers a good
proxy, with potentially no loss in practical performances.

A few alternative scalings, accompanied by suitable initializations
and learning rates, are proposed in \cite{golikov2020dynamically,yang2021feature,Hajjar:aa}
to avoid the NTK-like behavior. Theoretical understanding of feature
learning in these scalings is currently limited to just a single SGD
step, unlike our work which studies the full learning trajectory of
MF networks and proves the presence of meaningful learning via global
convergence. As said, all these scalings are proxies of practical
finite-width networks. Furthermore it is argued in \cite{mei2019mean}
that for two-layer infinite-width networks that are close to practical
networks, the behavior near initialization is more NTK-like, while
that in the long-time horizon is more MF-like. We expect a similar
situation for the multilayer case, in which case it is insufficient
to understand neural networks by analyzing only a few initial SGD
steps. Our work also demonstrates the goodness of well-designed non-i.i.d.
initializations, which thus far have been under-explored in the literature.

In a later follow-up work \cite{pham2021limiting}, our neuronal embedding
framework is extended to study a finite-width correction to the infinite-width
MF limit and the implicit bias of gradient descent training in this
finite-width regime, hence paving the path to address the aforementioned
limitation of the infinite-width viewpoint.

\section*{Acknowledgment}

The work of P.-M. Nguyen was partially supported by grants NSF IIS-1741162
and ONR N00014-18-1-2729. H. T. Pham would like to thank Jan Vondrak
for many helpful discussions and in particular for the shorter proof
of Lemma \ref{lem:square hoeffding}. We would like to thank Andrea
Montanari for the succinct description of the difficulty in extending
the mean field formulation to the multilayer case, in that there are
multiple symmetry group actions in a multilayer network.

\appendix

\section{Useful tools}

We state a martingale concentration result, which is a special case
of \cite[Theorem 3.2]{pinelis1994optimum} which applies to a more
general Banach space.
\begin{thm}[Concentration of martingales in separable Hilbert spaces.]
\label{thm:azuma-hilbert}Consider a martingale $Z_{n}\in\mathbb{Z}$
a separable Hilbert space, adapted to ${\cal F}_{n}$, such that $\left|Z_{n}-Z_{n-1}\right|\leq R$
and $Z_{0}=0$. Then for any $t>0$,
\[
\mathbb{P}\left(\max_{k\leq n}\left|Z_{k}\right|\geq t\right)\leq2\inf_{\lambda>0}\exp\left(-\lambda t+{\rm ess\text{-}sup}\sum_{k=1}^{n}\mathbb{E}\left[e^{\lambda|Z_{k}-Z_{k-1}|}-1-\lambda|Z_{k}-Z_{k-1}|\middle|{\cal F}_{k-1}\right]\right).
\]
In particular, for $t<nR$,
\[
\mathbb{P}\left(\max_{k\leq n}\left|Z_{k}\right|\geq t\right)\leq2\exp\left(-\frac{t^{2}}{16nR^{2}}\right).
\]
\end{thm}

\begin{proof}
The first part follows from \cite[Theorem 3.2]{pinelis1994optimum}.
The second part follows from the observation that for $\lambda<1/(2R)$,
\[
\mathbb{E}\left[e^{\lambda|Z_{k}-Z_{k-1}|}-1-\lambda|Z_{k}-Z_{k-1}|\middle|{\cal F}_{k-1}\right]<4\lambda^{2}R^{2},
\]
and as such we have for $t<nR$,
\[
\mathbb{P}\left(\max_{k\leq n}\left|Z_{k}\right|\geq t\right)\le2\inf_{0<\lambda<1/(2R)}\exp(-\lambda t+4n\lambda^{2}R^{2})\le2\exp\left(-\frac{t^{2}}{16nR^{2}}\right).
\]
\end{proof}
Next we state two results for $\eta$-independent random variables
in separable Hilbert spaces.
\begin{thm}[Concentration of $\eta$-independent bounded sum in separable Hilbert
spaces.]
\label{thm:iid-hilbert-concentration}Consider $n$ $\eta$-independent
random variables $X_{1},...,X_{n}$ in a separable Hilbert space,
where $\eta\in[0,1]$. Suppose that $\left|X_{i}-\mathbb{E}\left[X_{i}\right]\right|\leq R$
almost surely. Then for $\delta>2\eta R$, we have 
\[
\mathbb{P}\left(\frac{1}{n}\left|\sum_{i=1}^{n}X_{i}-\mathbb{E}\left[X_{i}\right]\right|\geq\delta\right)\leq2\exp\left(-\frac{n\delta^{2}}{64R^{2}}\right).
\]
\end{thm}

\begin{proof}
Since $|X_{i}-\mathbb{E}[X_{i}]|\le R$, the claims are immediate
for $\delta\geq R$. Let 
\[
Z_{i}=X_{1}-\mathbb{E}[X_{1}]+X_{2}-\mathbb{E}[X_{2}|X_{1}]+\dots+X_{i}-\mathbb{E}[X_{i}|X_{1},\dots,X_{i-1}].
\]
Then $Z_{i}$ is a martingale adapted to ${\cal F}_{i}=\sigma(X_{1},\dots,X_{i-1})$.
By Theorem \ref{thm:azuma-hilbert},
\[
\mathbb{P}\left(|Z_{n}|\ge t\right)\le2\exp\left(-\frac{t^{2}}{16nR^{2}}\right),
\]
assuming that $t<nR$. Using the $\eta$-independence property, we
have that for $\delta\in\left(2\eta R,\left(1+\eta\right)R\right)$,
\begin{align*}
\mathbb{P}\left(\left|\sum_{i=1}^{n}X_{i}-\mathbb{E}\left[X_{i}\right]\right|\geq\delta n\right) & \leq\mathbb{P}\left(|Z_{n}|\ge\delta n-n\eta R\right)\\
 & \leq2\exp\left(-\frac{(\delta-\eta R)^{2}n}{16R^{2}}\right)\\
 & \le2\exp\left(-\delta^{2}n/\left(64R^{2}\right)\right).
\end{align*}
\end{proof}
\begin{thm}[Moments of $\eta$-independent heavy-tailed sum in separable Hilbert
spaces.]
\label{thm:iid-hilbert-higher-moment}Consider $\left(X_{1},X_{2},...\right)$
being $\eta$-independent random variables in a separable Hilbert
space. Suppose that for some constant $k\geq1$ (with $k\leq K$),
for any $i\in\mathbb{N}_{>0}$,
\begin{align*}
\sup_{m\geq1}m^{-k/2}\mathbb{E}\left[\left|X_{i}\right|^{m}\right]^{1/m} & \leq K.
\end{align*}
Then for $m\geq1$,
\[
\mathbb{E}\left[\left|\frac{1}{n}\sideset{}{_{i=1}^{n}}\sum X_{i}-\mathbb{E}\left[X_{i}\right]\right|^{m}\right]^{1/m}\leq Km^{1+k/2}\max\left(\eta n^{0.01},n^{-1/2}\right).
\]
\end{thm}

\begin{proof}
It is easy to see that it suffices to prove the claim for $m\geq2$.
Let us define ${\cal F}_{i}=\sigma(X_{1},\dots,X_{i-1})$ and
\[
Z_{i}=X_{1}-\mathbb{E}[X_{1}]+X_{2}-\mathbb{E}[X_{2}|{\cal F}_{2}]+\dots+X_{i}-\mathbb{E}[X_{i}|{\cal F}_{i}].
\]
Then $Z_{i}$ is a martingale adapted to ${\cal F}_{i}$. Note that
for any $m\geq1$ and $B>0$:
\begin{align*}
\mathbb{E}\left[\left|\sideset{}{_{i=1}^{n}}\sum X_{i}-\mathbb{E}\left[X_{i}\right]\right|^{m}\right]^{1/m} & \leq\mathbb{E}\left[\left|Z_{n}\right|^{m}\right]^{1/m}+\mathbb{E}\left[\left|\sideset{}{_{i=1}^{n}}\sum\mathbb{E}\left[X_{i}\middle|{\cal F}_{i}\right]-\mathbb{E}\left[X_{i}\right]\right|^{m}\right]^{1/m}\\
 & \leq\mathbb{E}\left[\left|Z_{n}\right|^{m}\right]^{1/m}+\mathbb{E}\left[\left|\sideset{}{_{i=1}^{n}}\sum\mathbb{E}\left[X_{i}\mathbb{I}_{\left|X_{i}\right|\geq B}\middle|{\cal F}_{i}\right]-\mathbb{E}\left[X_{i}\mathbb{I}_{\left|X_{i}\right|\geq B}\right]\right|^{m}\right]^{1/m}+n\eta B\\
 & \leq\mathbb{E}\left[\left|Z_{n}\right|^{m}\right]^{1/m}+n^{1-1/m}\bigg(\sideset{}{_{i=1}^{n}}\sum\mathbb{E}\left[\left|X_{i}\right|^{m}\mathbb{I}_{\left|X_{i}\right|\geq B}\right]\bigg)^{1/m}+n\eta B.
\end{align*}
By \cite[Theorem 4.1]{pinelis1994optimum}, for $m\geq2$,
\begin{align*}
\mathbb{E}\left[\left|Z_{n}\right|^{m}\right]^{1/m} & \leq Km\left(\sideset{}{_{i=1}^{n}}\sum\mathbb{E}\left[\left|X_{i}\right|^{m}\right]\right)^{1/m}+K\sqrt{m}\mathbb{E}\left[\bigg(\sideset{}{_{i=1}^{n}}\sum\mathbb{E}\left[\left|X_{i}-\mathbb{E}\left[X_{i}\middle|{\cal F}_{i}\right]\right|^{2}\middle|{\cal F}_{i}\right]\bigg)^{m/2}\right]^{1/m}\\
 & \leq Km\left(\sideset{}{_{i=1}^{n}}\sum\mathbb{E}\left[\left|X_{i}\right|^{m}\right]\right)^{1/m}+K\sqrt{m}\mathbb{E}\left[\bigg(\sideset{}{_{i=1}^{n}}\sum\mathbb{E}\left[\left|X_{i}\right|^{2}\middle|{\cal F}_{i}\right]\bigg)^{m/2}\right]^{1/m}\\
 & \leq K\left(m+\sqrt{m}n^{1/2-1/m}\right)\left(\sideset{}{_{i=1}^{n}}\sum\mathbb{E}\left[\left|X_{i}\right|^{m}\right]\right)^{1/m}\\
 & \leq K\left(mn^{1/m}+\sqrt{m}n^{1/2}\right)m^{k/2}\\
 & \leq Km^{1+k/2}n^{1/2}.
\end{align*}
We also have:
\[
\mathbb{E}\left[\left|X_{i}\right|^{m}\mathbb{I}_{\left|X_{i}\right|\geq B}\right]\leq\mathbb{E}\left[\left|X_{i}\right|^{2m}\right]^{1/2}\mathbb{P}\left(\left|X_{i}\right|\geq B\right)^{1/2}\leq K^{m}m^{mk/2}\exp\left(-KB^{2/k}\right),
\]
since $\left|X_{i}\right|^{1/k}$ is $K$-sub-Gaussian. Therefore,
\[
\mathbb{E}\left[\left|\sideset{}{_{i=1}^{n}}\sum X_{i}-\mathbb{E}\left[X_{i}\right]\right|^{m}\right]^{1/m}\leq Km^{1+k/2}n^{1/2}+Km^{k/2}\exp\left(-KB^{2/k}/m\right)n+n\eta B.
\]
The claim is satisfied for $B^{2/k}=cm\log n$ with a suitable constant
$c$.
\end{proof}

\section{Remaining proofs for Section \ref{sec:Existence-MF}\label{sec:Remaining-proofs-existence-MF}}

\subsection{Proof of Lemma \ref{lem:bounds MF a priori}}
\begin{proof}[Proof of Lemma \ref{lem:bounds MF a priori}]
Let $\eta\left(i\right)=2^{L-i}$ and $\bar{\eta}\left(i\right)=2^{i-1}$.
Let us define
\begin{align*}
\left\llbracket w_{i}\right\rrbracket _{m,t} & =\sqrt{\frac{50}{m}}\mathbb{E}\left[\sup_{s\leq t}\left|w_{i}\left(s,C_{i-1},C_{i}\right)\right|^{m}\right]^{1/m},\qquad i=2,...,L.
\end{align*}
We prove the following by backward induction, for $i=2,...,L$ and
any $m\geq50$:
\begin{align*}
 & \sqrt{\frac{50}{m}}\mathbb{E}\left[\sup_{s\leq t}\underset{Z\sim{\cal P}}{{\rm ess\text{-}sup}}\left|\Delta_{i}^{H}\left(s,Z,C_{i}\right)\right|^{m}\right]^{1/m}\leq K^{3\eta\left(i-1\right)-2}\left(1+t^{\eta\left(i\right)-1}\right)\prod_{j=i+1}^{L}\left(1+\left\llbracket w_{j}\right\rrbracket _{m,0}^{\bar{\eta}\left(j-i\right)}+\interleave b_{j}\interleave_{0}^{\bar{\eta}\left(j-i\right)}\right),\\
 & \left\llbracket w_{i}\right\rrbracket _{m,t}\leq K^{3\eta\left(i-1\right)-1}\left(1+\left\llbracket w_{i}\right\rrbracket _{m,0}\right)\left(1+t^{\eta\left(i\right)}\right)\prod_{j=i+1}^{L}\left(1+\left\llbracket w_{j}\right\rrbracket _{m,0}^{\bar{\eta}\left(j-i\right)}+\interleave b_{j}\interleave_{0}^{\bar{\eta}\left(j-i\right)}\right),\\
 & \interleave b_{i}\interleave_{t}\leq K^{3\eta\left(i-1\right)-1}\left(1+\interleave b_{i}\interleave_{0}\right)\left(1+t^{\eta\left(i\right)}\right)\prod_{j=i+1}^{L}\left(1+\left\llbracket w_{j}\right\rrbracket _{m,0}^{\bar{\eta}\left(j-i\right)}+\interleave b_{j}\interleave_{0}^{\bar{\eta}\left(j-i\right)}\right),
\end{align*}
for some immaterial constant $K\geq1$, where by standard convention
$\prod_{j=i+1}^{L}=1$ if $i=L$.

Let us start with $i=L$. By Assumption \ref{enu:Assump_backward},
for ${\cal P}$-almost every $z$,
\[
\sup_{t\geq0}\left|\Delta_{L}^{H}\left(t,z,1\right)\right|\leq K.
\]
Consequently, for ${\cal P}$-almost every $z$,
\[
\max\left(\sup_{t\geq0}\sup_{c_{L-1}\in\Omega_{L-1}}\left|\Delta_{L}^{w}\left(t,z,c_{L-1},1\right)\right|,\;\sup_{t\geq0}\left|\Delta_{L}^{b}\left(t,z,1\right)\right|\right)\leq K\left(1+\sup_{t\geq0}\left|\Delta_{L}^{H}\left(t,z,1\right)\right|\right)\leq K^{2}.
\]
Together with Assumption \ref{enu:Assump_lrSchedule} and the fact
$w_{L}$ and $b_{L}$ satisfy the MF ODEs, this implies:
\[
\left\llbracket w_{L}\right\rrbracket _{m,t}\leq\left\llbracket w_{L}\right\rrbracket _{m,0}+K^{2}t,\qquad\interleave b_{L}\interleave_{t}\leq\interleave b_{L}\interleave_{0}+K^{2}t.
\]
These prove the statement for $i=L$.

Next, assuming the statement for $i+1$, we prove the statement for
$i$, where $1<i<L$. Using Cauchy-Schwarz's inequality, we have from
Assumption \ref{enu:Assump_backward}, for $m\geq50$,
\begin{align*}
 & \mathbb{E}\left[\sup_{s\leq t}\underset{Z\sim{\cal P}}{{\rm ess\text{-}sup}}\left|\Delta_{i}^{H}\left(s,Z,C_{i}\right)\right|^{m}\right]\\
 & \leq K^{m}\mathbb{E}\left[\sup_{s\leq t}\underset{Z\sim{\cal P}}{{\rm ess\text{-}sup}}\left|\mathbb{E}_{C_{i+1}}\left[\left(1+\left|\Delta_{i+1}^{H}\left(s,Z,C_{i+1}\right)\right|\right)\left(1+\left|w_{i+1}\left(s,C_{i},C_{i+1}\right)\right|+\left|b_{i+1}\left(s,C_{i+1}\right)\right|\right)\right]\right|^{m}\right]\\
 & \leq K^{m}\mathbb{E}\bigg[\mathbb{E}_{C_{i+1}}\left[1+\sup_{s\leq t}\underset{Z\sim{\cal P}}{{\rm ess\text{-}sup}}\left|\Delta_{i+1}^{H}\left(s,Z,C_{i+1}\right)\right|^{2}\right]^{m/2}\\
 & \quad\qquad\times\mathbb{E}_{C_{i+1}}\left[1+\sup_{s\leq t}\left|w_{i+1}\left(s,C_{i},C_{i+1}\right)\right|^{2}+\sup_{s\leq t}\left|b_{i+1}\left(s,C_{i+1}\right)\right|^{2}\right]^{m/2}\bigg]\\
 & \leq K^{m}\left(1+\mathbb{E}\left[\sup_{s\leq t}\underset{Z\sim{\cal P}}{{\rm ess\text{-}sup}}\left|\Delta_{i+1}^{H}\left(s,Z,C_{i+1}\right)\right|^{2}\right]^{m/2}\right)\\
 & \quad\qquad\times\mathbb{E}\left[1+\mathbb{E}_{C_{i+1}}\left[\sup_{s\leq t}\left|w_{i+1}\left(s,C_{i},C_{i+1}\right)\right|^{m}\right]+\mathbb{E}\left[\sup_{s\leq t}\left|b_{i+1}\left(s,C_{i+1}\right)\right|^{2}\right]^{m/2}\right]\\
 & \leq K^{m}\left(1+\mathbb{E}\left[\sup_{s\leq t}\underset{Z\sim{\cal P}}{{\rm ess\text{-}sup}}\left|\Delta_{i+1}^{H}\left(s,Z,C_{i+1}\right)\right|^{2}\right]^{m/2}\right)\left(1+m^{m/2}\left\llbracket w_{i+1}\right\rrbracket _{m,t}^{m}+\interleave b_{i+1}\interleave_{t}^{m}\right),
\end{align*}
which implies, by the induction hypothesis,
\begin{align*}
 & \sqrt{\frac{50}{m}}\mathbb{E}\left[\sup_{s\leq t}\underset{Z\sim{\cal P}}{{\rm ess\text{-}sup}}\left|\Delta_{i}^{H}\left(s,Z,C_{i}\right)\right|^{m}\right]^{1/m}\\
 & \leq K\left(1+\mathbb{E}\left[\sup_{s\leq t}\underset{Z\sim{\cal P}}{{\rm ess\text{-}sup}}\left|\Delta_{i+1}^{H}\left(s,Z,C_{i+1}\right)\right|^{2}\right]^{1/2}\right)\left(1+\left\llbracket w_{i+1}\right\rrbracket _{m,t}+\interleave b_{i+1}\interleave_{t}\right)\\
 & \leq K\left[1+K^{3\eta\left(i\right)-2}\left(1+t^{\eta\left(i+1\right)-1}\right)\prod_{j=i+2}^{L}\left(1+\left\llbracket w_{j}\right\rrbracket _{m,0}^{\bar{\eta}\left(j-i-1\right)}+\interleave b_{j}\interleave_{0}^{\bar{\eta}\left(j-i-1\right)}\right)\right]\\
 & \quad\times\left[1+K^{3\eta\left(i\right)-1}\left(1+\left\llbracket w_{i+1}\right\rrbracket _{m,0}+\interleave b_{i+1}\interleave_{0}\right)\left(1+t^{\eta\left(i+1\right)}\right)\prod_{j=i+2}^{L}\left(1+\left\llbracket w_{j}\right\rrbracket _{m,0}^{\bar{\eta}\left(j-i-1\right)}+\interleave b_{j}\interleave_{0}^{\bar{\eta}\left(j-i-1\right)}\right)\right]\\
 & \leq K^{3\eta\left(i-1\right)-2}\left(1+t^{\eta\left(i\right)-1}\right)\prod_{j=i+1}^{L}\left(1+\left\llbracket w_{j}\right\rrbracket _{m,0}^{\bar{\eta}\left(j-i\right)}+\interleave b_{j}\interleave_{0}^{\bar{\eta}\left(j-i\right)}\right).
\end{align*}
Therefore, by Assumptions \ref{enu:Assump_backward} and \ref{enu:Assump_lrSchedule},
with the fact that $w_{i}$ satisfies the MF ODEs:
\begin{align*}
\left\llbracket w_{i}\right\rrbracket _{m,t} & =\sqrt{\frac{50}{m}}\mathbb{E}\left[\sup_{s\leq t}\left|w_{i}\left(0,C_{i-1},C_{i}\right)-\int_{0}^{s}\xi_{i}^{\mathbf{w}}\left(s'\right)\mathbb{E}_{Z}\left[\Delta_{i}^{w}\left(s',Z,C_{i-1},C_{i}\right)\right]ds'\right|^{m}\right]^{1/m}\\
 & \leq\left\llbracket w_{i}\right\rrbracket _{m,0}+\frac{K}{\sqrt{m}}\mathbb{E}\left[\sup_{s\leq t}\underset{Z\sim{\cal P}}{{\rm ess\text{-}sup}}\left|\Delta_{i}^{w}\left(s,Z,C_{i-1},C_{i}\right)\right|^{m}\right]^{1/m}t\\
 & \leq\left\llbracket w_{i}\right\rrbracket _{m,0}+K\left(1+\frac{1}{\sqrt{m}}\mathbb{E}\left[\sup_{s\leq t}\underset{Z\sim{\cal P}}{{\rm ess\text{-}sup}}\left|\Delta_{i}^{H}\left(s,Z,C_{i}\right)\right|^{m}\right]^{1/m}\right)t\\
 & \leq K^{3\eta\left(i-1\right)-1}\left(1+\left\llbracket w_{i}\right\rrbracket _{m,0}\right)\left(1+t^{\eta\left(i\right)}\right)\prod_{j=i+1}^{L}\left(1+\left\llbracket w_{j}\right\rrbracket _{m,0}^{\bar{\eta}\left(j-i\right)}+\interleave b_{j}\interleave_{0}^{\bar{\eta}\left(j-i\right)}\right).
\end{align*}
We obtain a similar bound for $\interleave b_{i}\interleave_{t}$.
This completes the backward induction. With the same argument, one
can obtain a similar bound for $i=1$:
\begin{align*}
 & \sqrt{\frac{50}{m}}\mathbb{E}\left[\sup_{s\leq t}\underset{Z\sim{\cal P}}{{\rm ess\text{-}sup}}\left|\Delta_{1}^{H}\left(s,Z,C_{1}\right)\right|^{m}\right]^{1/m}\leq K^{3\eta\left(0\right)-2}\left(1+t^{\eta\left(1\right)-1}\right)\prod_{j=2}^{L}\left(1+\left\llbracket w_{j}\right\rrbracket _{m,0}^{\bar{\eta}\left(j-1\right)}+\interleave b_{j}\interleave_{0}^{\bar{\eta}\left(j-1\right)}\right),\\
 & \interleave w_{1}\interleave_{t}\leq K^{3\eta\left(0\right)-1}\left(1+\interleave w_{1}\interleave_{0}\right)\left(1+t^{\eta\left(1\right)}\right)\prod_{j=2}^{L}\left(1+\left\llbracket w_{j}\right\rrbracket _{m,0}^{\bar{\eta}\left(j-1\right)}+\interleave b_{j}\interleave_{0}^{\bar{\eta}\left(j-1\right)}\right).
\end{align*}

By taking the supremum on $m$ or setting $m=50$, these bounds imply
the claimed bound on $\left\llbracket W\right\rrbracket _{\psi,t}$
and $\interleave W\interleave_{t}$. In addition, the bounds on $\left\llbracket w_{i}\right\rrbracket _{m,t}$
show that $\sup_{s\leq t}\left|w_{i}\left(s,C_{i-1},C_{i}\right)\right|$
is $K_{0}\left(t\right)$-sub-Gaussian for $2\leq i\leq L$. Together
with the union bound, we then get the claimed probability bound.
\end{proof}

\subsection{Proof of Lemma \ref{lem:difference MF}}

We state the following two useful auxiliary lemmas:
\begin{lem}
\label{lem:Lipschitz forward MF}Consider two collections of MF parameters
$W',W''\in{\cal W}_{T}$. Under Assumption \ref{enu:Assump_forward},
for any $t\leq T$ and $1\leq i\leq L$, the following hold: 
\begin{align*}
\mathbb{E}\left[\sup_{s\leq t}\underset{X\sim{\cal P}}{{\rm ess\text{-}sup}}\left|H_{i}\left(X,C_{i};W'\left(s\right)\right)-H_{i}\left(X,C_{i};W''\left(s\right)\right)\right|^{2}\right]^{1/2} & \leq K^{L}K_{0}^{L}\left(T\right)\left\Vert W'-W''\right\Vert _{t},\\
\sup_{s\leq t}\underset{X\sim{\cal P}}{{\rm ess\text{-}sup}}\left|\hat{y}\left(X;W'\left(s\right)\right)-\hat{y}\left(X;W''\left(s\right)\right)\right| & \leq K^{L}K_{0}^{L}\left(T\right)\left\Vert W'-W''\right\Vert _{t}.
\end{align*}
\end{lem}

\begin{lem}
\label{lem:Lipschitz backward MF}For a given $B\geq0$, consider
two collections of MF parameters $W',W''\in{\cal W}_{T}$ such that
\begin{align*}
\mathbb{P}\left(\mathsf{max}_{T}^{w}\left(W'\right)\geq K_{0}\left(T\right)B\right) & \leq2Le^{1-K_{1}B^{2}},\\
\mathbb{P}\left(\mathsf{max}_{T}^{w}\left(W''\right)\geq K_{0}\left(T\right)B\right) & \leq2Le^{1-K_{1}B^{2}}.
\end{align*}
Under Assumptions \ref{enu:Assump_forward} and \ref{enu:Assump_backward},
for any $t\leq T$ and $2\leq i\leq L$, the following hold:
\begin{align*}
\mathbb{E}\left[\sup_{s\leq t}\underset{Z\sim{\cal P}}{{\rm ess\text{-}sup}}\left|\Delta_{i}^{w}\left(Z,C_{i-1},C_{i};W'\left(s\right)\right)-\Delta_{i}^{w}\left(Z,C_{i-1},C_{i};W''\left(s\right)\right)\right|^{2}\right]^{1/2} & \leq D\left(t,W',W''\right),\\
\mathbb{E}\left[\sup_{s\leq t}\underset{Z\sim{\cal P}}{{\rm ess\text{-}sup}}\left|\Delta_{i}^{b}\left(Z,C_{i};W'\left(s\right)\right)-\Delta_{i}^{b}\left(Z,C_{i};W''\left(s\right)\right)\right|^{2}\right]^{1/2} & \leq D\left(t,W',W''\right),\\
\mathbb{E}\left[\sup_{s\leq t}\underset{Z\sim{\cal P}}{{\rm ess\text{-}sup}}\left|\Delta_{1}^{w}\left(Z,C_{1};W'\left(s\right)\right)-\Delta_{1}^{w}\left(Z,C_{1};W''\left(s\right)\right)\right|^{2}\right]^{1/2} & \leq D\left(t,W',W''\right),
\end{align*}
in which $D\left(t,W',W''\right)\stackrel{{\rm def}}{=}\left(KK_{0}\left(T\right)\right)^{2L+2}\left(\left(1+B\right)\left\Vert W'-W''\right\Vert _{t}+\sqrt{L}e^{-K_{1}B^{2}/2}\right).$
\end{lem}

These lemmas lay the foundation to prove Lemma \ref{lem:difference MF}.
\begin{proof}[Proof of Lemma \ref{lem:difference MF}]
Let us recall the quantity $D\left(t,W',W''\right)$ from Lemma \ref{lem:Lipschitz backward MF}.
Let us note a simple identity:
\begin{align*}
\mathbb{E}_{C}\left[\left(\int_{0}^{t}f\left(s,C\right)ds\right)^{2}\right] & =\int_{0}^{t}\int_{0}^{t}\mathbb{E}_{C}\left[f\left(s_{1},C\right)f\left(s_{2},C\right)\right]ds_{1}ds_{2}\\
 & \leq\int_{0}^{t}\int_{0}^{t}\mathbb{E}_{C}\left[\left|f\left(s_{1},C\right)\right|^{2}\right]^{1/2}\mathbb{E}_{C}\left[\left|f\left(s_{2},C\right)\right|^{2}\right]^{1/2}ds_{1}ds_{2}\\
 & =\left(\int_{0}^{t}\mathbb{E}_{C}\left[\left|f\left(s,C\right)\right|^{2}\right]^{1/2}ds\right)^{2}.
\end{align*}
Now for any $i\geq2$:

\begin{align*}
 & \mathbb{E}\left[\left(\int_{0}^{t}\left|\frac{\partial}{\partial t}F_{i}^{w}\left(W'\right)\left(s,C_{i-1},C_{i}\right)-\frac{\partial}{\partial t}F_{i}^{w}\left(W''\right)\left(s,C_{i-1},C_{i}\right)\right|ds\right)^{2}\right]^{1/2}\\
 & =\mathbb{E}\left[\left(\int_{0}^{t}\left|\xi_{i}^{\mathbf{w}}\left(s\right)\mathbb{E}_{Z}\left[\Delta_{i}^{w}\left(Z,C_{i-1},C_{i};W'\left(s\right)\right)-\Delta_{i}^{w}\left(Z,C_{i-1},C_{i};W''\left(s\right)\right)\right]\right|ds\right)^{2}\right]^{1/2}\\
 & \stackrel{\left(a\right)}{\leq}K\mathbb{E}\left[\left(\int_{0}^{t}\left|\mathbb{E}_{Z}\left[\Delta_{i}^{w}\left(Z,C_{i-1},C_{i};W'\left(s\right)\right)-\Delta_{i}^{w}\left(Z,C_{i-1},C_{i};W''\left(s\right)\right)\right]\right|ds\right)^{2}\right]^{1/2}\\
 & \stackrel{\left(b\right)}{\leq}K\int_{0}^{t}\mathbb{E}\left[\left|\mathbb{E}_{Z}\left[\Delta_{i}^{w}\left(Z,C_{i-1},C_{i};W'\left(s\right)\right)-\Delta_{i}^{w}\left(Z,C_{i-1},C_{i};W''\left(s\right)\right)\right]\right|^{2}\right]^{1/2}ds\\
 & \stackrel{\left(c\right)}{\leq}K\int_{0}^{t}D\left(s,W',W''\right)ds,
\end{align*}
where $\left(a\right)$ is due to Assumption \ref{enu:Assump_lrSchedule},
$\left(b\right)$ is by the aforementioned identity, and $\left(c\right)$
is an application of Lemma \ref{lem:Lipschitz backward MF}. Therefore,
\begin{align*}
 & \mathbb{E}\left[\sup_{s\leq t}\left|F_{i}^{w}\left(W'\right)\left(s,C_{i-1},C_{i}\right)-F_{i}^{w}\left(W''\right)\left(s,C_{i-1},C_{i}\right)\right|^{2}\right]^{1/2}\\
 & \leq\mathbb{E}\left[\sup_{s\leq t}\left(\int_{0}^{s}\left|\frac{\partial}{\partial t}F_{i}^{w}\left(W'\right)\left(s',C_{i-1},C_{i}\right)-\frac{\partial}{\partial t}F_{i}^{w}\left(W''\right)\left(s',C_{i-1},C_{i}\right)\right|ds'\right)^{2}\right]^{1/2}\\
 & \leq K\int_{0}^{t}D\left(s,W',W''\right)ds.
\end{align*}
One can show the same bound for $F_{i}^{b}$ and $F_{1}^{w}$. This
completes the proof.
\end{proof}
Lemmas \ref{lem:Lipschitz forward MF} and \ref{lem:Lipschitz backward MF}
are in fact special cases of the following lemmas.
\begin{lem}
\label{lem:Lipschitz forward MF - general}Consider two collections
of MF parameters $W',W''\in{\cal W}_{T}$. Suppose that we define
$\tilde{C}_{1},...,\tilde{C}_{L}$ independent random variables on
$\Omega_{1},...,\Omega_{L}$, such that $\tilde{C}_{i}$ is independent
of $C_{1},...C_{i-1},C_{i+1},\dots,C_{L}$, and that there exists
some $K_{*}\left(T\right)\geq K_{0}\left(T\right)$ such that all
following quantities are upper-bounded by $K_{*}\left(T\right)$ for
all $t\leq T$ and for $W=W'$ or $W=W''$:
\[
\max_{2\leq i\leq L}\mathbb{E}\left[\sup_{s\leq t}\left|w_{i}\left(s,\tilde{C}_{i-1},\tilde{C}_{i}\right)\right|^{50}\right]^{1/50},\;\max_{2\leq i\leq L}\mathbb{E}\left[\sup_{s\leq t}\left|w_{i}\left(s,C_{i-1},\tilde{C}_{i}\right)\right|^{50}\right]^{1/50},\;\max_{2\leq i\leq L}\mathbb{E}\left[\sup_{s\leq t}\left|w_{i}\left(s,\tilde{C}_{i-1},C_{i}\right)\right|^{50}\right]^{1/50},
\]
\[
\max_{2\leq i\leq L}\mathbb{E}\left[\sup_{s\leq t}\left|w_{i}\left(s,C_{i-1},C_{i}\right)\right|^{50}\right]^{1/50},\;\max_{2\leq i\leq L}\mathbb{E}\left[\sup_{s\leq t}\left|b_{i}\left(s,\tilde{C}_{i}\right)\right|^{50}\right]^{1/50},\;\max_{2\leq i\leq L}\mathbb{E}\left[\sup_{s\leq t}\left|b_{i}\left(s,C_{i}\right)\right|^{50}\right]^{1/50},
\]
\[
\mathbb{E}\left[\sup_{s\leq t}\left|w_{1}\left(s,\tilde{C}_{1}\right)\right|^{50}\right]^{1/50},\;\mathbb{E}\left[\sup_{s\leq t}\left|w_{1}\left(s,C_{1}\right)\right|^{50}\right]^{1/50}.
\]
Under Assumption \ref{enu:Assump_forward}, for any $t\leq T$ and
$1\leq i\leq L$, we have:
\[
\mathbb{E}\left[\sup_{s\leq t}\underset{X\sim{\cal P}}{{\rm ess\text{-}sup}}\left|H_{i}\left(X,\tilde{C}_{i};W'\left(s\right)\right)-H_{i}\left(X,\tilde{C}_{i};W''\left(s\right)\right)\right|^{2}\right]^{1/2}\leq K^{L}K_{*}^{L}\left(T\right)\tilde{d}_{t}\left(W',W''\right),
\]
and the same holds if we replace $\tilde{C}_{i}$ with $C_{i}$ in
the left-hand side of the above. Here we have defined the metrics:
\begin{align*}
\tilde{d}_{t}\left(W',W''\right) & =\max\left(\max_{2\leq i\leq L}\tilde{d}_{t}\left(w_{i}',w_{i}''\right),\;\max_{2\leq i\leq L}\tilde{d}_{t}\left(b_{i}',b_{i}''\right),\;\tilde{d}_{t}\left(w_{1}',w_{1}''\right)\right),\\
\tilde{d}_{t}\left(w_{i}',w_{i}''\right) & =\max\bigg(\mathbb{E}\left[\sup_{s\leq t}\left|w_{i}'\left(s,\tilde{C}_{i-1},\tilde{C}_{i}\right)-w_{i}''\left(s,\tilde{C}_{i-1},\tilde{C}_{i}\right)\right|^{2}\right]^{1/2},\\
 & \qquad\qquad\mathbb{E}\left[\sup_{s\leq t}\left|w_{i}'\left(s,C_{i-1},\tilde{C}_{i}\right)-w_{i}''\left(s,C_{i-1},\tilde{C}_{i}\right)\right|^{2}\right]^{1/2},\\
 & \qquad\qquad\mathbb{E}\left[\sup_{s\leq t}\left|w_{i}'\left(s,\tilde{C}_{i-1},C_{i}\right)-w_{i}''\left(s,\tilde{C}_{i-1},C_{i}\right)\right|^{2}\right]^{1/2},\\
 & \qquad\qquad\mathbb{E}\left[\sup_{s\leq t}\left|w_{i}'\left(s,C_{i-1},C_{i}\right)-w_{i}''\left(s,C_{i-1},C_{i}\right)\right|^{2}\right]^{1/2}\bigg),\\
\tilde{d}_{t}\left(b_{i}',b_{i}''\right) & =\max\bigg(\mathbb{E}\left[\sup_{s\leq t}\left|b_{i}'\left(s,\tilde{C}_{i}\right)-b_{i}''\left(s,\tilde{C}_{i}\right)\right|^{2}\right]^{1/2},\;\mathbb{E}\left[\sup_{s\leq t}\left|b_{i}'\left(s,C_{i}\right)-b_{i}''\left(s,C_{i}\right)\right|^{2}\right]^{1/2}\bigg)\\
\tilde{d}_{t}\left(w_{1}',w_{1}''\right) & =\max\bigg(\mathbb{E}\left[\sup_{s\leq t}\left|w_{1}'\left(s,\tilde{C}_{1}\right)-w_{1}''\left(s,\tilde{C}_{1}\right)\right|^{2}\right]^{1/2},\;\mathbb{E}\left[\sup_{s\leq t}\left|w_{1}'\left(s,C_{1}\right)-w_{1}''\left(s,C_{1}\right)\right|^{2}\right]^{1/2}\bigg).
\end{align*}
(Note that the random variables $\tilde{C}_{i}$ are general, and
may be chosen to be equal to $C_{i}$. The space ${\cal W}_{T}$ which
contains $W'$ and $W''$ is defined with respect to the random variables
$C_{1},\dots,C_{L}$.)
\end{lem}

\begin{lem}
\label{lem:Lipschitz backward MF - general}Consider two collections
of MF parameters $W',W''\in{\cal W}_{T}$. Suppose we define the random
variables $\tilde{C}_{1},...,\tilde{C}_{L}$, the bounding constant
$K_{*}\left(T\right)$ and the metric $\tilde{d}_{t}\left(W',W''\right)$
as given in the statement of Lemma \ref{lem:Lipschitz forward MF - general}.
Further assume that for some non-negative function $\Xi$ and some
$B\geq0$,
\begin{align*}
\mathbb{P}\left(\widetilde{\mathsf{max}}_{T}^{w}\left(W'\right)\geq K_{*}\left(T\right)B\right) & \leq\Xi\left(B\right),\\
\mathbb{P}\left(\widetilde{\mathsf{max}}_{T}^{w}\left(W''\right)\geq K_{*}\left(T\right)B\right) & \leq\Xi\left(B\right),
\end{align*}
in which we define
\[
\widetilde{\mathsf{max}}_{t}^{w}\left(W'\right)=\max_{2\leq i\leq L}\sup_{s\leq t}\max\left(\left|w_{i}'\left(s,C_{i-1},\tilde{C}_{i}\right)\right|,\;\left|w_{i}'\left(s,\tilde{C}_{i-1},C_{i}\right)\right|,\;\left|w_{i}'\left(s,\tilde{C}_{i-1},\tilde{C}_{i}\right)\right|,\;\left|w_{i}'\left(s,C_{i-1},C_{i}\right)\right|\right).
\]
Under Assumptions \ref{enu:Assump_forward} and \ref{enu:Assump_backward},
for any $t\leq T$ and $2\leq i\leq L$, we have:
\begin{align*}
\mathbb{E}\left[\sup_{s\leq t}\underset{Z\sim{\cal P}}{{\rm ess\text{-}sup}}\left|\Delta_{i}^{w}\left(Z,\tilde{C}_{i-1},\tilde{C}_{i};W'\left(s\right)\right)-\Delta_{i}^{w}\left(Z,\tilde{C}_{i-1},\tilde{C}_{i};W''\left(s\right)\right)\right|^{2}\right]^{1/2} & \leq\tilde{D}\left(t,W',W''\right),\\
\mathbb{E}\left[\sup_{s\leq t}\underset{Z\sim{\cal P}}{{\rm ess\text{-}sup}}\left|\Delta_{i}^{b}\left(Z,\tilde{C}_{i};W'\left(s\right)\right)-\Delta_{i}^{b}\left(Z,\tilde{C}_{i};W''\left(s\right)\right)\right|^{2}\right]^{1/2} & \leq\tilde{D}\left(t,W',W''\right),\\
\mathbb{E}\left[\sup_{s\leq t}\underset{Z\sim{\cal P}}{{\rm ess\text{-}sup}}\left|\Delta_{1}^{w}\left(Z,\tilde{C}_{1};W'\left(s\right)\right)-\Delta_{1}^{w}\left(Z,\tilde{C}_{1};W''\left(s\right)\right)\right|^{2}\right]^{1/2} & \leq\tilde{D}\left(t,W',W''\right),
\end{align*}
and the same holds if we replace $\tilde{C}_{i}$ or $\tilde{C}_{i-1}$
with $C_{i}$ or $C_{i-1}$ respectively in the left-hand side of
each line above. Here 
\[
\tilde{D}\left(t,W',W''\right)\stackrel{{\rm def}}{=}\left(KK_{*}\left(T\right)\right)^{3L+2}\left(\left(1+B\right)\tilde{d}_{t}\left(W',W''\right)+\sqrt{\Xi\left(B\right)}\right).
\]
\end{lem}

Next we prove each of the remaining lemmas.

\subsection{Proof of Lemmas \ref{lem:Lipschitz forward MF} and \ref{lem:Lipschitz forward MF - general}}
\begin{proof}[Proof of Lemma \ref{lem:Lipschitz forward MF}]
The first bound is a direct corollary of Lemma \ref{lem:Lipschitz forward MF - general}
by setting $\tilde{C}_{i}=C_{i}$ for all $i\in\left[L\right]$. In
addition, by Assumption \ref{enu:Assump_forward},
\[
\sup_{s\leq t}\underset{X\sim{\cal P}}{{\rm ess\text{-}sup}}\left|\hat{y}\left(X;W'\left(s\right)\right)-\hat{y}\left(X;W''\left(s\right)\right)\right|\leq KD_{L}\left(t\right)\leq K^{L}K_{0}^{L}\left(T\right)\left\Vert W'-W''\right\Vert _{t},
\]
completing the proof.
\end{proof}
\begin{proof}[Proof of Lemma \ref{lem:Lipschitz forward MF - general}]
Let us denote
\begin{align*}
D_{i}\left(t\right) & =\mathbb{E}\left[\sup_{s\leq t}\underset{X\sim{\cal P}}{{\rm ess\text{-}sup}}\left|H_{i}\left(X,C_{i};W'\left(s\right)\right)-H_{i}\left(X,C_{i};W''\left(s\right)\right)\right|^{2}\right]^{1/2},\\
\tilde{D}_{i}\left(t\right) & =\mathbb{E}\left[\sup_{s\leq t}\underset{X\sim{\cal P}}{{\rm ess\text{-}sup}}\left|H_{i}\left(X,\tilde{C}_{i};W'\left(s\right)\right)-H_{i}\left(X,\tilde{C}_{i};W''\left(s\right)\right)\right|^{2}\right]^{1/2}.
\end{align*}
By Assumption \ref{enu:Assump_forward}, for $\tilde{D}_{1}\left(t\right)$:
\[
\tilde{D}_{1}\left(t\right)\leq K\mathbb{E}\left[\sup_{s\leq t}\left|w_{1}'\left(s,\tilde{C}_{1}\right)-w_{1}''\left(s,\tilde{C}_{1}\right)\right|^{2}\right]^{1/2}\leq K\tilde{d}_{t}\left(W',W''\right).
\]
The same bound holds for $D_{1}\left(t\right)$. Next let us consider
$\tilde{D}_{i}\left(t\right)$. By Assumption \ref{enu:Assump_forward},
using Cauchy-Schwarz's inequality, we obtain:
\begin{align*}
\tilde{D}_{i}\left(t\right) & \leq K\mathbb{E}\Bigg[\sup_{s\leq t}\mathbb{E}_{C_{i-1}}\Big[\left(1+\left|w_{i}'\left(s,C_{i-1},\tilde{C}_{i}\right)\right|+\left|w_{i}''\left(s,C_{i-1},\tilde{C}_{i}\right)\right|+\left|b_{i}'\left(s,\tilde{C}_{i}\right)\right|+\left|b_{i}''\left(s,\tilde{C}_{i}\right)\right|\right)\\
 & \quad\qquad\times\underset{X\sim{\cal P}}{{\rm ess\text{-}sup}}\left|H_{i-1}\left(X,C_{i-1};W'\left(s\right)\right)-H_{i-1}\left(X,C_{i-1};W''\left(s\right)\right)\right|\Big]^{2}\Bigg]^{1/2}\\
 & \quad+K\mathbb{E}\left[\sup_{s\leq t}\mathbb{E}_{C_{i-1}}\left[\left|w_{i}'\left(s,C_{i-1},\tilde{C}_{i}\right)-w_{i}''\left(s,C_{i-1},\tilde{C}_{i}\right)\right|+\left|b_{i}'\left(s,\tilde{C}_{i}\right)-b_{i}''\left(s,\tilde{C}_{i}\right)\right|\right]^{2}\right]^{1/2}\\
 & \leq K\mathbb{E}\left[\sup_{s\leq t}\mathbb{E}_{C_{i-1}}\left[1+\left|w_{i}'\left(s,C_{i-1},\tilde{C}_{i}\right)\right|^{2}+\left|w_{i}''\left(s,C_{i-1},\tilde{C}_{i}\right)\right|^{2}+\left|b_{i}'\left(s,\tilde{C}_{i}\right)\right|^{2}+\left|b_{i}''\left(s,\tilde{C}_{i}\right)\right|^{2}\right]D_{i-1}^{2}\left(t\right)\right]^{1/2}\\
 & \quad+K\mathbb{E}\left[\sup_{s\leq t}\mathbb{E}_{C_{i-1}}\left[\left|w_{i}'\left(s,C_{i-1},\tilde{C}_{i}\right)-w_{i}''\left(s,C_{i-1},\tilde{C}_{i}\right)\right|+\left|b_{i}'\left(s,\tilde{C}_{i}\right)-b_{i}''\left(s,\tilde{C}_{i}\right)\right|\right]^{2}\right]^{1/2}\\
 & \leq KK_{*}\left(T\right)D_{i-1}\left(t\right)+K\tilde{d}_{t}\left(W',W''\right).
\end{align*}
We have the same bound for $D_{i}\left(t\right)$. Hence,
\[
\max\left(D_{i}\left(t\right),\tilde{D}_{i}\left(t\right)\right)\leq KK_{*}\left(T\right)\max\left(D_{i-1}\left(t\right),\tilde{D}_{i-1}\left(t\right)\right)+K\tilde{d}_{t}\left(W',W''\right).
\]
This in particular implies
\[
\max_{1\leq i\leq L}\max\left(D_{i}\left(t\right),\tilde{D}_{i}\left(t\right)\right)\leq K^{L}K_{*}^{L}\left(T\right)\tilde{d}_{t}\left(W',W''\right),
\]
which proves the statement.
\end{proof}

\subsection{Proof of Lemmas \ref{lem:Lipschitz backward MF} and \ref{lem:Lipschitz backward MF - general}}
\begin{proof}[Proof of Lemma \ref{lem:Lipschitz backward MF}]
This is a special case of Lemma \ref{lem:Lipschitz backward MF - general}
with $\tilde{C}_{i}=C_{i}$ for all $i\in\left[L\right]$, $K_{*}\left(T\right)=K_{0}\left(T\right)$
and $\Xi\left(B\right)=2Le^{1-K_{1}B^{2}}$.
\end{proof}
\begin{proof}[Proof of Lemma \ref{lem:Lipschitz backward MF - general}]
First of all, by Cauchy-Schwarz's inequality, we have from Assumption
\ref{enu:Assump_backward},
\begin{align*}
 & \mathbb{E}\left[\sup_{s\leq t}\underset{Z\sim{\cal P}}{{\rm ess\text{-}sup}}\left|\Delta_{i}^{H}\left(Z,\tilde{C}_{i};W'\left(s\right)\right)\right|^{50}\right]^{1/50}\\
 & \leq K\mathbb{E}\left[\sup_{s\leq t}\underset{Z\sim{\cal P}}{{\rm ess\text{-}sup}}\left|\mathbb{E}_{C_{i+1}}\left[\left(1+\left|\Delta_{i+1}^{H}\left(Z,C_{i+1};W'\left(s\right)\right)\right|\right)\left(1+\left|w_{i+1}'\left(s,\tilde{C}_{i},C_{i+1}\right)\right|+\left|b_{i+1}'\left(s,C_{i+1}\right)\right|\right)\right]\right|^{50}\right]^{1/50}\\
 & \leq K\mathbb{E}\bigg[\mathbb{E}_{C_{i+1}}\left[1+\sup_{s\leq t}\underset{Z\sim{\cal P}}{{\rm ess\text{-}sup}}\left|\Delta_{i+1}^{H}\left(Z,C_{i+1};W'\left(s\right)\right)\right|^{2}\right]^{25}\\
 & \quad\qquad\times\mathbb{E}_{C_{i+1}}\left[1+\sup_{s\leq t}\left|w_{i+1}'\left(s,\tilde{C}_{i},C_{i+1}\right)\right|^{2}+\sup_{s\leq t}\left|b_{i+1}'\left(s,C_{i+1}\right)\right|^{2}\right]^{25}\bigg]^{1/50}\\
 & \leq K\bigg(1+\mathbb{E}\left[\sup_{s\leq t}\underset{Z\sim{\cal P}}{{\rm ess\text{-}sup}}\left|\Delta_{i+1}^{H}\left(Z,C_{i+1};W'\left(s\right)\right)\right|^{2}\right]^{1/2}\bigg)\\
 & \quad\qquad\times\bigg(1+\mathbb{E}\left[\sup_{s\leq t}\left|w_{i+1}'\left(s,\tilde{C}_{i},C_{i+1}\right)\right|^{50}\right]^{1/50}+\mathbb{E}\left[\sup_{s\leq t}\left|b_{i+1}'\left(s,C_{i+1}\right)\right|^{2}\right]^{1/2}\bigg)\\
 & \leq KK_{*}\left(T\right)\bigg(1+\mathbb{E}\left[\sup_{s\leq t}\underset{Z\sim{\cal P}}{{\rm ess\text{-}sup}}\left|\Delta_{i+1}^{H}\left(Z,C_{i+1};W'\left(s\right)\right)\right|^{50}\right]^{1/50}\bigg).
\end{align*}
We have similarly,
\[
\mathbb{E}\left[\sup_{s\leq t}\underset{Z\sim{\cal P}}{{\rm ess\text{-}sup}}\left|\Delta_{i}^{H}\left(Z,C_{i};W'\left(s\right)\right)\right|^{50}\right]^{1/50}\leq KK_{*}\left(T\right)\bigg(1+\mathbb{E}\left[\sup_{s\leq t}\underset{Z\sim{\cal P}}{{\rm ess\text{-}sup}}\left|\Delta_{i+1}^{H}\left(Z,C_{i+1};W'\left(s\right)\right)\right|^{50}\right]^{1/50}\bigg).
\]
Therefore, for any $i\in\left[L\right]$,
\begin{align}
\mathbb{E}\left[\sup_{s\leq t}\underset{Z\sim{\cal P}}{{\rm ess\text{-}sup}}\left|\Delta_{i}^{H}\left(Z,C_{i};W'\left(s\right)\right)\right|^{50}\right]^{1/50},\;\mathbb{E}\left[\sup_{s\leq t}\underset{Z\sim{\cal P}}{{\rm ess\text{-}sup}}\left|\Delta_{i}^{H}\left(Z,\tilde{C}_{i};W'\left(s\right)\right)\right|^{50}\right]^{1/50} & \leq K^{L}K_{*}^{L}\left(T\right).\label{eq:lem_Lip_backward_1-1}
\end{align}
The same bound holds for $W''$. With this, let us proceed with two
steps.

\paragraph*{Step 1.}

For brevity, let us define
\begin{align*}
D_{i}^{H}\left(t\right) & =\mathbb{E}\left[\sup_{s\leq t}\underset{Z\sim{\cal P}}{{\rm ess\text{-}sup}}\left|\Delta_{i}^{H}\left(Z,C_{i};W'\left(s\right)\right)-\Delta_{i}^{H}\left(Z,C_{i};W''\left(s\right)\right)\right|^{2}\right]^{1/2},\\
\tilde{D}_{i}^{H}\left(t\right) & =\mathbb{E}\left[\sup_{s\leq t}\underset{Z\sim{\cal P}}{{\rm ess\text{-}sup}}\left|\Delta_{i}^{H}\left(Z,\tilde{C}_{i};W'\left(s\right)\right)-\Delta_{i}^{H}\left(Z,\tilde{C}_{i};W''\left(s\right)\right)\right|^{2}\right]^{1/2}.
\end{align*}
We first have from Assumption \ref{enu:Assump_backward} and Lemma
\ref{lem:Lipschitz forward MF - general}:
\begin{align*}
D_{L}^{H}\left(t\right)=\tilde{D}_{L}^{H}\left(t\right) & \leq K\sup_{s\leq t}\underset{X\sim{\cal P}}{{\rm ess\text{-}sup}}\left|H_{L}\left(X,1;W'\left(s\right)\right)-H_{L}\left(X,1;W''\left(s\right)\right)\right|\\
 & \quad+K\sup_{s\leq t}\underset{X\sim{\cal P}}{{\rm ess\text{-}sup}}\left|\hat{y}\left(X;W'\left(s\right)\right)-\hat{y}\left(X;W''\left(s\right)\right)\right|\\
 & \leq K^{L}K_{*}^{L}\left(T\right)\tilde{d}_{t}\left(W',W''\right).
\end{align*}
Next we consider $\tilde{D}_{i-1}^{H}$ and $D_{i-1}^{H}$ for $i\geq2$.
By Assumption \ref{enu:Assump_backward}:
\[
\tilde{D}_{i-1}^{H}\left(t\right)\leq K\left(\tilde{D}_{i-1}^{H,1}\left(t\right)+\tilde{D}_{i-1}^{H,2}\left(t\right)+\tilde{D}_{i-1}^{H,3}\left(t\right)+\tilde{D}_{i-1}^{H,4}\left(t\right)+\tilde{D}_{i-1}^{H,5}\left(t\right)\right),
\]
in which
\begin{align*}
\tilde{D}_{i-1}^{H,1}\left(t\right) & =\mathbb{E}_{\tilde{C}_{i-1}}\bigg[\sup_{s\leq t}\underset{Z\sim{\cal P}}{{\rm ess\text{-}sup}}\mathbb{E}_{C_{i}}\Big[\left(1+\left|w_{i}'\left(s,\tilde{C}_{i-1},C_{i}\right)\right|+\left|w_{i}''\left(s,\tilde{C}_{i-1},C_{i}\right)\right|+\left|b_{i}'\left(s,C_{i}\right)\right|+\left|b_{i}''\left(s,C_{i}\right)\right|\right)\\
 & \qquad\times\left|\Delta_{i}^{H}\left(z,C_{i};W'\left(s\right)\right)-\Delta_{i}^{H}\left(z,C_{i};W''\left(s\right)\right)\right|\Big]^{2}\bigg]^{1/2},\\
\tilde{D}_{i-1}^{H,2}\left(t\right) & =\mathbb{E}_{\tilde{C}_{i-1}}\bigg[\sup_{s\leq t}\underset{Z\sim{\cal P}}{{\rm ess\text{-}sup}}\mathbb{E}_{C_{i}}\Big[\left(1+\left|\Delta_{i}^{H}\left(Z,C_{i};W'\left(s\right)\right)\right|+\left|\Delta_{i}^{H}\left(Z,C_{i};W''\left(s\right)\right)\right|\right)\\
 & \qquad\times\left(\left|w_{i}'\left(s,\tilde{C}_{i-1},C_{i}\right)-w_{i}''\left(s,\tilde{C}_{i-1},C_{i}\right)\right|+\left|b_{i}'\left(s,C_{i}\right)-b_{i}''\left(s,C_{i}\right)\right|\right)\Big]^{2}\bigg]^{1/2},\\
\tilde{D}_{i-1}^{H,3}\left(t\right) & =\mathbb{E}_{\tilde{C}_{i-1}}\bigg[\sup_{s\leq t}\underset{Z\sim{\cal P}}{{\rm ess\text{-}sup}}\mathbb{E}_{C_{i}}\Big[\left(1+\left|\Delta_{i}^{H}\left(Z,C_{i};W'\left(s\right)\right)\right|+\left|\Delta_{i}^{H}\left(Z,C_{i};W''\left(s\right)\right)\right|\right)\\
 & \qquad\times\left(1+\left|w_{i}'\left(s,\tilde{C}_{i-1},C_{i}\right)\right|+\left|w_{i}''\left(s,\tilde{C}_{i-1},C_{i}\right)\right|+\left|b_{i}'\left(s,C_{i}\right)\right|+\left|b_{i}''\left(s,C_{i}\right)\right|\right)\\
 & \qquad\times\left|H_{i}\left(X,C_{i};W'\left(s\right)\right)-H_{i}\left(X,C_{i};W''\left(s\right)\right)\right|\Big]^{2}\bigg]^{1/2},\\
\tilde{D}_{i-1}^{H,4}\left(t\right) & =\mathbb{E}_{\tilde{C}_{i-1}}\bigg[\sup_{s\leq t}\underset{Z\sim{\cal P}}{{\rm ess\text{-}sup}}\mathbb{E}_{C_{i}}\Big[\left(1+\left|\Delta_{i}^{H}\left(Z,C_{i};W'\left(s\right)\right)\right|+\left|\Delta_{i}^{H}\left(Z,C_{i};W''\left(s\right)\right)\right|\right)\\
 & \qquad\times\left(1+\left|b_{i}'\left(s,C_{i}\right)\right|+\left|b_{i}''\left(s,C_{i}\right)\right|\right)\Big]^{2}\\
 & \qquad\times\left|H_{i-1}\left(X,\tilde{C}_{i-1};W'\left(s\right)\right)-H_{i-1}\left(X,\tilde{C}_{i-1};W''\left(s\right)\right)\right|^{2}\bigg]^{1/2},\\
\tilde{D}_{i-1}^{H,5}\left(t\right) & =\mathbb{E}_{\tilde{C}_{i-1}}\bigg[\sup_{s\leq t}\underset{Z\sim{\cal P}}{{\rm ess\text{-}sup}}\mathbb{E}_{C_{i}}\Big[\left(1+\left|\Delta_{i}^{H}\left(Z,C_{i};W'\left(s\right)\right)\right|+\left|\Delta_{i}^{H}\left(Z,C_{i};W''\left(s\right)\right)\right|\right)\\
 & \qquad\times\left(\left|w_{i}'\left(s,\tilde{C}_{i-1},C_{i}\right)\right|+\left|w_{i}''\left(s,\tilde{C}_{i-1},C_{i}\right)\right|\right)\Big]^{2}\\
 & \qquad\times\left|H_{i-1}\left(X,\tilde{C}_{i-1};W'\left(s\right)\right)-H_{i-1}\left(X,\tilde{C}_{i-1};W''\left(s\right)\right)\right|^{2}\bigg]^{1/2}.
\end{align*}
We bound each term. For $\tilde{D}_{i-1}^{H,1}$, we use Cauchy-Schwarz's
inequality to obtain:
\begin{align*}
\tilde{D}_{i-1}^{H,1}\left(t\right) & \leq K\mathbb{E}_{\tilde{C}_{i-1}}\bigg[\sup_{s\leq t}\underset{Z\sim{\cal P}}{{\rm ess\text{-}sup}}\mathbb{E}_{C_{i}}\Big[\Big(1+\left|w_{i}'\left(s,\tilde{C}_{i-1},C_{i}\right)\right|^{2}+\left|w_{i}''\left(s,\tilde{C}_{i-1},C_{i}\right)\right|^{2}\\
 & \quad\qquad+\left|b_{i}'\left(s,C_{i}\right)\right|^{2}+\left|b_{i}''\left(s,C_{i}\right)\right|^{2}\Big)\left|D_{i}^{H}\left(t\right)\right|^{2}\Big]\bigg]^{1/2}\\
 & \leq KK_{*}\left(T\right)D_{i}^{H}\left(t\right).
\end{align*}
Similarly, using Eq. (\ref{eq:lem_Lip_backward_1-1}), 
\begin{align*}
\tilde{D}_{i-1}^{H,2}\left(t\right) & \leq K\sup_{s\leq t}\underset{Z\sim{\cal P}}{{\rm ess\text{-}sup}}\mathbb{E}_{C_{i}}\left[1+\left|\Delta_{i}^{H}\left(Z,C_{i};W'\left(s\right)\right)\right|^{2}+\left|\Delta_{i}^{H}\left(Z,C_{i};W''\left(s\right)\right)\right|^{2}\right]^{1/2}\tilde{d}_{t}\left(W',W''\right)\\
 & \leq K^{L}K_{*}^{L}\left(T\right)\tilde{d}_{t}\left(W',W''\right).
\end{align*}
To bound $\tilde{D}_{i-1}^{H,3}$, we use Lemma \ref{lem:Lipschitz forward MF - general}
and Eq. (\ref{eq:lem_Lip_backward_1-1}): 
\begin{align*}
\tilde{D}_{i-1}^{H,3}\left(t\right) & \leq K\mathbb{E}_{\tilde{C}_{i-1}}\bigg[\sup_{s\leq t}\underset{Z\sim{\cal P}}{{\rm ess\text{-}sup}}\mathbb{E}_{C_{i}}\left[1+\left|\Delta_{i}^{H}\left(Z,C_{i};W'\left(s\right)\right)\right|^{4}+\left|\Delta_{i}^{H}\left(Z,C_{i};W''\left(s\right)\right)\right|^{4}\right]^{1/2}\\
 & \quad\qquad\times\mathbb{E}_{C_{i}}\left[1+\left|w_{i}'\left(s,\tilde{C}_{i-1},C_{i}\right)\right|^{4}+\left|w_{i}''\left(s,\tilde{C}_{i-1},C_{i}\right)\right|^{4}+\left|b_{i}'\left(s,C_{i}\right)\right|^{4}+\left|b_{i}''\left(s,C_{i}\right)\right|^{4}\right]^{1/2}\\
 & \quad\qquad\times\mathbb{E}_{C_{i}}\left[\left|H_{i}\left(X,C_{i};W'\left(s\right)\right)-H_{i}\left(X,C_{i};W''\left(s\right)\right)\right|^{2}\right]\bigg]^{1/2}\\
 & \leq K^{2L+2}K_{*}^{2L+2}\left(T\right)\tilde{d}_{t}\left(W',W''\right),
\end{align*}
and similarly, for $\tilde{D}_{i-1}^{H,4}$,
\begin{align*}
\tilde{D}_{i-1}^{H,4}\left(t\right) & \leq K\mathbb{E}_{\tilde{C}_{i-1}}\bigg[\sup_{s\leq t}\underset{Z\sim{\cal P}}{{\rm ess\text{-}sup}}\mathbb{E}_{C_{i}}\left[1+\left|\Delta_{i}^{H}\left(Z,C_{i};W'\left(s\right)\right)\right|^{2}+\left|\Delta_{i}^{H}\left(Z,C_{i};W''\left(s\right)\right)\right|^{2}\right]\\
 & \quad\qquad\times\mathbb{E}_{C_{i}}\left[1+\left|b_{i}'\left(s,C_{i}\right)\right|^{2}+\left|b_{i}''\left(s,C_{i}\right)\right|^{2}\right]\\
 & \quad\qquad\times\left|H_{i-1}\left(X,\tilde{C}_{i-1};W'\left(s\right)\right)-H_{i-1}\left(X,\tilde{C}_{i-1};W''\left(s\right)\right)\right|^{2}\bigg]^{1/2}\\
 & \leq K^{2L+2}K_{*}^{2L+2}\left(T\right)\tilde{d}_{t}\left(W',W''\right).
\end{align*}
The treatment of $\tilde{D}_{i-1}^{H,5}$ requires more care. Cauchy-Schwarz's
inequality and Eq. (\ref{eq:lem_Lip_backward_1-1}) give us:
\begin{align*}
\tilde{D}_{i-1}^{H,5}\left(t\right) & \leq K\mathbb{E}_{\tilde{C}_{i-1}}\bigg[\sup_{s\leq t}\underset{Z\sim{\cal P}}{{\rm ess\text{-}sup}}\mathbb{E}_{C_{i}}\left[1+\left|\Delta_{i}^{H}\left(Z,C_{i};W'\left(s\right)\right)\right|^{2}+\left|\Delta_{i}^{H}\left(Z,C_{i};W''\left(s\right)\right)\right|^{2}\right]\\
 & \quad\qquad\times\mathbb{E}_{C_{i}}\left[\left|w_{i}'\left(s,\tilde{C}_{i-1},C_{i}\right)\right|^{2}+\left|w_{i}''\left(s,\tilde{C}_{i-1},C_{i}\right)\right|^{2}\right]\\
 & \quad\qquad\times\left|H_{i-1}\left(X,\tilde{C}_{i-1};W'\left(s\right)\right)-H_{i-1}\left(X,\tilde{C}_{i-1};W''\left(s\right)\right)\right|^{2}\bigg]^{1/2}\\
 & \leq K^{L}K_{*}^{L}\left(T\right)\mathbb{E}\bigg[\sup_{s\leq t}\left(\left|w_{i}'\left(s,\tilde{C}_{i-1},C_{i}\right)\right|^{2}+\left|w_{i}''\left(s,\tilde{C}_{i-1},C_{i}\right)\right|^{2}\right)\\
 & \quad\qquad\times\underset{Z\sim{\cal P}}{{\rm ess\text{-}sup}}\left|H_{i-1}\left(X,\tilde{C}_{i-1};W'\left(s\right)\right)-H_{i-1}\left(X,\tilde{C}_{i-1};W''\left(s\right)\right)\right|^{2}\bigg]^{1/2}.
\end{align*}
Recall our assumption:
\begin{align*}
\mathbb{P}\left(\widetilde{\mathsf{max}}_{T}^{w}\left(W'\right)\geq K_{*}\left(T\right)B\right) & \leq\Xi\left(B\right),\\
\mathbb{P}\left(\widetilde{\mathsf{max}}_{T}^{w}\left(W''\right)\geq K_{*}\left(T\right)B\right) & \leq\Xi\left(B\right).
\end{align*}
 We also have from Assumption \ref{enu:Assump_forward}:
\begin{align*}
 & \mathbb{E}\left[\sup_{s\leq t}\underset{Z\sim{\cal P}}{{\rm ess\text{-}sup}}\left|H_{i-1}\left(X,\tilde{C}_{i-1};W'\left(s\right)\right)\right|^{8}\right]^{1/8}\\
 & \leq K\bigg(1+\mathbb{E}\left[\sup_{s\leq t}\left|w_{i-1}'\left(s,C_{i-2},\tilde{C}_{i-1}\right)\right|^{8}\right]^{1/8}+\mathbb{E}\left[\sup_{s\leq t}\left|b_{i-1}'\left(s,\tilde{C}_{i-1}\right)\right|^{8}\right]^{1/8}\bigg)\\
 & \leq KK_{*}\left(T\right),
\end{align*}
and similarly,
\begin{align*}
\mathbb{E}\left[\sup_{s\leq t}\underset{Z\sim{\cal P}}{{\rm ess\text{-}sup}}\left|H_{i-1}\left(X,\tilde{C}_{i-1};W''\left(s\right)\right)\right|^{8}\right]^{1/8} & \leq KK_{*}\left(T\right).
\end{align*}
As such, denoting the event
\[
E=\left\{ \sup_{s\leq t}\left|w_{i}'\left(s,\tilde{C}_{i-1},C_{i}\right)\right|\geq K_{*}\left(T\right)B,\quad\sup_{s\leq t}\left|w_{i}''\left(s,\tilde{C}_{i-1},C_{i}\right)\right|\geq K_{*}\left(T\right)B\right\} ,
\]
we obtain from Lemma \ref{lem:Lipschitz forward MF - general} :
\begin{align*}
\tilde{D}_{i-1}^{H,5}\left(t\right) & \leq K^{L}K_{*}^{L}\left(T\right)\mathbb{E}\bigg[\sup_{s\leq t}\left(\left|w_{i}'\left(s,\tilde{C}_{i-1},C_{i}\right)\right|^{2}+\left|w_{i}''\left(s,\tilde{C}_{i-1},C_{i}\right)\right|^{2}\right)\\
 & \quad\qquad\times\underset{Z\sim{\cal P}}{{\rm ess\text{-}sup}}\left|H_{i-1}\left(X,\tilde{C}_{i-1};W'\left(s\right)\right)-H_{i-1}\left(X,\tilde{C}_{i-1};W''\left(s\right)\right)\right|^{2}\left(\mathbb{I}\left(\neg E\right)+\mathbb{I}\left(E\right)\right)\bigg]^{1/2}.\\
 & \leq K^{L}K_{*}^{L+1}\left(T\right)B\mathbb{E}\left[\sup_{s\leq t}\underset{Z\sim{\cal P}}{{\rm ess\text{-}sup}}\left|H_{i-1}\left(X,\tilde{C}_{i-1};W'\left(s\right)\right)-H_{i-1}\left(X,\tilde{C}_{i-1};W''\left(s\right)\right)\right|^{2}\right]^{1/2}\\
 & \quad+K^{L}K_{*}^{L}\left(T\right)\mathbb{E}\left[\sup_{s\leq t}\left|w_{i}'\left(s,\tilde{C}_{i-1},C_{i}\right)\right|^{8}+\sup_{s\leq t}\left|w_{i}''\left(s,\tilde{C}_{i-1},C_{i}\right)\right|^{8}\right]^{1/8}\\
 & \quad\qquad\times\mathbb{E}\left[\sup_{s\leq t}\underset{Z\sim{\cal P}}{{\rm ess\text{-}sup}}\left(\left|H_{i-1}\left(X,\tilde{C}_{i-1};W'\left(s\right)\right)\right|^{8}+\left|H_{i-1}\left(X,\tilde{C}_{i-1};W''\left(s\right)\right)\right|^{8}\right)\right]^{1/8}\mathbb{P}\left(E\right)^{1/2}\\
 & \leq K^{2L+2}K_{*}^{2L+2}\left(T\right)B\tilde{d}_{t}\left(W',W''\right)+K^{L}K_{*}^{L+2}\left(T\right)\sqrt{\Xi\left(B\right)}.
\end{align*}
Putting all the bounds together:
\[
\tilde{D}_{i-1}^{H}\left(t\right)\leq KK_{*}\left(T\right)D_{i}^{H}\left(t\right)+\left(KK_{*}\left(T\right)\right)^{2L+2}\left(1+B\right)\tilde{d}_{t}\left(W',W''\right)+K^{L}K_{*}^{L+2}\left(T\right)\sqrt{\Xi\left(B\right)}.
\]
Similarly,
\[
D_{i-1}^{H}\left(t\right)\leq KK_{*}\left(T\right)D_{i}^{H}\left(t\right)+\left(KK_{*}\left(T\right)\right)^{2L+2}\left(1+B\right)\tilde{d}_{t}\left(W',W''\right)+K^{L}K_{*}^{L+2}\left(T\right)\sqrt{\Xi\left(B\right)}.
\]
Together with the bound on $D_{L}^{H}$ and $\tilde{D}_{L}^{H}$,
we thus obtain:
\begin{equation}
\max_{1\leq i\leq L}\max\left(\tilde{D}_{i}^{H}\left(t\right),D_{i}^{H}\left(t\right)\right)\leq\left(KK_{*}\left(T\right)\right)^{3L+2}\left(\left(1+B\right)\tilde{d}_{t}\left(W',W''\right)+\sqrt{\Xi\left(B\right)}\right).\label{eq:lem_Lip_backward_2-1}
\end{equation}
This completes the first step.

\paragraph*{Step 2.}

We now prove the main claims of the lemma. For brevity, for $i\geq2$,
let us denote
\begin{align*}
\tilde{D}_{i}^{w}\left(t\right) & =\mathbb{E}\left[\sup_{s\leq t}\underset{Z\sim{\cal P}}{{\rm ess\text{-}sup}}\left|\Delta_{i}^{w}\left(Z,\tilde{C}_{i-1},\tilde{C}_{i};W'\left(s\right)\right)-\Delta_{i}^{w}\left(Z,\tilde{C}_{i-1},\tilde{C}_{i};W''\left(s\right)\right)\right|^{2}\right]^{1/2},\\
\tilde{D}_{i}^{b}\left(t\right) & =\mathbb{E}\left[\sup_{s\leq t}\underset{Z\sim{\cal P}}{{\rm ess\text{-}sup}}\left|\Delta_{i}^{b}\left(Z,\tilde{C}_{i};W'\left(s\right)\right)-\Delta_{i}^{b}\left(Z,\tilde{C}_{i};W''\left(s\right)\right)\right|^{2}\right]^{1/2},\\
\tilde{D}_{1}^{w}\left(t\right) & =\mathbb{E}\left[\sup_{s\leq t}\underset{Z\sim{\cal P}}{{\rm ess\text{-}sup}}\left|\Delta_{1}^{w}\left(Z,\tilde{C}_{1};W'\left(s\right)\right)-\Delta_{1}^{w}\left(Z,\tilde{C}_{1};W''\left(s\right)\right)\right|^{2}\right]^{1/2}.
\end{align*}
By Assumption \ref{enu:Assump_forward},
\[
\tilde{D}_{i}^{w}\left(t\right)\leq K\left(\tilde{D}_{i}^{w,1}\left(t\right)+\tilde{D}_{i}^{w,2}\left(t\right)\right),
\]
in which
\begin{align*}
\tilde{D}_{i}^{w,1}\left(t\right) & =\mathbb{E}\bigg[\sup_{s\leq t}\underset{Z\sim{\cal P}}{{\rm ess\text{-}sup}}\left(1+\left|\Delta_{i}^{H}\left(Z,\tilde{C}_{i};W'\left(s\right)\right)\right|^{2}+\left|\Delta_{i}^{H}\left(Z,\tilde{C}_{i};W''\left(s\right)\right)\right|^{2}\right)\\
 & \qquad\times\left(\left|H_{i-1}\left(X,\tilde{C}_{i-1};W'\left(s\right)\right)-H_{i-1}\left(X,\tilde{C}_{i-1};W''\left(s\right)\right)\right|^{2}\right)\bigg]^{1/2},\\
\tilde{D}_{i}^{w,2}\left(t\right) & =\mathbb{E}\bigg[\sup_{s\leq t}\underset{Z\sim{\cal P}}{{\rm ess\text{-}sup}}\bigg(\left|\Delta_{i}^{H}\left(Z,\tilde{C}_{i};W'\left(s\right)\right)-\Delta_{i}^{H}\left(Z,\tilde{C}_{i};W''\left(s\right)\right)\right|^{2}+\left|w_{i}'\left(s,\tilde{C}_{i-1},\tilde{C}_{i}\right)-w_{i}''\left(s,\tilde{C}_{i-1},\tilde{C}_{i}\right)\right|^{2}\\
 & \qquad+\left|b_{i}'\left(s,\tilde{C}_{i}\right)-b_{i}''\left(s,\tilde{C}_{i}\right)\right|^{2}+\left|H_{i}\left(X,\tilde{C}_{i};W'\left(s\right)\right)-H_{i}\left(X,\tilde{C}_{i};W''\left(s\right)\right)\right|^{2}\bigg)\bigg]^{1/2}.
\end{align*}
We bound $\tilde{D}_{i}^{w,1}$:
\begin{align*}
\tilde{D}_{i}^{w,1}\left(t\right) & \leq\mathbb{E}\left[\sup_{s\leq t}\underset{Z\sim{\cal P}}{{\rm ess\text{-}sup}}\left(1+\left|\Delta_{i}^{H}\left(Z,\tilde{C}_{i};W'\left(s\right)\right)\right|^{2}+\left|\Delta_{i}^{H}\left(Z,\tilde{C}_{i};W''\left(s\right)\right)\right|^{2}\right)\right]^{1/2}\\
 & \qquad\times\mathbb{E}\left[\sup_{s\leq t}\underset{Z\sim{\cal P}}{{\rm ess\text{-}sup}}\left|H_{i-1}\left(X,\tilde{C}_{i-1};W'\left(s\right)\right)-H_{i-1}\left(X,\tilde{C}_{i-1};W''\left(s\right)\right)\right|^{2}\right]^{1/2}\\
 & \leq K^{2L}K_{*}^{2L}\left(T\right)\tilde{d}_{t}\left(W',W''\right),
\end{align*}
where we have used Eq. (\ref{eq:lem_Lip_backward_1-1}) and Lemma
\ref{lem:Lipschitz forward MF - general}. We also have the following
bound on $\tilde{D}_{i}^{w,2}$ from Lemma \ref{lem:Lipschitz forward MF - general}
and Eq. (\ref{eq:lem_Lip_backward_2-1}):
\[
\tilde{D}_{i}^{w,2}\left(t\right)\leq\left(KK_{*}\left(T\right)\right)^{3L+2}\left(\left(1+B\right)\tilde{d}_{t}\left(W',W''\right)+\sqrt{\Xi\left(B\right)}\right),
\]
which therefore leads to
\[
\tilde{D}_{i}^{w}\left(t\right)\leq\left(KK_{*}\left(T\right)\right)^{3L+2}\left(\left(1+B\right)\tilde{d}_{t}\left(W',W''\right)+\sqrt{\Xi\left(B\right)}\right).
\]
The same bound similarly applies to $\tilde{D}_{i}^{b}\left(t\right)$
and $\tilde{D}_{1}^{w}\left(t\right)$.
\end{proof}

\section{Remaining proofs for Section \ref{sec:Main-result}\label{sec:Remaining-proofs-main-MF}}

\subsection{Proofs of Propositions \ref{prop:particle coupling - bounded}, \ref{prop:gradient descent - bounded}
and \ref{prop:truncation}}

Before delving into the proofs, we introduce some auxiliary results.
We first present a useful concentration result. In fact, the tail
bound can be improved using the argument in \cite{feldman2018generalization},
but the following simpler version is sufficient for our purposes.
\begin{lem}
\label{lem:square hoeffding}Consider an integer $n\geq2$; let $(c_{1},c_{2},\dots,c_{n})$
be $\eta$-independent for $\eta\in\left[0,1/2\right]$ and let $x$
be another independent random variable. Let $\mathbb{E}_{x}$ and
$\mathbb{E}_{c}$ denote the expectations w.r.t. $x$ only and $\left\{ c_{i}\right\} _{i\in\left[n\right]}$
only, respectively. Consider a collection of mappings $\left\{ f_{i}\right\} _{i\in\left[n\right]}$,
which map to the same separable Hilbert space. Let $f_{i}\left(x\right)=\mathbb{E}_{c}\left[f_{i}\left(c_{i},x\right)\right]$.
Assume that $|f_{i}(c_{i},x)-f_{i}(x)|\le R$ for almost every $x$
and $c_{i}$, then for any $\delta>2\eta R$,
\[
\mathbb{P}\left(\mathbb{E}_{x}\left[\left|\frac{1}{n}\sum_{i=1}^{n}f_{i}\left(c_{i},x\right)-f_{i}\left(x\right)\right|\right]\geq\delta\right)\leq\frac{4R}{\delta}\exp\left(-\frac{n\delta^{2}}{512R^{2}}\right).
\]
\end{lem}

\begin{proof}
For brevity, let us define
\[
Z_{n}\left(x\right)=\sum_{i=1}^{n}\left(f_{i}\left(c_{i},x\right)-f_{i}\left(x\right)\right).
\]
By Theorem \ref{thm:iid-hilbert-concentration}, for $\delta>2\eta R$,
\[
\mathbb{P}\left(\left|Z_{n}\left(x\right)\right|\geq n\delta\middle|x\right)\leq2\exp\left(-n\delta^{2}/\left(64R^{2}\right)\right),
\]
and therefore,
\[
\mathbb{P}\left(\left|Z_{n}\left(x\right)\right|\geq n\delta\right)\leq2\exp\left(-n\delta^{2}/\left(64R^{2}\right)\right),
\]
since the right-hand side is uniform in $x$. Next note that, w.r.t.
the randomness of $x$ only,
\begin{align*}
\mathbb{E}_{x}\left[\left|Z_{n}\left(x\right)\right|\right] & =\mathbb{E}_{x}\left[\left|Z_{n}\left(x\right)\right|\mathbb{I}\left(\left|Z_{n}\left(x\right)\right|\geq n\delta/2\right)\right]+\mathbb{E}_{x}\left[\left|Z_{n}\left(x\right)\right|\mathbb{I}\left(\left|Z_{n}\left(x\right)\right|<n\delta/2\right)\right]\\
 & \leq\mathbb{E}_{x}\left[\left|Z_{n}\left(x\right)\right|\mathbb{I}\left(\left|Z_{n}\left(x\right)\right|\geq n\delta/2\right)\right]+n\delta/2.
\end{align*}
As such, by Markov's inequality and Cauchy-Schwarz's inequality,
\begin{align*}
\mathbb{P}\left(\mathbb{E}_{x}\left[\left|Z_{n}\left(x\right)\right|\right]\geq n\delta\right) & \leq\mathbb{P}\left(\mathbb{E}_{x}\left[\left|Z_{n}\left(x\right)\right|\mathbb{I}\left(\left|Z_{n}\left(x\right)\right|\geq n\delta/2\right)\right]\geq n\delta/2\right)\\
 & \leq\frac{2}{n\delta}\mathbb{E}\left[\left|Z_{n}\left(x\right)\right|\mathbb{I}\left(\left|Z_{n}\left(x\right)\right|\geq n\delta/2\right)\right]\\
 & \leq\frac{2}{n\delta}\mathbb{E}\left[\left|Z_{n}\left(x\right)\right|^{2}\right]^{1/2}\mathbb{P}\left(\left|Z_{n}\left(x\right)\right|\geq n\delta/2\right)^{1/2}\\
 & \leq\frac{4}{n\delta}\mathbb{E}\left[\left|Z_{n}\left(x\right)\right|^{2}\right]^{1/2}\exp\left(-\frac{n\delta^{2}}{512R^{2}}\right).
\end{align*}
Notice that since $c_{1},...,c_{n}$ are $\eta$-independent and $f_{i}\left(x\right)=\mathbb{E}_{c}\left[f_{i}\left(c_{i},x\right)\right]$,
\[
\mathbb{E}\left[\left|Z_{n}\left(x\right)\right|^{2}\right]\le nR^{2}+\eta n^{2}R^{2}.
\]
We thus get:
\begin{align*}
\mathbb{P}\left(\mathbb{E}_{x}\left[\left|Z_{n}\left(x\right)\right|\right]\geq n\delta\right) & \leq\frac{4\sqrt{1+\eta n}R}{\sqrt{n}\delta}\exp\left(-\frac{n\delta^{2}}{512R^{2}}\right)\le\frac{4R}{\delta}\exp\left(-\frac{n\delta^{2}}{512R^{2}}\right).
\end{align*}
This proves the claim.
\end{proof}
The next useful result concerns with the sampling at initialization.
\begin{lem}
\label{lem:initialization_compare}Under Assumption \ref{assump:init},
following the coupling procedure, we have for any $\delta>0$ and
$B>0$, with probability at least $1-KLn_{\max}\exp\left(-K\left(\delta\land\delta^{1/26}\right)n_{\min}^{1/52}\right)$,
the following hold:
\begin{itemize}
\item moment bounds:
\begin{align*}
 & \interleave\tilde{W}\interleave_{0}=\interleave\mathbf{W}\interleave_{0}\leq\interleave W\interleave_{0}+\delta^{1/50},\\
 & \bigg(\frac{1}{n_{i}}\sum_{j_{i}=1}^{n_{i}}\mathbb{E}_{C_{i-1}}\left[\left|w_{i}^{0}\left(C_{i-1},C_{i}\left(j_{i}\right)\right)\right|^{50}\right]\bigg)^{1/50}\leq\mathbb{E}\left[\left|w_{i}^{0}\left(C_{i-1},C_{i}\right)\right|^{50}\right]^{1/50}+\delta^{1/50},\quad2\leq i\leq L,\\
 & \bigg(\frac{1}{n_{i-1}}\sum_{j_{i-1}=1}^{n_{i-1}}\mathbb{E}_{C_{i}}\left[\left|w_{i}^{0}\left(C_{i-1}\left(j_{i-1}\right),C_{i}\right)\right|^{50}\right]\bigg)^{1/50}\leq\mathbb{E}\left[\left|w_{i}^{0}\left(C_{i-1},C_{i}\right)\right|^{50}\right]^{1/50}+\delta^{1/50},\quad2\leq i\leq L,
\end{align*}
\item excess bounds:
\begin{align*}
 & \left|\frac{1}{n_{i-1}n_{i}}\sum_{j_{i-1}=1}^{n_{i-1}}\sum_{j_{i}=1}^{n_{i}}\mathbb{I}\left(\left|w_{i}^{0}\left(C_{i-1}\left(j_{i-1}\right),C_{i}\left(j_{i}\right)\right)\right|\geq B\right)-\mathbb{P}\left(\left|w_{i}^{0}\left(C_{i-1},C_{i}\right)\right|\geq B\right)\right|\leq\delta,\quad2\leq i\leq L,\\
 & \left|\frac{1}{n_{i}}\sum_{j_{i}=1}^{n_{i}}\mathbb{I}\left(\left|b_{i}^{0}\left(C_{i}\left(j_{i}\right)\right)\right|\geq B\right)-\mathbb{P}\left(\left|b_{i}^{0}\left(C_{i}\right)\right|\geq B\right)\right|\leq\delta,\quad2\leq i\leq L.
\end{align*}
\end{itemize}
Here $n_{\max}=\max\left(n_{1},...,n_{L}\right)$ and $n_{\min}=\min\left(n_{1},...,n_{L-1}\right)$.
\end{lem}

\begin{proof}
We treat the bounds separately.

\paragraph*{The moment bounds.}

We recall that
\begin{align*}
\interleave\tilde{W}\interleave_{0}=\max\Bigg( & \max_{2\leq i\leq L}\bigg(\frac{1}{n_{i-1}n_{i}}\sum_{j_{i-1}=1}^{n_{i-1}}\sum_{j_{i}=1}^{n_{i}}\left|w_{i}^{0}\left(C_{i-1}\left(j_{i-1}\right),C_{i}\left(j_{i}\right)\right)\right|^{50}\bigg)^{1/50},\\
 & \max_{2\leq i\leq L}\bigg(\frac{1}{n_{i}}\sum_{j_{i}=1}^{n_{i}}\left|b_{i}^{0}\left(C_{i}\left(j_{i}\right)\right)\right|^{50}\bigg)^{1/50},\;\bigg(\frac{1}{n_{1}}\sum_{j_{1}=1}^{n_{1}}\left|w_{1}^{0}\left(C_{1}\left(j_{1}\right)\right)\right|^{50}\bigg)^{1/50}\Bigg).
\end{align*}
Let us first prove the following:
\begin{align*}
\mathbb{P}\left(Z^{\left(1\right)}\geq\delta\right) & \leq e\cdot\exp\left(-K\delta^{1/26}n_{1}^{1/52}\right),\\
Z^{\left(1\right)} & =\left|\frac{1}{n_{1}}\sum_{j_{1}=1}^{n_{1}}\left|w_{1}^{0}\left(C_{1}\left(j_{1}\right)\right)\right|^{50}-\mathbb{E}\left[\left|w_{1}^{0}\left(C_{1}\right)\right|^{50}\right]\right|.
\end{align*}
Indeed we note that for any $m\geq1$, $\mathbb{E}\left[\left|w_{1}^{0}\left(C_{1}\right)\right|^{50m}\right]^{1/m}\leq Km^{25}$.
As such, by Theorem \ref{thm:iid-hilbert-higher-moment},
\[
\mathbb{E}\left[\left|Z^{\left(1\right)}\right|^{m}\right]^{1/m}\leq Km^{26}n_{1}^{-1/2}.
\]
This implies $\left|Z^{\left(1\right)}\right|^{1/52}$ is $Kn_{1}^{-1/104}$-sub-Gaussian,
from which the claim follows. Using the same argument, we get
\begin{align*}
\mathbb{P}\left(Z^{\left(L\right)}\geq\delta\right) & \leq e\cdot\exp\left(-K\delta^{1/26}n_{L-1}^{1/52}\right),\\
Z^{\left(L\right)} & =\left|\frac{1}{n_{L-1}}\sum_{j_{L-1}=1}^{n_{L-1}}\left|w_{L}^{0}\left(C_{L-1}\left(j_{L-1}\right),1\right)\right|^{50}-\mathbb{E}\left[\left|w_{L}^{0}\left(C_{L-1},1\right)\right|^{50}\right]\right|,
\end{align*}
as well as that
\begin{align*}
\mathbb{P}\left(A^{\left(i\right)}\geq\delta\right) & \leq e\cdot\exp\left(-K\delta^{1/26}n_{i}^{1/52}\right),\\
A^{\left(i\right)} & =\left|\frac{1}{n_{i}}\sum_{j_{i}=1}^{n_{i}}\left|b_{i}^{0}\left(C_{i}\left(j_{i}\right)\right)\right|^{50}-\mathbb{E}\left[\left|b_{i}^{0}\left(C_{i}\right)\right|^{50}\right]\right|,
\end{align*}
for $2\leq i\leq L-1$. In addition, since $\Omega_{L}=\left\{ 1\right\} $
and $n_{L}=1$, it is obvious that $\left|b_{L}^{0}\left(C_{L}\left(j_{L}\right)\right)\right|=\left|b_{L}^{0}\left(C_{L}\right)\right|$.

Next for $2\leq i\leq L-1$, without loss of generality, suppose $n_{i}\geq n_{i-1}$.
Let us prove the following:
\begin{align*}
\mathbb{P}\left(Z^{\left(i\right)}\geq\delta\right) & \leq en_{i}\cdot\exp\left(-K\delta^{1/26}n_{i-1}^{1/52}\right),\\
Z^{\left(i\right)} & =\left|\frac{1}{n_{i-1}n_{i}}\sum_{j_{i-1}=1}^{n_{i-1}}\sum_{j_{i}=1}^{n_{i}}\left|w_{i}^{0}\left(C_{i-1}\left(j_{i-1}\right),C_{i}\left(j_{i}\right)\right)\right|^{50}-\mathbb{E}\left[\left|w_{i}^{0}\left(C_{i-1},C_{i}\right)\right|^{50}\right]\right|.
\end{align*}
For fixed $j_{i}\in\left[n_{i}\right]$, let us first consider
\[
\tilde{C}_{j_{i}}\left(j_{i-1}\right)=\left(C_{i-1}\left(j_{i-1}\right),C_{i}\left(\left(j_{i-1}+j_{i}\right)\mod n_{i}\right)\right).
\]
For any $1$-bounded function $f$, due to independence between $C_{i}\left(j_{i}\right)$'s
and $C_{i-1}\left(j_{i-1}\right)$'s and Assumption \ref{assump:neuronal-embedding},
we have:
\begin{align*}
 & \bigg|\mathbb{E}\left[f\left(\tilde{C}_{j_{i}}\left(j_{i-1}\right)\right)\middle|\left\{ \tilde{C}_{j_{i}}\left(j_{i-1}'\right)\right\} _{j_{i-1}'\ne j_{i-1}},\;C_{i}\left(\left(j_{i-1}+j_{i}\right)\mod n_{i}\right)\right]\\
 & \qquad-\mathbb{E}\left[f\left(\tilde{C}_{j_{i}}\left(j_{i-1}\right)\right)\middle|C_{i}\left(\left(j_{i-1}+j_{i}\right)\mod n_{i}\right)\right]\bigg|\leq\eta_{i-1},
\end{align*}
which implies
\[
\left|\mathbb{E}\left[f\left(\tilde{C}_{j_{i}}\left(j_{i-1}\right)\right)\middle|\left\{ \tilde{C}_{j_{i}}\left(j_{i-1}'\right)\right\} _{j_{i-1}'\ne j_{i-1}}\right]-\mathbb{E}\left[f\left(\tilde{C}_{j_{i}}\left(j_{i-1}\right)\right)\right]\right|\leq\eta_{i-1}.
\]
That is, $\left\{ \tilde{C}_{j_{i}}\left(j_{i-1}\right)\right\} _{j_{i-1}\in\left[n_{i-1}\right]}$
is $\eta_{i-1}$-independent. Hence by the same argument, by letting
\begin{align*}
Z_{j_{i}}^{\left(i\right)} & =\left|\frac{1}{n_{i-1}}\sum_{j_{i-1}=1}^{n_{i-1}}\left|w_{i}^{0}\left(\tilde{C}_{j_{i}}\left(j_{i-1}\right)\right)\right|^{50}-\mathbb{E}\left[\left|w_{i}^{0}\left(\tilde{C}_{j_{i}}\left(j_{i-1}\right)\right)\right|^{50}\right]\right|,
\end{align*}
we have:
\[
\mathbb{P}\left(Z_{j_{i}}^{\left(i\right)}\geq\delta\right)\leq e\cdot\exp\left(-K\delta^{1/26}n_{i-1}^{1/52}\right).
\]
By the union bound,
\[
\mathbb{P}\left(Z^{\left(i\right)}\geq\delta\right)\leq\mathbb{P}\left(\max_{j_{i}\leq n_{i}}Z_{j_{i}}^{\left(i\right)}\geq\delta\right)\leq en_{i}\cdot\exp\left(-K\delta^{1/26}n_{i-1}^{1/52}\right),
\]
which is the desired claim.

Upon an application of the union bound, these probability bounds imply
the bound on the probability of the event $\interleave\tilde{W}\interleave_{0}=\interleave\mathbf{W}\interleave_{0}\leq\interleave W\interleave_{0}+\delta^{1/50}$.
The rest of the bounds are similarly proven.

\paragraph*{The excess bounds.}

Without loss of generality, assume $n_{i}\geq n_{i-1}$ for $2\leq i\leq L-1$.
Let us denote
\[
D_{j_{i}}^{\left(i\right)}=\frac{1}{n_{i-1}}\sum_{j_{i-1}=1}^{n_{i-1}}\mathbb{I}\left(\left|w_{i}^{0}\left(\tilde{C}_{j_{i}}\left(j_{i-1}\right)\right)\right|\geq B\right)-\mathbb{P}\left(\left|w_{i}^{0}\left(C_{i-1},C_{i}\right)\right|\geq B\right).
\]
Recall previously that $\left\{ \tilde{C}_{j_{i}}\left(j_{i-1}\right)\right\} _{j_{i-1}\in\left[n_{i-1}\right]}$
is $\eta_{i-1}$-independent. As such, by Theorem \ref{thm:iid-hilbert-higher-moment},
\[
\mathbb{E}\left[\left|D_{j_{i}}^{\left(i\right)}\right|^{m}\right]^{1/m}\leq Kmn_{i-1}^{-1/2}.
\]
This implies $\left|D_{j_{i}}^{\left(i\right)}\right|^{1/2}$ is $Kn_{i-1}^{-1/4}$-sub-Gaussian
and hence for any $\delta>0$,
\[
\mathbb{P}\left(\left|D_{j_{i}}^{\left(i\right)}\right|\geq\delta\right)\leq e\cdot\exp\left(-K\delta n_{i-1}^{1/2}\right).
\]
The union bound yields
\[
\mathbb{P}\left(\left|\frac{1}{n_{i}}\sum_{j_{i}=1}^{n_{i}}D_{j_{i}}^{\left(i\right)}\right|\geq\delta\right)\leq Kn_{i}\exp\left(-K\delta n_{i-1}^{1/2}\right).
\]
Note that this holds for any $B>0$. The rest of the bounds are similarly
proven.
\end{proof}
Similar to Lemma \ref{lem:bounds MF a priori}, one can prove the
following:
\begin{lem}
\label{lem:bounds NN a priori}Under Assumptions \ref{enu:Assump_lrSchedule}
and \ref{enu:Assump_backward}, for any $t\in[0,\infty)$,
\[
\interleave\tilde{W}\interleave_{t}\leq K^{\kappa_{L}}\left(1+t^{\kappa_{L}}\right)\left(1+\interleave\tilde{W}\interleave_{0}^{\kappa_{L}}\right),\qquad\interleave\mathbf{W}\interleave_{\left\lfloor t/\epsilon\right\rfloor }\leq K^{\kappa_{L}}\left(1+t^{\kappa_{L}}\right)\left(1+\interleave\mathbf{W}\interleave_{0}^{\kappa_{L}}\right),
\]
where $\kappa_{L}=K^{L}$ for some constant $K>1$ sufficiently large.
In particular, for any $i\in\left[L\right]$:
\begin{align*}
\bigg(\frac{1}{n_{i}}\sum_{j_{i}=1}^{n_{i}}\sup_{s\leq t}\underset{Z\sim{\cal P}}{{\rm ess\text{-}sup}}\left|\Delta_{i}^{\mathbf{H}}\left(Z,j_{i};\tilde{W}\left(s\right)\right)\right|^{50}\bigg)^{1/50} & \leq K^{\kappa_{L}}\left(1+t^{\kappa_{L}}\right)\left(1+\interleave\tilde{W}\interleave_{0}^{\kappa_{L}}\right),\\
\bigg(\frac{1}{n_{i}}\sum_{j_{i}=1}^{n_{i}}\sup_{s\leq t}\underset{Z\sim{\cal P}}{{\rm ess\text{-}sup}}\left|\Delta_{i}^{\mathbf{H}}\left(Z,j_{i};\mathbf{W}\left(\left\lfloor s/\epsilon\right\rfloor \right)\right)\right|^{50}\bigg)^{1/50} & \leq K^{\kappa_{L}}\left(1+t^{\kappa_{L}}\right)\left(1+\interleave\mathbf{W}\interleave_{0}^{\kappa_{L}}\right).
\end{align*}
Furthermore, by defining
\begin{align*}
\interleave W\interleave_{{\rm samp},t}=\max\Bigg( & \max_{2\leq i\leq L}\bigg(\frac{1}{n_{i}}\sum_{j_{i}=1}^{n_{i}}\mathbb{E}_{C_{i-1}}\left[\sup_{s\leq t}\left|w_{i}\left(s,C_{i-1},C_{i}\left(j_{i}\right)\right)\right|^{50}\right]\bigg)^{1/50},\\
 & \max_{2\leq i\leq L}\bigg(\frac{1}{n_{i-1}}\sum_{j_{i-1}=1}^{n_{i-1}}\mathbb{E}_{C_{i}}\left[\sup_{s\leq t}\left|w_{i}\left(s,C_{i-1}\left(j_{i-1}\right),C_{i}\right)\right|^{50}\right]\bigg)^{1/50},\\
 & \max_{2\leq i\leq L}\bigg(\frac{1}{n_{i-1}n_{i}}\sum_{j_{i-1}=1}^{n_{i-1}}\sum_{j_{i}=1}^{n_{i}}\sup_{s\leq t}\left|w_{i}\left(s,C_{i-1}\left(j_{i-1}\right),C_{i}\left(j_{i}\right)\right)\right|^{50}\bigg)^{1/50},\\
 & \max_{2\leq i\leq L}\bigg(\frac{1}{n_{i}}\sum_{j_{i}=1}^{n_{i}}\sup_{s\leq t}\left|b_{i}\left(s,C_{i}\left(j_{i}\right)\right)\right|^{50}\bigg)^{1/50}\\
 & \bigg(\frac{1}{n_{1}}\sum_{j_{1}=1}^{n_{1}}\sup_{s\leq t}\left|w_{1}\left(s,C_{1}\left(j_{1}\right)\right)\right|^{50}\bigg)^{1/50}\Bigg),
\end{align*}
we also have:
\begin{align*}
\interleave W\interleave_{{\rm samp},t} & \leq K^{\kappa_{L}}\left(1+t^{\kappa_{L}}\right)\left(1+\max\left(\interleave W\interleave_{0}^{\kappa_{L}},\interleave W\interleave_{{\rm samp},0}^{\kappa_{L}}\right)\right),\\
\bigg(\frac{1}{n_{i}}\sum_{j_{i}=1}^{n_{i}}\sup_{s\leq t}\underset{Z\sim{\cal P}}{{\rm ess\text{-}sup}}\left|\Delta_{i}^{H}\left(Z,C_{i}\left(j_{i}\right);W\left(s\right)\right)\right|^{50}\bigg)^{1/50} & \leq K^{\kappa_{L}}\left(1+t^{\kappa_{L}}\right)\left(1+\max\left(\interleave W\interleave_{0}^{\kappa_{L}},\interleave W\interleave_{{\rm samp},0}^{\kappa_{L}}\right)\right).
\end{align*}
\end{lem}

\begin{proof}
The proof follows the same argument as Lemma \ref{lem:bounds MF a priori}.
This is obvious for the statements concerning $\tilde{W}$ and $\mathbf{W}$.
To prove the latter claims that involve $\interleave W\interleave_{{\rm samp},t}$,
the argument follows similarly. In particular, let us denote
\begin{align*}
\interleave w_{i}\interleave_{{\rm right},t} & =\frac{1}{n_{i}}\sum_{j_{i}=1}^{n_{i}}\mathbb{E}_{C_{i-1}}\left[\sup_{s\leq t}\left|w_{i}\left(s,C_{i-1},C_{i}\left(j_{i}\right)\right)\right|^{50}\right]\bigg)^{1/50},\qquad2\leq i\leq L,\\
\interleave w_{i}\interleave_{{\rm left},t} & =\bigg(\frac{1}{n_{i-1}}\sum_{j_{i-1}=1}^{n_{i-1}}\mathbb{E}_{C_{i}}\left[\sup_{s\leq t}\left|w_{i}\left(s,C_{i-1}\left(j_{i-1}\right),C_{i}\right)\right|^{50}\right]\bigg)^{1/50},\qquad2\leq i\leq L,\\
\interleave w_{i}\interleave_{{\rm cen},t} & =\bigg(\frac{1}{n_{i-1}n_{i}}\sum_{j_{i-1}=1}^{n_{i-1}}\sum_{j_{i}=1}^{n_{i}}\sup_{s\leq t}\left|w_{i}\left(s,C_{i-1}\left(j_{i-1}\right),C_{i}\left(j_{i}\right)\right)\right|^{50}\bigg)^{1/50},\qquad2\leq i\leq L,\\
\interleave b_{i}\interleave_{{\rm samp},t} & =\bigg(\frac{1}{n_{i}}\sum_{j_{i}=1}^{n_{i}}\sup_{s\leq t}\left|b_{i}\left(s,C_{i}\left(j_{i}\right)\right)\right|^{50}\bigg)^{1/50},\qquad2\leq i\leq L,\\
\interleave w_{1}\interleave_{{\rm samp},t} & =\bigg(\frac{1}{n_{1}}\sum_{j_{1}=1}^{n_{1}}\sup_{s\leq t}\left|w_{1}\left(s,C_{1}\left(j_{1}\right)\right)\right|^{50}\bigg)^{1/50},\\
\interleave\Delta_{i}^{H}\interleave_{{\rm samp},t} & =\bigg(\frac{1}{n_{i}}\sum_{j_{i}=1}^{n_{i}}\sup_{s\leq t}\underset{Z\sim{\cal P}}{{\rm ess\text{-}sup}}\left|\Delta_{i}^{H}\left(Z,C_{i}\left(j_{i}\right);W\left(s\right)\right)\right|^{50}\bigg)^{1/50},\qquad1\leq i\leq L,\\
\interleave\Delta_{i}^{H}\interleave_{t} & =\mathbb{E}\left[\frac{1}{n_{i}}\sum_{j_{i}=1}^{n_{i}}\sup_{s\leq t}\underset{Z\sim{\cal P}}{{\rm ess\text{-}sup}}\left|\Delta_{i}^{H}\left(Z,C_{i};W\left(s\right)\right)\right|^{50}\right]^{1/50},\qquad1\leq i\leq L.
\end{align*}
Then similar to Lemma \ref{lem:bounds MF a priori}, we obtain for
$2\leq i\leq L$:
\begin{align*}
\interleave\Delta_{L}^{H}\interleave_{{\rm samp},t} & \leq K,\\
\interleave\Delta_{i-1}^{H}\interleave_{{\rm samp},t} & \leq K\left(1+\interleave\Delta_{i}^{H}\interleave_{t}\right)\left(1+\interleave w_{i}\interleave_{{\rm left},t}+\interleave b_{i}\interleave_{t}\right),\\
\interleave w_{i}\interleave_{{\rm left},t} & \leq\interleave w_{i}\interleave_{{\rm left},0}+K\left(1+\interleave\Delta_{i}^{H}\interleave_{t}\right)t,\\
\interleave w_{i}\interleave_{{\rm right},t} & \leq\interleave w_{i}\interleave_{{\rm right},0}+K\left(1+\interleave\Delta_{i}^{H}\interleave_{{\rm samp},t}\right)t,\\
\interleave w_{i}\interleave_{{\rm cen},t} & \leq\interleave w_{i}\interleave_{{\rm cen},0}+K\left(1+\interleave\Delta_{i}^{H}\interleave_{{\rm samp},t}\right)t,\\
\interleave b_{i}\interleave_{{\rm samp},t} & \leq\interleave b_{i}\interleave_{{\rm samp},0}+K\left(1+\interleave\Delta_{i}^{H}\interleave_{{\rm samp},t}\right)t,\\
\interleave w_{1}\interleave_{{\rm samp},t} & \leq\interleave w_{1}\interleave_{{\rm samp},0}+K\left(1+\interleave\Delta_{1}^{H}\interleave_{{\rm samp},t}\right)t.
\end{align*}
Note that 
\[
\interleave W\interleave_{{\rm samp},t}=\max\left(\max_{2\leq i\leq L}\interleave w_{i}\interleave_{{\rm left},t},\;\max_{2\leq i\leq L}\interleave w_{i}\interleave_{{\rm right},t},\;\max_{2\leq i\leq L}\interleave w_{i}\interleave_{{\rm cen},t},\;\max_{2\leq i\leq L}\interleave b_{i}\interleave_{{\rm samp},t},\;\interleave w_{1}\interleave_{{\rm samp},t}\right).
\]
Together with the bound on $\interleave\Delta_{i}^{H}\interleave_{t}$
given by Lemma \ref{lem:bounds MF a priori}, one can derive the claims.
\end{proof}

\subsubsection{Proof of Proposition \ref{prop:particle coupling - bounded}}
\begin{proof}[Proof of Proposition \ref{prop:particle coupling - bounded}]
In the following, let $K_{t}$ denote an immaterial positive constant
that takes the form
\[
K_{t}=K^{\kappa_{L}}\left(1+t^{\kappa_{L}}\right),
\]
where $\kappa_{L}=K^{L}$, such that $K_{t}\geq1$ and $K_{t}\leq K_{T}$
for all $t\leq T$. We note that the terminal time $T$, the constant
$K_{t}$, as well as the usual immaterial constant $K$, do not depend
on $B$. We start with some preliminary facts:

\paragraph{Fact 1: moment bounds.}

We first note that at initialization, $\mathscr{D}_{0}\left(W,\tilde{W}\right)=0$
and $\interleave W\interleave_{0}\leq K$. By Assumption \ref{assump:init}
and Lemma \ref{lem:bounds MF a priori}, $\interleave W\interleave_{T}\leq K_{T}$.
Furthermore, by Lemma \ref{lem:initialization_compare}, with probability
at least $1-KLn_{\max}\exp\left(-Kn_{\min}^{1/52}\right)$, we have
$\interleave\tilde{W}\interleave_{0}\leq K$ and $\interleave W\interleave_{{\rm samp},0}\leq K$,
recalling the definition of $\interleave W\interleave_{{\rm samp},t}$
from the statement of Lemma \ref{lem:bounds NN a priori}. Let this
event be denoted by ${\cal E}$. Unless noticed otherwise, we shall
place most of the contexts of our proof upon ${\cal E}$. By Lemma
\ref{lem:bounds NN a priori}, one deduces that
\begin{align*}
\interleave\tilde{W}\interleave_{T},\;\interleave W\interleave_{{\rm samp},T} & \leq K_{T},\\
\bigg(\frac{1}{n_{i}}\sum_{j_{i}=1}^{n_{i}}\sup_{s\leq t}\underset{Z\sim{\cal P}}{{\rm ess\text{-}sup}}\left|\Delta_{i}^{\mathbf{H}}\left(Z,j_{i};\tilde{W}\left(s\right)\right)\right|^{50}\bigg)^{1/50} & \leq K_{T},\\
\bigg(\frac{1}{n_{i}}\sum_{j_{i}=1}^{n_{i}}\sup_{s\leq t}\underset{Z\sim{\cal P}}{{\rm ess\text{-}sup}}\left|\Delta_{i}^{H}\left(Z,C_{i}\left(j_{i}\right);W\left(s\right)\right)\right|^{50}\bigg)^{1/50} & \leq K_{T}
\end{align*}
on the event ${\cal E}$. We also remark that the fact $\interleave W\interleave_{T}\leq K_{T}$
holds irrespective of ${\cal E}$.

\paragraph*{Fact 2: maximal bounds for $W$.}

We note that the assumption ${\rm ess\text{-}sup}\mathsf{max}_{0}^{w}\left(W\right)$
and ${\rm ess\text{-}sup}\mathsf{max}_{0}^{b}\left(W\right)$ has
an interesting consequence:
\begin{align*}
{\rm ess\text{-}sup}\mathsf{max}_{T}^{w}\left(W\right) & \leq K_{T}\left(1+B\right),\\
{\rm ess\text{-}sup}\mathsf{max}_{T}^{b}\left(W\right) & \leq K_{T}\left(1+B\right),\\
{\rm ess\text{-}sup}\max_{1\leq i\leq L}\sup_{t\leq T}\left|\Delta_{i}^{H}\left(Z,C_{i};W\left(t\right)\right)\right| & \leq K_{T}\left(1+B\right).
\end{align*}
 We note that this claim holds irrespective of the event ${\cal E}$
from Fact 1. Following this claim, it is immediate that almost surely,
\begin{align*}
\underset{C_{i}}{{\rm ess\text{-}sup}}\max_{2\leq i\leq L}\max_{j_{i-1}\in\left[n_{i-1}\right]}\left|w_{i}\left(t,C_{i-1}\left(j_{i-1}\right),C_{i}\right)\right| & \leq K_{T}\left(1+B\right),\\
\underset{C_{i-1}}{{\rm ess\text{-}sup}}\max_{2\leq i\leq L}\max_{j_{i}\in\left[n_{i}\right]}\left|w_{i}\left(t,C_{i-1},C_{i}\left(j_{i}\right)\right)\right| & \leq K_{T}\left(1+B\right),\\
\max_{2\leq i\leq L}\max_{j_{i-1}\in\left[n_{i-1}\right],\;j_{i}\in\left[n_{i}\right]}\left|w_{i}\left(t,C_{i-1}\left(j_{i-1}\right),C_{i}\left(j_{i}\right)\right)\right| & \leq K_{T}\left(1+B\right),\\
\max_{2\leq i\leq L}\max_{j_{i}\in\left[n_{i}\right]}\left|b_{i}\left(t,C_{i}\left(j_{i}\right)\right)\right| & \leq K_{T}\left(1+B\right),\\
\underset{Z\sim{\cal P}}{{\rm ess\text{-}sup}}\max_{1\leq i\leq L}\max_{j_{i}\in\left[n_{i}\right]}\sup_{t\leq T}\left|\Delta_{i}^{H}\left(Z,C_{i}\left(j_{i}\right);W\left(t\right)\right)\right| & \leq K_{T}\left(1+B\right),
\end{align*}
since $C_{i}\left(j_{i}\right)$ is a copy of $C_{i}$. Now we prove
the claim. First consider $\mathsf{max}_{t}^{w}\left(W\right)$. By
Assumption \ref{enu:Assump_backward}, for ${\cal P}$-almost every
$z$,
\[
\sup_{t\geq0}\sup_{c_{L-1}\in\Omega_{L-1}}\left|\Delta_{L}^{w}\left(z,c_{L-1},1;W\left(t\right)\right)\right|\leq K\left(1+\sup_{t\geq0}\left|\Delta_{L}^{H}\left(z,1;W\left(t\right)\right)\right|\right)\leq K,
\]
which implies, by Assumption \ref{enu:Assump_lrSchedule}, that
\[
{\rm ess\text{-}sup}\sup_{t\leq T}\left|w_{L}\left(t,C_{L-1},1\right)\right|\leq{\rm ess\text{-}sup}\left|w_{L}^{0}\left(C_{L-1},1\right)\right|+KT\leq K_{T}\left(1+B\right).
\]
Next assuming that ${\rm ess\text{-}sup}\sup_{t\leq T}\left|w_{i}\left(t,C_{i-1},C_{i}\right)\right|\leq K_{T}\left(1+B\right)$
for a given $i\geq2$, by Assumption \ref{enu:Assump_backward}, we
have for ${\cal P}$-almost every $z$ and all $t\leq T$,
\begin{align*}
\left|\Delta_{i-1}^{H}\left(z,C_{i-1};W\left(t\right)\right)\right| & \leq K\mathbb{E}_{C_{i}}\left[\left(1+\left|\Delta_{i}^{H}\left(z,C_{i};W\left(t\right)\right)\right|\right)\left(1+\left|w_{i}\left(t,C_{i-1},C_{i}\right)\right|+\left|b_{i}\left(t,C_{i}\right)\right|\right)\right]\\
 & \leq K\mathbb{E}_{C_{i}}\left[\left(1+\left|\Delta_{i}^{H}\left(z,C_{i};W\left(t\right)\right)\right|\right)\left(K_{T}\left(1+B\right)+\left|b_{i}\left(t,C_{i}\right)\right|\right)\right]\\
 & \leq K\left[\left(1+\mathbb{E}\left[\left|\Delta_{i}^{H}\left(z,C_{i};W\left(t\right)\right)\right|^{2}\right]^{1/2}\right)\left(K_{T}\left(1+B\right)+\mathbb{E}\left[\left|b_{i}\left(t,C_{i}\right)\right|^{2}\right]^{1/2}\right)\right]\\
 & \leq K_{T}\left(1+B\right),
\end{align*}
where the last step follows from the fact $\interleave W\interleave_{T}\leq K_{T}$
and Lemma \ref{lem:bounds MF a priori}. Again by Assumption \ref{enu:Assump_backward},
we then obtain:
\[
\left|\Delta_{i-1}^{w}\left(z,C_{i-1},C_{i};W\left(t\right)\right)\right|\leq K_{T}\left(1+B\right),
\]
which implies, by Assumption \ref{enu:Assump_lrSchedule}, that
\[
{\rm ess\text{-}sup}\sup_{t\leq T}\left|w_{i-1}\left(t,C_{i-1},C_{i}\right)\right|\leq{\rm ess\text{-}sup}\left|w_{i-1}^{0}\left(C_{i-1},C_{i}\right)\right|+K_{T}\left(1+B\right)T\leq K_{T}\left(1+B\right).
\]
This completes the induction argument to show that ${\rm ess\text{-}sup}\mathsf{max}_{t}^{w}\left(W\right)\leq K_{T}\left(1+B\right)$.
We have also showed that
\begin{align*}
{\rm ess\text{-}sup}\max_{1\leq i\leq L}\sup_{t\leq T}\left|\Delta_{i}^{H}\left(Z,C_{i};W\left(t\right)\right)\right| & \leq K_{T}\left(1+B\right).
\end{align*}
We thus obtain from Assumption \ref{enu:Assump_backward}:
\[
\left|\Delta_{i}^{b}\left(z,C_{i};W\left(t\right)\right)\right|\leq K_{T}\left(1+B\right),
\]
for $2\leq i\leq L$ and ${\cal P}$-almost every $z$. This implies:
\[
{\rm ess\text{-}sup}\sup_{t\leq T}\left|b_{i}\left(t,C_{i}\right)\right|\leq{\rm ess\text{-}sup}\left|b_{i}^{0}\left(C_{i}\right)\right|+K_{T}\left(1+B\right)T\leq K_{T}\left(1+B\right),
\]
which shows ${\rm ess\text{-}sup}\mathsf{max}_{t}^{b}\left(W\right)\leq K_{T}\left(1+B\right)$,
as claimed.

\paragraph*{Fact 3: maximal bounds for $\tilde{W}$.}

We also have on the event ${\cal E}$, almost surely,
\begin{align*}
\max_{2\leq i\leq L}\max_{j_{i-1}\in\left[n_{i-1}\right],\;j_{i}\in\left[n_{i}\right]}\sup_{t\leq T}\left|\tilde{w}_{i}\left(t,j_{i-1},j_{i}\right)\right| & \leq K_{T}\left(1+B\right),\\
\max_{2\leq i\leq L}\max_{j_{i}\in\left[n_{i}\right]}\sup_{t\leq T}\left|\tilde{b}_{i}\left(t,j_{i}\right)\right| & \leq K_{T}\left(1+B\right).
\end{align*}
A proof of this fact is similar to the argument for Fact 2. We note
that this argument requires the use of the fact $\interleave\tilde{W}\interleave_{T}\leq K_{T}$,
which holds on the event ${\cal E}$, and the application of Lemma
\ref{lem:bounds MF a priori}. The latter application holds by noticing
that $\tilde{W}$ can be viewed as a collection of MF parameter whose
neuronal ensemble $\left(\Omega_{{\rm new}},P_{{\rm new}}\right)=\prod_{i=1}^{L}\left(\Omega_{i,{\rm new}},P_{i,{\rm new}}\right)$
takes the following specific form: $\Omega_{i,{\rm new}}=\left\{ C_{i}\left(1\right),...,C_{i}\left(n_{i}\right)\right\} $
and $P_{i,{\rm new}}$ is a uniform probability measure on $\Omega_{i,{\rm new}}$.

We now decompose the proof into several steps.

\paragraph*{Step 1 - Main proof.}

Let us first define some quantities that represent the difference
between $W$ and $\tilde{W}$:\footnote{To simplify our notation, here and in the following argument, we denote
by $\partial_{1}$ the partial derivative with respect to the first
variable, so for example, $\partial_{1}w_{i}(t,C_{i-1}(j_{i-1}),C_{i}(j_{i}))=\frac{\partial}{\partial t}w_{i}(t,C_{i-1}(j_{i-1}),C_{i}(j_{i}))$.}
\begin{align*}
D_{i}^{w}\left(t\right) & =\bigg(\frac{1}{n_{i-1}n_{i}}\sum_{j_{i-1}=1}^{n_{i-1}}\sum_{j_{i}=1}^{n_{i}}\left|\partial_{1}\tilde{w}_{i}\left(t,j_{i-1},j_{i}\right)-\partial_{1}w_{i}\left(t,C_{i-1}\left(j_{i-1}\right),C_{i}\left(j_{i}\right)\right)\right|^{2}\bigg)^{1/2},\quad2\leq i\leq L,\\
D_{i}^{b}\left(t\right) & =\bigg(\frac{1}{n_{i}}\sum_{j_{i}=1}^{n_{i}}\left|\partial_{1}\tilde{b}_{i}\left(t,j_{i}\right)-\partial_{1}b_{i}\left(t,C_{i}\left(j_{i}\right)\right)\right|^{2}\bigg)^{1/2},\quad2\leq i\leq L,\\
D_{1}^{w}\left(t\right) & =\bigg(\frac{1}{n_{1}}\sum_{j_{1}=1}^{n_{1}}\left|\partial_{1}\tilde{w}_{1}\left(t,j_{1}\right)-\partial_{1}w_{1}\left(t,C_{1}\left(j_{1}\right)\right)\right|^{2}\bigg)^{1/2}.
\end{align*}
We are also interested in the following quantities that represent
the smoothness in the time evolution of $W\left(t\right)$ and $\tilde{W}\left(t\right)$:
\begin{align*}
A_{i}^{w}\left(t,\zeta\right) & =\bigg(\frac{1}{n_{i-1}n_{i}}\sum_{j_{i-1}=1}^{n_{i-1}}\sum_{j_{i}=1}^{n_{i}}\left|\partial_{1}w_{i}\left(t+\zeta,C_{i-1}\left(j_{i-1}\right),C_{i}\left(j_{i}\right)\right)-\partial_{1}w_{i}\left(t,C_{i-1}\left(j_{i-1}\right),C_{i}\left(j_{i}\right)\right)\right|^{2}\bigg)^{1/2},\quad2\leq i\leq L,\\
\tilde{A}_{i}^{w}\left(t,\zeta\right) & =\bigg(\frac{1}{n_{i-1}n_{i}}\sum_{j_{i-1}=1}^{n_{i-1}}\sum_{j_{i}=1}^{n_{i}}\left|\partial_{1}\tilde{w}_{i}\left(t+\zeta,j_{i-1},j_{i}\right)-\partial_{1}\tilde{w}_{i}\left(t,j_{i-1},j_{i}\right)\right|^{2}\bigg)^{1/2},\quad2\leq i\leq L,\\
A_{i}^{b}\left(t,\zeta\right) & =\bigg(\frac{1}{n_{i}}\sum_{j_{i}=1}^{n_{i}}\left|\partial_{1}b_{i}\left(t+\zeta,C_{i}\left(j_{i}\right)\right)-\partial_{1}b_{i}\left(t,C_{i}\left(j_{i}\right)\right)\right|^{2}\bigg)^{1/2},\quad2\leq i\leq L,\\
\tilde{A}_{i}^{b}\left(t,\zeta\right) & =\bigg(\frac{1}{n_{i}}\sum_{j_{i}=1}^{n_{i}}\left|\partial_{1}\tilde{b}_{i}\left(t+\zeta,j_{i}\right)-\partial_{1}\tilde{b}_{i}\left(t,j_{i}\right)\right|^{2}\bigg)^{1/2},\quad2\leq i\leq L,\\
A_{1}^{w}\left(t,\zeta\right) & =\bigg(\frac{1}{n_{1}}\sum_{j_{1}=1}^{n_{1}}\left|\partial_{1}w_{1}\left(t+\zeta,C_{1}\left(j_{1}\right)\right)-\partial_{1}w_{1}\left(t,C_{1}\left(j_{1}\right)\right)\right|^{2}\bigg)^{1/2},\\
\tilde{A}_{1}^{w}\left(t,\zeta\right) & =\bigg(\frac{1}{n_{1}}\sum_{j_{1}=1}^{n_{1}}\left|\partial_{1}\tilde{w}_{1}\left(t+\zeta,j_{1}\right)-\partial_{1}\tilde{w}_{1}\left(t,j_{1}\right)\right|^{2}\bigg)^{1/2}.
\end{align*}
These quantities give a bound on $\mathscr{D}_{t}\left(W,\tilde{W}\right)$:
\begin{align*}
\mathscr{D}_{t}\left(W,\tilde{W}\right) & \leq K\int_{0}^{t}\max_{1\leq i\leq L}D_{i}^{w}\left(\left\lfloor s/\zeta\right\rfloor \zeta\right)ds+K\int_{0}^{t}\max_{2\leq i\leq L}D_{i}^{b}\left(\left\lfloor s/\zeta\right\rfloor \zeta\right)ds\\
 & \quad+Kt\sup_{s\leq T-\zeta}\sup_{0\leq\zeta'\leq\zeta}\max_{i}\left(A_{i}^{w}\left(s,\zeta'\right),A_{i}^{b}\left(s,\zeta'\right),\tilde{A}_{i}^{w}\left(s,\zeta'\right),\tilde{A}_{i}^{b}\left(s,\zeta'\right)\right),
\end{align*}
where we have used the fact $\mathscr{D}_{0}\left(W,\tilde{W}\right)=0$.
The next task is to bound the terms inside the integral.

To find bounds on $D_{i}^{w}\left(t\right)$, we introduce the quantities
for $1\leq i\leq L$:
\begin{align*}
G_{i}\left(t\right) & =\bigg(\frac{1}{n_{i}}\sum_{j_{i}=1}^{n_{i}}\mathbb{E}_{Z}\left[\left|\Delta_{i}^{{\bf H}}\left(Z,j_{i};\tilde{W}\left(t\right)\right)-\Delta_{i}^{H}\left(Z,C_{i}\left(j_{i}\right);W\left(t\right)\right)\right|\right]^{2}\bigg)^{1/2},\\
F_{i}\left(t\right) & =\bigg(\frac{1}{n_{i}}\sum_{j_{i}=1}^{n_{i}}\mathbb{E}_{Z}\left[\left|\mathbf{H}_{i}\left(X,j_{i};\tilde{W}\left(t\right)\right)-H_{i}\left(X,C_{i}\left(j_{i}\right);W\left(t\right)\right)\right|\right]^{2}\bigg)^{1/2}.
\end{align*}
We specify their connection in the following. By Assumptions \ref{enu:Assump_lrSchedule}
and \ref{enu:Assump_backward}, for $i\geq2$,
\[
D_{i}^{w}\left(t\right)\leq K\left(D_{i}^{w,1}\left(t\right)+G_{i}\left(t\right)+\mathscr{D}_{t}\left(W,\tilde{W}\right)+F_{i}\left(t\right)\right),
\]
in which
\begin{align*}
D_{i}^{w,1}\left(t\right) & =\bigg(\frac{1}{n_{i}}\sum_{j_{i}=1}^{n_{i}}\underset{Z\sim{\cal P}}{{\rm ess\text{-}sup}}\left(1+\left|\Delta_{i}^{\mathbf{H}}\left(Z,j_{i};\tilde{W}\left(t\right)\right)\right|^{2}+\left|\Delta_{i}^{H}\left(Z,C_{i}\left(j_{i}\right);W\left(t\right)\right)\right|^{2}\right)\\
 & \qquad\times\frac{1}{n_{i-1}}\sum_{j_{i-1}=1}^{n_{i-1}}\mathbb{E}_{Z}\left[\left|\mathbf{H}_{i-1}\left(X,j_{i-1};\tilde{W}\left(t\right)\right)-H_{i-1}\left(X,C_{i-1}\left(j_{i-1}\right);W\left(t\right)\right)\right|\right]^{2}\bigg)^{1/2}.
\end{align*}
By Lemma \ref{lem:bounds NN a priori}, on the event ${\cal E}$,
\[
D_{i}^{w,1}\left(t\right)\leq K_{T}F_{i-1}\left(t\right).
\]
As such, on the event ${\cal E}$,
\[
D_{i}^{w}\left(t\right)\leq K_{T}\left(F_{i-1}\left(t\right)+G_{i}\left(t\right)+\mathscr{D}_{t}\left(W,\tilde{W}\right)+F_{i}\left(t\right)\right).
\]
Similarly, we also have:
\begin{align*}
D_{1}^{w}\left(t\right) & \leq K\left(G_{1}\left(t\right)+\mathscr{D}_{t}\left(W,\tilde{W}\right)\right).
\end{align*}
Together with the previously derived bound on $\mathscr{D}_{t}\left(W,\tilde{W}\right)$,
we obtain on the event ${\cal E}$:
\begin{align*}
\mathscr{D}_{t}\left(W,\tilde{W}\right) & \leq K_{T}\int_{0}^{t}\mathscr{D}_{s}\left(W,\tilde{W}\right)ds+K_{T}\int_{0}^{t}\max_{i}\left(G_{i}\left(\left\lfloor s/\zeta\right\rfloor \zeta\right)+F_{i}\left(\left\lfloor s/\zeta\right\rfloor \zeta\right)\right)ds\\
 & \quad+K\int_{0}^{t}\max_{2\leq i\leq L}D_{i}^{b}\left(\left\lfloor s/\zeta\right\rfloor \zeta\right)ds\\
 & \quad+Kt\sup_{s\leq T-\zeta}\sup_{0\leq\zeta'\leq\zeta}\max_{i}\left(A_{i}^{w}\left(s,\zeta'\right),A_{i}^{b}\left(s,\zeta'\right),\tilde{A}_{i}^{w}\left(s,\zeta'\right),\tilde{A}_{i}^{b}\left(s,\zeta'\right)\right),
\end{align*}
which holds for all $t\leq T$.

Next we make the following claims:
\begin{itemize}
\item \textbf{Claim 1:} For any $\zeta\in\left[0,T\right]$, on the event
${\cal E}$, almost surely,
\[
\sup_{t\leq T-\zeta}\sup_{0\leq\zeta'\leq\zeta}\max_{i}\left(A_{i}^{w}\left(t,\zeta'\right),A_{i}^{b}\left(t,\zeta'\right),\tilde{A}_{i}^{w}\left(t,\zeta'\right),\tilde{A}_{i}^{b}\left(t,\zeta'\right)\right)\leq K_{T}\left(1+B\right)\zeta.
\]
\item \textbf{Claim 2:} For a sequence $\left\{ \gamma_{j}>0,\;j=2,...,L\right\} $
and $t\leq T$, let ${\cal E}_{t,i}^{\mathbf{H}}$ denote the event
in which for all $k\in\left\{ 1,2,...,i\right\} $, 
\begin{align*}
F_{k}\left(t\right) & \leq K_{T}^{k}\bigg(\mathscr{D}_{t}\left(W,\tilde{W}\right)+\left(1+B\right)\sum_{j=1}^{k-1}\gamma_{j+1}\bigg).
\end{align*}
(The summation $\sum_{j=1}^{k-1}$ equals $0$ if $k=1$.) We claim
that for each $i=1,...,L$:
\[
\mathbb{P}\left({\cal E}_{t,i}^{\mathbf{H}};{\cal E}\right)\geq1-\sum_{j=1}^{i-1}\frac{n_{j+1}}{\gamma_{j+1}}\exp\left(-n_{j}\gamma_{j+1}^{2}/K_{T}\right).
\]
\item \textbf{Claim 3:} For a sequence $\left\{ \beta_{j}>0,\;j=1,...,L-2\right\} $
and $t\leq T$, let ${\cal E}_{t,i}^{\Delta}$ denote the event that
for all $k\in\left\{ i,i+1,...,L\right\} $, 
\[
G_{k}\left(t\right)\leq K_{T}^{2L-k+1}\bigg(\left(1+B\right)\mathscr{D}_{t}\left(W,\tilde{W}\right)+\left(1+B^{2}\right)\bigg(\delta_{L}^{\Delta}+\sum_{j=k}^{L-2}\beta_{j}\bigg)\bigg),
\]
where $\delta_{L}^{\Delta}=\sum_{j=1}^{L-1}\gamma_{j+1}$. (The summation
$\sum_{j=k}^{L-2}$ equals $0$ if $k\geq L-1$.) We claim that for
each $i=1,...,L$: 
\[
\mathbb{P}\left({\cal E}_{t,L}^{\mathbf{H}}\cap{\cal E}_{t,i}^{\Delta};{\cal E}\right)\geq\mathbb{P}\left({\cal E}_{t,L}^{\mathbf{H}};{\cal E}\right)-\sum_{j=i}^{L-2}\frac{n_{j}}{\beta_{j}}\exp\left(-n_{j+1}\beta_{j}^{2}/K_{T}\right).
\]
\item \textbf{Claim 4:} For $t\leq T$, let ${\cal E}_{t}^{b}$ denote the
event that for all $k\in\left\{ 2,...,L\right\} $,
\[
D_{k}^{b}\left(t\right)\leq K_{T}\left(\left(1+B\right)\mathscr{D}_{t}\left(W,\tilde{W}\right)+\left(1+B^{2}\right)\delta_{L}^{b}\right),
\]
where $\delta_{L}^{b}=\delta_{L}^{\Delta}+\sum_{j=1}^{L-2}\beta_{j}$.
We claim that 
\[
\mathbb{P}\left({\cal E}_{t,L}^{\mathbf{H}}\cap{\cal E}_{t,1}^{\Delta}\cap{\cal E}_{t}^{b};{\cal E}\right)\geq\mathbb{P}\left({\cal E}_{t,L}^{\mathbf{H}}\cap{\cal E}_{t,1}^{\Delta};{\cal E}\right)-\sum_{j=1}^{i-1}\frac{n_{j+1}}{\gamma_{j+1}}\exp\left(-n_{j}\gamma_{j+1}^{2}/K_{T}\right).
\]
\end{itemize}
Let us assume these claims. Using the bounds on $\mathbb{P}\left({\cal E}_{t,L}^{\mathbf{H}};{\cal E}\right)$,
$\mathbb{P}\left({\cal E}_{t,L}^{\mathbf{H}}\cap{\cal E}_{t,1}^{\Delta};{\cal E}\right)$
and $\mathbb{P}\left({\cal E}_{t,L}^{\mathbf{H}}\cap{\cal E}_{t,1}^{\Delta}\cap{\cal E}_{t}^{b};{\cal E}\right)$,
combining the previous bound, applying the union bound over $t\in\left\{ 0,\zeta,2\zeta,...,\left\lfloor T/\zeta\right\rfloor \zeta\right\} $
and recalling $\mathscr{D}_{0}\left(W,\tilde{W}\right)=0$, we then
get: 
\[
\mathscr{D}_{t}\left(W,\tilde{W}\right)\leq K_{T}\int_{0}^{t}\left[\left(1+B\right)\mathscr{D}_{s}\left(W,\tilde{W}\right)+\left(1+B^{2}\right)\bigg(\sum_{j=1}^{L-1}\gamma_{j+1}+\sum_{j=1}^{L-2}\beta_{j}\bigg)+\left(1+B\right)\zeta\right]ds,
\]
for all $t\leq T$, with probability at least 
\[
1-\frac{T}{\zeta}\left(\sum_{j=1}^{L-1}\frac{n_{j+1}}{\gamma_{j+1}}\exp\left(-n_{j}\gamma_{j+1}^{2}/K_{T}\right)+\sum_{j=1}^{L-2}\frac{n_{j}}{\beta_{j}}\exp\left(-n_{j+1}\beta_{j}^{2}/K_{T}\right)\right)-KLn_{\max}\exp\left(-Kn_{\min}^{1/52}\right),
\]
for any $\zeta\in\left[0,T\right]$. By Gronwall's lemma, the above
implies that for all $t\leq T$,
\begin{align*}
\mathscr{D}_{t}\left(W,\tilde{W}\right) & \leq K_{T}\left[\left(1+B^{2}\right)\bigg(\sum_{j=1}^{L-1}\gamma_{j+1}+\sum_{j=1}^{L-2}\beta_{j}\bigg)+\left(1+B\right)\zeta\right]\exp\left(K_{T}\left(1+B\right)T\right)\\
 & \leq K_{T}\bigg(\sum_{j=1}^{L-1}\gamma_{j+1}+\sum_{j=1}^{L-2}\beta_{j}+\zeta\bigg)\exp\left(K_{T}\left(1+B\right)\right).
\end{align*}
The proposition statement is then easily obtained by choosing 
\begin{align*}
\gamma_{j+1} & =\sqrt{\frac{1}{K_{T}n_{j}}\log\left(\frac{2TLn_{\max}^{2}}{\delta}+e\right)},\quad j=1,...,L-1,\\
\beta_{j} & =\sqrt{\frac{1}{K_{T}n_{j+1}}\log\left(\frac{2TLn_{\max}^{2}}{\delta}+e\right)},\quad j=1,...,L-2,\\
\zeta & =1/\sqrt{n_{\max}}.
\end{align*}
We are left with verifying the claims.

\paragraph*{Step 2 - Claim 1.}

We first note that by Assumptions \ref{enu:Assump_lrSchedule} and
\ref{enu:Assump_backward}, Lemma \ref{lem:bounds NN a priori}, and
the fact $\interleave W\interleave_{0},\;\interleave W\interleave_{{\rm samp},0}\leq K$
on the event ${\cal E}$:
\begin{align*}
 & \bigg(\frac{1}{n_{i-1}n_{i}}\sum_{j_{i-1}=1}^{n_{i-1}}\sum_{j_{i}=1}^{n_{i}}\sup_{s\leq t}\left|\partial_{1}w_{i}\left(s,C_{i-1}\left(j_{i-1}\right),C_{i}\left(j_{i}\right)\right)\right|^{50}\bigg)^{1/50}\\
 & \leq K+K\bigg(\frac{1}{n_{i}}\sum_{j_{i}=1}^{n_{i}}\sup_{s\leq t}\underset{Z\sim{\cal P}}{{\rm ess\text{-}sup}}\left|\Delta_{i}^{H}\left(Z,C_{i}\left(j_{i}\right);W\left(s\right)\right)\right|^{50}\bigg)^{1/50}\leq K_{T},
\end{align*}
for any $t\leq T$. Therefore,
\[
\bigg(\frac{1}{n_{i-1}n_{i}}\sum_{j_{i-1}=1}^{n_{i-1}}\sum_{j_{i}=1}^{n_{i}}\sup_{s\leq T-\zeta}\sup_{0\leq\zeta'\leq\zeta}\left|w_{i}\left(s+\zeta',C_{i-1}\left(j_{i-1}\right),C_{i}\left(j_{i}\right)\right)-w_{i}\left(s,C_{i-1}\left(j_{i-1}\right),C_{i}\left(j_{i}\right)\right)\right|^{2}\bigg)^{1/2}\leq K_{T}\zeta.
\]
We also have similarly that on the event ${\cal E}$,
\begin{align*}
\bigg(\frac{1}{n_{i}}\sum_{j_{i}=1}^{n_{i}}\mathbb{E}_{C_{i-1}}\left[\sup_{s\leq T-\zeta}\sup_{0\leq\zeta'\leq\zeta}\left|w_{i}\left(s+\zeta',C_{i-1},C_{i}\left(j_{i}\right)\right)-w_{i}\left(s,C_{i-1},C_{i}\left(j_{i}\right)\right)\right|^{2}\right]\bigg)^{1/2} & \leq K_{T}\zeta,\\
\bigg(\frac{1}{n_{i-1}}\sum_{j_{i-1}=1}^{n_{i-1}}\mathbb{E}_{C_{i}}\left[\sup_{s\leq T-\zeta}\sup_{0\leq\zeta'\leq\zeta}\left|w_{i}\left(s+\zeta',C_{i-1}\left(j_{i-1}\right),C_{i}\right)-w_{i}\left(s,C_{i-1}\left(j_{i-1}\right),C_{i}\right)\right|^{2}\right]\bigg)^{1/2} & \leq K_{T}\zeta,\\
\mathbb{E}\left[\sup_{s\leq T-\zeta}\sup_{0\leq\zeta'\leq\zeta}\left|w_{i}\left(s+\zeta',C_{i-1},C_{i}\right)-w_{i}\left(s,C_{i-1},C_{i}\right)\right|^{2}\right]^{1/2} & \leq K_{T}\zeta,\\
\bigg(\frac{1}{n_{i}}\sum_{j_{i}=1}^{n_{i}}\sup_{s\leq T-\zeta}\sup_{0\leq\zeta'\leq\zeta}\left|b_{i}\left(s+\zeta',C_{i}\left(j_{i}\right)\right)-b_{i}\left(s,C_{i}\left(j_{i}\right)\right)\right|^{2}\bigg)^{1/2} & \leq K_{T}\zeta,\\
\mathbb{E}\left[\sup_{s\leq T-\zeta}\sup_{0\leq\zeta'\leq\zeta}\left|b_{i}\left(s+\zeta',C_{i}\right)-b_{i}\left(s,C_{i}\right)\right|^{2}\right]^{1/2} & \leq K_{T}\zeta,\\
\bigg(\frac{1}{n_{1}}\sum_{j_{1}=1}^{n_{1}}\sup_{s\leq t}\sup_{0\leq\zeta'\leq\zeta}\left|w_{1}\left(s+\zeta',C_{1}\left(j_{1}\right)\right)-w_{1}\left(s,C_{1}\left(j_{1}\right)\right)\right|^{2}\bigg)^{1/2} & \leq K_{T}\zeta,\\
\mathbb{E}\left[\sup_{s\leq t}\sup_{0\leq\zeta'\leq\zeta}\left|w_{1}\left(s+\zeta',C_{1}\right)-w_{1}\left(s,C_{1}\right)\right|^{2}\right]^{1/2} & \leq K_{T}\zeta.
\end{align*}
Together with Lemma \ref{lem:Lipschitz backward MF - general}, this
fact gives us a bound on $A_{i}^{w}\left(t,\zeta\right)$. In particular,
defining $W_{\zeta}\left(t\right)=W\left(t+\zeta\right)$, we apply
Lemma \ref{lem:Lipschitz backward MF - general} to the two MF parameter
collections $W$ and $W_{\zeta}$ along with the new random variable
$\tilde{C}_{i}$ that is drawn uniformly from the set $\left\{ C_{i}\left(1\right),...,C_{i}\left(n_{i}\right)\right\} $.
Recalling the metric $\tilde{d}_{t}\left(W,W_{\zeta}\right)$ in this
lemma, the above fact shows that $\tilde{d}_{T-\zeta}\left(W,W_{\zeta}\right)\leq K_{T}\zeta$
on the event ${\cal E}$. The lemma holds owing to Fact 1 and Fact
2. The conclusion of the lemma then reads as
\[
\sup_{t\leq T-\zeta}\sup_{0\leq\zeta'\leq\zeta}\max_{i}\left(A_{i}^{w}\left(t,\zeta'\right),A_{i}^{b}\left(t,\zeta'\right)\right)\leq K_{T}\left(1+B\right)\tilde{d}_{T-\zeta}\left(W,W_{\zeta}\right)\leq K_{T}\left(1+B\right)\zeta,
\]
almost surely on the event ${\cal E}$.

By a similar argument, we have almost surely on the event ${\cal E}$:
\[
\sup_{t\leq T-\zeta}\sup_{0\leq\zeta'\leq\zeta}\max_{i}\left(\tilde{A}_{i}^{w}\left(t,\zeta'\right),\tilde{A}_{i}^{b}\left(t,\zeta'\right)\right)\leq K_{T}\left(1+B\right)\zeta.
\]
Indeed one can repeat the argument here by noticing that $\tilde{W}$
can be viewed as a collection of MF parameter whose neuronal ensemble
$\left(\Omega_{{\rm new}},P_{{\rm new}}\right)=\prod_{i=1}^{L}\left(\Omega_{i,{\rm new}},P_{i,{\rm new}}\right)$
takes the following specific form: $\Omega_{i,{\rm new}}=\left\{ C_{i}\left(1\right),...,C_{i}\left(n_{i}\right)\right\} $
and $P_{i,{\rm new}}$ is a uniform probability measure on $\Omega_{i,{\rm new}}$.

\paragraph*{Step 3 - Claim 2.}

We show the claim by induction. Consider $F_{1}$:
\begin{align*}
\left|\mathbf{H}_{1}\left(x,j_{1};\tilde{W}\left(t\right)\right)-H_{1}\left(x,C_{i}\left(j_{i}\right);W\left(t\right)\right)\right| & =\left|\phi_{1}\left(\tilde{w}_{1}\left(t,j_{1}\right),x\right)-\phi_{1}\left(w_{1}\left(t,C_{1}\left(j_{1}\right)\right),x\right)\right|\\
 & \leq K\mathscr{D}_{t}\left(W,\tilde{W}\right)
\end{align*}
for ${\cal P}$-almost every $x$ by Assumption \ref{enu:Assump_forward},
and therefore,
\[
F_{1}\left(t\right)\leq K\mathscr{D}_{t}\left(W,\tilde{W}\right).
\]
That is, $\mathbb{P}\left({\cal E}_{t,1}^{\mathbf{H}}\right)=1$.

Now let us assume the claim for $F_{i-1}$ with $i\geq2$ and consider
the claim for $F_{i}$. We have the following decomposition:
\begin{align*}
 & \left|{\bf H}_{i}\left(X,j_{i};\tilde{W}\left(t\right)\right)-H_{i}\left(X,C_{i}\left(j_{i}\right);W\left(t\right)\right)\right|\\
 & =\Bigg|\frac{1}{n_{i-1}}\sum_{j_{i-1}=1}^{n_{i-1}}\phi_{i}\left(\tilde{w}_{i}\left(t,j_{i-1},j_{i}\right),\tilde{b}_{i}\left(t,j_{i}\right),{\bf H}_{i-1}\left(X,j_{i-1};\tilde{W}\left(t\right)\right)\right)\\
 & \qquad-\mathbb{E}_{C_{i-1}}\left[\phi_{i}\left(w_{i}\left(t,C_{i-1},C_{i}\left(j_{i}\right)\right),b_{i}\left(t,C_{i}\left(j_{i}\right)\right),H_{i-1}\left(X,C_{i-1};W\left(t\right)\right)\right)\right]\Bigg|\\
 & \leq Q_{1,i}\left(t\right)+Q_{2,i}\left(t\right),
\end{align*}
which gives
\[
F_{i}\left(t\right)\leq\bigg(\frac{1}{n_{i}}\sum_{j_{i}=1}^{n_{i}}\mathbb{E}_{Z}\left[\left|Q_{1,i}\left(s\right)\right|+\left|Q_{2,i}\left(s\right)\right|\right]^{2}\bigg)^{1/2},
\]
where we define 
\begin{align*}
Q_{1,i}\left(t\right) & =\frac{1}{n_{i-1}}\sum_{j_{i-1}=1}^{n_{i-1}}\Big|\phi_{i}\left(\tilde{w}_{i}\left(t,j_{i-1},j_{i}\right),\tilde{b}_{i}\left(t,j_{i}\right),{\bf H}_{i-1}\left(X,j_{i-1};\tilde{W}\left(t\right)\right)\right)\\
 & \qquad-\phi_{i}\left(w_{i}\left(t,C_{i-1}\left(j_{i-1}\right),C_{i}\left(j_{i}\right)\right),b_{i}\left(t,C_{i}\left(j_{i}\right)\right),H_{i-1}\left(X,C_{i-1}\left(j_{i-1}\right);W\left(t\right)\right)\right)\Big|,\\
Q_{2,i}\left(t\right) & =\Bigg|\frac{1}{n_{i-1}}\sum_{j_{i-1}=1}^{n_{i-1}}\phi_{i}\left(w_{i}\left(t,C_{i-1}\left(j_{i-1}\right),C_{i}\left(j_{i}\right)\right),b_{i}\left(t,C_{i}\left(j_{i}\right)\right),H_{i-1}\left(X,C_{i-1}\left(j_{i-1}\right);W\left(t\right)\right)\right)\\
 & \qquad-\mathbb{E}_{C_{i-1}}\left[\phi_{i}\left(w_{i}\left(t,C_{i-1},C_{i}\left(j_{i}\right)\right),b_{i}\left(t,C_{i}\left(j_{i}\right)\right),H_{i-1}\left(X,C_{i-1};W\left(t\right)\right)\right)\right]\Bigg|.
\end{align*}
By Assumption \ref{enu:Assump_forward} and Cauchy-Schwarz's inequality,
we obtain a bound on $Q_{1,i}$:
\begin{align*}
\mathbb{E}_{Z}\left[\left|Q_{1,i}\left(t\right)\right|\right]^{2} & \leq\frac{K}{n_{i-1}}\sum_{j_{i-1}=1}^{n_{i-1}}\left(1+\left|\tilde{w}_{i}\left(t,j_{i-1},j_{i}\right)\right|^{2}+\left|w_{i}\left(t,C_{i-1}\left(j_{i-1}\right),C_{i}\left(j_{i}\right)\right)\right|^{2}+\left|\tilde{b}_{i}\left(t,j_{i}\right)\right|^{2}+\left|b_{i}\left(t,C_{i}\left(j_{i}\right)\right)\right|^{2}\right)\\
 & \quad\qquad\times\frac{1}{n_{i-1}}\sum_{j_{i-1}=1}^{n_{i-1}}\mathbb{E}_{Z}\left[\left|{\bf H}_{i-1}\left(X,j_{i-1};\tilde{W}\left(t\right)\right)-H_{i-1}\left(X,C_{i-1}\left(j_{i-1}\right);W\left(t\right)\right)\right|\right]^{2}\\
 & \quad+\frac{K}{n_{i-1}}\sum_{j_{i-1}=1}^{n_{i-1}}\left|\tilde{w}_{i}\left(t,j_{i-1},j_{i}\right)-w_{i}\left(t,C_{i-1}\left(j_{i-1}\right),C_{i}\left(j_{i}\right)\right)\right|^{2}\\
 & \quad+K\left|\tilde{b}_{i}\left(t,j_{i}\right)-b_{i}\left(t,C_{i}\left(j_{i}\right)\right)\right|^{2},
\end{align*}
and therefore, by Fact 1, under the events ${\cal E}_{t,i-1}^{\mathbf{H}}$
and ${\cal E}$,
\[
\bigg(\frac{1}{n_{i}}\sum_{j_{i}=1}^{n_{i}}\mathbb{E}_{Z}\left[\left|Q_{1,i}\left(t\right)\right|\right]^{2}\bigg)^{1/2}\leq K_{T}F_{i-1}\left(t\right)+K\mathscr{D}_{t}\left(W,\tilde{W}\right).
\]
Let us bound $Q_{2,i}$. For brevity, let us write 
\[
Z_{i}^{H}\left(t,c_{i-1},c_{i}\right)=\phi_{i}\left(w_{i}\left(t,c_{i-1},c_{i}\right),b_{i}\left(t,c_{i}\right),H_{i-1}\left(x,c_{i-1};W\left(t\right)\right)\right).
\]
Recall that $C_{i-1}\left(j_{i-1}\right)$ and $C_{i}\left(j_{i}\right)$
are independent. We thus have: 
\[
\mathbb{E}\left[Z_{i}^{H}\left(t,C_{i-1}\left(j_{i-1}\right),C_{i}\left(j_{i}\right)\right)\middle|C_{i}\left(j_{i}\right)\right]=\mathbb{E}_{C_{i-1}}\left[Z_{i}^{H}\left(t,C_{i-1},C_{i}\left(j_{i}\right)\right)\right].
\]
Furthermore $\left\{ C_{i-1}\left(j_{i-1}\right)\right\} _{j_{i-1}\in\left[n_{i-1}\right]}$
are $\eta_{i-1}$-independent by Assumption \ref{assump:neuronal-embedding}.
We also have that for ${\cal P}$-almost every $x$, almost surely,
\[
\left|Z_{i}^{H}\left(t,C_{i-1}\left(j_{i-1}\right),C_{i}\left(j_{i}\right)\right)\right|\leq K_{T}\left(1+B\right),
\]
by Assumption \ref{enu:Assump_forward} and Fact 2. Then by Lemma
\ref{lem:square hoeffding}, noting that $\gamma_{i}\ge K\eta_{i-1}$,
\[
\mathbb{P}\left(\mathbb{E}_{Z}\left[Q_{2,i}\right]\geq K_{T}\left(1+B\right)\gamma_{i}\right)\leq\left(1/\gamma_{i}\right)\exp\left(-n_{i-1}\gamma_{i}^{2}/K_{T}\right).
\]
By taking a union bound of the above probabilistic bound over $j_{i}\in\left[n_{i}\right]$,
we thus have, on the events ${\cal E}_{t,i-1}^{\mathbf{H}}$ and ${\cal E}$,
\begin{align*}
F_{i}\left(t\right) & \le K_{T}F_{i-1}\left(t\right)+K\mathscr{D}_{t}\left(W,\tilde{W}\right)+K_{T}\left(1+B\right)\gamma_{i}\\
 & \le K_{T}^{i}\bigg(\mathscr{D}_{t}\left(W,\tilde{W}\right)+\left(1+B\right)\sum_{j=1}^{i-1}\gamma_{j+1}\bigg)
\end{align*}
with probability at least $1-\left(n_{i}/\gamma_{i}\right)\exp\left(-n_{i-1}\gamma_{i}^{2}/K_{T}\right)$.
We thus get:
\[
\mathbb{P}\left({\cal E}_{t,i}^{\mathbf{H}};{\cal E}\right)\ge\mathbb{P}\left({\cal E}_{t,i-1}^{\mathbf{H}};{\cal E}\right)-\frac{n_{j}}{\gamma_{j}}\exp\left(-n_{i-1}\gamma_{i}^{2}/K_{T}\right)\ge1-\sum_{j=1}^{i-1}\frac{n_{j+1}}{\gamma_{j+1}}\exp\left(-n_{j}\gamma_{j+1}^{2}/K_{T}\right),
\]
which proves the claim.

\paragraph*{Step 4 - Claim 3.}

We show the claim by backward induction. The proof is similar to Claim
2. Consider $i=L$. Notice that on the event ${\cal E}_{t,L}^{\mathbf{H}}$,
\[
\mathbb{E}_{Z}\left[\left|\hat{\mathbf{y}}\left(X;\tilde{W}\left(t\right)\right)-\hat{y}\left(X;W\left(t\right)\right)\right|\right]\leq KF_{L}\left(t\right)\leq K_{T}^{L}\bigg(\mathscr{D}_{t}\left(W,\tilde{W}\right)+\left(1+B\right)\delta_{L}^{\Delta}\bigg),
\]
by Assumption \ref{enu:Assump_forward}. We thus get from Assumption
\ref{enu:Assump_backward} that on the events ${\cal E}_{t,L}^{\mathbf{H}}$
and ${\cal E}$, 
\begin{align*}
G_{L}\left(t\right) & \leq K\left(F_{L}\left(t\right)+\mathbb{E}_{Z}\left[\left|\hat{\mathbf{y}}\left(X;\tilde{W}\left(t\right)\right)-\hat{y}\left(X;W\left(t\right)\right)\right|\right]\right)\\
 & \leq K_{T}^{L+1}\bigg(\mathscr{D}_{t}\left(W,\tilde{W}\right)+\left(1+B\right)\delta_{L}^{\Delta}\bigg).
\end{align*}
That is, $\mathbb{P}\left({\cal E}_{t,L}^{\mathbf{H}}\cap{\cal E}_{t,L}^{\Delta};{\cal E}\right)=\mathbb{P}\left({\cal E}_{t,L}^{\mathbf{H}};{\cal E}\right)$.

Considering $i=L-1$, by Assumption \ref{enu:Assump_backward}, we
have: 
\[
G_{L-1}\left(t\right)\leq K\left(G_{L-1}^{\left(1\right)}\left(t\right)+G_{L-1}^{\left(2\right)}\left(t\right)+G_{L-1}^{\left(3\right)}\left(t\right)+G_{L-1}^{\left(4\right)}\left(t\right)\right),
\]
in which
\begin{align*}
G_{L-1}^{\left(1\right)}\left(t\right) & =\bigg(\frac{1}{n_{L-1}}\sum_{j_{L-1}=1}^{n_{L-1}}\left(1+\left|\tilde{w}_{L}\left(t,j_{L-1},1\right)\right|^{2}+\left|w_{L}\left(t,C_{L-1}\left(j_{L-1}\right),1\right)\right|^{2}+\left|\tilde{b}_{L}\left(t,1\right)\right|^{2}+\left|b_{L}\left(t,1\right)\right|^{2}\right)\bigg)^{1/2}G_{L}\left(t\right),\\
G_{L-1}^{\left(2\right)}\left(t\right) & =\left(1+\mathbb{E}_{Z}\left[\left|\Delta_{L}^{\mathbf{H}}\left(Z,1;\tilde{W}\left(t\right)\right)\right|\right]+\mathbb{E}_{Z}\left[\left|\Delta_{L}^{H}\left(Z,1;W\left(t\right)\right)\right|\right]\right)\\
 & \qquad\times\bigg(\frac{1}{n_{L-1}}\sum_{j_{L-1}=1}^{n_{L-1}}\left|\tilde{w}_{L}\left(t,j_{L-1},1\right)-w_{L}\left(t,C_{L-1}\left(j_{L-1}\right),1\right)\right|^{2}+\left|\tilde{b}_{L}\left(t,1\right)-b_{L}\left(t,1\right)\right|^{2}\bigg)^{1/2},\\
G_{L-1}^{\left(3\right)}\left(t\right) & =\bigg(\frac{1}{n_{L-1}}\sum_{j_{L-1}=1}^{n_{L-1}}\mathbb{E}_{Z}\Big[\left(1+\left|\Delta_{L}^{\mathbf{H}}\left(Z,1;\tilde{W}\left(t\right)\right)\right|+\left|\Delta_{L}^{H}\left(Z,1;W\left(t\right)\right)\right|\right)\\
 & \qquad\times\left(1+\left|\tilde{w}_{L}\left(t,j_{L-1},1\right)\right|+\left|w_{L}\left(t,C_{L-1}\left(j_{L-1}\right),1\right)\right|+\left|\tilde{b}_{L}\left(t,1\right)\right|+\left|b_{L}\left(t,1\right)\right|\right)\\
 & \qquad\times\left|\mathbf{H}_{L}\left(X,1;\tilde{W}\left(t\right)\right)-H_{L}\left(X,1;W\left(t\right)\right)\right|\Big]^{2}\bigg)^{1/2},\\
G_{L-1}^{\left(4\right)}\left(t\right) & =\bigg(\frac{1}{n_{L-1}}\sum_{j_{L-1}=1}^{n_{L-1}}\mathbb{E}_{Z}\Big[\left(1+\left|\Delta_{L}^{\mathbf{H}}\left(Z,1;\tilde{W}\left(t\right)\right)\right|+\left|\Delta_{L}^{H}\left(Z,1;W\left(t\right)\right)\right|\right)\\
 & \qquad\times\left(1+\left|\tilde{w}_{L}\left(t,j_{L-1},1\right)\right|+\left|w_{L}\left(t,C_{L-1}\left(j_{L-1}\right),1\right)\right|+\left|\tilde{b}_{L}\left(t,1\right)\right|+\left|b_{L}\left(t,1\right)\right|\right)\\
 & \qquad\times\left|\mathbf{H}_{L-1}\left(X,j_{L-1};\tilde{W}\left(t\right)\right)-H_{L}\left(X,C_{L-1}\left(j_{L-1}\right);W\left(t\right)\right)\right|\Big]^{2}\bigg)^{1/2}.
\end{align*}
Due to Fact 1, on the event ${\cal E}$,
\[
G_{L-1}^{\left(1\right)}\left(t\right)\leq K_{T}G_{L}\left(t\right).
\]
By Assumption \ref{enu:Assump_backward}, we have for ${\cal P}$-almost
every $z$,
\[
\left|\Delta_{L}^{\mathbf{H}}\left(z,1;\tilde{W}\left(t\right)\right)\right|,\;\left|\Delta_{L}^{H}\left(z,1;W\left(t\right)\right)\right|\leq K.
\]
Using this fact,
\[
G_{L-1}^{\left(2\right)}\left(t\right)\leq K\mathscr{D}_{t}\left(W,\tilde{W}\right).
\]
The same fact also applies to $G_{L-1}^{\left(3\right)}$, $G_{L-1}^{\left(4\right)}$
and $G_{L-1}^{\left(5\right)}$. In particular, we obtain for $G_{L-1}^{\left(3\right)}$,
on the event ${\cal E}$:
\begin{align*}
G_{L-1}^{\left(3\right)}\left(t\right) & \leq K\bigg(\frac{1}{n_{L-1}}\sum_{j_{L-1}=1}^{n_{L-1}}\left(1+\left|\tilde{w}_{L}\left(t,j_{L-1},1\right)\right|+\left|w_{L}\left(t,C_{L-1}\left(j_{L-1}\right),1\right)\right|+\left|\tilde{b}_{L}\left(t,1\right)\right|+\left|b_{L}\left(t,1\right)\right|\right)^{2}\bigg)^{1/2}F_{L}\left(t\right)\\
 & \leq K_{T}F_{L}\left(t\right),
\end{align*}
where the last display follows from Fact 1. Similarly, by using Fact
2 and Fact 3, we have on the event ${\cal E}$,
\[
G_{L-1}^{\left(4\right)}\left(t\right)\leq K_{T}\left(1+B\right)F_{L-1}\left(t\right).
\]
Hence on the events ${\cal E}_{t,L}^{\mathbf{H}}$, ${\cal E}_{t,L}^{\Delta}$
and ${\cal E}$,
\begin{align*}
G_{L-1}\left(t\right) & \leq K_{T}\left(G_{L}\left(t\right)+\mathscr{D}_{t}\left(W,\tilde{W}\right)+F_{L}\left(t\right)+\left(1+B\right)F_{L-1}\left(t\right)\right)\\
 & \leq K_{T}^{L+2}\bigg(\left(1+B\right)\mathscr{D}_{t}\left(W,\tilde{W}\right)+\left(1+B^{2}\right)\delta_{L}^{\Delta}\bigg).
\end{align*}
In other words, $\mathbb{P}\left({\cal E}_{t,L}^{\mathbf{H}}\cap{\cal E}_{t,L-1}^{\Delta};{\cal E}\right)=\mathbb{P}\left({\cal E}_{t,L}^{\mathbf{H}};{\cal E}\right)$.

Next let us assume the claim for $i$, and we consider the claim for
$i-1$, for $2\leq i\leq L-1$. For notational brevity, in the following,
we let 
\begin{align*}
\Delta_{i}^{{\bf H}}\left(j_{i}\right) & =\Delta_{i}^{{\bf H}}\left(Z,j_{i};\tilde{W}\left(t\right)\right), & {\bf H}_{i}\left(j_{i}\right) & ={\bf H}_{i}\left(X,j_{i};\tilde{W}\left(t\right)\right),\\
\Delta_{i}^{H}\left(c_{i}\right) & =\Delta_{i}^{H}\left(Z,c_{i};W\left(t\right)\right), & H_{i}\left(c_{i}\right) & =H_{i}\left(X,c_{i};W\left(t\right)\right).
\end{align*}
We have: 
\begin{align*}
 & \left|\Delta_{i-1}^{{\bf H}}\left(j_{i-1}\right)-\Delta_{i-1}^{H}\left(C_{i-1}\left(j_{i-1}\right)\right)\right|\\
 & =\Bigg|\frac{1}{n_{i}}\sum_{j_{i}=1}^{n_{i}}\sigma_{i-1}^{\mathbf{H}}\left(\Delta_{i}^{{\bf H}}\left(j_{i}\right),\tilde{w}_{i}\left(t,j_{i-1},j_{i}\right),\tilde{b}_{i}\left(t,j_{i}\right),{\bf H}_{i}\left(j_{i}\right),{\bf H}_{i-1}\left(j_{i-1}\right)\right)\\
 & \qquad-\mathbb{E}_{C_{i}}\left[\sigma_{i-1}^{\mathbf{H}}\left(\Delta_{i}^{H}\left(C_{i}\right),w_{i}\left(t,C_{i-1}\left(j_{i-1}\right),C_{i}\right),b_{i}\left(t,C_{i}\right),H_{i}\left(C_{i}\right),H_{i-1}\left(C_{i-1}\left(j_{i-1}\right)\right)\right)\right]\Bigg|\\
 & \leq Q_{3,i}\left(t\right)+Q_{4,i}\left(t\right),
\end{align*}
which gives
\[
G_{i-1}\left(t\right)\leq\bigg(\frac{1}{n_{i-1}}\sum_{j_{i-1}=1}^{n_{i-1}}\mathbb{E}_{Z}\left[\left|Q_{3,i}\left(t\right)\right|+\left|Q_{4,i}\left(t\right)\right|\right]^{2}\bigg)^{1/2},
\]
in which we define 
\begin{align*}
Q_{3,i}\left(t\right) & =\frac{1}{n_{i}}\sum_{j_{i}=1}^{n_{i}}\Big|\sigma_{i-1}^{\mathbf{H}}\left(\Delta_{i}^{{\bf H}}\left(j_{i}\right),\tilde{w}_{i}\left(t,j_{i-1},j_{i}\right),\tilde{b}_{i}\left(t,j_{i}\right),{\bf H}_{i}\left(j_{i}\right),{\bf H}_{i-1}\left(j_{i-1}\right)\right)\\
 & \qquad-\sigma_{i-1}^{\mathbf{H}}\left(\Delta_{i}^{H}\left(C_{i}\left(j_{i}\right)\right),w_{i}\left(t,C_{i-1}\left(j_{i-1}\right),C_{i}\left(j_{i}\right)\right),b_{i}\left(t,C_{i}\left(j_{i}\right)\right),H_{i}\left(C_{i}\left(j_{i}\right)\right),H_{i-1}\left(C_{i-1}\left(j_{i-1}\right)\right)\right)\Big|,\\
Q_{4,i}\left(t\right) & =\Bigg|\frac{1}{n_{i}}\sum_{j_{i}=1}^{n_{i}}\sigma_{i-1}^{\mathbf{H}}\left(\Delta_{i}^{H}\left(C_{i}\left(j_{i}\right)\right),w_{i}\left(t,C_{i-1}\left(j_{i-1}\right),C_{i}\left(j_{i}\right)\right),b_{i}\left(t,C_{i}\left(j_{i}\right)\right),H_{i}\left(C_{i}\left(j_{i}\right)\right),H_{i-1}\left(C_{i-1}\left(j_{i-1}\right)\right)\right)\\
 & \qquad-\mathbb{E}_{C_{i}}\left[\sigma_{i-1}^{\mathbf{H}}\left(\Delta_{i}^{H}\left(C_{i}\right),w_{i}\left(t,C_{i-1}\left(j_{i-1}\right),C_{i}\right),b_{i}\left(t,C_{i}\right),H_{i}\left(C_{i}\right),H_{i-1}\left(C_{i-1}\left(j_{i-1}\right)\right)\right)\right]\Bigg|.
\end{align*}
Let us first bound $Q_{3,i}$. This is similar to the bounding of
$G_{L-1}$. In particular, by Assumption \ref{enu:Assump_backward},
we have: 
\[
\bigg(\frac{1}{n_{i-1}}\sum_{j_{i-1}=1}^{n_{i-1}}\mathbb{E}_{Z}\left[\left|Q_{3,i}\left(t\right)\right|\right]^{2}\bigg)^{1/2}\leq K\left(Q_{3,i}^{\left(1\right)}\left(t\right)+Q_{3,i}^{\left(2\right)}\left(t\right)+Q_{3,i}^{\left(3\right)}\left(t\right)+Q_{3,i}^{\left(4\right)}\left(t\right)\right),
\]
in which
\begin{align*}
Q_{3,i}^{\left(1\right)}\left(t\right) & =\bigg(\frac{1}{n_{i-1}}\sum_{j_{i-1}=1}^{n_{i-1}}\bigg(\frac{1}{n_{i}}\sum_{j_{i}=1}^{n_{i}}\left(1+\left|\tilde{w}_{i}\left(t,j_{i-1},j_{i}\right)\right|+\left|w_{i}\left(t,C_{i-1}\left(j_{i-1}\right),C_{i}\left(j_{i}\right)\right)\right|+\left|\tilde{b}_{i}\left(t,j_{i}\right)\right|+\left|b_{i}\left(t,C_{i}\left(j_{i}\right)\right)\right|\right)\\
 & \qquad\times\mathbb{E}_{Z}\left[\left|\Delta_{i}^{{\bf H}}\left(j_{i}\right)-\Delta_{i}^{H}\left(C_{i}\left(j_{i}\right)\right)\right|\right]\bigg)^{2}\bigg)^{1/2},\\
Q_{3,i}^{\left(2\right)}\left(t\right) & =\bigg(\frac{1}{n_{i-1}}\sum_{j_{i-1}=1}^{n_{i-1}}\bigg(\frac{1}{n_{i}}\sum_{j_{i}=1}^{n_{i}}\mathbb{E}_{Z}\left[1+\left|\Delta_{i}^{{\bf H}}\left(j_{i}\right)\right|+\left|\Delta_{i}^{H}\left(C_{i}\left(j_{i}\right)\right)\right|\right]\\
 & \qquad\times\left(\left|\tilde{w}_{i}\left(t,j_{i-1},j_{i}\right)-w_{i}\left(t,C_{i-1}\left(j_{i-1}\right),C_{i}\left(j_{i}\right)\right)\right|+\left|\tilde{b}_{i}\left(t,j_{i}\right)-b_{i}\left(t,C_{i}\left(j_{i}\right)\right)\right|\right)\bigg)^{2}\bigg)^{1/2},\\
Q_{3,i}^{\left(3\right)}\left(t\right) & =\bigg(\frac{1}{n_{i-1}}\sum_{j_{i-1}=1}^{n_{i-1}}\bigg(\frac{1}{n_{i}}\sum_{j_{i}=1}^{n_{i}}\mathbb{E}_{Z}\Big[\left(1+\left|\Delta_{i}^{{\bf H}}\left(j_{i}\right)\right|+\left|\Delta_{i}^{H}\left(C_{i}\left(j_{i}\right)\right)\right|\right)\\
 & \qquad\times\left(1+\left|\tilde{w}_{i}\left(t,j_{i-1},j_{i}\right)\right|+\left|w_{i}\left(t,C_{i-1}\left(j_{i-1}\right),C_{i}\left(j_{i}\right)\right)\right|+\left|\tilde{b}_{i}\left(t,j_{i}\right)\right|+\left|b_{i}\left(t,C_{i}\left(j_{i}\right)\right)\right|\right)\\
 & \qquad\times\left|{\bf H}_{i}\left(j_{i}\right)-H_{i}\left(C_{i}\left(j_{i}\right)\right)\right|\Big]\bigg)^{2}\bigg)^{1/2},\\
Q_{3,i}^{\left(4\right)}\left(t\right) & =\bigg(\frac{1}{n_{i-1}}\sum_{j_{i-1}=1}^{n_{i-1}}\bigg(\frac{1}{n_{i}}\sum_{j_{i}=1}^{n_{i}}\mathbb{E}_{Z}\Big[\left(1+\left|\Delta_{i}^{{\bf H}}\left(j_{i}\right)\right|+\left|\Delta_{i}^{H}\left(C_{i}\left(j_{i}\right)\right)\right|\right)\\
 & \qquad\times\left(1+\left|\tilde{w}_{i}\left(t,j_{i-1},j_{i}\right)\right|+\left|w_{i}\left(t,C_{i-1}\left(j_{i-1}\right),C_{i}\left(j_{i}\right)\right)\right|+\left|\tilde{b}_{i}\left(t,j_{i}\right)\right|+\left|b_{i}\left(t,C_{i}\left(j_{i}\right)\right)\right|\right)\\
 & \qquad\times\left|{\bf H}_{i-1}\left(j_{i-1}\right)-H_{i-1}\left(C_{i-1}\left(j_{i-1}\right)\right)\right|\Big]\bigg)^{2}\bigg)^{1/2}.
\end{align*}
To bound $Q_{3,i}^{\left(1\right)}$, we use Cauchy-Schwarz's inequality
and Fact 1 to obtain that on the event ${\cal E}$:
\begin{align*}
Q_{3,i}^{\left(1\right)}\left(t\right) & \leq\bigg(\frac{1}{n_{i-1}n_{i}}\sum_{j_{i-1}=1}^{n_{i-1}}\sum_{j_{i}=1}^{n_{i}}\left(1+\left|\tilde{w}_{i}\left(t,j_{i-1},j_{i}\right)\right|+\left|w_{i}\left(t,C_{i-1}\left(j_{i-1}\right),C_{i}\left(j_{i}\right)\right)\right|+\left|\tilde{b}_{i}\left(t,j_{i}\right)\right|+\left|b_{i}\left(t,C_{i}\left(j_{i}\right)\right)\right|\right)^{2}\\
 & \qquad\times\frac{1}{n_{i}}\sum_{j_{i}=1}^{n_{i}}\mathbb{E}_{Z}\left[\left|\Delta_{i}^{{\bf H}}\left(j_{i}\right)-\Delta_{i}^{H}\left(C_{i}\left(j_{i}\right)\right)\right|\right]^{2}\bigg)^{1/2}\\
 & \leq K_{T}G_{i}\left(t\right).
\end{align*}
By Cauchy-Schwarz's inequality and Fact 1, we have a bound on $Q_{3,i}^{\left(2\right)}$
on the event ${\cal E}$:
\begin{align*}
Q_{3,i}^{\left(2\right)}\left(t\right) & \leq K\bigg(\frac{1}{n_{i}}\sum_{j_{i}=1}^{n_{i}}\mathbb{E}_{Z}\left[1+\left|\Delta_{i}^{{\bf H}}\left(j_{i}\right)\right|^{2}+\left|\Delta_{i}^{H}\left(C_{i}\left(j_{i}\right)\right)\right|^{2}\right]\bigg)^{1/2}\\
 & \qquad\times\bigg(\frac{1}{n_{i-1}n_{i}}\sum_{j_{i-1}=1}^{n_{i-1}}\sum_{j_{i}=1}^{n_{i}}\left|\tilde{w}_{i}\left(t,j_{i-1},j_{i}\right)-w_{i}\left(t,C_{i-1}\left(j_{i-1}\right),C_{i}\left(j_{i}\right)\right)\right|^{2}+\left|\tilde{b}_{i}\left(t,j_{i}\right)-b_{i}\left(t,C_{i}\left(j_{i}\right)\right)\right|^{2}\bigg)^{1/2}\\
 & \leq K_{T}\mathscr{D}_{t}\left(W,\tilde{W}\right).
\end{align*}
Similarly, on the event ${\cal E}$:
\begin{align*}
Q_{3,i}^{\left(3\right)}\left(t\right) & \stackrel{\left(a\right)}{\leq}K_{T}\left(1+B\right)\frac{1}{n_{i}}\sum_{j_{i}=1}^{n_{i}}\mathbb{E}_{Z}\left[\left(1+\left|\Delta_{i}^{{\bf H}}\left(j_{i}\right)\right|+\left|\Delta_{i}^{H}\left(C_{i}\left(j_{i}\right)\right)\right|\right)\left|{\bf H}_{i}\left(j_{i}\right)-H_{i}\left(C_{i}\left(j_{i}\right)\right)\right|\right]\\
 & \stackrel{\left(b\right)}{\leq}K_{T}\left(1+B\right)\bigg(\frac{1}{n_{i}}\sum_{j_{i}=1}^{n_{i}}\mathbb{E}_{Z}\left[1+\left|\Delta_{i}^{{\bf H}}\left(j_{i}\right)\right|^{2}+\left|\Delta_{i}^{H}\left(C_{i}\left(j_{i}\right)\right)\right|^{2}\right]\bigg)^{1/2}F_{i}\left(t\right)\\
 & \stackrel{\left(c\right)}{\leq}K_{T}\left(1+B\right)F_{i}\left(t\right),
\end{align*}
where we use Fact 2 and Fact 3 in step $\left(a\right)$, Cauchy-Schwarz's
inequality in step $\left(b\right)$ and Fact 1 in step $\left(c\right)$.
With the same argument, on the event ${\cal E}$:
\begin{align*}
Q_{3,i}^{\left(4\right)}\left(t\right) & \stackrel{\left(a\right)}{\leq}K_{T}\left(1+B\right)\bigg(\frac{1}{n_{i-1}}\sum_{j_{i-1}=1}^{n_{i-1}}\bigg(\frac{1}{n_{i}}\sum_{j_{i}=1}^{n_{i}}\mathbb{E}_{Z}\Big[\left(1+\left|\Delta_{i}^{{\bf H}}\left(j_{i}\right)\right|+\left|\Delta_{i}^{H}\left(C_{i}\left(j_{i}\right)\right)\right|\right)\\
 & \qquad\times\left|{\bf H}_{i-1}\left(j_{i-1}\right)-H_{i-1}\left(C_{i-1}\left(j_{i-1}\right)\right)\right|\Big]\bigg)^{2}\bigg)^{1/2}\\
 & \stackrel{\left(b\right)}{\leq}K_{T}\left(1+B\right)\bigg(\frac{1}{n_{i}}\sum_{j_{i}=1}^{n_{i}}\mathbb{E}_{Z}\left[1+\left|\Delta_{i}^{{\bf H}}\left(j_{i}\right)\right|^{2}+\left|\Delta_{i}^{H}\left(C_{i}\left(j_{i}\right)\right)\right|^{2}\right]\\
 & \qquad\times\frac{1}{n_{i-1}}\sum_{j_{i-1}=1}^{n_{i-1}}\mathbb{E}_{Z}\left[\left|{\bf H}_{i-1}\left(j_{i-1}\right)-H_{i-1}\left(C_{i-1}\left(j_{i-1}\right)\right)\right|^{2}\right]\bigg)^{1/2}\\
 & \stackrel{\left(c\right)}{\leq}K_{T}\left(1+B\right)F_{i-1}\left(t\right),
\end{align*}
where again we use Fact 2 and Fact 3 in step $\left(a\right)$, Cauchy-Schwarz's
inequality in step $\left(b\right)$ and Fact 1 in step $\left(c\right)$.
Therefore on the events ${\cal E}$, ${\cal E}_{t,i}^{\Delta}$ and
${\cal E}_{t,L}^{\mathbf{H}}$, 
\begin{align*}
\bigg(\frac{1}{n_{i-1}}\sum_{j_{i-1}=1}^{n_{i-1}}\mathbb{E}_{Z}\left[\left|Q_{3,i}\left(t\right)\right|\right]^{2}\bigg)^{1/2} & \leq K_{T}\left(G_{i}\left(t\right)+\mathscr{D}_{t}\left(W,\tilde{W}\right)+\left(1+B\right)\left(F_{i}\left(t\right)+F_{i-1}\left(t\right)\right)\right)\\
 & \leq K_{T}^{2L-i+2}\bigg(\left(1+B\right)\mathscr{D}_{t}\left(W,\tilde{W}\right)+\left(1+B^{2}\right)\bigg(\delta_{L}^{\Delta}+\sum_{j=i}^{L-2}\beta_{j}\bigg)\bigg).
\end{align*}
Next let us bound $Q_{4,i}$. For brevity, let us write 
\[
Z_{i}^{\Delta}\left(t,c_{i-1},c_{i}\right)=\sigma_{i-1}^{\mathbf{H}}\left(\Delta_{i}^{H}\left(c_{i}\right),w_{i}\left(t,c_{i-1},c_{i}\right),b_{i}\left(t,c_{i}\right),H_{i}\left(c_{i}\right),H_{i-1}\left(c_{i-1}\right)\right).
\]
Recall that $C_{i-1}\left(j_{i-1}\right)$ and $C_{i}\left(j_{i}\right)$
are independent. We thus have: 
\[
\mathbb{E}_{C_{i}\left(j_{i}\right)}\left[Z_{i}^{\Delta}\left(t,C_{i-1}\left(j_{i-1}\right),C_{i}\left(j_{i}\right)\right)\middle|C_{i-1}\left(j_{i-1}\right)\right]=\mathbb{E}_{C_{i}}\left[Z_{i}^{\Delta}\left(t,C_{i-1}\left(j_{i-1}\right),C_{i}\right)\right].
\]
Furthermore $\left\{ C_{i}\left(j_{i}\right)\right\} _{j_{i}\in\left[n_{i}\right]}$
are $\eta_{i}$-independent by Assumption \ref{assump:neuronal-embedding}.
We also have that almost surely,
\begin{align*}
\left|Z_{i}^{\Delta}\left(t,C_{i-1}\left(j_{i-1}\right),C_{i}\left(j_{i}\right)\right)\right| & \leq K\left(1+\left|\Delta_{i}^{H}\left(C_{i}\left(j_{i}\right)\right)\right|\right)\left(1+\left|w_{i}\left(t,C_{i-1}\left(j_{i-1}\right),C_{i}\left(j_{i}\right)\right)\right|+\left|b_{i}\left(t,C_{i}\left(j_{i}\right)\right)\right|\right)\\
 & \leq K_{T}\left(1+B^{2}\right),
\end{align*}
by Assumption \ref{enu:Assump_backward} and Fact 2. Then by Lemma
\ref{lem:square hoeffding}, noting that $\beta_{i-1}\ge K\eta_{i}$,
\[
\mathbb{P}\left(\mathbb{E}_{Z}\left[Q_{4,i}\left(t\right)\right]\geq K_{T}\left(1+B^{2}\right)\beta_{i-1}\right)\leq\left(1/\beta_{i-1}\right)\exp\left(-n_{i}\beta_{i-1}^{2}/K_{T}\right).
\]
We thus have, by taking a union bound over $j_{i-1}\in\left[n_{i-1}\right]$,
on the events ${\cal E}$, ${\cal E}_{t,i}^{\Delta}$ and ${\cal E}_{t,L}^{\mathbf{H}}$,
\[
G_{i-1}\left(t\right)\le K_{T}^{2L-i+2}\bigg(\left(1+B\right)\mathscr{D}_{t}\left(W,\tilde{W}\right)+\left(1+B^{2}\right)\bigg(\delta_{L}^{\Delta}+\sum_{j=i-1}^{L-2}\beta_{j}\bigg)\bigg)
\]
with probability at least $1-\left(n_{i-1}/\beta_{i-1}\right)\exp\left(-n_{i}\beta_{i-1}^{2}/K_{T}\right)$.
We thus get: 
\[
\mathbb{P}\left({\cal E}_{t,L}^{\mathbf{H}}\cap{\cal E}_{t,i-1}^{\Delta};{\cal E}\right)\ge\mathbb{P}\left({\cal E}_{t,L}^{\mathbf{H}}\cap{\cal E}_{t,i}^{\Delta};{\cal E}\right)-\frac{n_{i-1}}{\beta_{i-1}}\exp\left(-n_{i-1}\beta_{i-1}^{2}/K_{T}\right)\ge\mathbb{P}\left({\cal E}_{t,L}^{\mathbf{H}};{\cal E}\right)-\sum_{j=i-1}^{L-2}\frac{n_{j}}{\beta_{j}}\exp\left(-n_{j+1}\beta_{j}^{2}/K_{T}\right),
\]
which proves the claim.

\paragraph*{Step 5 - Claim 4.}

We reuse the notations introduced in the previous step. For $2\leq i\leq L$,
we have: 
\begin{align*}
 & \left|\Delta_{i}^{\mathbf{b}}\left(Z,j_{i};\tilde{W}\left(t\right)\right)-\Delta_{i}^{b}\left(Z,C_{i}\left(j_{i}\right);W\left(t\right)\right)\right|\\
 & =\Bigg|\frac{1}{n_{i-1}}\sum_{j_{i-1}=1}^{n_{i-1}}\sigma_{i}^{\mathbf{b}}\left(\Delta_{i}^{{\bf H}}\left(j_{i}\right),\tilde{w}_{i}\left(t,j_{i-1},j_{i}\right),\tilde{b}_{i}\left(t,j_{i}\right),{\bf H}_{i}\left(j_{i}\right),{\bf H}_{i-1}\left(j_{i-1}\right)\right)\\
 & \qquad-\mathbb{E}_{C_{i-1}}\left[\sigma_{i}^{\mathbf{b}}\left(\Delta_{i}^{H}\left(C_{i}\left(j_{i}\right)\right),w_{i}\left(t,C_{i-1},C_{i}\left(j_{i}\right)\right),b_{i}\left(t,C_{i}\left(j_{i}\right)\right),H_{i}\left(C_{i}\left(j_{i}\right)\right),H_{i-1}\left(C_{i-1}\right)\right)\right]\Big|\\
 & \leq Q_{5,i}\left(t\right)+Q_{6,i}\left(t\right),
\end{align*}
which gives, by Assumption \ref{enu:Assump_lrSchedule},
\[
D_{i}^{b}\left(t\right)\leq K\bigg(\frac{1}{n_{i}}\sum_{j_{i}=1}^{n_{i}}\mathbb{E}_{Z}\left[\left|Q_{5,i}\left(t\right)\right|+\left|Q_{6,i}\left(t\right)\right|\right]^{2}\bigg)^{1/2},
\]
in which we define 
\begin{align*}
Q_{5,i}\left(t\right) & =\frac{1}{n_{i-1}}\sum_{j_{i-1}=1}^{n_{i-1}}\Big|\sigma_{i}^{\mathbf{b}}\left(\Delta_{i}^{{\bf H}}\left(j_{i}\right),\tilde{w}_{i}\left(t,j_{i-1},j_{i}\right),\tilde{b}_{i}\left(t,j_{i}\right),{\bf H}_{i}\left(j_{i}\right),{\bf H}_{i-1}\left(j_{i-1}\right)\right)\\
 & \qquad-\sigma_{i}^{\mathbf{b}}\left(\Delta_{i}^{H}\left(C_{i}\left(j_{i}\right)\right),w_{i}\left(t,C_{i-1}\left(j_{i-1}\right),C_{i}\left(j_{i}\right)\right),b_{i}\left(t,C_{i}\left(j_{i}\right)\right),H_{i}\left(C_{i}\left(j_{i}\right)\right),H_{i-1}\left(C_{i-1}\left(j_{i-1}\right)\right)\right)\Big|,\\
Q_{6,i}\left(t\right) & =\Bigg|\frac{1}{n_{i-1}}\sum_{j_{i-1}=1}^{n_{i-1}}\sigma_{i}^{\mathbf{b}}\left(\Delta_{i}^{H}\left(C_{i}\left(j_{i}\right)\right),w_{i}\left(t,C_{i-1}\left(j_{i-1}\right),C_{i}\left(j_{i}\right)\right),b_{i}\left(t,C_{i}\left(j_{i}\right)\right),H_{i}\left(C_{i}\left(j_{i}\right)\right),H_{i-1}\left(C_{i-1}\left(j_{i-1}\right)\right)\right)\\
 & \qquad-\mathbb{E}_{C_{i-1}}\left[\sigma_{i}^{\mathbf{b}}\left(\Delta_{i}^{H}\left(C_{i}\left(j_{i}\right)\right),w_{i}\left(t,C_{i-1},C_{i}\left(j_{i}\right)\right),b_{i}\left(t,C_{i}\left(j_{i}\right)\right),H_{i}\left(C_{i}\left(j_{i}\right)\right),H_{i-1}\left(C_{i-1}\right)\right)\right]\Bigg|.
\end{align*}
Similar to the bounding of $D_{i}^{w}\left(t\right)$, we have on
the event ${\cal E}$,
\[
\bigg(\frac{1}{n_{i}}\sum_{j_{i}=1}^{n_{i}}\mathbb{E}_{Z}\left[\left|Q_{5,i}\left(t\right)\right|\right]^{2}\bigg)^{1/2}\leq K_{T}\left(F_{i-1}\left(t\right)+G_{i}\left(t\right)+\mathscr{D}_{t}\left(W,\tilde{W}\right)+F_{i}\left(t\right)\right).
\]
To bound $Q_{6,i}$, for brevity, let us write 
\[
Z_{i}^{b}\left(t,c_{i-1},c_{i}\right)=\sigma_{i}^{\mathbf{b}}\left(\Delta_{i}^{H}\left(c_{i}\right),w_{i}\left(t,c_{i-1},c_{i}\right),b_{i}\left(t,c_{i}\right),H_{i}\left(c_{i}\right),H_{i-1}\left(c_{i-1}\right)\right).
\]
Recall that $C_{i-1}\left(j_{i-1}\right)$ and $C_{i}\left(j_{i}\right)$
are independent. We thus have: 
\[
\mathbb{E}_{C_{i-1}\left(j_{i-1}\right)}\left[Z_{i}^{b}\left(t,C_{i-1}\left(j_{i-1}\right),C_{i}\left(j_{i}\right)\right)\middle|C_{i}\left(j_{i}\right)\right]=\mathbb{E}_{C_{i-1}}\left[Z_{i}^{b}\left(t,C_{i-1},C_{i}\left(j_{i}\right)\right)\right].
\]
Furthermore $\left\{ C_{i-1}\left(j_{i-1}\right)\right\} _{j_{i-1}\in\left[n_{i-1}\right]}$
are $\eta_{i-1}$-independent by Assumption \ref{assump:neuronal-embedding}.
We also have that almost surely,
\begin{align*}
\left|Z_{i}^{b}\left(t,C_{i-1}\left(j_{i-1}\right),C_{i}\left(j_{i}\right)\right)\right| & \leq K\left(1+\left|\Delta_{i}^{H}\left(C_{i}\left(j_{i}\right)\right)\right|\right)\\
 & \leq K_{T}\left(1+B\right),
\end{align*}
by Assumption \ref{enu:Assump_backward} and Fact 2. Then by Lemma
\ref{lem:square hoeffding}, and since $\gamma_{i}\ge K\eta_{i-1}$,
\[
\mathbb{P}\left(\mathbb{E}_{Z}\left[Q_{6,i}\right]\geq K_{T}\left(1+B\right)\gamma_{i}\right)\leq\left(1/\gamma_{i}\right)\exp\left(-n_{i-1}\gamma_{i}^{2}/K_{T}\right).
\]
Notice that $\delta_{L}^{b}\geq\gamma_{i}$. We thus have, by taking
a union bound over $j_{i}\in\left[n_{i}\right]$, on the events ${\cal E}$,
${\cal E}_{t,1}^{\Delta}$ and ${\cal E}_{t,L}^{\mathbf{H}}$, 
\[
D_{i}^{b}\left(t\right)\le K_{T}\left(\left(1+B\right)\mathscr{D}_{t}\left(W,\tilde{W}\right)+\left(1+B^{2}\right)\delta_{L}^{b}\right)
\]
with probability at least $1-\left(n_{i}/\gamma_{i}\right)\exp\left(-n_{i-1}\gamma_{i}^{2}/K_{T}\right)$.
The claim then follows again from the union bound.
\end{proof}

\subsubsection{Proof of Proposition \ref{prop:gradient descent - bounded}}
\begin{proof}[Proof of Proposition \ref{prop:gradient descent - bounded}]
We consider $t\leq T$, for a given terminal time $T\in\epsilon\mathbb{N}_{\geq0}$.
We again reuse the notation $K_{t}$ from the proof of Proposition
\ref{prop:particle coupling - bounded}. Note that $K_{t}\leq K_{T}$
for all $t\leq T$. We also note that at initialization, $\mathscr{D}_{0}\left(\mathbf{W},\tilde{W}\right)=0$.
We start with a few preliminary fact:

\paragraph*{Fact 1: moment bounds.}

We recall a useful fact from the proof of Proposition \ref{prop:particle coupling - bounded}:
with probability at least $1-KLn_{\max}\exp\left(-Kn_{\min}^{1/52}\right)$,
the event ${\cal E}$ occurs, and ${\cal E}$ contains the following:
\begin{align*}
\interleave\mathbf{W}\interleave_{0}=\interleave\tilde{W}\interleave_{0} & \leq K,\\
\max_{1\leq i\leq L}\bigg(\frac{1}{n_{i}}\sum_{j_{i}=1}^{n_{i}}\sup_{t\leq T}\underset{Z\sim{\cal P}}{{\rm ess\text{-}sup}}\left|\Delta_{i}^{\mathbf{H}}\left(Z,j_{i};\tilde{W}\left(t\right)\right)\right|^{50}\bigg)^{1/50},\quad\interleave\tilde{W}\interleave_{T} & \leq K_{T}.
\end{align*}
We further remark that since $\interleave\mathbf{W}\interleave_{0}\leq K$,
from Lemma \ref{lem:bounds NN a priori}, we have on the event ${\cal E}$:
\[
\max_{1\leq i\leq L}\bigg(\frac{1}{n_{i}}\sum_{j_{i}=1}^{n_{i}}\sup_{t\leq T}\underset{Z\sim{\cal P}}{{\rm ess\text{-}sup}}\left|\Delta_{i}^{\mathbf{H}}\left(Z,j_{i};\mathbf{W}\left(\left\lfloor t/\epsilon\right\rfloor \right)\right)\right|^{50}\bigg)^{1/50},\quad\interleave\mathbf{W}\interleave_{\left\lfloor T/\epsilon\right\rfloor }\leq K_{T}.
\]
We also observe that the randomness of the event ${\cal E}$ is entirely
by the samples of the coupling procedure $\left\{ C_{1}\left(j_{1}\right),...,C_{L}\left(j_{L}\right):\;j_{i}\in\left[n_{i}\right],\;i=1,...,L\right\} $.

\paragraph*{Fact 2: maximal bounds.}

We also recall another useful fact from the proof of Proposition \ref{prop:particle coupling - bounded}:
on the event ${\cal E}$, almost surely,
\begin{align*}
\max_{2\leq i\leq L}\max_{j_{i-1}\in\left[n_{i-1}\right],\;j_{i}\in\left[n_{i}\right]}\sup_{t\leq T}\left|\tilde{w}_{i}\left(t,j_{i-1},j_{i}\right)\right| & \leq K_{T}\left(1+B\right),\\
\max_{2\leq i\leq L}\max_{j_{i}\in\left[n_{i}\right]}\sup_{t\leq T}\left|\tilde{b}_{i}\left(t,j_{i}\right)\right| & \leq K_{T}\left(1+B\right),
\end{align*}
In fact, the same extends to $\mathbf{W}$: on the event ${\cal E}$,
almost surely,
\begin{align*}
\max_{2\leq i\leq L}\max_{j_{i-1}\in\left[n_{i-1}\right],\;j_{i}\in\left[n_{i}\right]}\sup_{t\leq T}\left|\mathbf{w}_{i}\left(\left\lfloor t/\epsilon\right\rfloor ,j_{i-1},j_{i}\right)\right| & \leq K_{T}\left(1+B\right),\\
\max_{2\leq i\leq L}\max_{j_{i}\in\left[n_{i}\right]}\sup_{t\leq T}\left|\mathbf{b}_{i}\left(\left\lfloor t/\epsilon\right\rfloor ,j_{i}\right)\right| & \leq K_{T}\left(1+B\right),\\
\max_{1\leq i\leq L}\max_{j_{i}\in\left[n_{i}\right]}\sup_{t\leq T}\underset{Z\sim{\cal P}}{{\rm ess\text{-}sup}}\left|\Delta_{i}^{\mathbf{H}}\left(Z,j_{i};\mathbf{W}\left(\left\lfloor t/\epsilon\right\rfloor \right)\right)\right| & \leq K_{T}\left(1+B\right).
\end{align*}
Indeed, let us consider the claim for $\mathbf{w}_{i}$. By Assumption
\ref{enu:Assump_backward}, for ${\cal P}$-almost every $z$,
\[
\sup_{t\geq0}\max_{j_{L-1}\in\left[n_{L-1}\right]}\left|\Delta_{L}^{\mathbf{w}}\left(z,j_{L-1},1;\mathbf{W}\left(\left\lfloor t/\epsilon\right\rfloor \right)\right)\right|\leq K\left(1+\sup_{t\geq0}\left|\Delta_{L}^{\mathbf{H}}\left(z,1;\mathbf{W}\left(\left\lfloor t/\epsilon\right\rfloor \right)\right)\right|\right)\leq K,
\]
which implies, by Assumption \ref{enu:Assump_lrSchedule}, that almost
surely, for any $j_{L-1}\in\left[n_{L-1}\right]$,
\[
\sup_{t\leq T}\left|\mathbf{w}_{L}\left(\left\lfloor t/\epsilon\right\rfloor ,j_{L-1},1\right)\right|\leq{\rm ess\text{-}sup}\left|w_{L}^{0}\left(C_{L-1},1\right)\right|+KT\leq K_{T}\left(1+B\right).
\]
Next assuming that $\sup_{t\leq T}\left|\mathbf{w}_{i}\left(\left\lfloor t/\epsilon\right\rfloor ,j_{i-1},j_{i}\right)\right|\leq K_{T}\left(1+B\right)$
almost surely for a given $i\geq2$, by Assumption \ref{enu:Assump_backward},
we have on the event ${\cal E}$, for any $j_{i-1}\in\left[n_{i-1}\right]$,
$t\leq T$ and ${\cal P}$-almost every $z$:
\begin{align*}
 & \left|\Delta_{i-1}^{\mathbf{H}}\left(z,j_{i-1};\mathbf{W}\left(\left\lfloor t/\epsilon\right\rfloor \right)\right)\right|\\
 & \leq\frac{K}{n_{i}}\sum_{j_{i}=1}^{n_{i}}\left(1+\sup_{t\leq T}\underset{Z\sim{\cal P}}{{\rm ess\text{-}sup}}\left|\Delta_{i}^{\mathbf{H}}\left(Z,j_{i};\mathbf{W}\left(\left\lfloor t/\epsilon\right\rfloor \right)\right)\right|\right)\left(1+\left|\mathbf{w}_{i}\left(\left\lfloor t/\epsilon\right\rfloor ,j_{i-1},j_{i}\right)\right|+\left|\mathbf{b}_{i}\left(\left\lfloor t/\epsilon\right\rfloor ,j_{i}\right)\right|\right)\\
 & \leq\frac{K}{n_{i}}\sum_{j_{i}=1}^{n_{i}}\left(1+\sup_{t\leq T}\underset{Z\sim{\cal P}}{{\rm ess\text{-}sup}}\left|\Delta_{i}^{\mathbf{H}}\left(Z,j_{i};\mathbf{W}\left(\left\lfloor t/\epsilon\right\rfloor \right)\right)\right|\right)\left(K_{T}\left(1+B\right)+\left|\mathbf{b}_{i}\left(\left\lfloor t/\epsilon\right\rfloor ,j_{i}\right)\right|\right)\\
 & \leq K\bigg(1+\bigg(\frac{1}{n_{i}}\sum_{j_{i}=1}^{n_{i}}\sup_{t\leq T}\underset{Z\sim{\cal P}}{{\rm ess\text{-}sup}}\left|\Delta_{i}^{\mathbf{H}}\left(Z,j_{i};\mathbf{W}\left(\left\lfloor t/\epsilon\right\rfloor \right)\right)\right|^{2}\bigg)^{1/2}\bigg)\bigg(K_{T}\left(1+B\right)+\bigg(\frac{1}{n_{i}}\sum_{j_{i}=1}^{n_{i}}\left|\mathbf{b}_{i}\left(\left\lfloor t/\epsilon\right\rfloor ,j_{i}\right)\right|^{2}\bigg)^{1/2}\bigg)\\
 & \leq K_{T}\left(1+B\right),
\end{align*}
where the last step follows from Fact 1. Again by Assumption \ref{enu:Assump_backward},
we then obtain:
\[
\left|\Delta_{i-1}^{\mathbf{w}}\left(z,j_{i-1},j_{i};\mathbf{W}\left(\left\lfloor t/\epsilon\right\rfloor \right)\right)\right|\leq K_{T}\left(1+B\right),
\]
which implies, by Assumption \ref{enu:Assump_lrSchedule}, that almost
surely on the event ${\cal E}$, for any $j_{i-1}\in\left[n_{i-1}\right]$
and $j_{i}\in\left[n_{i}\right]$:
\[
\sup_{t\leq T}\left|\mathbf{w}_{i}\left(\left\lfloor t/\epsilon\right\rfloor ,j_{i-1},j_{i}\right)\right|\leq{\rm ess\text{-}sup}\left|w_{i-1}^{0}\left(C_{i-1},C_{i}\right)\right|+K_{T}\left(1+B\right)T\leq K_{T}\left(1+B\right).
\]
This proves the claim for $\mathbf{w}_{i}$, and the rest of the claims
are similarly proven.

Now let us consider $2\leq i\leq L$ and particularly the task of
bounding
\[
\bigg(\frac{1}{n_{i-1}n_{i}}\sum_{j_{i-1}=1}^{n_{i-1}}\sum_{j_{i}=1}^{n_{i}}\left|\mathbf{w}_{i}\left(\left\lfloor t/\epsilon\right\rfloor ,j_{i-1},j_{i}\right)-\tilde{w}_{i}\left(t,j_{i-1},j_{i}\right)\right|^{2}\bigg)^{1/2},
\]
which is a quantity in $\mathscr{D}_{T}\left(\tilde{W},\mathbf{W}\right)$.
As shown in the proof of Proposition \ref{prop:particle coupling - bounded}:
\[
\sup_{t\leq T-\zeta}\sup_{0\leq\zeta'\leq\zeta}\max_{i}\left(\tilde{A}_{i}^{w}\left(t,\zeta'\right),\tilde{A}_{i}^{b}\left(t,\zeta'\right)\right)\leq K_{T}\left(1+B\right)\zeta
\]
almost surely, where we recall
\begin{align*}
\tilde{A}_{i}^{w}\left(t,\zeta\right) & =\bigg(\frac{1}{n_{i-1}n_{i}}\sum_{j_{i-1}=1}^{n_{i-1}}\sum_{j_{i}=1}^{n_{i}}\left|\partial_{1}\tilde{w}_{i}\left(t+\zeta,j_{i-1},j_{i}\right)-\partial_{1}\tilde{w}_{i}\left(t,j_{i-1},j_{i}\right)\right|^{2}\bigg)^{1/2},
\end{align*}
As such, by Assumption \ref{enu:Assump_lrSchedule}, we have the decomposition:
\begin{align*}
 & \bigg(\frac{1}{n_{i-1}n_{i}}\sum_{j_{i-1}=1}^{n_{i-1}}\sum_{j_{i}=1}^{n_{i}}\left|\mathbf{w}_{i}\left(\left\lfloor t/\epsilon\right\rfloor ,j_{i-1},j_{i}\right)-\tilde{w}_{i}\left(t,j_{i-1},j_{i}\right)\right|^{2}\bigg)^{1/2}\\
 & =\bigg(\frac{1}{n_{i-1}n_{i}}\sum_{j_{i-1}=1}^{n_{i-1}}\sum_{j_{i}=1}^{n_{i}}\bigg|\epsilon\sum_{k=0}^{\left\lfloor t/\epsilon\right\rfloor -1}\xi_{i}^{\mathbf{w}}\left(k\epsilon\right)\Delta_{i}^{{\bf w}}\left(z\left(k\right),j_{i-1},j_{i};\mathbf{W}\left(k\right)\right)-\int_{s=0}^{t}\partial_{1}\tilde{w}_{i}\left(s,j_{i-1},j_{i}\right)ds\bigg|^{2}\bigg)^{1/2}\\
 & \leq\bigg(\frac{1}{n_{i-1}n_{i}}\sum_{j_{i-1}=1}^{n_{i-1}}\sum_{j_{i}=1}^{n_{i}}\bigg|\epsilon\sum_{k=0}^{\left\lfloor t/\epsilon\right\rfloor -1}\xi_{i}^{\mathbf{w}}\left(k\epsilon\right)\Delta_{i}^{{\bf w}}\left(z\left(k\right),j_{i-1},j_{i};\mathbf{W}\left(k\right)\right)-\epsilon\sum_{k=0}^{\left\lfloor t/\epsilon\right\rfloor -1}\partial_{1}\tilde{w}_{i}\left(k\epsilon,j_{i-1},j_{i}\right)\bigg|^{2}\bigg)^{1/2}+tK_{T}\left(1+B\right)\epsilon\\
 & \leq K\bigg(\frac{1}{n_{i-1}n_{i}}\sum_{j_{i-1}=1}^{n_{i-1}}\sum_{j_{i}=1}^{n_{i}}\left|Q_{1}\left(\left\lfloor t/\epsilon\right\rfloor ,j_{i-1},j_{i}\right)\right|^{2}+\left|Q_{2}\left(\left\lfloor t/\epsilon\right\rfloor ,j_{i-1},j_{i}\right)\right|^{2}\bigg)^{1/2}+tK_{T}\left(1+B\right)\epsilon,
\end{align*}
where we define 
\begin{align*}
Q_{1,i}\left(\left\lfloor t/\epsilon\right\rfloor ,j_{i-1},j_{i}\right) & =\epsilon\sum_{k=0}^{\left\lfloor t/\epsilon\right\rfloor -1}\mathbb{E}_{Z}\left[\left|\Delta_{i}^{\mathbf{w}}\left(Z,j_{i-1},j_{i};\mathbf{W}\left(k\right)\right)-\Delta_{i}^{\mathbf{w}}\left(Z,j_{i-1},j_{i};\tilde{W}\left(k\epsilon\right)\right)\right|\right],\\
Q_{2,i}\left(\left\lfloor t/\epsilon\right\rfloor ,j_{i-1},j_{i}\right) & =\bigg|\epsilon\sum_{k=0}^{\left\lfloor t/\epsilon\right\rfloor -1}\xi_{i}^{\mathbf{w}}\left(k\epsilon\right)\left(\Delta_{i}^{{\bf w}}\left(z\left(k\right),j_{i-1},j_{i};\mathbf{W}\left(k\right)\right)-\mathbb{E}_{Z}\left[\Delta_{i}^{\mathbf{w}}\left(Z,j_{i-1},j_{i};\mathbf{W}\left(k\right)\right)\right]\right)\bigg|.
\end{align*}
(Here $\sum_{k=0}^{\left\lfloor t/\epsilon\right\rfloor -1}=0$ if
$\left\lfloor t/\epsilon\right\rfloor =0$.) The task is then to bound
$Q_{1,i}$ and $Q_{2,i}$.

\paragraph{Bounding $Q_{1,i}$.}

We take note of a simple identity:
\begin{align*}
\frac{1}{\left|J\right|}\sum_{j\in J}\bigg(\epsilon\sum_{k=0}^{\left\lfloor t/\epsilon\right\rfloor -1}f\left(j,k\right)\bigg)^{2} & =\epsilon^{2}\sum_{k_{1}=0}^{\left\lfloor t/\epsilon\right\rfloor -1}\sum_{k_{2}=0}^{\left\lfloor t/\epsilon\right\rfloor -1}\frac{1}{\left|J\right|}\sum_{j\in J}f\left(j,k_{1}\right)f\left(j,k_{2}\right)\\
 & \leq\epsilon^{2}\sum_{k_{1}=0}^{\left\lfloor t/\epsilon\right\rfloor -1}\sum_{k_{2}=0}^{\left\lfloor t/\epsilon\right\rfloor -1}\bigg(\frac{1}{\left|J\right|}\sum_{j\in J}\left|f\left(j,k_{1}\right)\right|^{2}\bigg)^{1/2}\bigg(\frac{1}{\left|J\right|}\sum_{j\in J}\left|f\left(j,k_{2}\right)\right|^{2}\bigg)^{1/2}\\
 & =\bigg(\epsilon\sum_{k=0}^{\left\lfloor t/\epsilon\right\rfloor -1}\bigg(\frac{1}{\left|J\right|}\sum_{j\in J}\left|f\left(j,k\right)\right|^{2}\bigg)^{1/2}\bigg)^{2}.
\end{align*}
As such, by Assumption \ref{enu:Assump_backward}:
\begin{align*}
 & \bigg(\frac{1}{n_{i-1}n_{i}}\sum_{j_{i-1}=1}^{n_{i-1}}\sum_{j_{i}=1}^{n_{i}}\left|Q_{1,i}\left(\left\lfloor t/\epsilon\right\rfloor ,j_{i-1},j_{i}\right)\right|^{2}\bigg)^{1/2}\leq\epsilon\sum_{k=0}^{\left\lfloor t/\epsilon\right\rfloor -1}D_{i}\left(k\right)\\
 & \qquad\leq K\epsilon\sum_{k=0}^{\left\lfloor t/\epsilon\right\rfloor -1}\left(D_{i}^{\left(1\right)}\left(k\right)+G_{i}\left(k\right)+\mathscr{D}_{k\epsilon}\left(\mathbf{W},\tilde{W}\right)+F_{i}\left(k\right)\right),
\end{align*}
in which we define:
\begin{align*}
D_{i}\left(k\right) & =\bigg(\frac{1}{n_{i-1}n_{i}}\sum_{j_{i-1}=1}^{n_{i-1}}\sum_{j_{i}=1}^{n_{i}}\underset{Z\sim{\cal P}}{{\rm ess\text{-}sup}}\left|\Delta_{i}^{\mathbf{w}}\left(Z,j_{i-1},j_{i};\mathbf{W}\left(k\right)\right)-\Delta_{i}^{\mathbf{w}}\left(Z,j_{i-1},j_{i};\tilde{W}\left(k\epsilon\right)\right)\right|^{2}\bigg)^{1/2},\\
D_{i}^{\left(1\right)}\left(k\right) & =\bigg(\frac{1}{n_{i}}\sum_{j_{i}=1}^{n_{i}}\underset{Z\sim{\cal P}}{{\rm ess\text{-}sup}}\left(1+\left|\Delta_{i}^{\mathbf{H}}\left(Z,j_{i};\tilde{W}\left(k\epsilon\right)\right)\right|^{2}+\left|\Delta_{i}^{\mathbf{H}}\left(Z,j_{i};\mathbf{W}\left(k\right)\right)\right|^{2}\right)\\
 & \qquad\times\frac{1}{n_{i-1}}\sum_{j_{i-1}=1}^{n_{i-1}}\underset{Z\sim{\cal P}}{{\rm ess\text{-}sup}}\left|\mathbf{H}_{i-1}\left(X,j_{i-1};\tilde{W}\left(k\epsilon\right)\right)-\mathbf{H}_{i-1}\left(X,j_{i-1};\mathbf{W}\left(k\right)\right)\right|^{2}\bigg)^{1/2},\\
G_{i}\left(k\right) & =\bigg(\frac{1}{n_{i}}\sum_{j_{i}=1}^{n_{i}}\underset{Z\sim{\cal P}}{{\rm ess\text{-}sup}}\left|\Delta_{i}^{{\bf H}}\left(Z,j_{i};\tilde{W}\left(k\epsilon\right)\right)-\Delta_{i}^{\mathbf{H}}\left(Z,j_{i};\mathbf{W}\left(k\right)\right)\right|^{2}\bigg)^{1/2},\\
F_{i}\left(k\right) & =\bigg(\frac{1}{n_{i}}\sum_{j_{i}=1}^{n_{i}}\underset{Z\sim{\cal P}}{{\rm ess\text{-}sup}}\left|\mathbf{H}_{i}\left(X,j_{i};\tilde{W}\left(k\epsilon\right)\right)-\mathbf{H}_{i}\left(X,j_{i};\mathbf{W}\left(k\right)\right)\right|^{2}\bigg)^{1/2}.
\end{align*}
By Lemma \ref{lem:bounds NN a priori} and Fact 1, on the event ${\cal E}$:
\[
D_{i}^{\left(1\right)}\left(k\right)\leq K_{T}F_{i-1}\left(k\right),
\]
which implies
\[
\bigg(\frac{1}{n_{i-1}n_{i}}\sum_{j_{i-1}=1}^{n_{i-1}}\sum_{j_{i}=1}^{n_{i}}\left|Q_{1,i}\left(\left\lfloor t/\epsilon\right\rfloor ,j_{i-1},j_{i}\right)\right|^{2}\bigg)^{1/2}\leq K_{T}\epsilon\sum_{k=0}^{\left\lfloor t/\epsilon\right\rfloor -1}\left(F_{i-1}\left(k\right)+G_{i}\left(k\right)+\mathscr{D}_{k\epsilon}\left(\mathbf{W},\tilde{W}\right)+F_{i}\left(k\right)\right).
\]
We proceed with bounding $F_{i}$ and $G_{i}$.

To bound $F_{i}$, by Assumption \ref{enu:Assump_forward} and Cauchy-Schwarz's
inequality:
\begin{align*}
\left|F_{i}\left(k\right)\right|^{2} & \leq\frac{K}{n_{i-1}n_{i}}\sum_{j_{i-1}=1}^{n_{i-1}}\sum_{j_{i}=1}^{n_{i}}\left(1+\left|\tilde{w}_{i}\left(k\epsilon,j_{i-1},j_{i}\right)\right|^{2}+\left|\mathbf{w}_{i}\left(k,j_{i-1},j_{i}\right)\right|^{2}+\left|\tilde{b}_{i}\left(k\epsilon,j_{i}\right)\right|^{2}+\left|\mathbf{b}_{i}\left(k,j_{i}\right)\right|^{2}\right)\left|F_{i-1}\left(k\right)\right|^{2}\\
 & \quad+K\mathscr{D}_{k\epsilon}^{2}\left(\mathbf{W},\tilde{W}\right)\\
 & \leq K_{T}\left|F_{i-1}\left(k\right)\right|^{2}+K\mathscr{D}_{k\epsilon}^{2}\left(\mathbf{W},\tilde{W}\right),
\end{align*}
where the last display holds on the event ${\cal E}$ by Fact 1. Notice
that by Assumption \ref{enu:Assump_forward}, $\left|F_{1}\left(k\right)\right|\leq K\mathscr{D}_{k\epsilon}\left(\mathbf{W},\tilde{W}\right)$.
Therefore, on the event ${\cal E}$,
\[
\max_{1\leq i\leq L}\left|F_{i}\left(k\right)\right|\leq K_{T}\mathscr{D}_{k\epsilon}\left(\mathbf{W},\tilde{W}\right),
\]
which is the desired bound for $F_{i}$.

Next let us bound $G_{i}$. By Assumption \ref{enu:Assump_backward},
we have: 
\[
G_{i-1}\left(k\right)\leq K\left(G_{i}^{\left(1\right)}\left(k\right)+G_{i}^{\left(2\right)}\left(k\right)+G_{i}^{\left(3\right)}\left(k\right)+G_{i}^{\left(4\right)}\left(k\right)\right),
\]
in which
\begin{align*}
G_{i-1}^{\left(1\right)}\left(k\right) & =\bigg(\frac{1}{n_{i-1}}\sum_{j_{i-1}=1}^{n_{i-1}}\bigg(\frac{1}{n_{i}}\sum_{j_{i}=1}^{n_{i}}\left(1+\left|\tilde{w}_{i}\left(k\epsilon,j_{i-1},j_{i}\right)\right|+\left|\mathbf{w}_{i}\left(k,j_{i-1},j_{i}\right)\right|+\left|\tilde{b}_{i}\left(k\epsilon,j_{i}\right)\right|+\left|\mathbf{b}_{i}\left(k,j_{i}\right)\right|\right)\\
 & \qquad\times\underset{Z\sim{\cal P}}{{\rm ess\text{-}sup}}\left|\Delta_{i}^{{\bf H}}\left(Z,j_{i};\tilde{W}\left(k\epsilon\right)\right)-\Delta_{i}^{{\bf H}}\left(Z,j_{i};\mathbf{W}\left(k\right)\right)\right|\bigg)^{2}\bigg)^{1/2},\\
G_{i-1}^{\left(2\right)}\left(k\right) & =\bigg(\frac{1}{n_{i-1}}\sum_{j_{i-1}=1}^{n_{i-1}}\bigg(\frac{1}{n_{i}}\sum_{j_{i}=1}^{n_{i}}\underset{Z\sim{\cal P}}{{\rm ess\text{-}sup}}\left(1+\left|\Delta_{i}^{{\bf H}}\left(Z,j_{i};\tilde{W}\left(k\epsilon\right)\right)\right|+\left|\Delta_{i}^{{\bf H}}\left(Z,j_{i};\mathbf{W}\left(k\right)\right)\right|\right)\\
 & \qquad\times\left(\left|\tilde{w}_{i}\left(k\epsilon,j_{i-1},j_{i}\right)-\mathbf{w}_{i}\left(k,j_{i-1},j_{i}\right)\right|+\left|\tilde{b}_{i}\left(k\epsilon,j_{i}\right)-\mathbf{b}_{i}\left(k,j_{i}\right)\right|\right)\bigg)^{2}\bigg)^{1/2},\\
G_{i-1}^{\left(3\right)}\left(k\right) & =\bigg(\frac{1}{n_{i-1}}\sum_{j_{i-1}=1}^{n_{i-1}}\bigg(\frac{1}{n_{i}}\sum_{j_{i}=1}^{n_{i}}\underset{Z\sim{\cal P}}{{\rm ess\text{-}sup}}\left(1+\left|\Delta_{i}^{{\bf H}}\left(Z,j_{i};\tilde{W}\left(k\epsilon\right)\right)\right|+\left|\Delta_{i}^{{\bf H}}\left(Z,j_{i};\mathbf{W}\left(k\right)\right)\right|\right)\\
 & \qquad\times\left(1+\left|\tilde{w}_{i}\left(k\epsilon,j_{i-1},j_{i}\right)\right|+\left|\mathbf{w}_{i}\left(k,j_{i-1},j_{i}\right)\right|+\left|\tilde{b}_{i}\left(k\epsilon,j_{i}\right)\right|+\left|\mathbf{b}_{i}\left(k,j_{i}\right)\right|\right)\\
 & \qquad\times\underset{Z\sim{\cal P}}{{\rm ess\text{-}sup}}\left|\mathbf{H}_{i}\left(X,j_{i};\tilde{W}\left(k\epsilon\right)\right)-\mathbf{H}_{i}\left(X,j_{i};\mathbf{W}\left(k\right)\right)\right|\bigg)^{2}\bigg)^{1/2},\\
G_{i-1}^{\left(4\right)}\left(k\right) & =\bigg(\frac{1}{n_{i-1}}\sum_{j_{i-1}=1}^{n_{i-1}}\bigg(\frac{1}{n_{i}}\sum_{j_{i}=1}^{n_{i}}\underset{Z\sim{\cal P}}{{\rm ess\text{-}sup}}\left(1+\left|\Delta_{i}^{{\bf H}}\left(Z,j_{i};\tilde{W}\left(k\epsilon\right)\right)\right|+\left|\Delta_{i}^{{\bf H}}\left(Z,j_{i};\mathbf{W}\left(k\right)\right)\right|\right)\\
 & \qquad\times\left(1+\left|\tilde{w}_{i}\left(k\epsilon,j_{i-1},j_{i}\right)\right|+\left|\mathbf{w}_{i}\left(k,j_{i-1},j_{i}\right)\right|+\left|\tilde{b}_{i}\left(k\epsilon,j_{i}\right)\right|+\left|\mathbf{b}_{i}\left(k,j_{i}\right)\right|\right)\\
 & \qquad\times\underset{Z\sim{\cal P}}{{\rm ess\text{-}sup}}\left|\mathbf{H}_{i-1}\left(X,j_{i-1};\tilde{W}\left(k\epsilon\right)\right)-\mathbf{H}_{i-1}\left(X,j_{i-1};\mathbf{W}\left(k\right)\right)\right|\bigg)^{2}\bigg)^{1/2}.
\end{align*}
To bound $G_{i-1}^{\left(1\right)}$, by Cauchy-Schwarz's inequality
and Fact 1, on the event ${\cal E}$:
\begin{align*}
G_{i-1}^{\left(1\right)}\left(k\right) & \leq\bigg(\frac{1}{n_{i-1}n_{i}}\sum_{j_{i-1}=1}^{n_{i-1}}\sum_{j_{i}=1}^{n_{i}}\left(1+\left|\tilde{w}_{i}\left(k\epsilon,j_{i-1},j_{i}\right)\right|+\left|\mathbf{w}_{i}\left(k,j_{i-1},j_{i}\right)\right|+\left|\tilde{b}_{i}\left(k\epsilon,j_{i}\right)\right|+\left|\mathbf{b}_{i}\left(k,j_{i}\right)\right|\right)^{2}\\
 & \qquad\times\frac{1}{n_{i}}\sum_{j_{i}=1}^{n_{i}}\underset{Z\sim{\cal P}}{{\rm ess\text{-}sup}}\left|\Delta_{i}^{{\bf H}}\left(Z,j_{i};\tilde{W}\left(k\epsilon\right)\right)-\Delta_{i}^{{\bf H}}\left(Z,j_{i};\mathbf{W}\left(k\right)\right)\right|\bigg)^{2}\bigg)^{1/2}\\
 & \leq K_{T}G_{i}\left(k\right).
\end{align*}
We also have a bound on $G_{i-1}^{\left(2\right)}$ on the event ${\cal E}$:
\begin{align*}
G_{i-1}^{\left(2\right)}\left(k\right) & \leq K\bigg(\frac{1}{n_{i}}\sum_{j_{i}=1}^{n_{i}}\underset{Z\sim{\cal P}}{{\rm ess\text{-}sup}}\left(1+\left|\Delta_{i}^{{\bf H}}\left(Z,j_{i};\tilde{W}\left(k\epsilon\right)\right)\right|^{2}+\left|\Delta_{i}^{{\bf H}}\left(Z,j_{i};\mathbf{W}\left(k\right)\right)\right|^{2}\right)\bigg)^{1/2}\\
 & \qquad\times\bigg(\frac{1}{n_{i-1}n_{i}}\sum_{j_{i-1}=1}^{n_{i-1}}\sum_{j_{i}=1}^{n_{i}}\left|\tilde{w}_{i}\left(k\epsilon,j_{i-1},j_{i}\right)-\mathbf{w}_{i}\left(k,j_{i-1},j_{i}\right)\right|^{2}+\left|\tilde{b}_{i}\left(k\epsilon,j_{i}\right)-\mathbf{b}_{i}\left(k,j_{i}\right)\right|^{2}\bigg)^{1/2}\\
 & \leq K_{T}\mathscr{D}_{k\epsilon}\left(\mathbf{W},\tilde{W}\right).
\end{align*}
Similarly, by Fact 1 and Fact 2, on the event ${\cal E}$:
\begin{align*}
G_{i-1}^{\left(3\right)}\left(k\right) & \leq K_{T}\left(1+B\right)\frac{1}{n_{i}}\sum_{j_{i}=1}^{n_{i}}\underset{Z\sim{\cal P}}{{\rm ess\text{-}sup}}\left(1+\left|\Delta_{i}^{{\bf H}}\left(Z,j_{i};\tilde{W}\left(k\epsilon\right)\right)\right|+\left|\Delta_{i}^{{\bf H}}\left(Z,j_{i};\mathbf{W}\left(k\right)\right)\right|\right)\\
 & \qquad\times\underset{Z\sim{\cal P}}{{\rm ess\text{-}sup}}\left|\mathbf{H}_{i}\left(X,j_{i};\tilde{W}\left(k\epsilon\right)\right)-\mathbf{H}_{i}\left(X,j_{i};\mathbf{W}\left(k\right)\right)\right|\\
 & \leq K_{T}\left(1+B\right)\bigg(\frac{1}{n_{i}}\sum_{j_{i}=1}^{n_{i}}\underset{Z\sim{\cal P}}{{\rm ess\text{-}sup}}\left(1+\left|\Delta_{i}^{{\bf H}}\left(Z,j_{i};\tilde{W}\left(k\epsilon\right)\right)\right|^{2}+\left|\Delta_{i}^{{\bf H}}\left(Z,j_{i};\mathbf{W}\left(k\right)\right)\right|^{2}\right)\bigg)^{1/2}F_{i}\left(k\right)\\
 & \leq K_{T}\left(1+B\right)F_{i}\left(k\right),\\
G_{i-1}^{\left(4\right)}\left(k\right) & \leq K_{T}\left(1+B\right)\frac{1}{n_{i}}\sum_{j_{i}=1}^{n_{i}}\underset{Z\sim{\cal P}}{{\rm ess\text{-}sup}}\left(1+\left|\Delta_{i}^{{\bf H}}\left(Z,j_{i};\tilde{W}\left(k\epsilon\right)\right)\right|+\left|\Delta_{i}^{{\bf H}}\left(Z,j_{i};\mathbf{W}\left(k\right)\right)\right|\right)F_{i-1}\left(k\right)\\
 & \leq K_{T}\left(1+B\right)F_{i-1}\left(k\right).
\end{align*}
Therefore on the event ${\cal E}$: 
\begin{align*}
G_{i-1}\left(k\right) & \leq K_{T}\left(G_{i}\left(k\right)+\mathscr{D}_{k\epsilon}\left(\mathbf{W},\tilde{W}\right)+\left(1+B\right)\left(F_{i}\left(k\right)+F_{i-1}\left(k\right)\right)\right)\\
 & \leq K_{T}\left(G_{i}\left(k\right)+\left(1+B\right)\mathscr{D}_{k\epsilon}\left(\mathbf{W},\tilde{W}\right)\right).
\end{align*}
Notice that by Assumption \ref{enu:Assump_backward},
\[
G_{L}\left(k\right)\leq KF_{L}\left(k\right)\leq K_{T}\mathscr{D}_{k\epsilon}\left(\mathbf{W},\tilde{W}\right).
\]
Therefore, on the event ${\cal E}$,
\[
\max_{1\leq i\leq L}\left|G_{i}\left(k\right)\right|\leq K_{T}\left(1+B\right)\mathscr{D}_{k\epsilon}\left(\mathbf{W},\tilde{W}\right),
\]
which is the desired bound for $G_{i}$.

Together these bounds yield
\[
\max_{2\leq i\leq L}\bigg(\frac{1}{n_{i-1}n_{i}}\sum_{j_{i-1}=1}^{n_{i-1}}\sum_{j_{i}=1}^{n_{i}}\left|Q_{1,i}\left(\left\lfloor t/\epsilon\right\rfloor ,j_{i-1},j_{i}\right)\right|^{2}\bigg)^{1/2}\leq\epsilon K_{T}\left(1+B\right)\sum_{k=0}^{\left\lfloor t/\epsilon\right\rfloor -1}\mathscr{D}_{k\epsilon}\left(\mathbf{W},\tilde{W}\right).
\]

\paragraph{Bounding $Q_{2,i}$.}

For brevity, let us write 
\begin{align*}
Z_{k} & =\xi_{i}^{\mathbf{w}}\left(k\epsilon\right)\left(\Delta_{i}^{{\bf w}}\left(z\left(k\right),j_{i-1},j_{i};\mathbf{W}\left(k\right)\right)-\mathbb{E}_{Z}\left[\Delta_{i}^{\mathbf{w}}\left(Z,j_{i-1},j_{i};\mathbf{W}\left(k\right)\right)\right]\right),\\
\underline{Z}_{k} & =\sum_{\ell=0}^{k-1}Z_{\ell},\qquad\underline{Z}_{0}=0.
\end{align*}
Let ${\cal F}_{k}$ be the sigma-algebra generated by $\left\{ z\left(s\right):\;s\in\left\{ 0,...,k-1\right\} \right\} $.
Recall that it is independent of the samples $\left\{ C_{1}\left(j_{1}\right),...,C_{L}\left(j_{L}\right):\;j_{i}\in\left[n_{i}\right],\;i=1,...,L\right\} $
and hence the event ${\cal E}$. Note that $\left\{ \underline{Z}_{k}\right\} _{k\in\mathbb{N}}$
is a martingale adapted to $\left\{ {\cal F}_{k}\right\} _{k\in\mathbb{N}}$.
Furthermore, for $k\leq T/\epsilon$, the martingale difference is
bounded: 
\[
\left|Z_{k}\right|\leq K\underset{Z\sim{\cal P}}{{\rm ess\text{-}sup}}\left|\Delta_{i}^{{\bf w}}\left(Z,j_{i-1},j_{i};\mathbf{W}\left(k\right)\right)\right|\leq K\left(1+\underset{Z\sim{\cal P}}{{\rm ess\text{-}sup}}\left|\Delta_{i}^{{\bf H}}\left(Z,j_{i};\mathbf{W}\left(k\right)\right)\right|\right)\leq K_{T}\left(1+B\right),
\]
which holds on the event ${\cal E}$, by Assumptions \ref{enu:Assump_lrSchedule}
and \ref{enu:Assump_backward} and Fact 2. Therefore, by Theorem \ref{thm:azuma-hilbert},
we have: 
\[
\mathbb{P}\left(\max_{u\in\left\{ 0,1,...,T/\epsilon\right\} }Q_{2,i}\left(u,j_{i-1},j_{i}\right)\geq\left(1+B\right)\zeta;{\cal E}\right)\leq2\exp\left(-\frac{\zeta^{2}}{K_{T}T\epsilon}\right).
\]

\paragraph{Putting together.}

Applying the union bound to the bound on $Q_{2,i}$, we then get that
on the event ${\cal E}$, with probability at least $1-2n_{i}n_{i-1}\exp\left(-\zeta^{2}/\left(K_{T}T\epsilon\right)\right)$,
for all $t\leq T$, 
\begin{align*}
 & \sup_{s\leq t,\;j_{i-1}\in\left[n_{i-1}\right],\;j_{i}\in\left[n_{i}\right]}\bigg(\frac{1}{n_{i-1}n_{i}}\sum_{j_{i-1}=1}^{n_{i-1}}\sum_{j_{i}=1}^{n_{i}}\left|\mathbf{w}_{i}\left(\left\lfloor s/\epsilon\right\rfloor ,j_{i-1},j_{i}\right)-\tilde{w}_{i}\left(s,j_{i-1},j_{i}\right)\right|^{2}\bigg)^{1/2}\\
 & \leq K_{T}\left(1+B\right)\bigg(\epsilon\sum_{k=0}^{\left\lfloor t/\epsilon\right\rfloor -1}\mathscr{D}_{k\epsilon}\left(\mathbf{W},\tilde{W}\right)+\zeta+\epsilon\bigg).
\end{align*}
One can obtain similar bounds for $\mathbf{b}_{i}$ and $\mathbf{w}_{1}$.
Together these bounds yield that with probability at least 
\[
1-2\bigg(n_{1}+\sum_{i=2}^{L}n_{i}n_{i-1}\bigg)\exp\left(-\zeta^{2}/\left(K_{T}T\epsilon\right)\right)-KLn_{\max}\exp\left(-Kn_{\min}^{1/52}\right),
\]
we have for all $t\leq T$, 
\[
\mathscr{D}_{\left\lfloor t/\epsilon\right\rfloor \epsilon}\left(\tilde{W},\mathbf{W}\right)\leq K_{T}\left(1+B\right)\bigg(\epsilon\sum_{k=0}^{\left\lfloor t/\epsilon\right\rfloor -1}\mathscr{D}_{k\epsilon}\left(\mathbf{W},\tilde{W}\right)+\zeta+\epsilon\bigg),
\]
which implies, by Gronwall's lemma, 
\[
\mathscr{D}_{T}\left(\tilde{W},\mathbf{W}\right)\leq\left(\zeta+\epsilon\right)\exp\left(K_{T}\left(1+B\right)\right).
\]
Choosing $\zeta=\sqrt{K_{T}\epsilon\log\left(2\left(n_{1}+\sum_{i=2}^{L}n_{i}n_{i-1}\right)/\delta\right)}$
completes the proof.
\end{proof}

\subsubsection{Proof of Proposition \ref{prop:truncation}}
\begin{proof}[Proof of Proposition \ref{prop:truncation}]
We again reuse the notation $K_{t}$ from the proof of Proposition
\ref{prop:particle coupling - bounded}. Note that $K_{t}\leq K_{T}$
for all $t\leq T$. It is easy to see from Theorem \ref{thm:existence ODE}
that the trajectory $\underline{W}\left(t\right)$ exists and is unique.
Let us recall the mapping $F$ and the space ${\cal W}_{T}$ from
the proof Theorem \ref{thm:existence ODE}; we note that $F$ is associated
with the initialization $W\left(0\right)$. Since $F\left(\underline{W}\right)\left(t\right)-W\left(0\right)=\underline{W}\left(t\right)-\underline{W}\left(0\right)$
and $W$ is a fixed point of $F$, we have:
\[
\left\Vert W-\underline{W}\right\Vert _{t}\leq\left\Vert W-\underline{W}\right\Vert _{0}+\left\Vert W-F\left(\underline{W}\right)\right\Vert _{t}=\left\Vert W-\underline{W}\right\Vert _{0}+\left\Vert F\left(W\right)-F\left(\underline{W}\right)\right\Vert _{t}.
\]
Due to truncation, it is immediate that
\[
\left|\underline{w}_{1}\left(0,c_{1}\right)\right|=\left|w_{1}\left(0,c_{1}\right)\right|,\;\left|\underline{w}_{i}\left(0,c_{i-1},c_{i}\right)\right|\leq\left|w_{i}\left(0,c_{i-1},c_{i}\right)\right|,\;\left|\underline{b}_{i}\left(0,c_{i}\right)\right|\leq\left|b_{i}\left(0,c_{i}\right)\right|,\quad2\leq i\leq L.
\]
As such, by repeating the argument of Lemma \ref{lem:bounds MF a priori},
one can show that $\underline{W}\in{\cal W}_{T}$ and that
\[
\mathbb{P}\left(\mathsf{max}_{T}^{w}\left(\underline{W}\right)\geq K_{0}\left(T\right)u\right)\leq2Le^{1-K_{1}u^{2}}\qquad\forall u\geq0.
\]
Thus, Lemma \ref{lem:difference MF} gives:
\[
\left\Vert F\left(W\right)-F\left(\underline{W}\right)\right\Vert _{t}\leq K_{T}\left(\left(1+B\right)\int_{0}^{t}\left\Vert W-\underline{W}\right\Vert _{s}ds+e^{-KB^{2}}\right),
\]
which implies, by the previous bound,
\[
\left\Vert W-\underline{W}\right\Vert _{t}\leq K_{T}\left(\left(1+B\right)\int_{0}^{t}\left\Vert W-\underline{W}\right\Vert _{s}ds+e^{-KB^{2}}\right)+\left\Vert W-\underline{W}\right\Vert _{0}.
\]
Hence Gronwall's lemma yields:
\[
\left\Vert W-\underline{W}\right\Vert _{T}\leq\left(\left\Vert W-\underline{W}\right\Vert _{0}+e^{-KB^{2}}\right)e^{K_{T}\left(1+B\right)}.
\]
Notice that for $2\leq i\leq L$:
\begin{align*}
\mathbb{E}\left[\left|w_{i}^{0}\left(C_{i-1},C_{i}\right)-\underline{w}_{i}\left(0,C_{i-1},C_{i}\right)\right|^{2}\right] & =\mathbb{E}\left[\left|w_{i}^{0}\left(C_{i-1},C_{i}\right)-B\right|^{2}\mathbb{I}\left(\left|w_{i}^{0}\left(C_{i-1},C_{i}\right)\right|>B\right)\right]\\
 & \leq\mathbb{E}\left[\left|w_{i}^{0}\left(C_{i-1},C_{i}\right)\right|^{2}\mathbb{I}\left(\left|w_{i}^{0}\left(C_{i-1},C_{i}\right)\right|>B\right)\right]\\
 & \leq\mathbb{E}\left[\left|w_{i}^{0}\left(C_{i-1},C_{i}\right)\right|^{4}\right]^{1/2}\mathbb{P}\left(\left|w_{i}^{0}\left(C_{i-1},C_{i}\right)\right|>B\right)^{1/2}\\
 & \leq Ke^{-KB^{2}},
\end{align*}
where the last displays comes from Assumption \ref{assump:init} and
in particular we have \cite{vershynin2010introduction}:
\[
\mathbb{P}\left(\left|w_{i}^{0}\left(C_{i-1},C_{i}\right)\right|\geq r\right)\leq Ke^{-Kr^{2}},\qquad\forall r\geq0.
\]
Similarly,
\[
\mathbb{E}\left[\left|b_{i}^{0}\left(C_{i}\right)-\underline{b}_{i}\left(0,C_{i}\right)\right|^{2}\right]\leq Ke^{-KB^{2}}.
\]
Also recall that $\underline{w}_{1}\left(0,c_{1}\right)=w_{1}^{0}\left(c_{1}\right)$.
As such, a similar bound holds for $\left\Vert W-\underline{W}\right\Vert _{0}$
and this gives the desired bound on $\left\Vert W-\underline{W}\right\Vert _{T}$.

The derivation for $\left\Vert \tilde{W}-\underline{\tilde{W}}\right\Vert _{T}$
is similar. Indeed Lemma \ref{lem:initialization_compare} indicates
that for any fixed $r\geq0$, with probability at least $1-KLn_{\max}\exp\left(-Ke^{-Kr^{2}}n_{\min}^{1/52}\right)$,
we have:
\[
\interleave\tilde{W}\interleave_{0}\leq\interleave W\interleave_{0}+e^{-Kr^{2}}\leq K,
\]
as well as that for all $i\in\left\{ 2,...,L\right\} $,
\begin{align*}
\frac{1}{n_{i-1}n_{i}}\sum_{j_{i-1}=1}^{n_{i-1}}\sum_{j_{i}=1}^{n_{i}}\mathbb{I}\left(\left|w_{i}^{0}\left(C_{i-1}\left(j_{i-1}\right),C_{i}\left(j_{i}\right)\right)\right|\geq r\right) & \leq\mathbb{P}\left(\left|w_{i}^{0}\left(C_{i-1},C_{i}\right)\right|\geq r\right)+e^{-Kr^{2}}\leq Ke^{-Kr^{2}},\\
\frac{1}{n_{i}}\sum_{j_{i}=1}^{n_{i}}\mathbb{I}\left(\left|b_{i}^{0}\left(C_{i}\left(j_{i}\right)\right)\right|\geq r\right) & \leq\mathbb{P}\left(\left|b_{i}^{0}\left(C_{i}\right)\right|\geq r\right)+e^{-Kr^{2}}\leq Ke^{-Kr^{2}}.
\end{align*}
By taking $r=B$ and performing an argument similar to the bounding
of $\left\Vert W-\underline{W}\right\Vert _{T}$, we obtain:
\[
\left\Vert \tilde{W}-\underline{\tilde{W}}\right\Vert _{T}\leq\left(\left\Vert \tilde{W}-\underline{\tilde{W}}\right\Vert _{0}+e^{-KB^{2}}\right)e^{K_{T}\left(1+B\right)}\leq Ke^{-KB^{2}+K_{T}\left(1+B\right)},
\]
with probability at least $1-KLn_{\max}\exp\left(-Ke^{-KB^{2}}n_{\min}^{1/52}\right)$.
The derivation for $\left\Vert \mathbf{W}-\underline{\mathbf{W}}\right\Vert _{T}$
is also similar.
\end{proof}

\subsection{Proofs of Corollaries \ref{cor:gradient descent quality}, \ref{cor:MF_insensitivity}
and \ref{cor:MF-twolayers}}
\begin{lem}
\label{lem:a priori MF - time difference}Consider the MF trajectory
$W\left(t\right)$, $t\leq T$, under Assumptions \ref{enu:Assump_lrSchedule}-\ref{enu:Assump_backward}
and \ref{assump:init}. For any $\zeta\geq0$, $\left\Vert W-W_{\zeta}\right\Vert _{T}\leq K_{T+\zeta}\zeta$,
where $W_{\zeta}\left(t\right)=W\left(t+\zeta\right)$ and $K_{T+\zeta}$
is a finite constant that depends on the initialization $W\left(0\right)$
and grows continuously with $\zeta$.
\end{lem}

\begin{proof}
We reuse the notation $K_{t}$ from the proof of Proposition \ref{prop:particle coupling - bounded}.
By Assumption \ref{enu:Assump_backward} and Lemma \ref{lem:bounds NN a priori},
for $2\leq i\leq L$,
\[
\mathbb{E}\left[\sup_{t\leq T+\zeta}\mathbb{E}_{Z}\left[\left|\Delta_{i}^{w}\left(t,Z,C_{i-1},C_{i}\right)\right|^{2}\right]\right]\leq K\bigg(1+\mathbb{E}\left[\sup_{t\leq T+\zeta}\mathbb{E}_{Z}\left[\left|\Delta_{i}^{H}\left(t,Z,C_{i}\right)\right|^{2}\right]\right]\bigg)\leq K_{T+\zeta},
\]
and therefore, by Assumption \ref{enu:Assump_lrSchedule},
\[
\mathbb{E}\left[\sup_{t\leq T}\left|w_{i}\left(t+\zeta,C_{i-1},C_{i}\right)-w_{i}\left(t,C_{i-1},C_{i}\right)\right|^{2}\right]^{1/2}\leq K_{T+\zeta}\zeta.
\]
One can also deduce a similar bound for $b_{i}$ and $w_{1}$.
\end{proof}
\begin{proof}[Proof of Corollary \ref{cor:gradient descent quality}]
We reuse the notation $K_{t}$ from the proof of Proposition \ref{prop:particle coupling - bounded}.
We have the following decomposition for $i\in\left[L\right]$:
\begin{align*}
 & \bigg|\frac{1}{n_{i}}\sum_{j_{i}=1}^{n_{i}}\mathbb{E}_{Z}\left[\psi\left({\bf H}_{i}\left(\left\lfloor t/\epsilon\right\rfloor ,X,j_{i}\right)\right)\right]-\mathbb{E}_{Z}\mathbb{E}_{C_{i}}\left[\psi\left(H_{i}\left(t,X,C_{i}\right)\right)\right]\bigg|\\
 & \leq\frac{1}{n_{i}}\sum_{j_{i}=1}^{n_{i}}\left|\mathbb{E}_{Z}\left[\psi\left({\bf H}_{i}\left(\left\lfloor t/\epsilon\right\rfloor ,X,j_{i}\right)\right)\right]-\mathbb{E}_{Z}\left[\psi\left(H_{i}\left(\left\lfloor t/\epsilon\right\rfloor \epsilon,X,C_{i}\left(j_{i}\right)\right)\right)\right]\right|\\
 & \qquad+\bigg|\frac{1}{n_{i}}\sum_{j_{i}=1}^{n_{i}}\mathbb{E}_{Z}\left[\psi\left(H_{i}\left(\left\lfloor t/\epsilon\right\rfloor \epsilon,X,C_{i}\left(j_{i}\right)\right)\right)\right]-\mathbb{E}_{Z}\mathbb{E}_{C_{i}}\left[\psi\left(H_{i}\left(\left\lfloor t/\epsilon\right\rfloor \epsilon,X,C_{i}\right)\right)\right]\bigg|\\
 & \qquad+\mathbb{E}_{Z}\mathbb{E}_{C_{i}}\left[\left|\psi\left(H_{i}\left(\left\lfloor t/\epsilon\right\rfloor \epsilon,X,C_{i}\right)\right)-\psi\left(H_{i}\left(t,X,C_{i}\right)\right)\right|\right]\\
 & \leq K\bigg(\frac{1}{n_{i}}\sum_{j_{i}=1}^{n_{i}}\mathbb{E}_{Z}\left[\left|{\bf H}_{i}\left(\left\lfloor t/\epsilon\right\rfloor ,X,j_{i}\right)-H_{i}\left(\left\lfloor t/\epsilon\right\rfloor \epsilon,X,C_{i}\left(j_{i}\right)\right)\right|^{2}\right]\bigg)^{1/2}\\
 & \qquad+\bigg|\frac{1}{n_{i}}\sum_{j_{i}=1}^{n_{i}}\mathbb{E}_{Z}\left[\psi\left(H_{i}\left(\left\lfloor t/\epsilon\right\rfloor \epsilon,X,C_{i}\left(j_{i}\right)\right)\right)\right]-\mathbb{E}_{Z}\mathbb{E}_{C_{i}}\left[\psi\left(H_{i}\left(\left\lfloor t/\epsilon\right\rfloor \epsilon,X,C_{i}\right)\right)\right]\bigg|\\
 & \qquad+K\mathbb{E}_{Z}\mathbb{E}_{C_{i}}\left[\left|H_{i}\left(\left\lfloor t/\epsilon\right\rfloor \epsilon,X,C_{i}\right)-H_{i}\left(t,X,C_{i}\right)\right|\right]\\
 & =Q_{1,i}\left(t\right)+Q_{2,i}\left(t\right)+Q_{3,i}\left(t\right),
\end{align*}
where we use the fact $\psi$ is $K$-Lipschitz and Cauchy-Schwarz's
inequality. We provide bounds on each term. Note that by the fact
$\psi$ is $K$-Lipschitz and Assumption \ref{enu:Assump_forward}:
\[
\left|\mathbb{E}_{Z}\left[\psi\left(Y,\hat{\mathbf{y}}\left(\left\lfloor t/\epsilon\right\rfloor ,X\right)\right)\right]-\mathbb{E}_{Z}\left[\psi\left(Y,\hat{y}\left(t,X\right)\right)\right]\right|\leq K\mathbb{E}_{Z}\left[\left|\mathbf{H}_{L}\left(\left\lfloor t/\epsilon\right\rfloor ,X,1\right)-H_{L}\left(t,X,1\right)\right|\right],
\]
and as such, bounding $Q_{1,L}$ gives the last claim in the corollary.

\paragraph*{Bounding $Q_{1,i}$.}

By Assumption \ref{enu:Assump_forward} and Cauchy-Schwarz's inequality,
for $i\geq2$:
\begin{align*}
\left|Q_{1,i}\left(t\right)\right|^{2} & \leq\frac{K}{n_{i-1}n_{i}}\sum_{j_{i-1}=1}^{n_{i-1}}\sum_{j_{i}=1}^{n_{i}}\Big(1+\left|\mathbf{w}_{i}\left(\left\lfloor t/\epsilon\right\rfloor ,j_{i-1},j_{i}\right)\right|^{2}+\left|w_{i}\left(t,C_{i-1}\left(j_{i-1}\right),C_{i}\left(j_{i}\right)\right)\right|^{2}\\
 & \qquad\qquad\qquad+\left|\mathbf{b}_{i}\left(\left\lfloor t/\epsilon\right\rfloor ,j_{i}\right)\right|^{2}+\left|b_{i}\left(t,C_{i}\left(j_{i}\right)\right)\right|^{2}\Big)\left|Q_{1,i-1}\left(t\right)\right|^{2}\\
 & \qquad+K\mathscr{D}_{t}^{2}\left(W,\mathbf{W}\right)\\
 & \leq\frac{K}{n_{i-1}n_{i}}\sum_{j_{i-1}=1}^{n_{i-1}}\sum_{j_{i}=1}^{n_{i}}\left(1+\left|w_{i}\left(t,C_{i-1}\left(j_{i-1}\right),C_{i}\left(j_{i}\right)\right)\right|^{2}+\left|b_{i}\left(t,C_{i}\left(j_{i}\right)\right)\right|^{2}+\mathscr{D}_{t}^{2}\left(W,\mathbf{W}\right)\right)\left|Q_{1,i-1}\left(t\right)\right|^{2}\\
 & \qquad+K\mathscr{D}_{t}^{2}\left(W,\mathbf{W}\right)\\
 & \stackrel{\left(a\right)}{\leq}K_{T}\left(1+\mathscr{D}_{t}^{2}\left(W,\mathbf{W}\right)\right)\left|Q_{1,i-1}\left(t\right)\right|^{2}+K\mathscr{D}_{t}^{2}\left(W,\mathbf{W}\right)\\
 & \leq K_{T}\left(1+\mathscr{D}_{t}^{2}\left(W,\mathbf{W}\right)\right)\left|Q_{1,i-1}\left(t\right)\right|^{2},
\end{align*}
where $\left(a\right)$ holds for all $t\leq T$ and all $i\geq2$
with probability at least $1-KLn_{\max}\exp\left(-Kn_{\min}^{1/52}\right)$,
by Lemmas \ref{lem:initialization_compare} and \ref{lem:bounds NN a priori}.
Also from Assumption \ref{enu:Assump_forward}: $Q_{1,1}\left(t\right)\leq K\mathscr{D}_{t}\left(W,\mathbf{W}\right)$.
We thus have:
\[
\max_{i\in\left[L\right]}\sup_{t\leq T}Q_{1,i}\left(t\right)\leq K_{T}\left(1+\mathscr{D}_{T}^{L}\left(W,\mathbf{W}\right)\right)\mathscr{D}_{T}\left(W,\mathbf{W}\right).
\]
with probability at least $1-KLn_{\max}\exp\left(-Kn_{\min}^{1/52}\right)$.

\paragraph*{Bounding $Q_{2,i}$.}

Recall that $\left\{ C_{i}\left(j_{i}\right)\right\} _{j_{i}\in\left[n_{i}\right]}$
are $\eta_{i}$-independent and $\eta_{i}\leq n_{i}^{-1/2}$. Since
$\psi$ is $K$-bounded, we have by Theorem \ref{thm:iid-hilbert-concentration}
and the union bound that
\[
\sup_{t\le T}Q_{2,i}\left(t\right)\leq\sqrt{\frac{K}{n_{i}}\log\left(\frac{KT}{\epsilon\delta}\right)}
\]
with probability at least $1-\delta$.

\paragraph*{Bounding $Q_{3,i}$.}

By Assumption \ref{enu:Assump_forward} and Lemma \ref{lem:bounds NN a priori},
for all $t\leq T$ and $i\geq2$:
\begin{align*}
\left|Q_{3,i}\left(t\right)\right|^{2} & \leq K\mathbb{E}\left[1+\left|w_{i}\left(t,C_{i-1},C_{i}\right)\right|^{2}+\left|b_{i}\left(t,C_{i}\right)\right|^{2}+\left|w_{i}\left(\left\lfloor t/\epsilon\right\rfloor \epsilon,C_{i-1},C_{i}\right)\right|^{2}+\left|b_{i}\left(\left\lfloor t/\epsilon\right\rfloor \epsilon,C_{i}\right)\right|^{2}\right]\left|Q_{3,i-1}\left(t\right)\right|^{2}\\
 & \quad+K\mathbb{E}\left[\left|w_{i}\left(t,C_{i-1},C_{i}\right)-w_{i}\left(\left\lfloor t/\epsilon\right\rfloor \epsilon,C_{i-1},C_{i}\right)\right|^{2}+\left|b_{i}\left(t,C_{i}\right)-b_{i}\left(\left\lfloor t/\epsilon\right\rfloor \epsilon,C_{i}\right)\right|^{2}\right]\\
 & \leq K_{T}\left|Q_{3,i-1}\left(t\right)\right|^{2}+K\mathbb{E}\left[\left|w_{i}\left(t,C_{i-1},C_{i}\right)-w_{i}\left(\left\lfloor t/\epsilon\right\rfloor \epsilon,C_{i-1},C_{i}\right)\right|^{2}+\left|b_{i}\left(t,C_{i}\right)-b_{i}\left(\left\lfloor t/\epsilon\right\rfloor \epsilon,C_{i}\right)\right|^{2}\right].
\end{align*}
Similarly,
\[
\left|Q_{3,1}\left(t\right)\right|^{2}\leq K\mathbb{E}\left[\left|w_{1}\left(t,C_{1}\right)-w_{1}\left(\left\lfloor t/\epsilon\right\rfloor \epsilon,C_{1}\right)\right|^{2}\right].
\]
We thus obtain from Lemma \ref{lem:a priori MF - time difference}:
\[
\max_{i\in\left[L\right]}\sup_{t\leq T}Q_{3,i}\left(t\right)\leq K_{T}\epsilon.
\]

\paragraph*{Putting together.}

All previous bounds show that
\begin{align*}
 & \sup_{t\leq T}\bigg|\frac{1}{n_{i}}\sum_{j_{i}=1}^{n_{i}}\mathbb{E}_{Z}\left[\psi\left({\bf H}_{i}\left(\left\lfloor t/\epsilon\right\rfloor ,X,j_{i}\right)\right)\right]-\mathbb{E}_{Z}\mathbb{E}_{C_{i}}\left[\psi\left(H_{i}\left(t,X,C_{i}\right)\right)\right]\bigg|\\
 & \leq K_{T}\left(1+\mathscr{D}_{T}^{L}\left(W,\mathbf{W}\right)\right)\mathscr{D}_{T}\left(W,\mathbf{W}\right)+\sqrt{\frac{K}{n_{i}}\log\left(\frac{KT}{\epsilon\delta}\right)}+K_{T}\epsilon
\end{align*}
with probability at least $1-\delta-KLn_{\max}\exp\left(-Kn_{\min}^{1/52}\right)$.
Together with Theorem \ref{thm:gradient descent coupling}, we obtain
the claim.
\end{proof}
\begin{proof}[Proof of Corollary \ref{cor:MF_insensitivity}]
In the following, for a set $J=\left\{ N_{1},...,N_{L}\right\} $
with $N_{L}=1$, we write $J\to\infty$ to mean that $N_{1},...,N_{L-1}\to\infty$
such that for $N_{\max}=\max J$ and $N_{\min}=\min\left\{ N_{1},...,N_{L-1}\right\} $,
we have $N_{\min}\to\infty$ and $N_{\min}^{-c}\log N_{\max}\to0$
for any $c>0$.

For a given $T\geq0$ and a set of integers $I=\left\{ n_{1},...,n_{L}\right\} $,
for any two sets ${\cal W}^{\left(1\right)}$ and ${\cal W}^{\left(2\right)}$
of the form
\[
{\cal W}^{\left(1\right)}=\Big\{ w_{1}^{\left(1\right)}\left(t,r_{1}\right),\;w_{i}^{\left(1\right)}\left(t,r_{i-1},r_{i}\right),\;b_{i}^{\left(1\right)}\left(t,r_{i}\right):\quad r_{i}\in\left[n_{i}\right],\;i\in\left[L\right],\;t\in\left[0,T\right]\Big\},
\]
and similar for ${\cal W}^{\left(2\right)}$, let us equip a distance
metric:
\begin{align*}
d_{I,T}\left({\cal W}^{\left(1\right)},{\cal W}^{\left(2\right)}\right) & =\max\bigg(\max_{1\leq i\leq L}d_{I,T}^{w,i}\left({\cal W}^{\left(1\right)},{\cal W}^{\left(2\right)}\right),\;\max_{2\leq i\leq L}d_{I,T}^{b,i}\left({\cal W}^{\left(1\right)},{\cal W}^{\left(2\right)}\right)\bigg),\\
d_{I,T}^{w,1}\left({\cal W}^{\left(1\right)},{\cal W}^{\left(2\right)}\right) & =\bigg(\frac{1}{n_{1}}\sum_{r_{1}=1}^{n_{1}}\sup_{t\leq T}\left|w_{1}^{\left(1\right)}\left(t,r_{1}\right)-w_{1}^{\left(2\right)}\left(t,r_{1}\right)\right|^{2}\bigg)^{1/2},\\
d_{I,T}^{w,i}\left({\cal W}^{\left(1\right)},{\cal W}^{\left(2\right)}\right) & =\bigg(\frac{1}{n_{i-1}n_{i}}\sum_{r_{i-1}=1}^{n_{i-1}}\sum_{r_{i}=1}^{n_{i}}\sup_{t\leq T}\left|w_{i}^{\left(1\right)}\left(t,r_{i-1},r_{i}\right)-w_{i}^{\left(2\right)}\left(t,r_{i-1},r_{i}\right)\right|^{2}\bigg)^{1/2},\\
d_{I,T}^{b,i}\left({\cal W}^{\left(1\right)},{\cal W}^{\left(2\right)}\right) & =\bigg(\frac{1}{n_{i}}\sum_{r_{i}=1}^{n_{i}}\sup_{t\leq T}\left|b_{i}^{\left(1\right)}\left(t,r_{i}\right)-b_{i}^{\left(2\right)}\left(t,r_{i}\right)\right|^{2}\bigg)^{1/2},\qquad2\leq i\leq L.
\end{align*}
Let us also consider the space ${\cal F}_{I,T}$ of $1$-bounded Lipchitz
functions $f$ w.r.t. to this distance metric:
\[
\left|f\left({\cal W}^{\left(1\right)}\right)-f\left({\cal W}^{\left(2\right)}\right)\right|\leq2\land d_{I,T}\left({\cal W}^{\left(1\right)},{\cal W}^{\left(2\right)}\right).
\]

\paragraph*{Step 1: Coupling via finite-width networks.}

Recall that $\left(\Omega,P,\left\{ w_{i}^{0}\right\} _{i\in\left[L\right]},\left\{ b_{i}^{0}\right\} _{2\leq i\leq L}\right)$
satisfies Assumption \ref{assump:neuronal-embedding}, i.e. $\bar{\eta}$-independence.
As such, for each index $J=\left\{ N_{1},...,N_{L}\right\} $ of $\mathsf{Init}$,
one can find a sampling rule for which the samples $\left\{ C_{i}\left(j_{i}\right)\right\} _{j_{i}\in\left[N_{i}\right]}$
are $\eta_{i}$-independent for $i\leq L-1$ and $\eta_{i}\to0$ as
$N_{i}\to\infty$. Then one obtains a neural network initialization
${\bf W}(0)$ with law $\rho$ by setting 
\[
{\bf w}_{1}\left(0,j_{1}\right)=w_{1}\left(0,C_{1}(j_{1})\right),\;{\bf w}_{i}\left(0,j_{i-1},j_{i}\right)=w_{i}\left(0,C_{i-1}\left(j_{i-1}\right),C_{i}\left(j_{i}\right)\right),
\]
\[
{\bf b}_{i}\left(0,j_{i}\right)=b_{i}\left(0,C_{i}\left(j_{i}\right)\right),\quad2\leq i\leq L.
\]
Similarly using $\left(\hat{\Omega},\hat{P},\left\{ \hat{w}_{i}^{0}\right\} _{i\in\left[L\right]},\left\{ \hat{b}_{i}^{0}\right\} _{2\leq i\leq L}\right)$we
obtain $\hat{{\bf W}}(0)$ with the same law $\rho$ by setting 
\[
\hat{{\bf w}}_{1}\left(0,j_{1}\right)=\hat{w}_{1}\left(0,\hat{C}_{1}(j_{1})\right),\;\hat{{\bf w}}_{i}\left(0,j_{i-1},j_{i}\right)=\hat{w}_{i}\left(0,\hat{C}_{i-1}\left(j_{i-1}\right),\hat{C}_{i}\left(j_{i}\right)\right),
\]
\[
\hat{{\bf b}}_{i}\left(0,j_{i}\right)=\hat{b}_{i}\left(0,\hat{C}_{i}\left(j_{i}\right)\right),\quad2\leq i\leq L,
\]
where $\left\{ \hat{C}_{i}\left(j_{i}\right)\right\} _{j_{i}\in\left[N_{i}\right]}$
are $\eta_{i}$-independent for $i\leq L-1$. We consider the evolution
${\bf W}\left(t\right)$ starting from ${\bf W}\left(0\right)$ (which
is independent of $W$ once ${\bf W}(0)$ is fixed). Note that $\mathbf{W}\left(t\right)$
is a deterministic function of its initialization $\mathbf{W}\left(0\right)$
and the data $\left\{ z\left(s\right)\right\} _{s\leq t}$. Similarly,
we consider the counterpart for $\hat{W}$: the evolution $\hat{\mathbf{W}}\left(t\right)$
as a function of the initialization $\hat{\mathbf{W}}\left(0\right)$
and the data $\left\{ \hat{z}\left(s\right)\right\} _{s\leq t}$.
Due to sharing the same distribution for both the initialization and
the data, these evolutions have the same law. In other words, for
any $\theta>0$,
\begin{align*}
\inf_{{\rm coupling\,of\,}\left(\mathbf{W},\hat{{\bf W}}\right)}\mathbb{P}\left(\left\Vert \mathbf{W}-\hat{{\bf W}}\right\Vert _{T}\geq\theta\right)= & 0,
\end{align*}
in which
\begin{align*}
\left\Vert \mathbf{W}-\hat{{\bf W}}\right\Vert _{T}=\max\bigg( & \bigg(\frac{1}{N_{1}}\sum_{j_{1}=1}^{N_{1}}\sup_{t\leq T}\left|{\bf w}_{1}\left(\left\lfloor t/\epsilon\right\rfloor ,j_{1}\right)-\hat{{\bf w}}_{1}\left(\left\lfloor t/\epsilon\right\rfloor ,j_{1}\right)\right|^{2}\bigg)^{1/2},\\
 & \max_{2\leq i\leq L}\bigg(\frac{1}{N_{i-1}N_{i}}\sum_{j_{i-1}=1}^{N_{i-1}}\sum_{j_{i}=1}^{N_{i}}\sup_{t\leq T}\left|{\bf w}_{i}\left(\left\lfloor t/\epsilon\right\rfloor ,j_{i-1},j_{i}\right)-\hat{{\bf w}}_{i}\left(\left\lfloor t/\epsilon\right\rfloor ,j_{i-1},j_{i}\right)\right|^{2}\bigg)^{1/2},\\
 & \max_{2\leq i\leq L}\bigg(\frac{1}{N_{i}}\sum_{j_{i}=1}^{N_{i}}\sup_{t\leq T}\left|{\bf b}_{i}\left(\left\lfloor t/\epsilon\right\rfloor ,j_{i}\right)-\hat{{\bf b}}_{i}\left(\left\lfloor t/\epsilon\right\rfloor ,j_{i}\right)\right|^{2}\bigg)^{1/2}\bigg).
\end{align*}
Theorem \ref{thm:gradient descent coupling} implies that following
the coupling procedure, for any $\delta>0$, with probability at least
$1-\delta-o_{J;L}$,
\[
\mathscr{D}_{T}\left(W,\mathbf{W}\right)\leq o_{\epsilon,J;\delta,T,L},
\]
where here and in the following, we denote by $o_{J;L}$ and $o_{\epsilon,J;\delta,T,L}$
appropriate quantities that may change from line to line with $o_{J;L}\to0$
and $o_{\epsilon,J;\delta,T,L}\to0$ as the learning rate $\epsilon\to0$
and $J\to\infty$. Here without loss of generality, we assume $o_{\epsilon,J;\delta,T,L}>0$.
We also have a similar result for $\mathscr{D}_{T}\left(\hat{W},\hat{\mathbf{W}}\right)$.
As such,
\begin{align*}
 & \inf_{{\rm coupling\,of\,}\left({\cal W},\hat{{\cal W}}\right)}\mathbb{P}\left(d_{J,T}\left({\cal W}_{J,T},\hat{{\cal W}}_{J,T}\right)\geq o_{\epsilon,J;\delta,T,L}\right)\\
 & \leq\mathbb{P}\left(\mathscr{D}_{T}\left(W,\mathbf{W}\right)\geq o_{\epsilon,J;\delta,T,L}\right)+\mathbb{P}\left(\mathscr{D}_{T}\left(\hat{W},\hat{\mathbf{W}}\right)\geq o_{\epsilon,J;\delta,T,L}\right)\\
 & \quad+\inf_{{\rm coupling\,of\,}\left(\mathbf{W},\hat{{\bf W}}\right)}\mathbb{P}\left(\left\Vert \mathbf{W}-\hat{{\bf W}}\right\Vert _{T}\geq o_{\epsilon,J;\delta,T,L}\right)\\
 & \leq2\delta+2o_{J;L},
\end{align*}
where we define
\begin{align*}
{\cal W}_{J,T} & =\Big\{ w_{1}\left(t,C_{1}\left(j_{1}\right)\right),\;w_{i}\left(t,C_{i-1}\left(j_{i-1}\right),C_{i}\left(j_{i}\right)\right),\;b_{i}\left(t,C_{i}\left(j_{i}\right)\right):\quad j_{i}\in\left[N_{i}\right],\;i\in\left[L\right],\;t\in\left[0,T\right]\Big\},\\
\hat{{\cal W}}_{J,T} & =\Big\{\hat{w}_{1}\left(t,\hat{C}_{1}\left(j_{1}\right)\right),\;\hat{w}_{i}\left(t,\hat{C}_{i-1}\left(j_{i-1}\right),\hat{C}_{i}\left(j_{i}\right)\right),\;\hat{b}_{i}\left(t,\hat{C}_{i}\left(j_{i}\right)\right):\quad j_{i}\in\left[N_{i}\right],\;i\in\left[L\right],\;t\in\left[0,T\right]\Big\}.
\end{align*}
This gives a sense of approximate closeness between $W$ and $\hat{W}$
on a set $J=\left\{ N_{1},...,N_{L}\right\} $ with sufficiently large
size $N_{i}$, importantly under the assumption of $\bar{\eta}$-independence.
To extend this to arbitrary finite sizes, we perform the following
argument.

\paragraph*{Step 2: Extension to finite sizes.}

For a given fixed set $I=\left\{ n_{1},...,n_{L}\right\} $ with $n_{L}=1$,
let us consider the following sub-sampling procedure: for each $i\in\left[L\right]$
and each $r_{i}\in\left[n_{i}\right]$, we independently sample $V_{i}\left(r_{i}\right)$
uniformly from $\left[N_{i}\right]$, and then set $S_{i}\left(r_{i}\right)=C_{i}\left(V_{i}\left(r_{i}\right)\right)$
and $\hat{S}_{i}\left(r_{i}\right)=\hat{C}_{i}\left(V_{i}\left(r_{i}\right)\right)$.
Let us define
\[
{\cal W}_{I,T}=\Big\{ w_{1}\left(t,S_{1}\left(r_{1}\right)\right),\;w_{i}\left(t,S_{i-1}\left(r_{i-1}\right),S_{i}\left(r_{i}\right)\right),\;b_{i}\left(t,S_{i}\left(r_{i}\right)\right):\quad r_{i}\in\left[n_{i}\right],\;i\in\left[L\right],\;t\in\left[0,T\right]\Big\},
\]
and $\hat{{\cal W}}_{I,T}$ similarly. We prove that ${\rm Law}({\cal W}_{I,T})$
and ${\rm Law}(\hat{{\cal W}}_{I,T})$ are close in an appropriate
sense. This shall be done via a connection with $d_{J,T}\left({\cal W}_{J,T},\hat{{\cal W}}_{J,T}\right)$
on the set $J$.

Let $\mathbb{E}_{V}$ denote the expectation w.r.t. the sub-sampling
procedure only (i.e. w.r.t. the randomness of $\left\{ V_{i}\left(r_{i}\right):\;r_{i}\in\left[n_{i}\right],\;i\in\left[L\right]\right\} $).
Notice that
\begin{align*}
\mathbb{E}_{V}\left[\left|d_{I,T}^{w,i}\left({\cal W}_{I,T},\hat{{\cal W}}_{I,T}\right)\right|^{2}\right] & =\left|d_{J,T}^{w,i}\left({\cal W}_{J,T},\hat{{\cal W}}_{J,T}\right)\right|^{2},\\
\mathbb{E}_{V}\left[\left|d_{I,T}^{b,i}\left({\cal W}_{I,T},\hat{{\cal W}}_{I,T}\right)\right|^{2}\right] & =\left|d_{J,T}^{b,i}\left({\cal W}_{J,T},\hat{{\cal W}}_{J,T}\right)\right|^{2}.
\end{align*}
Using this fact and Markov's inequality, for any $\theta>0$:
\begin{align*}
 & \mathbb{P}\left(d_{I,T}^{w,i}\left({\cal W}_{I,T},\hat{{\cal W}}_{I,T}\right)\geq\theta\right)\\
 & \leq\mathbb{P}\left(d_{J,T}^{w,i}\left({\cal W}_{J,T},\hat{{\cal W}}_{J,T}\right)\geq\theta^{2}\right)+\mathbb{E}\left[\mathbb{E}_{V}\left[\mathbb{I}\left(d_{I,T}^{w,i}\left({\cal W}_{I,T},\hat{{\cal W}}_{I,T}\right)\geq\theta\right)\right]\mathbb{I}\left(d_{J,T}^{w,i}\left({\cal W}_{J,T},\hat{{\cal W}}_{J,T}\right)\leq\theta^{2}\right)\right]\\
 & \leq\mathbb{P}\left(d_{J,T}^{w,i}\left({\cal W}_{J,T},\hat{{\cal W}}_{J,T}\right)\geq\theta^{2}\right)+\theta^{-2}\mathbb{E}\left[\mathbb{E}_{V}\left[\left|d_{I,T}^{w,i}\left({\cal W}_{I,T},\hat{{\cal W}}_{I,T}\right)\right|^{2}\right]\mathbb{I}\left(d_{J,T}^{w,i}\left({\cal W}_{J,T},\hat{{\cal W}}_{J,T}\right)\leq\theta^{2}\right)\right]\\
 & \leq\mathbb{P}\left(d_{J,T}^{w,i}\left({\cal W}_{J,T},\hat{{\cal W}}_{J,T}\right)\geq\theta^{2}\right)+\theta^{2}.
\end{align*}
The bound on $d_{J,T}\left({\cal W}_{J,T},\hat{{\cal W}}_{J,T}\right)$
gives:
\[
\inf_{{\rm coupling\,of\,}\left({\cal W},\hat{{\cal W}}\right)}\mathbb{P}\left(d_{J,T}^{w,i}\left({\cal W}_{J,T},\hat{{\cal W}}_{J,T}\right)\geq\tilde{O}_{\epsilon,J;\delta,T,L}\left(1\right)\right)\leq2\delta+2o_{J;L}.
\]
Then by taking $\theta=\sqrt{o_{\epsilon,J;\delta,T,L}}$, we have:
\begin{align*}
\inf_{{\rm coupling\,of\,}\left({\cal W},\hat{{\cal W}}\right)}\mathbb{P}\left(d_{I,T}^{w,i}\left({\cal W}_{I,T},\hat{{\cal W}}_{I,T}\right)\geq\sqrt{o_{\epsilon,J;\delta,T,L}}\right) & \leq e_{T,L}\left(\delta,\epsilon,J\right),\\
e_{T,L}\left(\delta,\epsilon,J\right) & =2\delta+2o_{J;L}+o_{\epsilon,J;\delta,T,L}.
\end{align*}
A similar fact holds for $d_{I,T}^{b,i}$, and therefore by the union
bound,
\[
\inf_{{\rm coupling\,of\,}\left({\cal W},\hat{{\cal W}}\right)}\mathbb{P}\left(d_{I,T}\left({\cal W}_{I,T},\hat{{\cal W}}_{I,T}\right)\geq\sqrt{o_{\epsilon,J;\delta,T,L}\left(\epsilon,J\right)}\right)\leq2Le_{T,L}\left(\delta,\epsilon,J\right).
\]
In particular, this implies for any $f\in{\cal F}_{I,T}$:
\[
\left|\mathbb{E}\left[f\left({\cal W}_{I,T}\right)\right]-\mathbb{E}\left[f\left({\cal \hat{W}}_{I,T}\right)\right]\right|\leq4Le_{T,L}\left(\delta,\epsilon,J\right)+\sqrt{o_{\epsilon,J;\delta,T,L}}.
\]
This describes closeness between ${\rm Law}({\cal W}_{I,T})$ and
${\rm Law}(\hat{{\cal W}}_{I,T})$. Note that this is not sufficient
to conclude the proof (via taking $\epsilon\to0$, $J\to\infty$,
$\delta\to0$): the left-hand side involves the random variables $C_{i}\left(V_{i}\left(r_{i}\right)\right)$,
which firstly does not remove $\bar{\eta}$-independence and secondly
is not independent of $J$ since $V_{i}\left(r_{i}\right)\in\left[N_{i}\right]$.

\paragraph*{Step 3: Removing $\bar{\eta}$-independence.}

Let $\left\{ U_{i}\left(j_{i}\right)\right\} _{j_{i}\in\left[N_{i}\right]}$
be drawn i.i.d. from $P_{i}$, independently for each $i\in\left[L\right]$,
as in the statement of the corollary. We recall that $\left\{ C_{i}\left(j_{i}\right)\right\} _{j_{i}\in\left[N_{i}\right]}$
are $\eta_{i}$-independent for $i\leq L-1$ with $\eta_{i}\to0$
as $N_{i}\to\infty$. We also note $N_{L}=n_{L}=1$ and hence $U_{L}\left(1\right)=C_{L}\left(1\right)$.
As such, for any $1$-bounded function $g$:
\[
\left|\mathbb{E}_{\sim V}\left[g\left(S_{i}\left(r_{i}\right):\;r_{i}\in\left[n_{i}\right],\;i\in\left[L\right]\right)\right]-\mathbb{E}_{\sim V}\left[g\left(U_{i}\left(V_{i}\left(r_{i}\right)\right):\;r_{i}\in\left[n_{i}\right],\;i\in\left[L\right]\right)\right]\right|\le o_{J;L,I},
\]
where $o_{J;L,I}$ is a deterministic quantity such that $o_{J;L,I}\to0$
as $J\to\infty$ and $\mathbb{E}_{\sim V}$ denotes the expectation
w.r.t. everything excluding the sub-sampling procedure. Indeed supposing
that $\left\{ \tilde{V}_{1}\left(1\right),...,\tilde{V}_{1}\left(n_{1}\right)\right\} $
is a permutation of $\left\{ V_{1}\left(1\right),...,V_{1}\left(n_{1}\right)\right\} $
such that $\tilde{V}_{1}\left(1\right)>...>\tilde{V}_{1}\left(n_{1}\right)$.
Using the $\bar{\eta}$-independence property, we have the following
for $|\zeta_{1}|\le\eta_{1}$:
\begin{align*}
 & \mathbb{E}_{\sim V}\left[g\left(S_{i}\left(r_{i}\right):\;r_{i}\in\left[n_{i}\right],\;i\in\left[L\right]\right)\right]\\
 & =\mathbb{E}_{\sim V}\mathbb{E}_{\sim C_{1}\left(\tilde{V}_{1}\left(1\right)\right)}\mathbb{E}_{C_{1}\left(\tilde{V}_{1}\left(1\right)\right)|\sim C_{1}\left(\tilde{V}_{1}\left(1\right)\right)}\left[g\left(C_{1}\left(\tilde{V}_{1}\left(r_{1}\right)\right),\;C_{i}\left(V_{i}\left(r_{i}\right)\right):\;r_{1}\in\left[n_{1}\right],\;r_{i}\in\left[n_{i}\right]\;{\rm for}\;i\geq2\right)\right]\\
 & =\mathbb{E}_{\sim V}\mathbb{E}_{\sim C_{1}\left(\tilde{V}_{1}\left(1\right)\right)}\mathbb{E}_{U_{1}\left(\tilde{V}_{1}\left(1\right)\right)}\left[g\left(U_{1}\left(\tilde{V}_{1}\left(1\right)\right),\;C_{1}\left(\tilde{V}_{1}\left(r_{1}\right)\right),\;C_{i}\left(V_{i}\left(r_{i}\right)\right):\;r_{1}\in\left[n_{1}\right]\backslash\left\{ 1\right\} ,\;r_{i}\in\left[n_{i}\right]\;{\rm for}\;i\geq2\right)\right]+\zeta_{1},
\end{align*}
where conditioning on the sub-sampling, $\mathbb{E}_{\sim C_{1}\left(\tilde{V}_{1}\left(1\right)\right)}$
is the expectation w.r.t. everything excluding $C_{1}\left(\tilde{V}_{1}\left(1\right)\right)$,
and $\mathbb{E}_{C_{1}\left(\tilde{V}_{1}\left(1\right)\right)|\sim C_{1}\left(\tilde{V}_{1}\left(1\right)\right)}$
is w.r.t. $C_{1}\left(\tilde{V}_{1}\left(1\right)\right)$ conditioning
on everything else. (Here we have assumed that $V_{1}\left(1\right),...,V_{1}\left(n_{1}\right)$
are all distinct, since any repeated elements can be removed without
affecting the argument.) By iterating this decomposition, we obtain
the claim.

On the other hand, since $\left\{ U_{i}\left(j_{i}\right)\right\} _{j_{i}\in\left[N_{i}\right]}$
are i.i.d., it is easy to see that
\[
\mathbb{E}\left[g\left(U_{i}\left(V_{i}\left(r_{i}\right)\right):\;r_{i}\in\left[n_{i}\right],\;i\in\left[L\right]\right)\right]=\mathbb{E}\left[g\left(U_{i}\left(r_{i}\right):\;r_{i}\in\left[n_{i}\right],\;i\in\left[L\right]\right)\right].
\]
Together with the result from the previous step, we thus have for
any $f\in{\cal F}_{I,T}$:
\[
\left|\mathbb{E}\left[f\left({\cal W}\left(I,T\right)\right)\right]-\mathbb{E}\left[f\left(\hat{{\cal W}}\left(I,T\right)\right)\right]\right|\leq2o_{J;L,I}+8Le_{T,L}\left(\delta,\epsilon,J\right)+2\sqrt{o_{\epsilon,J;\delta,T,L}},
\]
where we recall
\begin{align*}
{\cal W}\left(I,T\right) & =\Big\{ w_{1}\left(t,U_{1}\left(j_{1}\right)\right),\;w_{i}\left(t,U_{i-1}\left(j_{i-1}\right),U_{i}\left(j_{i}\right)\right),\;b_{i}\left(t,U_{i}\left(j_{i}\right)\right):\quad j_{i}\in\left[n_{i}\right],\;i\in\left[L\right],\;t\in\left[0,T\right]\Big\},
\end{align*}
and similarly for $\hat{{\cal W}}\left(I,T\right)$. Note that the
left-hand side is completely independent of $J$, $\epsilon$ and
$\delta$. So by taking $\epsilon\to0$, $J\to\infty$, $\delta\to0$,
we have:
\[
\mathbb{E}\left[f\left({\cal W}\left(I,T\right)\right)\right]=\mathbb{E}\left[f\left(\hat{{\cal W}}\left(I,T\right)\right)\right],
\]
which completes the proof.
\end{proof}
\begin{proof}[Proof of Corollary \ref{cor:MF-twolayers}]
For any test function $\psi:\;\mathbb{W}_{1}\times\mathbb{W}_{2}\to\mathbb{R}$
that is bounded with bounded gradient, we have:
\begin{align*}
\frac{d}{dt}\int\psi\left(u_{1},u_{2}\right)d\rho_{t}\left(u_{1},u_{2}\right) & =\frac{d}{dt}\mathbb{E}_{C_{1}}\left[\psi\left(w_{1}\left(t,C_{1}\right),w_{2}\left(t,C_{1},1\right)\right)\right]\\
 & =-\mathbb{E}_{C_{1}}\left[\left\langle \nabla_{1}\psi\left(w_{1}\left(t,C_{1}\right),w_{2}\left(t,C_{1},1\right)\right),\xi_{1}^{\mathbf{w}}\left(t\right)\mathbb{E}_{Z}\left[\Delta_{1}^{w}\left(t,Z,C_{1}\right)\right]\right\rangle \right]\\
 & \qquad-\mathbb{E}_{C_{1}}\left[\left\langle \nabla_{2}\psi\left(w_{1}\left(t,C_{1}\right),w_{2}\left(t,C_{1},1\right)\right),\xi_{2}^{\mathbf{w}}\left(t\right)\mathbb{E}_{Z}\left[\Delta_{2}^{w}\left(t,Z,C_{1},1\right)\right]\right\rangle \right]\\
 & \stackrel{\left(a\right)}{=}-\mathbb{E}_{C_{1}}\left[\left\langle \nabla_{1}\psi\left(w_{1}\left(t,C_{1}\right),w_{2}\left(t,C_{1},1\right)\right),\xi_{1}^{\mathbf{w}}\left(t\right)\mathbb{E}_{Z}\left[\underline{\Delta}_{1}^{w}\left(u_{1},u_{2};Z,\rho_{t}\right)\right]\right\rangle \right]\\
 & \qquad-\mathbb{E}_{C_{1}}\left[\left\langle \nabla_{2}\psi\left(w_{1}\left(t,C_{1}\right),w_{2}\left(t,C_{1},1\right)\right),\xi_{2}^{\mathbf{w}}\left(t\right)\mathbb{E}_{Z}\left[\underline{\Delta}_{2}^{w}\left(u_{1},u_{2};Z,\rho_{t}\right)\right]\right\rangle \right]\\
 & =-\int\left\langle \nabla\psi\left(u_{1},u_{2}\right),G\left(u_{1},u_{2};\rho_{t}\right)\right\rangle d\rho_{t}\left(u_{1},u_{2}\right),
\end{align*}
where step $\left(a\right)$ can be checked easily by inspection.
This shows that $\rho_{t}$ satisfies the claimed distributional partial
differential equation. The rest of the claims follow in a similar
vein to the proof of Corollary \ref{cor:gradient descent quality}.
\end{proof}

\section{Remaining details for Section \ref{sec:iid_init}\label{sec:Remaining-proofs-iid-init}}

\subsection{Infinite-$M$ limit of the canonical MF limit under i.i.d. initializations\label{subsec:Infinite-M-limit-full}}

We give the full description of the infinite-$M$ limit $W^{*}$ of
the canonical MF limit, described in Section \ref{subsec:Infinite-M-canonical}.
To that end, let us first consider depth $L\ge5$. Let $\left\{ w_{i}^{*}\right\} _{i=1}^{L}$
and $\left\{ b_{i}^{*}\right\} _{i=2}^{L}$ be functions satisfying
the following dynamics:
\begin{align*}
\frac{\partial}{\partial t}w_{1}^{*}\left(t,u_{1}\right) & =-\xi_{1}^{\mathbf{w}}\left(t\right)\mathbb{E}_{Z}\left[\Delta_{1}^{w*}\left(t,Z,u_{1}\right)\right],\\
\frac{\partial}{\partial t}w_{2}^{*}\left(t,u_{1},u_{2},v_{2}\right) & =-\xi_{2}^{\mathbf{w}}\left(t\right)\mathbb{E}_{Z}\left[\Delta_{2}^{w*}\left(t,Z,u_{1},u_{2},v_{2}\right)\right],\\
\frac{\partial}{\partial t}w_{i}^{*}\left(t,u_{i},v_{i-1},v_{i}\right) & =-\xi_{i}^{\mathbf{w}}\left(t\right)\mathbb{E}_{Z}\left[\Delta_{i}^{w*}\left(t,Z,u_{i},v_{i-1},v_{i}\right)\right],\qquad i=3,...,L-2,\\
\frac{\partial}{\partial t}w_{L-1}^{*}\left(t,u_{L-1},u_{L},v_{L-2},v_{L-1}\right) & =-\xi_{L-1}^{\mathbf{w}}\left(t\right)\mathbb{E}_{Z}\left[\Delta_{L-1}^{w*}\left(t,Z,u_{L-1},u_{L},v_{L-2},v_{L-1}\right)\right],\\
\frac{\partial}{\partial t}w_{L}^{*}\left(t,u_{L},v_{L-1}\right) & =-\xi_{L}^{\mathbf{w}}\left(t\right)\mathbb{E}_{Z}\left[\Delta_{L}^{w*}\left(t,Z,u_{L},v_{L-1}\right)\right],\\
\frac{\partial}{\partial t}b_{i}^{*}\left(t,v_{i}\right) & =-\xi_{i}^{\mathbf{b}}\left(t\right)\mathbb{E}_{Z}\left[\Delta_{i}^{b*}\left(t,Z\right)\right],\qquad i=2,...,L-2,\\
\frac{\partial}{\partial t}b_{L-1}^{*}\left(t,u_{L},v_{L-1}\right) & =-\xi_{L-1}^{\mathbf{b}}\left(t\right)\mathbb{E}_{Z}\left[\Delta_{L-1}^{b*}\left(t,Z,u_{L}\right)\right],\\
\frac{\partial}{\partial t}b_{L}^{*}\left(t\right) & =-\xi_{L}^{\mathbf{b}}\left(t\right)\mathbb{E}_{Z}\left[\Delta_{L}^{b*}\left(t,Z\right)\right],\\
 & \forall u_{i}\in{\rm supp}\left(\rho_{\mathbf{w}}^{i}\right)\;{\rm for}\;i=1,...,L,\\
 & \forall v_{i}\in{\rm supp}\left(\rho_{\mathbf{b}}^{i}\right)\;{\rm for}\;i=2,...,L-1,
\end{align*}
with the initialization $w_{1}^{*}(0,u_{1})=u_{1}$, $w_{2}^{*}\left(0,\cdot,u_{2},\cdot\right)=u_{2}$,
$w_{i}^{*}\left(0,u_{i},\cdot,\cdot\right)=u_{i}$ for $i=3,...,L-2$,
$w_{L-1}^{*}\left(0,u_{L-1},\cdot,\cdot,\cdot\right)=u_{L-1}$, $w_{L}^{*}(0,u_{L},\cdot)=u_{L}$,
$b_{i}^{*}\left(0,v_{i}\right)=v_{i}$ for $i=2,...,L-2$, $b_{L-1}^{*}\left(0,\cdot,v_{L-1}\right)=v_{L-1}$
and $b_{L}^{*}(0)$ a deterministic constant that $\rho_{\mathbf{b}}^{L}\left(b_{L}^{*}(0)\right)=1$
(i.e. $b_{L}^{*}(0)=\mathfrak{p}_{L}\left(1\right)$ according to
Eq. (\ref{eq:iid_init_embedding_constr2})). Here the quantities are
defined by the following forward and backward recursions:
\begin{itemize}
\item Forward recursion:
\end{itemize}
\begin{align*}
H_{1}^{*}\left(t,x,u_{1}\right) & =\phi_{1}\left(w_{1}^{*}\left(t,u_{1}\right),x\right),\\
H_{2}^{*}\left(t,x,v_{2}\right) & =\int\phi_{2}\left(w_{2}^{*}\left(t,u_{1},u_{2},v_{2}\right),b_{2}^{*}\left(t,v_{2}\right),H_{1}^{*}\left(t,x,u_{1}\right)\right)\rho_{\mathbf{w}}^{1}\left(du_{1}\right)\rho_{\mathbf{w}}^{2}\left(du_{2}\right),\\
H_{i}^{*}\left(t,x,v_{i}\right) & =\int\phi_{i}\left(w_{i}^{*}\left(t,u_{i},v_{i-1},v_{i}\right),b_{i}^{*}\left(t,v_{i}\right),H_{i-1}^{*}\left(t,x,v_{i-1}\right)\right)\rho_{\mathbf{w}}^{i}\left(du_{i}\right)\rho_{\mathbf{b}}^{i-1}\left(dv_{i-1}\right),\\
 & \qquad i=3,...,L-2,\\
H_{L-1}^{*}\left(t,x,u_{L},v_{L-1}\right) & =\int\phi_{L-1}\left(w_{L-1}^{*}\left(t,u_{L-1},u_{L},v_{L-2},v_{L-1}\right),b_{L-1}^{*}\left(t,u_{L},v_{L-1}\right),H_{L-2}^{*}\left(t,x,v_{L-2}\right)\right)\\
 & \qquad\times\rho_{\mathbf{w}}^{L-1}\left(du_{L-1}\right)\rho_{\mathbf{b}}^{L-2}\left(dv_{L-2}\right),\\
H_{L}^{*}\left(t,x\right) & =\int\phi_{L}\left(w_{L}^{*}\left(t,u_{L},v_{L-1}\right),b_{L}^{*}\left(t\right),H_{L-1}^{*}\left(t,x,u_{L},v_{L-1}\right)\right)\rho_{\mathbf{w}}^{L}\left(du_{L}\right)\rho_{\mathbf{b}}^{L-1}\left(dv_{L-1}\right),\\
\hat{y}^{*}\left(t,x\right) & =\phi_{L+1}\left(H_{L}^{*}\left(t,x\right)\right).
\end{align*}

\begin{itemize}
\item Backward recursion:
\end{itemize}
\begin{align*}
\Delta_{L}^{H*}\left(t,z\right) & =\sigma_{L}^{\mathbf{H}}\left(y,\hat{y}^{*}\left(t,x\right),H_{L}^{*}\left(t,x\right)\right),\\
\Delta_{L}^{w*}\left(t,z,u_{L},v_{L-1}\right) & =\sigma_{L}^{\mathbf{w}}\left(\Delta_{L}^{H*}\left(t,z\right),w_{L}^{*}\left(t,u_{L},v_{L-1}\right),b_{L}^{*}\left(t\right),H_{L}^{*}\left(t,x\right),H_{L-1}^{*}\left(t,x,u_{L},v_{L-1}\right)\right),\\
\Delta_{L}^{b*}\left(t,z\right) & =\int\sigma_{L}^{\mathbf{b}}\left(\Delta_{L}^{H*}\left(t,z\right),w_{L}^{*}\left(t,u_{L},v_{L-1}\right),b_{L}^{*}\left(t\right),H_{L}^{*}\left(t,x\right),H_{L-1}^{*}\left(t,x,u_{L},v_{L-1}\right)\right)\\
 & \qquad\times\rho_{\mathbf{w}}^{L}\left(du_{L}\right)\rho_{\mathbf{b}}^{L-1}\left(dv_{L-1}\right),\\
\Delta_{L-1}^{H*}\left(t,z,u_{L},v_{L-1}\right) & =\sigma_{L-1}^{\mathbf{H}}\left(\Delta_{L}^{H*}\left(t,z\right),w_{L}^{*}\left(t,u_{L},v_{L-1}\right),b_{L}^{*}\left(t\right),H_{L}^{*}\left(t,x\right),H_{L-1}^{*}\left(t,x,u_{L},v_{L-1}\right)\right),\\
\Delta_{L-1}^{w*}\left(t,z,u_{L-1},u_{L},v_{L-2},v_{L-1}\right) & =\sigma_{L-1}^{\mathbf{w}}\Big(\Delta_{L-1}^{H*}\left(t,z,u_{L},v_{L-1}\right),w_{L-1}^{*}\left(t,u_{L-1},u_{L},v_{L-2},v_{L-1}\right),b_{L-1}^{*}\left(t,u_{L},v_{L-1}\right),\\
 & \qquad H_{L-1}^{*}\left(t,x,u_{L},v_{L-1}\right),H_{L-2}^{*}\left(t,x,v_{L-2}\right)\Big),\\
\Delta_{L-1}^{b*}\left(t,z,u_{L},v_{L-1}\right) & =\int\sigma_{L-1}^{\mathbf{b}}\Big(\Delta_{L-1}^{H*}\left(t,z,u_{L},v_{L-1}\right),w_{L-1}^{*}\left(t,u_{L-1},u_{L},v_{L-2},v_{L-1}\right),b_{L-1}^{*}\left(t,u_{L},v_{L-1}\right),\\
 & \qquad H_{L-1}^{*}\left(t,x,u_{L},v_{L-1}\right),H_{L-2}^{*}\left(t,x,v_{L-2}\right)\Big)\rho_{\mathbf{w}}^{L-1}\left(du_{L-1}\right)\rho_{\mathbf{b}}^{L-2}\left(dv_{L-2}\right),\\
\Delta_{L-2}^{H*}\left(t,z,v_{L-2}\right) & =\int\sigma_{L-2}^{\mathbf{H}}\Big(\Delta_{L-1}^{H*}\left(t,z,u_{L},v_{L-1}\right),w_{L-1}^{*}\left(t,u_{L-1},u_{L},v_{L-2},v_{L-1}\right),b_{L-1}^{*}\left(t,u_{L},v_{L-1}\right),\\
 & \qquad H_{L-1}^{*}\left(t,x,u_{L},v_{L-1}\right),H_{L-2}^{*}\left(t,x,v_{L-2}\right)\Big)\rho_{\mathbf{w}}^{L}\left(du_{L}\right)\rho_{\mathbf{w}}^{L-1}\left(du_{L-1}\right)\rho_{\mathbf{b}}^{L-1}\left(dv_{L-1}\right),\\
\Delta_{i}^{w*}\left(t,z,u_{i},v_{i-1},v_{i}\right) & =\sigma_{i}^{\mathbf{w}}\left(\Delta_{i}^{H*}\left(t,z,v_{i}\right),w_{i}^{*}\left(t,u_{i},v_{i-1},v_{i}\right),b_{i}^{*}\left(t,v_{i}\right),H_{i}^{*}\left(t,x,v_{i}\right),H_{i-1}^{*}\left(t,x,v_{i-1}\right)\right),\\
\Delta_{i}^{b*}\left(t,z,v_{i}\right) & =\int\sigma_{i}^{\mathbf{b}}\left(\Delta_{i}^{H*}\left(t,z,v_{i}\right),w_{i}^{*}\left(t,u_{i},v_{i-1},v_{i}\right),b_{i}^{*}\left(t,v_{i}\right),H_{i}^{*}\left(t,x,v_{i}\right),H_{i-1}^{*}\left(t,x,v_{i-1}\right)\right)\\
 & \qquad\times\rho_{\mathbf{w}}^{i}\left(du_{i}\right)\rho_{\mathbf{b}}^{i-1}\left(dv_{i-1}\right),\\
\Delta_{i-1}^{H*}\left(t,z,v_{i-1}\right) & =\int\sigma_{i-1}^{\mathbf{H}}\left(\Delta_{i}^{H*}\left(t,z,v_{i}\right),w_{i}^{*}\left(t,u_{i},v_{i-1},v_{i}\right),b_{i}^{*}\left(t,v_{i}\right),H_{i}^{*}\left(t,x,v_{i}\right),H_{i-1}^{*}\left(t,x,v_{i-1}\right)\right)\\
 & \qquad\times\rho_{\mathbf{w}}^{i}\left(du_{i}\right)\rho_{\mathbf{b}}^{i}\left(dv_{i}\right),\\
 & \qquad{\rm for}\;i=L-2,...,3,\\
\Delta_{2}^{w*}\left(t,z,u_{1},u_{2},v_{2}\right) & =\sigma_{2}^{\mathbf{w}}\left(\Delta_{2}^{H*}\left(t,z,v_{2}\right),w_{2}^{*}\left(t,u_{1},u_{2},v_{2}\right),b_{2}^{*}\left(t,v_{2}\right),H_{2}^{*}\left(t,x,v_{2}\right),H_{1}^{*}\left(t,x,u_{1}\right)\right),\\
\Delta_{2}^{b*}\left(t,z,v_{2}\right) & =\int\sigma_{2}^{\mathbf{b}}\left(\Delta_{2}^{H*}\left(t,z,v_{2}\right),w_{2}^{*}\left(t,u_{1},u_{2},v_{2}\right),b_{2}^{*}\left(t,v_{2}\right),H_{2}^{*}\left(t,x,v_{2}\right),H_{1}^{*}\left(t,x,u_{1}\right)\right)\\
 & \qquad\times\rho_{\mathbf{w}}^{2}\left(du_{2}\right)\rho_{\mathbf{w}}^{1}\left(du_{1}\right),\\
\Delta_{1}^{H*}\left(t,z,u_{1}\right) & =\int\sigma_{1}^{\mathbf{H}}\left(\Delta_{2}^{H*}\left(t,z,v_{2}\right),w_{2}^{*}\left(t,u_{1},u_{2},v_{2}\right),b_{2}^{*}\left(t,v_{2}\right),H_{2}^{*}\left(t,x,v_{2}\right),H_{1}^{*}\left(t,x,u_{1}\right)\right)\\
 & \qquad\times\rho_{\mathbf{w}}^{2}\left(du_{2}\right)\rho_{\mathbf{b}}^{2}\left(dv_{2}\right),\\
\Delta_{1}^{w*}\left(t,z,u_{1}\right) & =\sigma_{1}^{\mathbf{w}}\left(\Delta_{1}^{H*}\left(t,z,u_{1}\right),w_{1}^{*}\left(t,u_{1}\right),x\right).
\end{align*}
In the case $L=3$ and $L=4$, we define the dynamics of $w_{i}^{*}$
and $b_{i}^{*}$ similarly. In particular, for $L=4$, one can simply
disregard all above equations that are with invalid indices. For $L=3$,
we define:
\begin{align*}
\frac{\partial}{\partial t}w_{1}^{*}\left(t,u_{1}\right) & =-\xi_{1}^{\mathbf{w}}\left(t\right)\mathbb{E}_{Z}\left[\Delta_{1}^{w*}\left(t,Z,u_{1}\right)\right],\\
\frac{\partial}{\partial t}w_{2}^{*}\left(t,u_{1},u_{2},u_{3},v_{2}\right) & =-\xi_{2}^{\mathbf{w}}\left(t\right)\mathbb{E}_{Z}\left[\Delta_{2}^{w*}\left(t,Z,u_{1},u_{2},u_{3},v_{2}\right)\right],\\
\frac{\partial}{\partial t}w_{3}^{*}\left(t,u_{3},v_{2}\right) & =-\xi_{3}^{\mathbf{w}}\left(t\right)\mathbb{E}_{Z}\left[\Delta_{3}^{w*}\left(t,Z,u_{3},v_{2}\right)\right],\\
\frac{\partial}{\partial t}b_{2}^{*}\left(t,u_{3},v_{2}\right) & =-\xi_{2}^{\mathbf{b}}\left(t\right)\mathbb{E}_{Z}\left[\Delta_{2}^{b*}\left(t,Z,u_{3},v_{2}\right)\right],\\
\frac{\partial}{\partial t}b_{3}^{*}\left(t\right) & =-\xi_{3}^{\mathbf{b}}\left(t\right)\mathbb{E}_{Z}\left[\Delta_{3}^{b*}\left(t,Z\right)\right],\\
 & \forall u_{i}\in{\rm supp}\left(\rho_{\mathbf{w}}^{i}\right)\;{\rm for}\;i=1,2,3,\\
 & \forall v_{i}\in{\rm supp}\left(\rho_{\mathbf{b}}^{i}\right)\;{\rm for}\;i=2,3,
\end{align*}
in which the quantities are:
\begin{align*}
H_{1}^{*}\left(t,x,u_{1}\right) & =\phi_{1}\left(w_{1}^{*}\left(t,u_{1}\right),x\right),\\
H_{2}^{*}\left(t,x,u_{3},v_{2}\right) & =\int\phi_{2}\left(w_{2}^{*}\left(t,u_{1},u_{2},u_{3},v_{2}\right),b_{2}^{*}\left(t,u_{3},v_{2}\right),H_{1}^{*}\left(t,x,u_{1}\right)\right)\rho_{\mathbf{w}}^{1}\left(du_{1}\right)\rho_{\mathbf{w}}^{2}\left(du_{2}\right),\\
H_{3}^{*}\left(t,x\right) & =\int\phi_{3}\left(w_{3}^{*}\left(t,u_{3},v_{2}\right),b_{3}^{*}\left(t\right),H_{2}^{*}\left(t,x,u_{3},v_{2}\right)\right)\rho_{\mathbf{w}}^{3}\left(du_{3}\right)\rho_{\mathbf{b}}^{2}\left(dv_{2}\right),\\
\hat{y}^{*}\left(t,x\right) & =\phi_{4}\left(H_{3}^{*}\left(t,x\right)\right),\\
\Delta_{3}^{H*}\left(t,z\right) & =\sigma_{3}^{\mathbf{H}}\left(y,\hat{y}^{*}\left(t,x\right),H_{3}^{*}\left(t,x\right)\right),\\
\Delta_{3}^{w*}\left(t,z,u_{3},v_{2}\right) & =\sigma_{3}^{\mathbf{w}}\left(\Delta_{3}^{H*}\left(t,z\right),w_{3}^{*}\left(t,u_{3},v_{2}\right),b_{3}^{*}\left(t\right),H_{3}^{*}\left(t,x\right),H_{2}^{*}\left(t,x,u_{3},v_{2}\right)\right),\\
\Delta_{3}^{b*}\left(t,z\right) & =\int\sigma_{3}^{\mathbf{b}}\left(\Delta_{3}^{H*}\left(t,z\right),w_{3}^{*}\left(t,u_{3},v_{2}\right),b_{3}^{*}\left(t\right),H_{3}^{*}\left(t,x\right),H_{2}^{*}\left(t,x,u_{3},v_{2}\right)\right)\rho_{\mathbf{w}}^{3}\left(du_{3}\right)\rho_{\mathbf{b}}^{2}\left(dv_{2}\right),\\
\Delta_{2}^{H*}\left(t,z,u_{3},v_{2}\right) & =\sigma_{3}^{\mathbf{H}}\left(\Delta_{3}^{H*}\left(t,z\right),w_{3}^{*}\left(t,u_{3},v_{2}\right),b_{3}^{*}\left(t\right),H_{3}^{*}\left(t,x\right),H_{2}^{*}\left(t,x,u_{3},v_{2}\right)\right),\\
\Delta_{2}^{w*}\left(t,z,u_{1},u_{2},u_{3},v_{2}\right) & =\sigma_{2}^{\mathbf{w}}\left(\Delta_{2}^{H*}\left(t,z,u_{3},v_{2}\right),w_{2}^{*}\left(t,u_{1},u_{2},u_{3},v_{2}\right),b_{2}^{*}\left(t,u_{3},v_{2}\right),H_{2}^{*}\left(t,x,u_{3},v_{2}\right),H_{1}^{*}\left(t,x,u_{1}\right)\right),\\
\Delta_{2}^{b*}\left(t,z,u_{3},v_{2}\right) & =\int\sigma_{2}^{\mathbf{b}}\left(\Delta_{2}^{H*}\left(t,z,u_{3},v_{2}\right),w_{2}^{*}\left(t,u_{1},u_{2},u_{3},v_{2}\right),b_{2}^{*}\left(t,u_{3},v_{2}\right),H_{2}^{*}\left(t,x,u_{3},v_{2}\right),H_{1}^{*}\left(t,x,u_{1}\right)\right)\\
 & \qquad\times\rho_{\mathbf{w}}^{1}\left(du_{1}\right)\rho_{\mathbf{w}}^{2}\left(du_{2}\right),\\
\Delta_{1}^{H*}\left(t,z,u_{1}\right) & =\int\sigma_{1}^{\mathbf{H}}\left(\Delta_{2}^{H*}\left(t,z,u_{3},v_{2}\right),w_{2}^{*}\left(t,u_{1},u_{2},u_{3},v_{2}\right),b_{2}^{*}\left(t,u_{3},v_{2}\right),H_{2}^{*}\left(t,x,u_{3},v_{2}\right),H_{1}^{*}\left(t,x,u_{1}\right)\right)\\
 & \qquad\times\rho_{\mathbf{w}}^{2}\left(du_{2}\right)\rho_{\mathbf{w}}^{3}\left(du_{3}\right),\rho_{\mathbf{b}}^{2}\left(dv_{2}\right),\\
\Delta_{1}^{w*}\left(t,z,u_{1}\right) & =\sigma_{1}^{\mathbf{w}}\left(\Delta_{1}^{H*}\left(t,z,u_{1}\right),w_{1}^{*}\left(t,u_{1}\right),x\right).
\end{align*}
Finally let $W^{*}\left(t\right)=\left\{ w_{1}^{*}\left(t,\cdot\right),w_{i}^{*}\left(t,\cdot\right),b_{i}^{*}\left(t,\cdot\right),\;\;i=2,...,L\right\} $.
The existence and uniqueness of such dynamics follow similarly to
the proof of Theorem \ref{thm:existence ODE}.
\begin{thm}[Complete statement of Theorem \ref{thm:iid dynamics}]
\label{thm:iid dynamics-full}Given $\left(\rho_{\mathbf{w}}^{1},...,\rho_{\mathbf{w}}^{L},\rho_{\mathbf{b}}^{2},...,\rho_{\mathbf{b}}^{L}\right)$
and an integer $M$, construct the canonical neuronal ensemble $\left(\Omega^{M},P^{M}\right)$,
the random variables $\left(C_{1},...,C_{L}\right)\sim P^{M}=\prod_{i=1}^{L}P_{i}^{M}$
and the canonical MF limit $W^{M}$ as described in Section \ref{subsec:Canonical-neuronal-embeddings}.
Also construct the dynamics $W^{*}$ described in Section \ref{subsec:Infinite-M-canonical}.

For $L\geq5$, define the following:
\begin{align*}
w_{1}^{\infty}\left(t,c_{1}\right) & =w_{1}^{*}\left(t,w_{1}^{0}\left(c_{1}\right)\right),\\
w_{2}^{\infty}\left(t,c_{1},c_{2}\right) & =w_{2}^{*}\left(t,w_{1}^{0}\left(c_{1}\right),w_{2}^{0}\left(c_{1},c_{2}\right),b_{2}^{0}\left(c_{2}\right)\right),\\
w_{i}^{\infty}\left(t,c_{i-1},c_{i}\right) & =w_{i}^{*}\left(t,w_{i}^{0}\left(c_{i-1},c_{i}\right),b_{i-1}^{0}\left(c_{i-1}\right),b_{i}^{0}\left(c_{i}\right)\right),\qquad i=3,...,L-2,\\
w_{L-1}^{\infty}\left(t,c_{L-2},c_{L-1}\right) & =w_{L-1}^{*}\left(t,w_{L-1}^{0}\left(c_{L-2},c_{L-1}\right),w_{L}^{0}\left(c_{L-1},1\right),b_{L-2}^{0}\left(c_{L-2}\right),b_{L-1}^{0}\left(c_{L-1}\right)\right),\\
w_{L}^{\infty}\left(t,c_{L-1},1\right) & =w_{L}^{*}\left(t,w_{L}^{0}\left(c_{L-1},1\right),b_{L-1}^{0}\left(c_{L-1}\right)\right),\\
b_{i}^{\infty}\left(t,c_{i}\right) & =b_{i}^{*}\left(t,b_{i}^{0}\left(c_{i}\right)\right),\qquad i=2,...,L-2,\\
b_{L-1}^{\infty}\left(t,c_{L-1}\right) & =b_{L-1}^{*}\left(t,w_{L}^{0}\left(c_{L-1},1\right),b_{L-1}^{0}\left(c_{L-1}\right)\right),\\
b_{L}^{\infty}\left(t,1\right) & =b_{L}^{*}\left(t\right),\\
 & c_{i}\in\Omega_{i}=\Lambda\times\mathbb{N}_{>0},\quad i=1,...,L-1.
\end{align*}
For $L=4$, we define similarly by disregarding the equations with
invalid indices. For $L=3$, we define:
\begin{align*}
w_{1}^{\infty}\left(t,c_{1}\right) & =w_{1}^{*}\left(t,w_{1}^{0}\left(c_{1}\right)\right),\\
w_{2}^{\infty}\left(t,c_{1},c_{2}\right) & =w_{2}^{*}\left(t,w_{2}^{0}\left(c_{1},c_{2}\right),w_{3}^{0}\left(c_{2},1\right),w_{1}^{0}\left(c_{1}\right),b_{2}^{0}\left(c_{2}\right)\right),\\
w_{3}^{\infty}\left(t,c_{2},1\right) & =w_{3}^{*}\left(t,w_{3}^{0}\left(c_{2},1\right),b_{2}^{0}\left(c_{2}\right)\right),\\
b_{2}^{\infty}\left(t,c_{2}\right) & =b_{2}^{*}\left(t,w_{3}^{0}\left(c_{2},1\right),b_{2}^{0}\left(c_{2}\right)\right),\\
b_{3}^{\infty}\left(t\right) & =b_{3}^{*}\left(t\right).
\end{align*}
We also let $W^{\infty}\left(t\right)=\left\{ w_{1}^{\infty}\left(t,\cdot\right),w_{i}^{\infty}\left(t,\cdot,\cdot\right),b_{i}^{\infty}\left(t,\cdot\right),\;\;i=2,...,L\right\} $.
Let us consider:
\begin{align*}
\left\langle W^{M}-W^{\infty}\right\rangle _{t} & =\max\left(\max_{1\leq i\leq L}\left\langle w_{i}^{M}-w_{i}^{\infty}\right\rangle _{t},\;\max_{2\leq i\leq L}\left\langle b_{i}^{M}-b_{i}^{\infty}\right\rangle _{t}\right),\\
\left\langle w_{i}^{M}-w_{i}^{\infty}\right\rangle _{t} & =\mathbb{E}\left[\left|w_{i}^{M}\left(t,C_{i-1},C_{i}\right)-w_{i}^{\infty}\left(t,C_{i-1},C_{i}\right)\right|^{2}\right]^{1/2},\\
\left\langle b_{i}^{M}-b_{i}^{\infty}\right\rangle _{t} & =\mathbb{E}\left[\left|b_{i}^{M}\left(t,C_{i}\right)-b_{i}^{\infty}\left(t,C_{i}\right)\right|^{2}\right]^{1/2},\qquad i=2,...,L,\\
\left\langle w_{1}^{M}-w_{1}^{\infty}\right\rangle _{t} & =\mathbb{E}\left[\left|w_{1}^{M}\left(t,C_{1}\right)-w_{1}^{\infty}\left(t,C_{1}\right)\right|^{2}\right]^{1/2}.
\end{align*}

Then under Assumptions \ref{enu:Assump_lrSchedule}-\ref{enu:Assump_backward}
and \ref{assump:init}, for any $T\geq0$ and $L\geq2$,
\[
\sup_{t\leq T}\left\langle W^{M}-W^{\infty}\right\rangle _{t}\le\frac{K_{T,L}}{M^{0.499}},
\]
for sufficiently large $M=M\left(T,L\right)$, where $K_{T,L}$ is
a constant that depends on $T$ and $L$. Furthermore, for $L\geq4$
and $2\leq i\leq L-2$,
\[
\sup_{t\leq T}\mathbb{E}\left[\left|H_{i}\left(X,C_{i};W^{M}(t)\right)-H_{i}^{*}\left(t,X,b_{i}^{0}(C_{i})\right)\right|^{2}\right]^{1/2}\leq\frac{K_{T,L}}{M^{0.499}}.
\]
\end{thm}

\subsection{Proof of Theorem \ref{thm:iid dynamics-full}\label{subsec:Proof-Theorem-iid-full}}
\begin{proof}[Proof of Theorem \ref{thm:iid dynamics-full}]
Let us consider the case $L\geq5$; the case where $L\leq4$ is similarly
proven. We use $K_{T,L}$ to denote a generic constant that depends
on $T$ and $L$ and may change from line to line.

\paragraph*{Step 1.}

By following the argument of Lemma \ref{lem:bounds MF a priori},
one can show that the following quantities are bounded by $K_{T,L}$:
\[
\left\Vert W^{\infty}\right\Vert _{T},\quad\mathbb{E}\left[\sup_{t\leq T}\underset{Z\sim{\cal P}}{{\rm ess\text{-}sup}}\left|H_{1}^{*}\left(t,X,w_{1}^{0}(C_{1})\right)\right|^{50}\right],\quad\max_{2\leq i\leq L-2}\mathbb{E}\left[\sup_{t\leq T}\underset{Z\sim{\cal P}}{{\rm ess\text{-}sup}}\left|H_{i}^{*}\left(t,X,b_{i}^{0}(C_{i})\right)\right|^{50}\right],
\]
\[
\mathbb{E}\left[\sup_{t\leq T}\underset{Z\sim{\cal P}}{{\rm ess\text{-}sup}}\left|H_{L-1}^{*}\left(t,X,w_{L}^{0}(C_{L-1},1),b_{L-1}^{0}(C_{L-1})\right)\right|^{50}\right],\quad\sup_{t\leq T}\underset{Z\sim{\cal P}}{{\rm ess\text{-}sup}}\left|H_{L}^{*}\left(t,X\right)\right|,
\]
\[
\mathbb{E}\left[\sup_{t\leq T}\underset{Z\sim{\cal P}}{{\rm ess\text{-}sup}}\left|\Delta_{1}^{H*}\left(t,Z,w_{1}^{0}(C_{1})\right)\right|^{50}\right],\quad\max_{2\leq i\leq L-2}\mathbb{E}\left[\sup_{t\leq T}\underset{Z\sim{\cal P}}{{\rm ess\text{-}sup}}\left|\Delta_{i}^{H*}\left(t,Z,b_{i}^{0}\left(C_{i}\right)\right)\right|^{50}\right],
\]
\[
\mathbb{E}\left[\sup_{t\leq T}\underset{Z\sim{\cal P}}{{\rm ess\text{-}sup}}\left|\Delta_{L-1}^{H*}\left(t,Z,w_{L}^{0}\left(C_{L-1},1\right),b_{L-1}^{0}(C_{L-1})\right)\right|^{50}\right],\quad\sup_{t\leq T}\underset{Z\sim{\cal P}}{{\rm ess\text{-}sup}}\left|\Delta_{L}^{H*}\left(t,Z\right)\right|.
\]
Likewise one can also show that for any $B\geq0$,
\[
\mathbb{P}\left(\mathsf{max}_{T}^{w}\left(W^{\infty}\right)\geq K_{T,L}B\right)\leq2K_{T,L}e^{-KB^{2}},
\]
in which
\begin{align*}
\mathsf{max}_{T}^{w}\left(W^{\infty}\right) & =\max_{2\leq i\leq L}\sup_{t\leq T}\left|w_{i}^{\infty}\left(t,C_{i-1},C_{i}\right)\right|.
\end{align*}

\paragraph*{Step 2.}

Let us define
\begin{align*}
D_{1}\left(t\right) & =\mathbb{E}\left[\left|H_{1}\left(X,C_{1};W^{\infty}(t)\right)-H_{1}^{*}\left(t,X,w_{1}^{0}(C_{1})\right)\right|^{2}\right],\\
D_{i}\left(t\right) & =\mathbb{E}\left[\left|H_{i}\left(X,C_{i};W^{\infty}(t)\right)-H_{i}^{*}\left(t,X,b_{i}^{0}(C_{i})\right)\right|^{2}\right],\quad i=2,...,L-2,\\
D_{L-1}\left(t\right) & =\mathbb{E}\left[\left|H_{L-1}\left(X,C_{L-1};W^{\infty}(t)\right)-H_{L-1}^{*}\left(t,X,w_{L}^{0}(C_{L-1},1),b_{L-1}^{0}(C_{L-1})\right)\right|^{2}\right],\\
D_{L}\left(t\right) & =\mathbb{E}\left[\left|H_{L}\left(X,1;W^{\infty}(t)\right)-H_{L}^{*}\left(t,X\right)\right|^{2}\right].
\end{align*}
We claim that for $t\le T$,
\[
D_{i}\left(t\right)\leq\frac{K_{T,L}}{M},\qquad i\in\left[L\right].
\]
Firstly it is immediate that
\[
H_{1}\left(X,C_{1};W^{\infty}(t)\right)=\phi_{1}\left(w_{1}^{*}(t,w_{1}^{0}(C_{1})),X\right)=H_{1}^{*}\left(t,X,w_{1}^{0}(C_{1})\right),
\]
and hence $D_{1}\left(t\right)=0$. For $i=2$, we have 
\begin{align*}
D_{2}\left(t\right) & =\mathbb{E}\bigg[\bigg|\mathbb{E}_{C_{1}}\left[\phi_{2}\left(w_{2}^{*}(t,w_{1}^{0}(C_{1}),w_{2}^{0}(C_{1},C_{2}),b_{2}^{0}(C_{2})),b_{2}^{*}(t,b_{2}^{0}(C_{2})),H_{1}(X,C_{1};W^{\infty}(t))\right)\right]\\
 & \qquad-\int\phi_{2}\left(w_{2}^{*}\left(t,u_{1},u_{2},b_{2}^{0}(C_{2})\right),b_{2}^{*}\left(t,b_{2}^{0}(C_{2})\right),H_{1}^{*}\left(t,x,u_{1}\right)\right)\rho_{\mathbf{w}}^{1}\left(du_{1}\right)\rho_{\mathbf{w}}^{2}\left(du_{2}\right)\bigg|^{2}\bigg].
\end{align*}
Recalling that $w_{2}^{0}(C_{1},C_{2})=\mathfrak{q}_{2}(\theta_{1},\theta_{2})(\lambda_{2})$,
$b_{2}^{0}(C_{2})=\mathfrak{p}_{2}(\theta_{2})(\lambda_{2})$ and
$w_{1}^{0}(C_{1})=\mathfrak{p}_{1}(\theta_{1})(\lambda_{1})$ from
the construction of Section \ref{subsec:Canonical-neuronal-embeddings},
we have: 
\begin{align*}
 & \mathbb{E}_{C_{2}}\bigg[\left|\mathbb{E}_{C_{1}}\left[\phi_{2}\left(w_{2}^{*}(t,w_{1}^{0}(C_{1}),w_{2}^{0}(C_{1},C_{2}),b_{2}^{0}(C_{2})),b_{2}^{*}(t,b_{2}^{0}(C_{2})),H_{1}(x,C_{1};W^{\infty}(t))\right)\right]\right|^{2}\bigg]\\
 & \stackrel{\left(a\right)}{=}\mathbb{E}_{\theta_{1},\lambda_{1},\theta_{1}',\lambda_{1}',\theta_{2},\lambda_{2}}\bigg[\Big\langle\phi_{2}\left(w_{2}^{*}(t,\mathfrak{p}_{1}(\theta_{1})(\lambda_{1}),\mathfrak{q}_{2}(\theta_{1},\theta_{2})(\lambda_{2}),\mathfrak{p}_{2}(\theta_{2})(\lambda_{2})),b_{2}^{*}(t,\mathfrak{p}_{2}(\theta_{2})(\lambda_{2})),H_{1}^{*}(t,x,\mathfrak{p}_{1}(\theta_{1})(\lambda_{1}))\right),\\
 & \qquad\phi_{2}\left(w_{2}^{*}\left(t,\mathfrak{p}_{1}(\theta_{1}')(\lambda_{1}'),\mathfrak{q}_{2}(\theta_{1}',\theta_{2})(\lambda_{2}),\mathfrak{p}_{2}(\theta_{2})(\lambda_{2})\right),b_{2}^{*}(t,\mathfrak{p}_{2}(\theta_{2})(\lambda_{2})),H_{1}^{*}(t,x,\mathfrak{p}_{1}(\theta_{1}')(\lambda_{1}'))\right)\Big\rangle\bigg]\\
 & \stackrel{\left(b\right)}{=}\mathbb{E}_{\theta_{1},\theta_{1}'}\bigg[\mathbb{I}(\theta_{1}=\theta_{1}')\int\left\langle \phi_{2}\left(w_{2}^{*}(t,u_{1},u_{2},v_{2}),b_{2}^{*}(t,v_{2}),H_{1}^{*}(t,x,u_{1})\right),\phi_{2}\left(w_{2}^{*}(t,u_{1}',u_{2},v_{2}),b_{2}^{*}(t,v_{2}),H_{1}^{*}(t,x,u_{1}')\right)\right\rangle \\
 & \qquad\qquad\qquad\qquad\qquad\times\rho_{\mathbf{w}}^{1}\left(du_{1}\right)\rho_{\mathbf{w}}^{1}\left(du_{1}'\right)\rho_{\mathbf{w}}^{2}\left(du_{2}\right)\rho_{\mathbf{b}}^{2}\left(dv_{2}\right)\bigg]\\
 & \quad+\mathbb{E}_{\theta_{1},\theta_{1}'}\bigg[\mathbb{I}(\theta_{1}\ne\theta_{1}')\int\phi_{2}\left(w_{2}^{*}(t,u_{1},u_{2},v_{2}),b_{2}^{*}(t,v_{2}),H_{1}^{*}(t,x,u_{1})\right)\phi_{2}\left(w_{2}^{*}(t,u_{1}',u_{2}',v_{2}),b_{2}^{*}(t,v_{2}),H_{1}^{*}(t,x,u_{1}')\right)\\
 & \qquad\qquad\qquad\qquad\qquad\times\rho_{\mathbf{w}}^{1}\left(du_{1}\right)\rho_{\mathbf{w}}^{2}\left(du_{2}\right)\rho_{\mathbf{w}}^{1}\left(du_{1}'\right)\rho_{\mathbf{w}}^{2}\left(du_{2}'\right)\rho_{\mathbf{b}}^{2}\left(dv_{2}\right)\bigg]\\
 & =\frac{1}{M}\int\left\langle \phi_{2}\left(w_{2}^{*}(t,u_{1},u_{2},v_{2}),b_{2}^{*}(t,v_{2}),H_{1}^{*}(t,x,u_{1})\right),\phi_{2}\left(w_{2}^{*}(t,u_{1}',u_{2},v_{2}),b_{2}^{*}(t,v_{2}),H_{1}^{*}(t,x,u_{1}')\right)\right\rangle \\
 & \qquad\qquad\times\rho_{\mathbf{w}}^{1}\left(du_{1}\right)\rho_{\mathbf{w}}^{1}\left(du_{1}'\right)\rho_{\mathbf{w}}^{2}\left(du_{2}\right)\rho_{\mathbf{b}}^{2}\left(dv_{2}\right)\\
 & \quad+\frac{M-1}{M}\int\left|\int\phi_{2}\left(w_{2}^{*}(t,u_{1},u_{2},v_{2}),b_{2}^{*}(t,v_{2}),H_{1}^{*}(t,x,u_{1})\right)\rho_{\mathbf{w}}^{1}\left(du_{1}\right)\rho_{\mathbf{w}}^{2}\left(du_{2}\right)\right|^{2}\rho_{\mathbf{b}}^{2}\left(dv_{2}\right),
\end{align*}
where in step $\left(a\right)$, $\left(\theta_{1}',\lambda_{1}'\right)\sim{\rm Unif}\left(\left[M\right]\right)\times P_{0}$
is an independent copy of $\left(\theta_{1},\lambda_{1}\right)$ and
is independent of $\left(\theta_{2},\lambda_{2}\right)$, and step
$\left(b\right)$ is by the construction of $\mathfrak{p}_{1}$, $\mathfrak{p}_{2}$
and $\mathfrak{q}_{2}$. It is also easy to see that
\begin{align*}
 & \mathbb{E}_{C_{2}}\bigg[\Big\langle\mathbb{E}_{C_{1}}\left[\phi_{2}\left(w_{2}^{*}(t,w_{1}^{0}(C_{1}),w_{2}^{0}(C_{1},C_{2}),b_{2}^{0}(C_{2})),b_{2}^{*}(t,b_{2}^{0}(C_{2})),H_{1}(x,C_{1};W^{\infty}(t))\right)\right],\\
 & \qquad\int\phi_{2}\left(w_{2}^{*}(t,u_{1},u_{2},b_{2}^{0}(C_{2})),b_{2}^{*}(t,b_{2}^{0}(C_{2})),H_{1}^{*}(t,x,u_{1})\right)\rho_{\mathbf{w}}^{1}\left(du_{1}\right)\rho_{\mathbf{w}}^{2}\left(du_{2}\right)\Big\rangle\bigg]\\
 & =\int\left|\int\phi_{2}\left(w_{2}^{*}(t,u_{1},u_{2},v_{2}),b_{2}^{*}(t,v_{2}),H_{1}^{*}(t,x,u_{1})\right)\rho_{\mathbf{w}}^{1}\left(du_{1}\right)\rho_{\mathbf{w}}^{2}\left(du_{2}\right)\right|^{2}\rho_{\mathbf{b}}^{2}\left(dv_{2}\right).
\end{align*}
Therefore, for $t\leq T$,
\begin{align*}
D_{2}\left(t\right) & \leq\frac{1}{M}\int\mathbb{E}\left[\left|\left\langle \phi_{2}\left(w_{2}^{*}(t,u_{1},u_{2},v_{2}),b_{2}^{*}(t,v_{2}),H_{1}^{*}(t,X,u_{1})\right),\phi_{2}\left(w_{2}^{*}(t,u_{1}',u_{2},v_{2}),b_{2}^{*}(t,v_{2}),H_{1}^{*}(t,X,u_{1}')\right)\right\rangle \right|\right]\\
 & \qquad\qquad\times\rho_{\mathbf{w}}^{1}\left(du_{1}\right)\rho_{\mathbf{w}}^{1}\left(du_{1}'\right)\rho_{\mathbf{w}}^{2}\left(du_{2}\right)\rho_{\mathbf{b}}^{2}\left(dv_{2}\right)\\
 & \quad+\frac{1}{M}\int\mathbb{E}\left[\left|\int\phi_{2}\left(w_{2}^{*}(t,u_{1},u_{2},v_{2}),b_{2}^{*}(t,v_{2}),H_{1}^{*}(t,X,u_{1})\right)\rho_{\mathbf{w}}^{1}\left(du_{1}\right)\rho_{\mathbf{w}}^{2}\left(du_{2}\right)\right|^{2}\right]\rho_{\mathbf{b}}^{2}\left(dv_{2}\right)\\
 & \leq\frac{K}{M}\int\mathbb{E}\left[\left|\phi_{2}\left(w_{2}^{*}(t,u_{1},u_{2},v_{2}),b_{2}^{*}(t,v_{2}),H_{1}^{*}(t,X,u_{1})\right)\right|^{2}\right]\rho_{\mathbf{w}}^{1}\left(du_{1}\right)\rho_{\mathbf{w}}^{2}\left(du_{2}\right)\rho_{\mathbf{b}}^{2}\left(dv_{2}\right)\\
 & \leq\frac{K_{T,L}}{M}
\end{align*}
where we use Step 1 and Assumption \ref{enu:Assump_forward} in the
last step. For $i\in\left\{ 3,...,L-2\right\} $, recall that $w_{i}^{0}(C_{i-1},C_{i})=\mathfrak{q}_{i}(\theta_{i-1},\theta_{i})(\lambda_{i})$,
$b_{i}^{0}(C_{i})=\mathfrak{p}_{i}(\theta_{i})(\lambda_{i})$ and
$b_{i-1}^{0}(C_{i-1})=\mathfrak{p}_{i-1}(\theta_{i-1})(\lambda_{i-1})$
from the construction of Section \ref{subsec:Canonical-neuronal-embeddings}.
Then similar to the argument for $i=2$: 
\begin{align*}
 & \mathbb{E}_{C_{i}}\bigg[\left|\mathbb{E}_{C_{i-1}}\left[\phi_{i}\left(w_{i}^{*}(t,w_{i}^{0}(C_{i-1},C_{i}),b_{i-1}^{0}(C_{i-1}),b_{i}^{0}(C_{i})),b_{i}^{*}(t,b_{i}^{0}(C_{i})),H_{i-1}^{*}(t,x,b_{i-1}^{0}(C_{i-1}))\right)\right]\right|^{2}\bigg]\\
 & =\frac{1}{M}\int\left\langle \phi_{i}\left(w_{i}^{*}(t,u_{i},v_{i-1},v_{i}),b_{i}^{*}(t,v_{i}),H_{i-1}^{*}(t,x,v_{i-1})\right),\phi_{i}\left(w_{i}^{*}(t,u_{i},v_{i-1}',v_{i}),b_{i}^{*}(t,v_{i}),H_{i-1}^{*}(t,x,v_{i-1}')\right)\right\rangle \\
 & \qquad\qquad\times\rho_{\mathbf{w}}^{i}\left(du_{i}\right)\rho_{\mathbf{b}}^{i-1}\left(dv_{i-1}\right)\rho_{\mathbf{b}}^{i-1}\left(dv_{i-1}'\right)\rho_{\mathbf{b}}^{i}\left(dv_{i}\right)\\
 & \quad+\frac{M-1}{M}\int\left|\int\phi_{2}\left(w_{i}^{*}(t,u_{i},v_{i-1},v_{i}),b_{i}^{*}(t,v_{i}),H_{i-1}^{*}(t,x,v_{i-1})\right)\rho_{\mathbf{w}}^{i}\left(du_{i}\right)\rho_{\mathbf{b}}^{i-1}\left(dv_{i-1}\right)\right|^{2}\rho_{\mathbf{b}}^{i}\left(dv_{i}\right),\\
 & \mathbb{E}_{C_{i}}\bigg[\Big\langle\mathbb{E}_{C_{i-1}}\left[\phi_{i}\left(w_{i}^{*}(t,w_{i}^{0}(C_{i-1},C_{i}),b_{i-1}^{0}(C_{i-1}),b_{i}^{0}(C_{i})),b_{i}^{*}(t,b_{i}^{0}(C_{i})),H_{i-1}^{*}(t,x,b_{i-1}^{0}(C_{i-1}))\right)\right],\\
 & \qquad\qquad\int\phi_{i}\left(w_{i}^{*}\left(t,u_{i},v_{i-1},b_{i}^{0}(C_{i})\right),b_{i}^{*}\left(t,b_{i}^{0}(C_{i})\right),H_{i-1}^{*}\left(t,x,v_{i-1}\right)\right)\rho_{\mathbf{w}}^{i}\left(du_{i}\right)\rho_{\mathbf{b}}^{i-1}\left(dv_{i-1}\right)\Big\rangle\bigg]\\
 & =\int\left|\int\phi_{i}\left(w_{i}^{*}\left(t,u_{i},v_{i-1},v_{i}\right),b_{i}^{*}\left(t,v_{i}\right),H_{i-1}^{*}\left(t,x,v_{i-1}\right)\right)\rho_{\mathbf{w}}^{i}\left(du_{i}\right)\rho_{\mathbf{b}}^{i-1}\left(dv_{i-1}\right)\right|^{2}\rho_{\mathbf{b}}^{i}\left(dv_{i}\right),
\end{align*}
which then gives, by Step 1 and Assumption \ref{enu:Assump_forward}:
\begin{align*}
 & \mathbb{E}_{C_{i}}\bigg[\bigg|\mathbb{E}_{C_{i-1}}\left[\phi_{i}\left(w_{i}^{*}(t,w_{i}^{0}(C_{i-1},C_{i}),b_{i-1}^{0}(C_{i-1}),b_{i}^{0}(C_{i})),b_{i}^{*}(t,b_{i}^{0}(C_{i})),H_{i-1}^{*}(t,x,b_{i-1}^{0}(C_{i-1}))\right)\right]\\
 & \qquad\qquad-H_{i}^{*}(t,X,b_{i}^{0}(C_{i}))\bigg|^{2}\bigg]\\
 & =\mathbb{E}_{C_{i}}\bigg[\bigg|\mathbb{E}_{C_{i-1}}\left[\phi_{i}\left(w_{i}^{*}(t,w_{i}^{0}(C_{i-1},C_{i}),b_{i-1}^{0}(C_{i-1}),b_{i}^{0}(C_{i})),b_{i}^{*}(t,b_{i}^{0}(C_{i})),H_{i-1}^{*}(t,x,b_{i-1}^{0}(C_{i-1}))\right)\right]\\
 & \qquad\qquad-\int\phi_{i}\left(w_{i}^{*}\left(t,u_{i},v_{i-1},b_{i}^{0}(C_{i})\right),b_{i}^{*}\left(t,b_{i}^{0}(C_{i})\right),H_{i-1}^{*}\left(t,x,v_{i-1}\right)\right)\rho_{\mathbf{w}}^{i}\left(du_{i}\right)\rho_{\mathbf{b}}^{i-1}\left(dv_{i-1}\right)\bigg|^{2}\bigg]\\
 & \leq\frac{K}{M}\int\left|\phi_{i}\left(w_{i}^{*}(t,u_{i},v_{i-1},v_{i}),b_{i}^{*}(t,v_{i}),H_{i-1}^{*}(t,x,v_{i-1})\right)\right|^{2}\rho_{\mathbf{w}}^{i}\left(du_{i}\right)\rho_{\mathbf{b}}^{i-1}\left(dv_{i-1}\right)\rho_{\mathbf{b}}^{i}\left(dv_{i}\right)\\
 & \leq\frac{K_{T,L}}{M}.
\end{align*}
Next, notice that again by Step 1 and Assumption \ref{enu:Assump_forward}:
\begin{align*}
 & \mathbb{E}\bigg[\bigg|H_{i}\left(X,C_{i};W^{\infty}(t)\right)\\
 & \qquad-\mathbb{E}_{C_{i-1}}\left[\phi_{i}\left(w_{i}^{*}(t,w_{i}^{0}(C_{i-1},C_{i}),b_{i-1}^{0}(C_{i-1}),b_{i}^{0}(C_{i})),b_{i}^{*}(t,b_{i}^{0}(C_{i})),H_{i-1}^{*}(t,X,b_{i-1}^{0}(C_{i-1}))\right)\right]\bigg|^{2}\bigg]\\
 & =\mathbb{E}\bigg[\bigg|\mathbb{E}_{C_{i-1}}\left[\phi_{i}(w_{i}^{*}(t,w_{i}^{0}(C_{i-1},C_{i}),b_{i-1}^{0}(C_{i-1}),b_{i}^{0}(C_{i})),b_{i}^{*}(t,b_{i}^{0}(C_{i})),H_{i-1}(X,C_{i-1};W^{\infty}(t)))\right]\\
 & \qquad\qquad-\mathbb{E}_{C_{i-1}}\left[\phi_{i}\left(w_{i}^{*}(t,w_{i}^{0}(C_{i-1},C_{i}),b_{i-1}^{0}(C_{i-1}),b_{i}^{0}(C_{i})),b_{i}^{*}(t,b_{i}^{0}(C_{i})),H_{i-1}^{*}(t,X,b_{i-1}^{0}(C_{i-1}))\right)\right]\bigg|^{2}\bigg]\\
 & \le K\mathbb{E}\bigg[\bigg|\mathbb{E}_{C_{i-1}}\Big[\left(1+\left|w_{i}^{*}(t,w_{i}^{0}(C_{i-1},C_{i}),b_{i-1}^{0}(C_{i-1}),b_{i}^{0}(C_{i}))\right|+\left|b_{i}^{*}(t,b_{i}^{0}(C_{i}))\right|\right)\\
 & \qquad\qquad\times\left|H_{i-1}(X,C_{i-1};W^{\infty}(t))-H_{i-1}^{*}(t,X,b_{i-1}^{0}(C_{i-1}))\right|\Big]\bigg|^{2}\bigg]\\
 & \le K\mathbb{E}\left[\left(1+\left|w_{i}^{*}(t,w_{i}^{0}(C_{i-1},C_{i}),b_{i-1}^{0}(C_{i-1}),b_{i}^{0}(C_{i}))\right|+\left|b_{i}^{*}(t,b_{i}^{0}(C_{i}))\right|\right)^{2}\right]\\
 & \qquad\times\mathbb{E}\left[\left|H_{i-1}(X,C_{i-1};W^{\infty}(t))-H_{i-1}^{*}(t,X,b_{i-1}^{0}(C_{i-1}))\right|^{2}\right]\\
 & \le K_{T,L}D_{i-1}\left(t\right).
\end{align*}
Hence, 
\[
D_{i}\left(t\right)\le\frac{K_{T,L}}{M}+K_{T,L}D_{i-1}\left(t\right).
\]
This proves the claim for $i\leq L-2$. The other claims are similar.

\paragraph*{Step 3.}

Let us define:
\begin{align*}
D_{1}^{H}\left(t\right) & =\mathbb{E}\left[\left|\Delta_{1}^{H}\left(Z,C_{1};W^{\infty}\left(t\right)\right)-\Delta_{1}^{H*}\left(t,Z,w_{1}^{0}(C_{1})\right)\right|^{2}\right],\\
D_{i}^{H}\left(t\right) & =\mathbb{E}\left[\left|\Delta_{i}^{H}\left(Z,C_{i};W^{\infty}\left(t\right)\right)-\Delta_{i}^{H*}\left(t,Z,b_{i}^{0}(C_{i})\right)\right|^{2}\right],\qquad i=2,...,L-2,\\
D_{L-1}^{H}\left(t\right) & =\mathbb{E}\left[\left|\Delta_{L-1}^{H}\left(Z,C_{L-1};W^{\infty}\left(t\right)\right)-\Delta_{L-1}^{H*}\left(t,Z,w_{L}^{0}\left(C_{L-1},1\right),b_{L-1}^{0}(C_{L-1})\right)\right|^{2}\right],\\
D_{L}^{H}\left(t\right) & =\mathbb{E}\left[\left|\Delta_{L}^{H}\left(Z,1;W^{\infty}\left(t\right)\right)-\Delta_{L}^{H*}\left(t,Z\right)\right|^{2}\right].
\end{align*}
We claim that for $t\leq T$,
\[
D_{i}^{H}\left(t\right)\leq K_{T,L}\frac{\log^{1/2}M}{M},\qquad i\in\left[L\right].
\]
The derivation is similar to Step 2; let us give a sketch and highlight
the difference. The last claim for $i=L$ is immediate from Assumption
\ref{enu:Assump_backward} and Step 2. Let us consider the claim for
$2\leq i\leq L-3$; the rest of the claims are similar. We have:
\begin{align*}
 & \mathbb{E}\bigg[\bigg|\mathbb{E}_{C_{i+1}}\Big[\sigma_{i}^{\mathbf{H}}\Big(\Delta_{i+1}^{H*}\left(t,Z,b_{i+1}^{0}\left(C_{i+1}\right)\right),w_{i+1}^{*}\left(t,w_{i+1}^{0}\left(C_{i},C_{i+1}\right),b_{i}^{0}\left(C_{i}\right),b_{i+1}^{0}\left(C_{i+1}\right)\right),\\
 & \qquad\qquad b_{i+1}^{*}\left(t,b_{i+1}^{0}\left(C_{i+1}\right)\right),H_{i+1}^{*}\left(t,X,b_{i+1}^{0}\left(C_{i+1}\right)\right),H_{i}^{*}\left(t,X,b_{i}^{0}\left(C_{i}\right)\right)\Big)\Big]-\Delta_{i}^{H*}\left(t,Z,b_{i}^{0}(C_{i})\right)\bigg|^{2}\bigg]\\
 & =\mathbb{E}\bigg[\bigg|\mathbb{E}_{C_{i+1}}\Big[\sigma_{i}^{\mathbf{H}}\Big(\Delta_{i+1}^{H*}\left(t,Z,b_{i+1}^{0}\left(C_{i+1}\right)\right),w_{i+1}^{*}\left(t,w_{i+1}^{0}\left(C_{i},C_{i+1}\right),b_{i}^{0}\left(C_{i}\right),b_{i+1}^{0}\left(C_{i+1}\right)\right),\\
 & \qquad\qquad b_{i+1}^{*}\left(t,b_{i+1}^{0}\left(C_{i+1}\right)\right),H_{i+1}^{*}\left(t,X,b_{i+1}^{0}\left(C_{i+1}\right)\right),H_{i}^{*}\left(t,X,b_{i}^{0}\left(C_{i}\right)\right)\Big)\Big]\\
 & \quad-\int\sigma_{i}^{\mathbf{H}}\left(\Delta_{i+1}^{H*}\left(t,Z,v_{i+1}\right),w_{i+1}^{*}\left(t,u_{i+1},b_{i}^{0}\left(C_{i}\right),v_{i+1}\right),b_{i+1}^{*}\left(t,v_{i+1}\right),H_{i+1}^{*}\left(t,X,v_{i+1}\right),H_{i}^{*}\left(t,X,b_{i}^{0}\left(C_{i}\right)\right)\right)\\
 & \qquad\qquad\times\rho_{\mathbf{w}}^{i+1}\left(du_{i+1}\right)\rho_{\mathbf{b}}^{i+1}\left(dv_{i+1}\right)\bigg|^{2}\bigg]\\
 & \stackrel{\left(a\right)}{\leq}\frac{K}{M}\int\mathbb{E}\left[\left|\sigma_{i}^{\mathbf{H}}\left(\Delta_{i+1}^{H*}\left(t,Z,v_{i+1}\right),w_{i+1}^{*}\left(t,u_{i+1},v_{i},v_{i+1}\right),b_{i+1}^{*}\left(t,v_{i+1}\right),H_{i+1}^{*}\left(t,X,v_{i+1}\right),H_{i}^{*}\left(t,X,v_{i}\right)\right)\right|^{2}\right]\\
 & \qquad\qquad\times\rho_{\mathbf{w}}^{i+1}\left(du_{i+1}\right)\rho_{\mathbf{b}}^{i}\left(dv_{i}\right)\rho_{\mathbf{b}}^{i+1}\left(dv_{i+1}\right)\\
 & \stackrel{\left(b\right)}{\leq}\frac{K}{M}\int\mathbb{E}\left[\left(1+\left|\Delta_{i+1}^{H*}\left(t,Z,v_{i+1}\right)\right|^{2}\right)\left(1+\left|w_{i+1}^{*}\left(t,u_{i+1},v_{i},v_{i+1}\right)\right|^{2}+\left|b_{i+1}^{*}\left(t,v_{i+1}\right)\right|^{2}\right)\right]\\
 & \qquad\qquad\times\rho_{\mathbf{w}}^{i+1}\left(du_{i+1}\right)\rho_{\mathbf{b}}^{i}\left(dv_{i}\right)\rho_{\mathbf{b}}^{i+1}\left(dv_{i+1}\right)\\
 & \stackrel{\left(c\right)}{\leq}\frac{K_{T,L}}{M},
\end{align*}
where $\left(a\right)$ is similar to Step 2 in which we use the fact
$w_{i+1}^{0}(C_{i},C_{i+1})=\mathfrak{q}_{i+1}(\theta_{i},\theta_{i+1})(\lambda_{i+1})$,
$b_{i}^{0}(C_{i})=\mathfrak{p}_{i}(\theta_{i})(\lambda_{i})$ and
$b_{i+1}^{0}(C_{i+1})=\mathfrak{p}_{i+1}(\theta_{i+1})(\lambda_{i+1})$
from the construction of Section \ref{subsec:Canonical-neuronal-embeddings},
$\left(b\right)$ is by Assumption \ref{enu:Assump_backward}, and
$\left(c\right)$ is follows from Step 1. We also note:
\begin{align*}
 & \mathbb{E}\bigg[\bigg|\mathbb{E}_{C_{i+1}}\Big[\sigma_{i}^{\mathbf{H}}\Big(\Delta_{i+1}^{H*}\left(t,Z,b_{i+1}^{0}\left(C_{i+1}\right)\right),w_{i+1}^{*}\left(t,w_{i+1}^{0}\left(C_{i},C_{i+1}\right),b_{i}^{0}\left(C_{i}\right),b_{i+1}^{0}\left(C_{i+1}\right)\right),\\
 & \qquad\qquad b_{i+1}^{*}\left(t,b_{i+1}^{0}\left(C_{i+1}\right)\right),H_{i+1}^{*}\left(t,X,b_{i+1}^{0}\left(C_{i+1}\right)\right),H_{i}^{*}\left(t,X,b_{i}^{0}\left(C_{i}\right)\right)\Big)\Big]-\Delta_{i}^{H}\left(Z,b_{i}^{0}(C_{i});W^{\infty}\left(t\right)\right)\bigg|^{2}\bigg]\\
 & =\mathbb{E}\bigg[\bigg|\mathbb{E}_{C_{i+1}}\Big[\sigma_{i}^{\mathbf{H}}\Big(\Delta_{i+1}^{H*}\left(t,Z,b_{i+1}^{0}\left(C_{i+1}\right)\right),w_{i+1}^{*}\left(t,w_{i+1}^{0}\left(C_{i},C_{i+1}\right),b_{i}^{0}\left(C_{i}\right),b_{i+1}^{0}\left(C_{i+1}\right)\right),\\
 & \qquad\qquad b_{i+1}^{*}\left(t,b_{i+1}^{0}\left(C_{i+1}\right)\right),H_{i+1}^{*}\left(t,X,b_{i+1}^{0}\left(C_{i+1}\right)\right),H_{i}^{*}\left(t,X,b_{i}^{0}\left(C_{i}\right)\right)\Big)\Big]\\
 & \quad-\mathbb{E}_{C_{i+1}}\Big[\sigma_{i}^{\mathbf{H}}\Big(\Delta_{i+1}^{H}\left(Z,C_{i+1};W^{\infty}\left(t\right)\right),w_{i+1}^{*}\left(t,w_{i+1}^{0}\left(C_{i},C_{i+1}\right),b_{i}^{0}\left(C_{i}\right),b_{i+1}^{0}\left(C_{i+1}\right)\right),\\
 & \qquad\qquad b_{i+1}^{*}\left(t,b_{i+1}^{0}\left(C_{i+1}\right)\right),H_{i+1}\left(X,C_{i+1};W^{\infty}\left(t\right)\right),H_{i}\left(X,C_{i};W^{\infty}\left(t\right)\right)\Big)\Big]\bigg|^{2}\bigg]\\
 & \stackrel{\left(a\right)}{\leq}K\mathbb{E}\bigg[\mathbb{E}_{C_{i+1}}\Big[\left(1+\left|w_{i+1}^{*}\left(t,w_{i+1}^{0}\left(C_{i},C_{i+1}\right),b_{i}^{0}\left(C_{i}\right),b_{i+1}^{0}\left(C_{i+1}\right)\right)\right|+\left|b_{i+1}^{*}\left(t,b_{i+1}^{0}\left(C_{i+1}\right)\right)\right|\right)\\
 & \qquad\qquad\times\left|\Delta_{i+1}^{H*}\left(t,Z,b_{i+1}^{0}\left(C_{i+1}\right)\right)-\Delta_{i+1}^{H}\left(Z,C_{i+1};W^{\infty}\left(t\right)\right)\right|\Big]^{2}\bigg]\\
 & \quad+K\mathbb{E}\bigg[\mathbb{E}_{C_{i+1}}\Big[\left(1+\left|w_{i+1}^{*}\left(t,w_{i+1}^{0}\left(C_{i},C_{i+1}\right),b_{i}^{0}\left(C_{i}\right),b_{i+1}^{0}\left(C_{i+1}\right)\right)\right|+\left|b_{i+1}^{*}\left(t,b_{i+1}^{0}\left(C_{i+1}\right)\right)\right|\right)\\
 & \qquad\qquad\times\left(1+\left|\Delta_{i+1}^{H*}\left(t,Z,b_{i+1}^{0}\left(C_{i+1}\right)\right)\right|+\left|\Delta_{i+1}^{H*}\left(t,Z,b_{i+1}^{0}\left(C_{i+1}\right)\right)-\Delta_{i+1}^{H}\left(Z,C_{i+1};W^{\infty}\left(t\right)\right)\right|\right)\\
 & \qquad\qquad\times\Big(\left|H_{i+1}^{*}\left(t,X,b_{i+1}^{0}\left(C_{i+1}\right)\right)-H_{i+1}\left(X,C_{i+1};W^{\infty}\left(t\right)\right)\right|\\
 & \qquad\qquad\quad+\left|H_{i}^{*}\left(t,X,b_{i}^{0}\left(C_{i}\right)\right)-H_{i}\left(X,C_{i};W^{\infty}\left(t\right)\right)\right|\Big)\Big]^{2}\bigg]\\
 & \stackrel{\left(b\right)}{\leq}K_{T,L}\left(D_{i+1}^{H}\left(t\right)+D_{i+1}\left(t\right)+D_{i}\left(t\right)\right)+KQ_{i}\left(t\right),
\end{align*}
where $\left(a\right)$ follows from Assumption \ref{enu:Assump_backward},
$\left(b\right)$ follows from Step 1, and we define
\begin{align*}
Q_{i}\left(t\right) & =\mathbb{E}\bigg[\mathbb{E}_{C_{i+1}}\Big[\left|w_{i+1}^{*}\left(t,w_{i+1}^{0}\left(C_{i},C_{i+1}\right),b_{i}^{0}\left(C_{i}\right),b_{i+1}^{0}\left(C_{i+1}\right)\right)\right|\\
 & \qquad\times\left(\left|\Delta_{i+1}^{H*}\left(t,Z,b_{i+1}^{0}\left(C_{i+1}\right)\right)\right|+\left|\Delta_{i+1}^{H*}\left(t,Z,b_{i+1}^{0}\left(C_{i+1}\right)\right)-\Delta_{i+1}^{H}\left(Z,C_{i+1};W^{\infty}\left(t\right)\right)\right|\right)\\
 & \qquad\times\left|H_{i}^{*}\left(t,X,b_{i}^{0}\left(C_{i}\right)\right)-H_{i}\left(X,C_{i};W^{\infty}\left(t\right)\right)\right|\Big)\Big]^{2}\bigg].
\end{align*}
The bounding of $Q_{i}\left(t\right)$ requires some more care. In
particular, for $B>0$, define
\[
E=\left\{ \left|w_{i+1}^{*}\left(t,w_{i+1}^{0}\left(C_{i},C_{i+1}\right),b_{i}^{0}\left(C_{i}\right),b_{i+1}^{0}\left(C_{i+1}\right)\right)\right|\geq B\right\} .
\]
Upon decomposing the inner expectation of $Q_{i}\left(t\right)$ into
the sum of $\mathbb{I}\left(E\right)$ and $\mathbb{I}\left(\neg E\right)$,
together with Step 1, via an appropriate use of Cauchy-Schwarz's inequality,
it is easy to see that
\begin{align*}
Q_{i}\left(t\right) & \leq K_{T,L}B\left(D_{i}\left(t\right)+D_{i+1}^{H}\left(t\right)\right)+K_{T,L}\left(1+D_{i}\left(t\right)+D_{i+1}^{H}\left(t\right)\right)\mathbb{P}\left(E\right)^{1/8}\\
 & \leq K_{T,L}\left(1+B\right)\left(D_{i}\left(t\right)+D_{i+1}^{H}\left(t\right)\right)+K_{T,L}e^{-KB^{2}},
\end{align*}
which holds for any $B>0$. Combining these bounds together and Step
2, we obtain:
\begin{align*}
D_{i}^{H}\left(t\right) & \leq\frac{K_{T,L}}{M}+K_{T,L}\left(1+B\right)\left(D_{i+1}^{H}\left(t\right)+D_{i+1}\left(t\right)+D_{i}\left(t\right)\right)+K_{T,L}e^{-KB^{2}}\\
 & \leq\frac{K_{T,L}}{M}+K_{T,L}\left(1+B\right)\left(D_{i+1}^{H}\left(t\right)+\frac{1}{M}\right)+K_{T,L}e^{-KB^{2}}.
\end{align*}
Then choosing $B=c\sqrt{\log M}$ for an appropriate constant $c$
leads to the desired conclusion.

\paragraph*{Step 4.}

Let us define:
\begin{align*}
D_{1}^{w}\left(t\right) & =\mathbb{E}\left[\left|\Delta_{1}^{w}\left(Z,C_{1};W^{\infty}\left(t\right)\right)-\Delta_{1}^{w*}\left(t,Z,w_{1}^{0}(C_{1})\right)\right|^{2}\right],\\
D_{2}^{w}\left(t\right) & =\mathbb{E}\left[\left|\Delta_{2}^{w}\left(Z,C_{1},C_{2};W^{\infty}\left(t\right)\right)-\Delta_{2}^{w*}\left(t,Z,w_{1}^{0}\left(C_{1}\right),w_{2}^{0}(C_{1},C_{2}),b_{2}^{0}(C_{2})\right)\right|^{2}\right],\\
D_{i}^{w}\left(t\right) & =\mathbb{E}\left[\left|\Delta_{i}^{w}\left(Z,C_{i-1},C_{i};W^{\infty}\left(t\right)\right)-\Delta_{i}^{w*}\left(t,Z,w_{i}^{0}(C_{i-1},C_{i}),b_{i-1}^{0}(C_{i-1}),b_{i}^{0}(C_{i})\right)\right|^{2}\right],\quad i=3,...,L-2\\
D_{L-1}^{w}\left(t\right) & =\mathbb{E}\Big[\Big|\Delta_{L-1}^{w}\left(Z,C_{L-2},C_{L-1};W^{\infty}\left(t\right)\right)\\
 & \qquad-\Delta_{L-1}^{w*}\left(t,Z,w_{L-1}^{0}(C_{L-2},C_{L-1}),w_{L}^{0}\left(C_{L-1},1\right),b_{L-2}^{0}(C_{L-2}),b_{L-1}^{0}(C_{L-1})\right)\Big|\Big],\\
D_{L}^{w}\left(t\right) & =\mathbb{E}\left[\left|\Delta_{L}^{w}\left(Z,C_{L-1},1;W^{\infty}\left(t\right)\right)-\Delta_{L}^{w*}\left(t,Z,w_{L}^{0}\left(C_{L-1},1\right),b_{L-1}^{0}(C_{L-1})\right)\right|^{2}\right],\\
D_{i}^{b}\left(t\right) & =\mathbb{E}\left[\left|\Delta_{i}^{b}\left(Z,C_{i};W^{\infty}\left(t\right)\right)-\Delta_{i}^{b*}\left(t,Z,b_{i}^{0}(C_{i})\right)\right|^{2}\right],\quad i=2,...,L-2,\\
D_{L-1}^{b}\left(t\right) & =\mathbb{E}\left[\left|\Delta_{L-1}^{b}\left(Z,C_{L-1};W^{\infty}\left(t\right)\right)-\Delta_{L-1}^{b*}\left(t,Z,w_{L}^{0}\left(C_{L-1},1\right),b_{L-1}^{0}(C_{L-1})\right)\right|^{2}\right],\\
D_{L}^{b}\left(t\right) & =\mathbb{E}\left[\left|\Delta_{L}^{b}\left(Z,1;W^{\infty}\left(t\right)\right)-\Delta_{L}^{b*}\left(t,Z\right)\right|^{2}\right].
\end{align*}
We claim that for any $t\leq T$,
\[
\max_{1\leq i\leq L}D_{i}^{w}\left(t\right)\leq K_{T,L}\frac{\log^{1/2}M}{M},\qquad\max_{2\leq i\leq L}D_{i}^{b}\left(t\right)\leq K_{T,L}\frac{\log^{1/2}M}{M}.
\]
Indeed by Assumption \ref{enu:Assump_backward}, for $3\leq i\leq L-2$,
\begin{align*}
D_{i}^{w}\left(t\right) & \leq K\mathbb{E}\left[1+\left|\Delta_{i}^{H*}\left(t,Z,b_{i}^{0}\left(C_{i}\right)\right)\right|^{2}+\left|\Delta_{i}^{H*}\left(t,Z,b_{i}^{0}\left(C_{i}\right)\right)-\Delta_{i}^{H}\left(Z,C_{i};W^{\infty}\left(t\right)\right)\right|^{2}\right]D_{i-1}\left(t\right)\\
 & \quad+K\left(D_{i}^{H}\left(t\right)+D_{i}\left(t\right)\right).
\end{align*}
The claim for $D_{i}^{w}$ then follows from Steps 1, 2 and 3. The
rest are similar.

\paragraph*{Step 5.}

With the same argument as Lemma \ref{lem:Lipschitz backward MF - general},
given Step 1, one gets that for $2\leq i\leq L$, any $t\leq T$ and
any $B\geq0$,
\[
\mathbb{E}\left[\left|\mathbb{E}_{Z}\left[\Delta_{i}^{w}\left(Z,C_{i-1},C_{i};W^{M}\left(t\right)\right)-\Delta_{i}^{w}\left(Z,C_{i-1},C_{i};W^{\infty}\left(t\right)\right)\right]\right|^{2}\right]^{1/2}\le K_{T,L}\left(\left(1+B\right)\left\langle W^{M}-W^{\infty}\right\rangle _{t}+e^{-KB^{2}}\right).
\]
As such, by Step 4,
\begin{align*}
 & \mathbb{E}\left[\left|\Delta_{i}^{w}(Z,C_{i-1},C_{i};W^{M}(t))-\Delta_{i}^{w*}(t,Z,w_{i}^{0}(C_{i-1},C_{i}),b_{i-1}^{0}(C_{i-1}),b_{i}^{0}(C_{i}))\right|^{2}\right]^{1/2}\\
 & \le\left|D_{i}^{w}\left(t\right)\right|^{1/2}+\mathbb{E}\left[\left|\Delta_{i}^{w}(Z,C_{i-1},C_{i};W^{M}(t))-\Delta_{i}^{w}(Z,C_{i-1},C_{i};W^{\infty}(t))\right|^{2}\right]^{1/2}\\
 & \le K_{T,L}\left(\frac{\log^{1/4}M}{M^{1/2}}+\left(1+B\right)\left\langle W^{M}-W^{\infty}\right\rangle _{t}+e^{-KB^{2}}\right).
\end{align*}
One can obtain similar results for $\Delta_{1}^{w}$ and $\Delta_{i}^{b}$.
Hence, we obtain that for all $t\le T$, 
\[
\left\langle W^{M}-W^{\infty}\right\rangle _{t}\le K_{T,L}\int_{0}^{t}\left(\frac{\log^{1/4}M}{M^{1/2}}+\left(1+B\right)\left\langle W^{M}-W^{\infty}\right\rangle _{t}+e^{-KB^{2}}\right)ds.
\]
Since $\left\langle W^{M}-W^{\infty}\right\rangle _{0}=0$, Gronwall's
inequality implies that 
\[
\sup_{t\leq T}\left\langle W^{M}-W^{\infty}\right\rangle _{t}\le K_{T,L}\inf_{B>0}\left[\left(\frac{\log^{1/4}M}{M^{1/2}}+e^{-KB^{2}}\right)e^{K_{T,L}\left(1+B\right)}\right]\leq K_{T,L}\frac{1}{M^{0.499}},
\]
for sufficiently large $M$.

Furthermore, with the same argument as Lemma \ref{lem:Lipschitz forward MF - general},
given Step 1, one gets that for $2\leq i\leq L-2$ and any $t\leq T$,
\[
\mathbb{E}\left[\left|H_{i}\left(X,C_{i};W^{M}(t)\right)-H_{i}\left(X,C_{i};W^{\infty}(t)\right)\right|^{2}\right]^{1/2}\le K_{T,L}\left\langle W^{M}-W^{\infty}\right\rangle _{t}.
\]
As such, together with Step 2, we get
\[
\sup_{t\leq T}\mathbb{E}\left[\left|H_{i}\left(X,C_{i};W^{M}(t)\right)-H_{i}^{*}\left(t,X,b_{i}^{0}(C_{i})\right)\right|^{2}\right]^{1/2}\le K_{T,L}\frac{1}{M^{0.499}},
\]
for sufficiently large $M$.
\end{proof}

\subsection{Proof of Proposition \ref{prop:iid_law_embedding}\label{subsec:Proof-iid-init-det-representable}}
\begin{proof}[Proof of Proposition \ref{prop:iid_law_embedding}]
It is easy to see that under the canonical neuronal ensemble $\left(\Omega^{M},P^{M}\right)$,
the functions $\left\{ w_{i}^{0}\right\} _{i=1}^{L}$ and $\left\{ b_{i}^{0}\right\} _{i=2}^{L}$
satisfy the i.i.d. initialization law, according to Eq. (\ref{eq:iid_init_embedding_constr1})-(\ref{eq:iid_init_embedding_constr2}).
To derive the $\bar{\eta}$-independence property, recall from the
construction that for $i\leq L-1$, $C_{i}\left(j_{i}\right)=\left(\lambda_{i}\left(j_{i}\right),\theta_{i}\left(j_{i}\right)\right)$
and $\left\{ C_{i}\left(j_{i}\right)\right\} _{j_{i}\in\left[n_{i}\right]}$
are sampled from $\left(P_{0}\times{\rm Unif}\left(\left[M\right]\right)\right)^{n_{i}}$
conditional on that $\left\{ \theta_{i}\left(j_{i}\right)\right\} _{j_{i}\in\left[n_{i}\right]}$
are all distinct. Notice then for $i\leq L-1$ and any $j\in\left[n_{i}\right]$:
\[
\mathbb{E}\left[f(C_{i}(j))\mid\left\{ C_{i}(h),\;h\ne j\right\} \right]=\frac{1}{M-n_{i}+1}\sum_{\theta\notin\{\theta_{i}(h):\;h\ne j\}}\mathbb{E}\left[f(C_{i}(j))\mid\theta_{i}(h)=\theta\right].
\]
Thus, for $1$-bounded function $f$, we have 
\begin{align*}
 & \left|\mathbb{E}\left[f(C_{i}(j))\mid\left\{ C_{i}(h),\;h\ne j\right\} \right]-\mathbb{E}\left[f(C_{i}(j))\right]\right|\\
 & \le\frac{1}{M}\sum_{\theta\in\{\theta_{i}(h):\;h\ne j\}}\left|\mathbb{E}\left[f(C_{i}(j))\mid\theta_{i}(h)=\theta\right]\right|+\frac{n_{i}-1}{M(M-n_{i}+1)}\sum_{\theta\notin\{\theta_{i}(h):\;h\ne j\}}\left|\mathbb{E}\left[f(C_{i}(j))\mid\theta_{i}(h)=\theta\right]\right|\\
 & \le2\frac{n_{i}-1}{M}.
\end{align*}
The claim is trivial for $i=L$.
\end{proof}

\subsection{Proofs of Corollaries \ref{cor:iid_same_neurons} and \ref{cor:iid_standard-network}\label{subsec:Proofs-Corollaries-i.i.d.}}
\begin{proof}[Proof of Corollary \ref{cor:iid_same_neurons}]
By Proposition \ref{prop:iid_law_embedding} and Corollary \ref{cor:gradient descent quality}
(in particular, one of the intermediate steps in its proof), we have
that for sufficiently large $M$, with probability at least $1-3\delta-KLn_{\max}\exp\left(-Kn_{\min}^{c_{2}}\right)$,
\[
\bigg(\frac{1}{n_{i}}\sum_{j_{i}=1}^{n_{i}}\mathbb{E}_{Z}\left[\left|{\bf H}_{i}\left(\left\lfloor t/\epsilon\right\rfloor ,X,j_{i}\right)-H_{i}\left(X,C_{i}\left(j_{i}\right);W^{M}\left(\left\lfloor t/\epsilon\right\rfloor \epsilon\right)\right)\right|^{2}\right]\bigg)^{1/2}=\tilde{O}\left(n_{\min}^{-c_{1}}+\epsilon^{c_{1}}\right),
\]
where we recall that $\left\{ C_{i}\left(j_{i}\right)\right\} _{j_{i}\in\left[n_{i}\right]}$
are sampled according to the sampling rule $\overline{P}_{{\bf n}}^{M}$
as described in Section \ref{subsec:Canonical-neuronal-embeddings}.
On the other hand, since ${\rm Law}\left(C_{i}\left(j_{i}\right)\right)=P_{i}^{M}$,
by Theorem \ref{thm:iid dynamics},
\[
\mathbb{E}\bigg[\frac{1}{n_{i}}\sum_{j_{i}=1}^{n_{i}}\mathbb{E}_{Z}\left[\left|H_{i}\left(X,C_{i}\left(j_{i}\right);W^{M}\left(\left\lfloor t/\epsilon\right\rfloor \epsilon\right)\right)-H_{i}^{*}\left(\left\lfloor t/\epsilon\right\rfloor \epsilon,X,B_{i}\right)\right|^{2}\right]\bigg]^{1/2}\leq\frac{K_{T,L}}{M^{0.499}},
\]
which yields for any $\gamma>0$,
\[
\mathbb{P}\bigg(\frac{1}{n_{i}}\sum_{j_{i}=1}^{n_{i}}\mathbb{E}_{Z}\left[\left|H_{i}\left(X,C_{i}\left(j_{i}\right);W^{M}\left(\left\lfloor t/\epsilon\right\rfloor \epsilon\right)\right)-H_{i}^{*}\left(\left\lfloor t/\epsilon\right\rfloor \epsilon,X,B_{i}\right)\right|^{2}\right]\geq\gamma\bigg)\leq\frac{K_{T,L}}{\gamma M^{0.9}}.
\]
Finally by following the argument in the proof of Corollary \ref{cor:gradient descent quality},
we have:
\[
\mathbb{E}\left[\left|H_{i}^{*}\left(\left\lfloor t/\epsilon\right\rfloor \epsilon,X,B_{i}\right)-H_{i}^{*}\left(t,X,B_{i}\right)\right|^{2}\right]^{1/2}\leq K_{T,L}\epsilon.
\]
The proof concludes by combining this with the previous two probability
bounds and taking $M\to\infty$ then $\text{\ensuremath{\gamma}}\to0$.
\end{proof}
\begin{proof}[Proof of Corollary \ref{cor:iid_standard-network}]
For $3\leq i\leq L-2$, since $b_{i}^{0}\left(C_{i}\right)=B_{i}$
a constant,
\begin{align*}
 & w_{i}^{\infty}\left(t,C_{i-1},C_{i}\right)-w_{i}^{\infty}\left(0,C_{i-1},C_{i}\right)\\
 & =-\int_{0}^{t}\xi_{i}^{\mathbf{w}}\left(s\right)\mathbb{E}_{Z}\Big[\sigma_{i}^{\mathbf{w}}\Big(\Delta_{i}^{H*}\left(s,Z,b_{i}^{0}\left(C_{i}\right)\right),w_{i}^{*}\left(s,w_{i}^{0}\left(C_{i-1},C_{i}\right),b_{i-1}^{0}\left(C_{i-1}\right),b_{i}^{0}\left(C_{i}\right)\right),\\
 & \qquad\quad b_{i}^{*}\left(s,b_{i}^{0}\left(C_{i}\right)\right),H_{i}^{*}\left(s,X,b_{i}^{0}\left(C_{i}\right)\right),H_{i-1}^{*}\left(s,X,b_{i-1}^{0}\left(C_{i-1}\right)\right)\Big)\Big]ds\\
 & =-\int_{0}^{t}\xi_{i}^{\mathbf{w}}\left(s\right)\mathbb{E}_{Z}\left[\bar{\sigma}_{i}^{\mathbf{w}}\left(\Delta_{i}^{H*}\left(s,Z,b_{i}^{0}\left(C_{i}\right)\right),b_{i}^{*}\left(s,b_{i}^{0}\left(C_{i}\right)\right),H_{i}^{*}\left(s,X,b_{i}^{0}\left(C_{i}\right)\right),H_{i-1}^{*}\left(s,X,b_{i-1}^{0}\left(C_{i-1}\right)\right)\right)\right]ds\\
 & =-\int_{0}^{t}\xi_{i}^{\mathbf{w}}\left(s\right)\mathbb{E}_{Z}\left[\bar{\sigma}_{i}^{\mathbf{w}}\left(\Delta_{i}^{H*}\left(s,Z,B_{i}\right),b_{i}^{*}\left(s,B_{i}\right),H_{i}^{*}\left(s,X,B_{i}\right),H_{i-1}^{*}\left(s,X,B_{i-1}\right)\right)\right]ds,
\end{align*}
which is independent of $C_{i-1}$ and $C_{i}$. The desired claim
readily follows.
\end{proof}

\section{Remaining proofs for Section \ref{sec:global_convergence_iid}\label{sec:Remaining-proofs-global-conv-iid}}

\subsection{Proof of Theorem \ref{thm:global-optimum-2}\label{subsec:Proof-two-layers}}

First we show that if $w_{1}\left(0,C_{1}\right)$ has full support,
then so is $w_{1}\left(t,C_{1}\right)$ at any time $t$. Note that
the following result holds beyond the setting of Theorem \ref{thm:global-optimum-2}.
\begin{lem}
\label{lem:full-support-2}Consider the MF ODEs (as described in Section
\ref{subsec:MF}) with $L=2$ and $\mathbb{W}_{1}=\mathbb{R}^{d}$
(for some positive integer $d$), under Assumptions \ref{enu:Assump_lrSchedule}-\ref{enu:Assump_backward}
and \ref{assump:init}. Let us disregard the bias of the second layer
by considering $\xi_{2}^{\mathbf{b}}\left(\cdot\right)=0$ and $b_{2}\left(0,\cdot\right)=0$.
Suppose that the support of ${\rm Law}\left(w_{1}\left(0,C_{1}\right),w_{2}\left(0,C_{1},1\right)\right)$
contains the graph of a continuous function $F:\;\mathbb{W}_{1}\to\mathbb{W}_{2}$
such that $\left|F\left(u\right)\right|\leq K$ for all $u\in\mathbb{W}_{1}$.
Then for all finite time $t$, the support of ${\rm Law}\left(w_{1}\left(t,C_{1}\right)\right)$
is $\mathbb{W}_{1}$.
\end{lem}

\begin{proof}
Since the support of ${\rm Law}\left(w_{1}\left(0,C_{1}\right),w_{2}\left(0,C_{1},1\right)\right)$
contains the graph of $F:\;\mathbb{W}_{1}\to\mathbb{W}_{2}$, we can
choose the neuronal embedding so that there is a choice $C_{1}\left(u\right)$
for each $u\in\mathbb{W}_{1}$ such that $w_{1}\left(0,C_{1}\left(u\right)\right)=u$
and $w_{2}\left(0,C_{1}\left(u\right),1\right)=F\left(u\right)$,
and furthermore for any neighborhood $U$ of $\left(u,F\left(u\right)\right)$,
$\left(w_{1}\left(0,C_{1}\right),w_{2}\left(0,C_{1},1\right)\right)$
lies in $U$ with positive probability. For an arbitrary $T\geq0$,
let us define $M:\;\left[0,T\right]\times\mathbb{W}_{1}\to\mathbb{W}_{1}$
by $M\left(t,u\right)=w_{1}\left(t,C_{1}\left(u\right)\right)$.

We show that $M$ is continuous. In the following, we define $K_{t}$
to be a generic constant that changes with $t$ and is finite with
finite $t$. We first have from Assumption \ref{enu:Assump_backward}
that:
\begin{align*}
\left|\Delta_{2}^{H}\left(t,z,1\right)\right| & \leq K,\\
\left|\Delta_{2}^{w}\left(t,z,c_{1},1\right)\right| & \leq K\left(1+\left|\Delta_{2}^{H}\left(t,z,1\right)\right|\right)\leq K,
\end{align*}
which implies, by Assumption \ref{enu:Assump_lrSchedule},
\[
\left|w_{2}\left(t,c_{1},1\right)\right|\leq\left|w_{2}\left(0,c_{1},1\right)\right|+K_{t}.
\]
In particular, for any $u\in\mathbb{W}_{1}$,
\[
\left|w_{2}\left(t,C_{1}\left(u\right),1\right)\right|\leq F\left(u\right)+K_{t}\leq K_{t}.
\]
We then have from Assumptions \ref{enu:Assump_forward}-\ref{enu:Assump_backward}
that
\begin{align*}
\left|H_{1}\left(t,x,c_{1}\right)-H_{1}\left(t,x,c_{1}'\right)\right| & \leq K\left|w_{1}\left(t,c_{1}\right)-w_{1}\left(t,c_{1}'\right)\right|,\\
\left|\Delta_{2}^{w}\left(t,z,c_{1},1\right)-\Delta_{2}^{w}\left(t,z,c_{1}',1\right)\right| & \leq K\left(1+\left|\Delta_{2}^{H}\left(t,z,1\right)\right|\right)\left|H_{1}\left(t,x,c_{1}\right)-H_{1}\left(t,x,c_{1}'\right)\right|\\
 & \quad+K\left|w_{2}\left(t,c_{1},1\right)-w_{2}\left(t,c_{1}',1\right)\right|\\
 & \leq K\left(\left|w_{2}\left(t,c_{1},1\right)-w_{2}\left(t,c_{1}',1\right)\right|+\left|w_{1}\left(t,c_{1}\right)-w_{1}\left(t,c_{1}'\right)\right|\right),\\
\left|\Delta_{1}^{H}\left(t,z,c_{1}\right)-\Delta_{1}^{H}\left(t,z,c_{1}'\right)\right| & \leq K\left(1+\left|\Delta_{2}^{H}\left(t,z,1\right)\right|\right)\Big(\left|w_{2}\left(t,c_{1},1\right)-w_{2}\left(t,c_{1}',1\right)\right|\\
 & \qquad+\left(1+\left|w_{2}\left(t,c_{1},1\right)\right|+\left|w_{2}\left(t,c_{1}',1\right)\right|\right)\left|H_{1}\left(t,x,c_{1}\right)-H_{1}\left(t,x,c_{1}'\right)\right|\Big)\\
 & \leq K\left|w_{2}\left(t,c_{1},1\right)-w_{2}\left(t,c_{1}',1\right)\right|\\
 & \quad+K_{t}\left(1+\left|w_{2}\left(0,c_{1},1\right)\right|+\left|w_{2}\left(0,c_{1}',1\right)\right|\right)\left|w_{1}\left(t,c_{1}\right)-w_{1}\left(t,c_{1}'\right)\right|,\\
\left|\Delta_{1}^{w}\left(t,z,c_{1}\right)-\Delta_{1}^{w}\left(t,z,c_{1}'\right)\right| & \leq K\left(\left|\Delta_{1}^{H}\left(t,z,c_{1}\right)-\Delta_{1}^{H}\left(t,z,c_{1}'\right)\right|+\left|w_{1}\left(t,c_{1}\right)-w_{1}\left(t,c_{1}'\right)\right|\right)\\
 & \leq K\left|w_{2}\left(t,c_{1},1\right)-w_{2}\left(t,c_{1}',1\right)\right|\\
 & \quad+K_{t}\left(1+\left|w_{2}\left(0,c_{1},1\right)\right|+\left|w_{2}\left(0,c_{1}',1\right)\right|\right)\left|w_{1}\left(t,c_{1}\right)-w_{1}\left(t,c_{1}'\right)\right|.
\end{align*}
Defining 
\[
R\left(t\right)=\left|w_{2}\left(t,C_{1}\left(u\right),1\right)-w_{2}\left(t,C_{1}\left(u'\right),1\right)\right|^{2}+\left|w_{1}\left(t,C_{1}\left(u\right)\right)-w_{1}\left(t,C_{1}\left(u'\right)\right)\right|^{2}
\]
for some $u,u'\in\mathbb{W}_{1}$, we then have for any $t\leq T$:
\begin{align*}
\frac{d}{dt}R\left(t\right) & \leq K_{T}\left(1+\left|w_{2}\left(0,C_{1}\left(u\right),1\right)\right|+\left|w_{2}\left(0,C_{1}\left(u'\right),1\right)\right|\right)^{2}R\left(t\right)\\
 & =K_{T}\left(1+\left|F\left(u\right)\right|+\left|F\left(u'\right)\right|\right)^{2}R\left(t\right)\\
 & \leq K_{T}R\left(t\right),
\end{align*}
which implies that $R\left(t\right)\leq R\left(0\right)\exp\left(K_{T}t\right)$.
In addition, by Assumption \ref{enu:Assump_backward},
\begin{align*}
\left|\Delta_{1}^{H}\left(t,z,c_{1}\right)\right| & \leq K\left(1+\left|\Delta_{2}^{H}\left(t,z,1\right)\right|\right)\left(1+\left|w_{2}\left(t,c_{1},1\right)\right|\right)\\
 & \leq K\left|w_{2}\left(0,c_{1},1\right)\right|+K_{t},\\
\left|\Delta_{1}^{w}\left(t,z,c_{1}\right)\right| & \leq K\left(1+\left|\Delta_{1}^{H}\left(t,z,c_{1}\right)\right|\right)\\
 & \leq K\left|w_{2}\left(0,c_{1},1\right)\right|+K_{t},
\end{align*}
which leads to
\[
\left|w_{1}\left(t,c_{1}\right)-w_{1}\left(t',c_{1}\right)\right|\leq K_{t\lor t'}\left(1+\left|w_{2}\left(0,c_{1},1\right)\right|\right)\left|t-t'\right|.
\]
Since $R\left(0\right)=\left|F\left(u\right)-F\left(u'\right)\right|^{2}+\left|u-u'\right|^{2}\to0$
as $u\to u'$, we deduce that, for $t,t'\leq T$,
\begin{align*}
\left|w_{1}\left(t,C_{1}\left(u\right)\right)-w_{1}\left(t',C_{1}\left(u'\right)\right)\right| & \leq\left|w_{1}\left(t,C_{1}\left(u\right)\right)-w_{1}\left(t',C_{1}\left(u\right)\right)\right|\\
 & \quad+\left|w_{1}\left(t',C_{1}\left(u\right)\right)-w_{1}\left(t',C_{1}\left(u'\right)\right)\right|\\
 & \leq K_{T}\left(1+\left|F\left(u\right)\right|\right)\left|t-t'\right|+\sqrt{R\left(0\right)}\exp\left(K_{T}T\right)\\
 & \to0
\end{align*}
as $\left(u,t\right)\to\left(u',t'\right)$. This shows that $M\left(t,u\right)=w_{1}\left(t,C_{1}\left(u\right)\right)$
is continuous.

Recall that $\mathbb{W}_{1}=\mathbb{R}^{d}$, and consider the sphere
$\mathbb{S}^{d}$ which is a compactification of $\mathbb{W}_{1}$.
We extend $M:\;\left[0,T\right]\times\mathbb{S}^{d}\to\mathbb{S}^{d}$
fixing the point at infinity, which remains a continuous map since
\[
\left|M\left(t,u\right)-u\right|=\left|M\left(t,u\right)-M\left(0,u\right)\right|=\left|w_{1}\left(t,C_{1}\left(u\right)\right)-w_{1}\left(0,C_{1}\left(u\right)\right)\right|\leq K_{T}\left(1+\left|F\left(u\right)\right|\right)t\leq K_{T}t.
\]
Let $M_{t}:\mathbb{W}_{1}\to\mathbb{W}_{1}$ be defined by $M_{t}(u)=M(t,u)$.
Observe that if $M_{t}$ is surjective for all $t$, then the support
of ${\rm Law}\left(w_{1}\left(t,C_{1}\right)\right)$ is $\mathbb{W}_{1}$,
since for a neighborhood $B$ of $M(t,u)=w_{1}(t,C_{1}(u))$, $\mathbb{P}(w_{1}(t,C_{1})\in B)=\mathbb{P}(w_{1}(0,C_{1})\in M_{t}^{-1}(B))>0$.
It is indeed true that $M_{t}$ is surjective for all $t$ for the
following reason. If $M_{t}$ fails to be surjective for some $t$,
then for some $p\in\mathbb{S}^{d}$, $M_{t}:\;\mathbb{S}^{d}\to\mathbb{S}^{d}\backslash\left\{ p\right\} \to\mathbb{S}^{d}$
is homotopic to the constant map, but $M$ then gives a homotopy from
the identity map $M_{0}$ on the sphere to a constant map, which is
a contradiction as the sphere $\mathbb{S}^{d}$ is not contractible.
This finishes the proof of the claim.
\end{proof}
We are ready to prove Theorem \ref{thm:global-optimum-2}. We recall
the setting of Theorem \ref{thm:global-optimum-2}, and in particular,
the neural network (\ref{eq:two-layer-nn}).
\begin{proof}[Proof of Theorem \ref{thm:global-optimum-2}]
It is easy to check that Assumptions \ref{enu:Assump_lrSchedule}-\ref{enu:Assump_backward}
hold. Therefore, by Theorem \ref{thm:existence ODE}, the solution
to the MF ODEs exists uniquely, and by Lemma \ref{lem:full-support-2},
the support of ${\rm Law}\left(w_{1}\left(t,C_{1}\right)\right)$
is $\mathbb{R}^{d}$ at all $t$. We recall from the convergence assumption
the limits $\bar{w}_{1}$ and $\bar{w}_{2}$, and we shall first prove
$\left(\bar{w}_{1},\bar{w}_{2}\right)$ is a global minimizer of $\mathscr{L}$
in Case 1 and $\mathscr{L}\left(\bar{w}_{1},\bar{w}_{2}\right)=0$
in Case 2.

By the convergence assumption, we have that for any $\epsilon>0$,
there exists $T\left(\epsilon\right)$ such that for all $t\geq T\left(\epsilon\right)$
and $P$-almost every $c_{1}$:
\begin{align*}
\epsilon & \geq\left|\mathbb{E}_{Z}\left[\partial_{2}{\cal L}\left(Y,\hat{y}\left(t,X\right)\right)\varphi_{2}'\left(H_{2}\left(t,X,1\right)\right)\varphi_{1}\left(\left\langle w_{1}\left(t,c_{1}\right),X\right\rangle \right)\right]\right|\\
 & =\left|\left\langle \mathbb{E}_{Z}\left[\partial_{2}{\cal L}\left(Y,\hat{y}\left(t,X\right)\right)\middle|X=x\right]\varphi_{2}'\left(H_{2}\left(t,x,1\right)\right),\varphi_{1}\left(\left\langle w_{1}\left(t,c_{1}\right),x\right\rangle \right)\right\rangle _{L^{2}\left({\cal P}_{X}\right)}\right|.
\end{align*}
Let ${\cal H}\left(f_{1},f_{2},x\right)=\mathbb{E}_{Z}\left[\partial_{2}{\cal L}\left(Y,\hat{y}\left(X;f_{1},f_{2}\right)\right)\middle|X=x\right]\varphi_{2}'\left(H_{2}\left(x;f_{1},f_{2}\right)\right)$.
Since ${\rm Law}\left(w_{1}\left(t,C_{1}\right)\right)$ has full
support, we obtain that for $u$ in a dense subset of $\mathbb{R}^{d}$,
\[
\left|\left\langle {\cal H}\left(w_{1}\left(t,\cdot\right),w_{2}\left(t,\cdot,1\right),x\right),\varphi_{1}\left(\left\langle u,x\right\rangle \right)\right\rangle _{L^{2}\left({\cal P}_{X}\right)}\right|\leq\epsilon.
\]
Since $\varphi_{1}'$ is bounded and $\left|X\right|\leq K$, $\mathbb{E}_{X}\left[\left|\varphi_{1}(\left\langle u',X\right\rangle )-\varphi_{1}(\left\langle u,X\right\rangle )\right|^{2}\right]\to0$
as $u'\to u$. Hence, 
\[
\left|\left\langle {\cal H}\left(w_{1}\left(t,\cdot\right),w_{2}\left(t,\cdot,1\right),x\right),\varphi_{1}\left(\left\langle u,x\right\rangle \right)\right\rangle _{L^{2}\left({\cal P}_{X}\right)}\right|\leq\epsilon,
\]
for all $u\in\mathbb{R}^{d}$. We claim that ${\cal H}\left(w_{1}\left(t,\cdot\right),w_{2}\left(t,\cdot,1\right),X\right)\to{\cal H}\left(\bar{w}_{1},\bar{w}_{2},X\right)$
in $L^{1}\left({\cal P}_{X}\right)$ as $t\to\infty$. Assuming this
claim, since $\varphi_{1}$ is bounded, we have for every $u\in\mathbb{R}^{d}$,
\[
\left\langle {\cal H}\left(\bar{w}_{1},\bar{w}_{2},x\right),\varphi_{1}\left(u,x\right)\right\rangle _{L^{2}\left({\cal P}_{X}\right)}=0.
\]
Since $\left\{ \varphi_{1}\left(\left\langle u,\cdot\right\rangle \right):\;u\in\mathbb{R}^{d}\right\} $
has dense span in $L^{2}\left({\cal P}_{X}\right)$,
\[
{\cal H}\left(\bar{w}_{1},\bar{w}_{2},x\right)=\mathbb{E}\left[\partial_{2}{\cal L}\left(Y,\hat{y}\left(X;\bar{w}_{1},\bar{w}_{2}\right)\right)\middle|X=x\right]\varphi_{2}'\left(H_{2}\left(x;\bar{w}_{1},\bar{w}_{2}\right)\right)=0,
\]
for ${\cal P}_{X}$-almost every $x$.

In Case 1, $\varphi_{2}'$ is non-zero, we get $\mathbb{E}\left[\partial_{2}{\cal L}\left(Y,\hat{y}\left(X;\bar{w}_{1},\bar{w}_{2}\right)\right)\middle|X=x\right]=0$
for ${\cal P}_{X}$-almost every $x$. For ${\cal L}$ convex in the
second variable, for any measurable function $\tilde{y}(x)$, 
\[
{\cal L}\left(y,\tilde{y}\left(x\right)\right)-{\cal L}\left(y,\hat{y}\left(x;\bar{w}_{1},\bar{w}_{2}\right)\right)\ge\partial_{2}{\cal L}\left(y,\hat{y}\left(x;\bar{w}_{1},\bar{w}_{2}\right)\right)\left(\tilde{y}\left(x\right)-\hat{y}\left(x;\bar{w}_{1},\bar{w}_{2}\right)\right).
\]
Taking expectation, we get $\mathbb{E}_{Z}\left[{\cal L}\left(Y,\tilde{y}\left(X\right)\right)\right]\geq\mathscr{L}\left(\bar{w}_{1},\bar{w}_{2}\right)$,
i.e. $\left(\bar{w}_{1},\bar{w}_{2}\right)$ is a global minimizer
of $\mathscr{L}$.

In Case 2, since $y$ is a function of $x$ and $\varphi_{2}'$ is
non-zero, we obtain $\partial_{2}{\cal L}\left(y,\hat{y}\left(x;\bar{w}_{1},\bar{w}_{2}\right)\right)=0$
and hence ${\cal L}\left(y,\hat{y}\left(x;\bar{w}_{1},\bar{w}_{2}\right)\right)=0$
for ${\cal P}_{X}$-almost every $x$. That is, $\mathscr{L}\left(\bar{w}_{1},\bar{w}_{2}\right)=0$.

We now prove the claim. Using the assumptions and recalling the coupling
$\pi_{t}$ in Assumption \ref{assump:two-layers}.4:
\begin{align*}
 & \mathbb{E}\left[\left|{\cal H}(w_{1}(t,\cdot),w_{2}(t,\cdot,1),X)-{\cal H}(\bar{w}_{1}(t,\cdot),\bar{w}_{2}(t,\cdot,1),X)\right|\right]\\
 & \le\mathbb{E}\left[\left|\partial_{2}{\cal L}\left(Y,\hat{y}\left(X;w_{1},w_{2}\right)\right)\varphi_{2}'\left(H_{2}\left(X;w_{1},w_{2}\right)\right)-\partial_{2}{\cal L}\left(Y,\hat{y}\left(X;\bar{w}_{1},\bar{w}_{2}\right)\right)\varphi_{2}'\left(H_{2}\left(X;\bar{w}_{1},\bar{w}_{2}\right)\right)\right|\right]\\
 & \le K\left(\mathbb{E}\left[\left|\varphi_{2}'\left(H_{2}\left(X;w_{1},w_{2}\right)\right)-\varphi_{2}'\left(H_{2}\left(X;\bar{w}_{1},\bar{w}_{2}\right)\right)\right|+\left|\varphi_{2}\left(H_{2}\left(X;w_{1},w_{2}\right)\right)-\varphi_{2}\left(H_{2}\left(X;\bar{w}_{1},\bar{w}_{2}\right)\right)\right|\right]\right)\\
 & \le K\mathbb{E}\left[\left|H_{2}\left(X;w_{1},w_{2}\right)-H_{2}\left(X;\bar{w}_{1},\bar{w}_{2}\right)\right|\right]\\
 & \le K\mathbb{E}_{\pi_{t}}\left[\left|\bar{w}_{2}(C_{1})-w_{2}(t,C_{1}',1)\right|+\left|\bar{w}_{2}(C_{1})\right|\left|\varphi_{1}\left(\left\langle w_{1}(t,C_{1}'),x\right\rangle \right)-\varphi_{1}\left(\left\langle \bar{w}_{1}(C_{1}),x\right\rangle \right)\right|\right]\\
 & \le K\mathbb{E}_{\pi_{t}}\left[\left|\bar{w}_{2}(C_{1})-w_{2}(t,C_{1}',1)\right|+\left|\bar{w}_{2}(C_{1})\right|\left|w_{1}(t,C_{1}')-\bar{w}_{1}(C_{1})\right|\right],
\end{align*}
which converges to $0$ by assumption. This proves the claim.

Finally to connect $\mathscr{L}\left(\bar{w}_{1},\bar{w}_{2}\right)$
with $\mathscr{L}\left(W\left(t\right)\right)$ in the limit $t\to\infty$,
we have:
\begin{align*}
\left|\mathscr{L}\left(W\left(t\right)\right)-\mathscr{L}\left(\bar{w}_{1},\bar{w}_{2}\right)\right| & =\left|\mathbb{E}_{Z}\left[{\cal L}\left(Y,\hat{y}\left(X;W\left(t\right)\right)\right)-{\cal L}\left(Y,\hat{y}\left(X;\bar{w}_{1},\bar{w}_{2}\right)\right)\right]\right|\\
 & \leq K\left|\mathbb{E}_{Z}\left[\hat{y}\left(X;W\left(t\right)\right)-\hat{y}\left(X;\bar{w}_{1},\bar{w}_{2}\right)\right]\right|\\
 & \leq K\mathbb{E}_{\pi_{t}}\left[\left|\bar{w}_{2}\left(C_{1}\right)\right|\left|w_{1}\left(t,C_{1}'\right)-\bar{w}_{1}\left(C_{1}\right)\right|+\left|w_{2}\left(t,C_{1}',1\right)-\bar{w}_{2}\left(C_{1}\right)\right|\right]
\end{align*}
which again converges to $0$ by assumption. This completes the proof.
\end{proof}

\subsection{Proof of Proposition \ref{prop:converse-three-layers}\label{subsec:Proof-converse-three-layers}}
\begin{proof}[Proof of Proposition \ref{prop:converse-three-layers}]
We recall
\begin{align*}
\frac{\partial}{\partial t}w_{2}^{*}\left(t,u_{1},u_{2},u_{3}\right) & =-\mathbb{E}_{Z}\Big[\partial_{2}{\cal L}\left(Y,\hat{y}^{*}\left(t,X\right)\right)w_{3}^{*}\left(t,u_{3}\right)\\
 & \qquad\quad\times\varphi_{3}'\left(H_{3}^{*}\left(t,X\right)\right)\varphi_{2}'\left(H_{2}^{*}\left(t,X,u_{3}\right)\right)\varphi_{1}\left(\left\langle w_{1}^{*}\left(t,u_{1}\right),X\right\rangle \right)\Big],
\end{align*}
for $u_{1}\in\mathbb{R}^{d}$, $u_{2}\in{\rm supp}\left(\rho^{2}\right)$,
$u_{3}\in{\rm supp}\left(\rho^{3}\right)$. By the regularity assumption,
\[
\left|\frac{\partial}{\partial t}w_{2}^{*}\left(t,u_{1},u_{2},u_{3}\right)\right|\leq K\mathbb{E}_{Z}\left[\left|\partial_{2}{\cal L}\left(Y,\hat{y}^{*}\left(t,X\right)\right)\right|\right]\left|w_{3}^{*}\left(t,u_{3}\right)\right|.
\]
Note that the right-hand side is independent of $u_{1}$ and $u_{2}$.
Since $\int\left|w_{3}^{*}\left(t,u_{3}'\right)-\bar{w}_{3}(u_{3})\right|d\pi_{t}^{3}\left(u_{3},u_{3}'\right)\to0$
as $t\to\infty$ for a coupling $\pi_{t}^{3}$ of $\rho_{3}$ and
itself, we have for some finite $t_{0}\leq K$,
\[
\mathbb{E}\left[\left|\bar{w}_{3}(U_{3})\right|\right]\leq\mathbb{E}\left[\left|w_{3}^{*}\left(t_{0},U_{3}\right)\right|\right]+K\leq K,
\]
where the last step is by an argument similar to the proof of Lemma
\ref{lem:bounds MF a priori} and the initialization assumption. As
such, for all $t$ sufficiently large, we have:
\begin{align*}
\sup_{u_{1}\in\mathbb{R}^{d},\;u_{2}\in{\rm supp}\left(\rho^{2}\right)}\mathbb{E}_{U_{3}\sim\rho^{3}}\left[\left|\frac{\partial}{\partial t}w_{2}^{*}\left(t,u_{1},u_{2},U_{3}\right)\right|\right] & \leq K\mathbb{E}_{Z}\left[\left|\partial_{2}{\cal L}\left(Y,\hat{y}^{*}\left(t,X\right)\right)\right|\right]\mathbb{E}\left[\left|w_{3}^{*}\left(t,U_{3}\right)\right|\right]\\
 & \leq K\mathbb{E}_{Z}\left[\left|\partial_{2}{\cal L}\left(Y,\hat{y}^{*}\left(t,X\right)\right)\right|\right]\left(K+\mathbb{E}\left[\left|\bar{w}_{3}\left(U_{3}\right)\right|\right]\right)\\
 & \leq K\mathbb{E}_{Z}\left[\left|\partial_{2}{\cal L}\left(Y,\hat{y}^{*}\left(t,X\right)\right)\right|\right].
\end{align*}
The proof concludes once we show that $\mathbb{E}_{Z}\left[\left|\partial_{2}{\cal L}\left(Y,\hat{y}^{*}\left(t,X\right)\right)\right|\right]\to0$
as $t\to\infty$.

For a fixed $x$, let us write ${\cal L}\left(t,x\right)=\mathbb{E}\left[{\cal L}(Y,\hat{y}^{*}(t,X))\middle|X=x\right]$
and $\partial_{2}{\cal L}(t,x)=\mathbb{E}\left[\partial_{2}{\cal L}(Y,\hat{y}^{*}(t,X))\middle|X=x\right]$
for brevity. Consider Case 1. We claim that if there is an increasing
sequence of time $t_{i}$ so that $\lim_{i\to\infty}\left[{\cal L}(t_{i},x)-\inf_{\hat{y}}\mathbb{E}\left[{\cal L}(Y,\hat{y})\middle|X=x\right]\right]=0$,
then $\lim_{i\to\infty}\left|\partial_{2}{\cal L}(t_{i},x)\right|=0$.
Indeed, it suffices to show that for any subsequence $t_{i_{j}}$
of $t_{i}$, there exists a further subsequence $t_{i_{j_{k}}}$ such
that $\lim_{k\to\infty}\left|\partial_{2}{\cal L}(t_{i_{j_{k}}},x)\right|=0$.
In any subsequence $t_{i_{j}}$ of $t_{i}$, using that ${\cal L}(t_{i_{j}},x)$
is convergent and the fact ${\cal L}(y,\hat{y})\to\infty$ as $|\hat{y}|\to\infty$,
we have $\hat{y}^{*}(t_{i_{j}},x)$ is bounded. Hence, we obtain a
subsequence $t_{i_{j_{k}}}$ for which $\hat{y}^{*}(t_{i_{j_{k}}},x)$
converges to some limit $\hat{y}^{*}$. By continuity, we have $\mathbb{E}[{\cal L}(Y,\hat{y}^{*})|X=x]=\lim_{k\to\infty}{\cal L}(t_{i_{j_{k}}},x)=\inf_{\hat{y}}\mathbb{E}[{\cal L}(Y,\hat{y})|X=x]$.
Thus, since ${\cal L}$ is convex in the second variable, we have
$\mathbb{E}[\partial_{2}{\cal L}(Y,\hat{y}^{*})|X=x]=0$. Thus, $\lim_{k\to\infty}\left|\partial_{2}{\cal L}(t_{i_{j_{k}}},x)\right|=\left|\mathbb{E}[\partial_{2}{\cal L}(Y,\hat{y}^{*})|X=x]\right|=0$,
as claimed. Similarly, we obtain in Case 2 that if there is an increasing
sequence of time $t_{i}$ so that $\lim_{i\to\infty}{\cal L}(t_{i},x)=0$,
then $\lim_{i\to\infty}\left|\partial_{2}{\cal L}(t_{i},x)\right|=0$.

To show that $\mathbb{E}_{Z}\left[\left|\partial_{2}{\cal L}\left(t,X\right)\right|\right]\to0$
as $t\to\infty$, it suffices to show that for any increasing sequence
of times $t_{i}$ tending to infinity, there exists a subsequence
$t_{i_{j}}$ of $t_{i}$ such that $\mathbb{E}_{Z}\left[\left|\partial_{2}{\cal L}\left(t_{i_{j}},X\right)\right|\right]\to0$.
In Case 1, we have $\lim_{i\to\infty}\mathscr{L}\left(W^{*}\left(t_{i}\right)\right)=\inf_{\tilde{y}}\mathbb{E}_{Z}\left[{\cal L}\left(Y,\tilde{y}\left(X\right)\right)\right]$,
so $\lim_{i\to\infty}\mathbb{E}_{Z}\left[{\cal L}\left(t_{i},X\right)-\inf_{\tilde{y}\left(X\right)}\mathbb{E}_{Z}\left[{\cal L}\left(Y,\tilde{y}\left(X\right)\right)\middle|X\right]\right]=0$.
Since ${\cal L}\left(t_{i},X\right)-\inf_{\tilde{y}\left(X\right)}\mathbb{E}_{Z}\left[{\cal L}\left(Y,\tilde{y}\left(X\right)\right)\middle|X\right]$
is nonnegative, it converges to $0$ in probability. Thus, there is
a further subsequence $t_{i_{j}}$ for which ${\cal L}\left(t_{i_{j}},X\right)-\inf_{\tilde{y}\left(X\right)}\mathbb{E}_{Z}\left[{\cal L}\left(Y,\tilde{y}\left(X\right)\right)\middle|X\right]$
converges to $0$ ${\cal P}$-almost surely. By the previous claim,
$\left|\partial_{2}{\cal L}\left(t_{i_{j}},X\right)\right|$ converges
to $0$ ${\cal P}$-almost surely. Since $\left|\partial_{2}{\cal L}\left(t_{i_{j}},X\right)\right|$
is bounded ${\cal P}$-almost surely, we obtain that $\mathbb{E}_{Z}\left[\left|\partial_{2}{\cal L}\left(t_{i_{j}},X\right)\right|\right]\to0$
from the bounded convergence theorem. The result in Case 2 can be
established similarly.
\end{proof}

\section{Remaining proofs for Section \ref{sec:global_convergence_general}\label{sec:Remaining-proofs-global-conv-general}}

\subsection{Proof of Proposition \ref{prop:converse-multilayer}\label{subsec:Proof-converse-multi}}
\begin{proof}[Proof of Proposition \ref{prop:converse-multilayer}]
We recall
\[
\frac{\partial}{\partial t}w_{L}\left(t,c_{L-1},1\right)=-\mathbb{E}_{Z}\left[\partial_{2}{\cal L}\left(Y,\hat{y}\left(X;W\left(t\right)\right)\right)\varphi_{L}'\left(H_{L}\left(X,1;W\left(t\right)\right)\right)\varphi_{L-1}\left(H_{L-1}\left(X,c_{L-1};W\left(t\right)\right)\right)\right],
\]
for $c_{L-1}\in\Omega_{L-1}$. By the regularity assumption,
\[
\left|\frac{\partial}{\partial t}w_{L}\left(t,c_{L-1},1\right)\right|\leq K\mathbb{E}_{Z}\left[\left|\partial_{2}{\cal L}\left(Y,\hat{y}\left(X;W\left(t\right)\right)\right)\right|\right].
\]
Note that the right-hand side is independent of $c_{L-1}$. Then as
argued in the proof of Proposition \ref{prop:converse-three-layers}
(Section \ref{subsec:Proof-converse-three-layers}), $\mathbb{E}_{Z}\left[\left|\partial_{2}{\cal L}\left(Y,\hat{y}\left(X;W\left(t\right)\right)\right)\right|\right]\to0$
as $t\to\infty$. This completes the proof.
\end{proof}

\section{Remaining proofs for Section \ref{sec:Global-convergence-ms}\label{sec:Remaining-proofs-global-conv-ms}}

\subsection{Proof of Theorem \ref{thm:global-optimum-2-ms}\label{subsec:Proof-two-layers-ms}}
\begin{proof}[Proof of Theorem \ref{thm:global-optimum-2-ms}]
For brevity, let us write
\[
\overline{H}_{2}(x)=H_{2}(x,1;\overline{W}),\qquad\bar{y}(x)=\hat{y}(x;\overline{W}).
\]
We also define
\begin{align*}
G_{2}(t,u_{1}) & =\mathbb{E}_{Z}\left[\partial_{2}{\cal L}(Y,\hat{y}(t,X))\varphi_{2}'(H_{2}(t,X,1))\varphi_{1}(\langle u_{1},X\rangle)\right],\\
G_{1}(t,u_{1}) & =\mathbb{E}_{Z}\left[\partial_{2}{\cal L}(Y,\hat{y}(t,X))\varphi_{2}'(H_{2}(t,X,1))\varphi_{1}'(\langle u_{1},X\rangle)X\right],\\
\overline{G}_{2}(u_{1}) & =\mathbb{E}_{Z}\left[\partial_{2}{\cal L}(Y,\overline{y}(X))\varphi_{2}'(\overline{H}_{2}(X))\varphi_{1}(\langle u_{1},X\rangle)\right],\\
\overline{G}_{1}(u_{1}) & =\mathbb{E}_{Z}\left[\partial_{2}{\cal L}(Y,\overline{y}(X))\varphi_{2}'(\overline{H}_{2}(X))\varphi_{1}'(\langle u_{1},X\rangle)X\right].
\end{align*}
We claim that as $t\to\infty$,
\[
\mathbb{E}\left[\left|H_{2}(t,X,1)-\overline{H}_{2}(X)\right|\right]\to0,\qquad\mathbb{E}\left[\left|\hat{y}(t,X)-\bar{y}(X)\right|\right]\to0,
\]
\[
\mathbb{E}\left[\left|\partial_{2}{\cal L}(Y,\hat{y}(t,X))-\partial_{2}{\cal L}(Y,\overline{y}(X))\right|\right]\to0,
\]
and uniformly in $u_{1}$,
\[
\left|G_{1}(t,u_{1})-\overline{G}_{1}(u_{1})\right|\to0,\qquad\left|G_{2}(t,u_{1})-\overline{G}_{2}(u_{1})\right|\to0.
\]
Indeed recall the coupling $\pi_{t}$ in Assumption \ref{assump:Morse-Sard-2},
we have from Assumption \ref{assump:two-layers}.3:
\begin{align*}
\mathbb{E}_{X}\left[\left|H_{2}(t,X,1)-\overline{H}_{2}(X)\right|\right] & =\mathbb{E}_{X}\left[\left|\mathbb{E}_{(C_{1},C_{1}')\sim\pi_{t}}\left[w_{2}(t,C_{1}',1)\varphi_{1}(\langle w_{1}(t,C_{1}'),X\rangle)-\overline{w}_{2}(C_{1})\varphi_{1}(\langle\overline{w}_{1}(C_{1}),X\rangle)\right]\right|\right]\\
 & \le K\mathbb{E}_{\pi_{t}}\left[\left|\overline{w}_{2}(C_{1})\right|\left|w_{1}(t,C_{1}')-\overline{w}_{1}(C_{1})\right|+\left|w_{2}(t,C_{1}',1)-\overline{w}_{2}(C_{1})\right|\right],
\end{align*}
which tends to $0$ as $t\to\infty$ by Assumption \ref{assump:Morse-Sard-2}.
The other claims can be derived similarly.

Consider the limit potential $\overline{{\cal F}}$ given by 
\[
\overline{{\cal F}}(u_{1})=\frac{1}{2}\left|\overline{G}_{2}(u_{1})\right|^{2}.
\]
By Assumption \ref{assump:two-layers}.3, $u_{1}\mapsto\overline{{\cal F}}\left(u_{1}\right)$
is continuous. Notice that 
\[
\nabla\overline{{\cal F}}(u_{1})=\frac{1}{2}\cdot2\overline{G}_{2}(u_{1})\nabla\left(\overline{G}_{2}(u_{1})\right)=\overline{G}_{2}(u_{1})\overline{G}_{1}(u_{1}).
\]
Let $\overline{{\cal F}}^{\infty}:\mathbb{S}^{d-1}\to\mathbb{R}$
be defined by $\overline{{\cal F}}^{\infty}(\tilde{u}_{1})=\lim_{r\to\infty}\overline{{\cal F}}(r\tilde{u}_{1})$,
which exists by Assumption \ref{assump:Morse-Sard-2}. We shall argue
that $\overline{{\cal F}}(u_{1})=0$ for all $u_{1}\in\mathbb{R}^{d}$,
by contradiction. To that end, let us assume that $\overline{{\cal F}}(u_{1})\ne0$
for some $u_{1}$. Note that $\overline{{\cal F}}$ is bounded by
a constant by Assumption \ref{assump:two-layers}.3. Thus, either
there is a local maximizer $u_{1}^{*}$ of $\overline{{\cal F}}$
with $\overline{{\cal F}}(u_{1}^{*})>0$ or there is a local maximizer
$\tilde{u}_{1}^{*}$ of $\overline{{\cal F}}^{\infty}$ with $\overline{{\cal F}}^{\infty}(\tilde{u}_{1}^{*})>0$.

First consider the case that $\overline{{\cal F}}$ has a local maximizer
$u_{1}^{*}$ with $\overline{{\cal F}}(u_{1}^{*})>0$. Under Assumption
\ref{assump:Morse-Sard-2}, there exists $\delta\in\left(0,\overline{{\cal F}}\left(u_{1}^{*}\right)\right)$
arbitrarily small so that for $S_{\delta}$ the connected component
of the set $\{u:\;\overline{{\cal F}}(u)>\overline{{\cal F}}(u_{1}^{*})-\delta\}$
that contains $u_{1}^{*}$, there is $\xi>0$ such that $\left|\nabla\overline{{\cal F}}(u_{1})\right|>\xi$
for all $u_{1}\in\partial{\rm cl}\left(S_{\delta}\right)$. Let $T_{0}$
be sufficiently large so that for $t\ge T_{0}$, we have if $u_{1}\in\partial{\rm cl}\left(S_{\delta}\right)$,
$\left|\overline{G}_{1}(u_{1})-G_{1}(t,u_{1})\right|\le\xi/\sqrt{8\overline{{\cal F}}(u_{1}^{*})}$,
which implies
\begin{align}
\left\langle \overline{G}_{1}(u_{1}),G_{1}(t,u_{1})\right\rangle  & \geq\left|\overline{G}_{1}(u_{1})\right|^{2}-\left|\overline{G}_{1}(u_{1})\right|\left|\overline{G}_{1}(u_{1})-G_{1}(t,u_{1})\right|\nonumber \\
 & \geq\left|\overline{G}_{1}(u_{1})\right|^{2}-\frac{\xi}{\sqrt{8\overline{{\cal F}}(u_{1}^{*})}}\left|\overline{G}_{1}(u_{1})\right|\nonumber \\
 & \stackrel{\left(a\right)}{\geq}\left|\overline{G}_{1}(u_{1})\right|^{2}-\frac{\left|\nabla\overline{{\cal F}}\left(u_{1}\right)\right|}{2\left|\overline{G}_{2}(u_{1})\right|}\left|\overline{G}_{1}(u_{1})\right|\nonumber \\
 & =\frac{\left|\nabla\overline{{\cal F}}\left(u_{1}\right)\right|^{2}}{4\overline{{\cal F}}\left(u_{1}\right)}\nonumber \\
 & >\frac{\xi^{2}}{4\overline{{\cal F}}(u_{1}^{*})},\label{eq:proof-global-opt-2ms-1}
\end{align}
where $\left(a\right)$ is because $2\overline{{\cal F}}(u_{1}^{*})>2\overline{{\cal F}}(u_{1})=\left|\overline{G}_{2}\left(u_{1}\right)\right|^{2}$
for any $u_{1}\in\partial{\rm cl}\left(S_{\delta}\right)$ by local
maximality of $u_{1}^{*}$ and continuity of $\overline{{\cal F}}$.
Also, we further enlarge $T_{0}$ so that for $t\ge T_{0}$ and any
$u_{1}\in{\rm cl}\left(S_{\delta}\right)$, $\left|\overline{G}_{2}(u_{1})-G_{2}(t,u_{1})\right|\le\frac{1}{2}\sqrt{\overline{{\cal F}}(u_{1}^{*})-\delta}$
and hence 
\begin{align}
G_{2}(t,u_{1}) & \geq\overline{G}_{2}(u_{1})-\frac{1}{2}\sqrt{\overline{{\cal F}}(u_{1}^{*})-\delta}>\overline{G}_{2}(u_{1})-\frac{1}{2}\sqrt{\overline{{\cal F}}(u_{1})}=\overline{G}_{2}(u_{1})-\frac{1}{2}\left|\overline{G}_{2}(u_{1})\right|,\label{eq:proof-global-opt-2ms-2.1}\\
G_{2}(t,u_{1}) & \leq\overline{G}_{2}(u_{1})+\frac{1}{2}\sqrt{\overline{{\cal F}}(u_{1}^{*})-\delta}<\overline{G}_{2}(u_{1})+\frac{1}{2}\sqrt{\overline{{\cal F}}(u_{1})}=\overline{G}_{2}(u_{1})+\frac{1}{2}\left|\overline{G}_{2}(u_{1})\right|.\label{eq:proof-global-opt-2ms-2.2}
\end{align}
Furthermore notice that
\begin{equation}
\frac{\partial}{\partial t}\overline{G}_{2}(w_{1}(t,C_{1}))=\left\langle \overline{G}_{1}(w_{1}(t,C_{1})),\frac{\partial}{\partial t}w_{1}(t,C_{1})\right\rangle =-w_{2}(t,C_{1},1)\left\langle \overline{G}_{1}(w_{1}(t,C_{1})),G_{1}(t,w_{1}(t,C_{1}))\right\rangle .\label{eq:proof-global-opt-2ms-3}
\end{equation}
Let $\tilde{\Omega}_{1}$ be the subset of $\Omega_{1}$ consisting
of $c_{1}$ where $\left|w_{2}(0,c_{1},1)\right|<\left|F(w_{1}(0,c_{1}))\right|+1$
for $F$ given in Assumption \ref{assump:two-layers}.2. The proof
of Lemma \ref{lem:full-support-2} in fact shows that for any $t\geq0$
and any open subset $B$ of $\mathbb{R}^{d}$, there exists a positive
mass of $C_{1}\in\tilde{\Omega}_{1}$ such that $w_{1}(t,C_{1})\in B$.
In the following, we consider $C_{1}\in\tilde{\Omega}_{1}$. We further
divide the argument into two cases: $\overline{G}_{2}(u_{1}^{*})>0$
and $\overline{G}_{2}(u_{1}^{*})<0$.

Let us consider the case that $\overline{G}_{2}(u_{1}^{*})>0$. Then
we can choose sufficiently small $\delta$ such that $\overline{G}_{2}(u_{1})>0$
for all $u_{1}\in{\rm cl}\left(S_{\delta}\right)$. Furthermore consider
the scenario that there exists $T\ge T_{0}$ such that a positive
mass of $(w_{1}(T,C_{1}),w_{2}(T,C_{1},1))$ with $C_{1}\in\tilde{\Omega}_{1}$
has $w_{1}(T,C_{1})\in S_{\delta}$ and $w_{2}(T,C_{1},1)<0$. Note
that if $w_{1}(t,C_{1})\in{\rm cl}\left(S_{\delta}\right)$,
\[
\frac{\partial}{\partial t}w_{2}(t,C_{1},1)=-G_{2}(t,w_{1}(t,C_{1}))\leq-\left(\overline{G}_{2}(w_{1}(t,C_{1}))-\frac{1}{2}\left|\overline{G}_{2}(w_{1}(t,C_{1}))\right|\right)<0
\]
by Eq. (\ref{eq:proof-global-opt-2ms-2.1}). Define $T_{1}=\inf\left\{ t\geq T:\;w_{1}(t,C_{1})\notin S_{\delta}\right\} $.
Then $t\mapsto w_{2}\left(t,C_{1},1\right)$ is decreasing on $t\in[T,T_{1})$.
Let us argue that $T_{1}=\infty$. Indeed, suppose $T_{1}$ is finite.
We then have, by continuity, $w_{1}(T_{1},C_{1})\in\partial{\rm cl}\left(S_{\delta}\right)$
and $w_{2}(T_{1},C_{1},1)\le w_{2}(T,C_{1},1)<0$. As such, $\frac{\partial}{\partial t}\overline{G}_{2}(w_{1}(T_{1},C_{1}))>0$
by Eq. (\ref{eq:proof-global-opt-2ms-1}) and (\ref{eq:proof-global-opt-2ms-3}).
By continuity, for some $\gamma>0$, $\frac{\partial}{\partial t}\overline{G}_{2}(w_{1}(T_{1}+t,C_{1}))>0$
for all $t\in\left[0,\gamma\right]$. But then $\overline{G}_{2}(w_{1}(T_{1}+t,C_{1}))\ge\overline{G}_{2}(w_{1}(T_{1},C_{1}))\ge\sqrt{2(\overline{{\cal F}}(u_{1}^{*})-\delta)}$,
and hence $w_{1}(T_{1}+t,C_{1})\in S_{\delta}$ for all $t\leq\gamma$,
contradicting the definition of $T_{1}$. Therefore $T_{1}=\infty$,
i.e. for $t\ge T$ and $C_{1}\in\tilde{\Omega}_{1}$ with $w_{1}(T,C_{1})\in S_{\delta}$
and $w_{2}(T,C_{1},1)<0$, we have $w_{1}(t,C_{1})\in S_{\delta}$,
which implies 
\[
G_{2}(t,w_{1}\left(t,C_{1}\right))\stackrel{\left(a\right)}{\ge}\frac{1}{2}\overline{G}_{2}\left(w_{1}\left(t,C_{1}\right)\right)=\sqrt{\frac{1}{2}\overline{{\cal F}}\left(w_{1}\left(t,C_{1}\right)\right)}\geq\sqrt{\frac{1}{2}(\overline{{\cal F}}(u_{1}^{*})-\delta)},
\]
where $\left(a\right)$ is by Eq. (\ref{eq:proof-global-opt-2ms-2.1})
and the fact $\overline{G}_{2}(u_{1})>0$ for all $u_{1}\in{\rm cl}\left(S_{\delta}\right)$.
In particular, there is a positive mass of $(w_{1}(t,C_{1}),w_{2}(t,C_{1},1))$
with $G_{2}(t,w_{1}\left(t,C_{1}\right))\ge\sqrt{(\overline{{\cal F}}(u_{1}^{*})-\delta)/2}$
for all $t\ge T$. Noting that 
\[
\frac{d}{dt}\mathbb{E}\left[{\cal L}(Y,\hat{y}(t,X))\right]=-\mathbb{E}\left[\left|G_{2}(t,w_{1}\left(t,C_{1}\right))\right|^{2}\right]-\mathbb{E}\left[\left|G_{1}(t,w_{1}\left(t,C_{1}\right))\right|^{2}\right]\leq-\mathbb{E}\left[\left|G_{2}(t,w_{1}\left(t,C_{1}\right))\right|^{2}\right],
\]
we obtain $\frac{d}{dt}\mathbb{E}\left[{\cal L}(Y,\hat{y}(t,X))\right]$
being bounded above by a strictly negative constant for all $t\ge T$,
which is a contradiction since ${\cal L}$ is bounded below.

Next consider the scenario that for all $t\ge T_{0}$, the probability
that $w_{1}(t,C_{1})\in S_{\delta}$ and $w_{2}(t,C_{1},1)<0$ on
$C_{1}\in\tilde{\Omega}_{1}$ is zero. Let us argue that for any $t\ge T_{0}$
and for a.e. $C_{1}\in\tilde{\Omega}_{1}$ with $w_{1}(t,C_{1})\in S_{\delta}$,
we have $w_{1}(s,C_{1})\in S_{\delta}$ for all $s\in[T_{0},t]$.
Indeed, consider $t$ and $C_{1}\in\tilde{\Omega}_{1}$ such that
$w_{1}(t,C_{1})\in S_{\delta}$ and $w_{1}(T',C_{1})\notin S_{\delta}$
for some $T'\in[T_{0},t)$. Let $t'=\sup\{s\in\left[T',t\right]:\;w_{1}(s,C_{1})\notin S_{\delta}\}<t$.
By continuity, $w_{1}(t',C_{1})\in\partial{\rm cl}\left(S_{\delta}\right)$
and so by Eq. (\ref{eq:proof-global-opt-2ms-1}),
\[
\left\langle \overline{G}_{1}(w_{1}\left(t',C_{1}\right)),G_{1}(t',w_{1}\left(t',C_{1}\right))\right\rangle >\frac{\xi^{2}}{4\overline{{\cal F}}(u_{1}^{*})}.
\]
By continuity, there exists $t''\in\left(t',t\right)$ such that for
all $s\in[t',t'']$,
\[
\left\langle \overline{G}_{1}(w_{1}\left(s,C_{1}\right)),G_{1}(s,w_{1}\left(s,C_{1}\right))\right\rangle \geq\frac{\xi^{2}}{100\overline{{\cal F}}(u_{1}^{*})}.
\]
By definition of $t'$, we also have $w_{1}(s,C_{1})\in S_{\delta}$
and therefore $w_{2}(s,C_{1},1)\ge0$ for any $s\in(t',t]$. Then
by Eq. (\ref{eq:proof-global-opt-2ms-3}), $\frac{\partial}{\partial t}\overline{G}_{2}(w_{1}(s,C_{1}))\leq0$
for all $s\in(t',t'']$ and therefore
\[
\overline{G}_{2}(w_{1}(t'',C_{1}))\leq\overline{G}_{2}(w_{1}(t',C_{1}))=\sqrt{2(\overline{{\cal F}}(u_{1}^{*})-\delta)},
\]
where the equality follows from $w_{1}(t',C_{1})\in\partial{\rm cl}\left(S_{\delta}\right)$.
However this contradicts with $w_{1}\left(t'',C_{1}\right)\in S_{\delta}$.
Therefore it holds that for any $t\ge T_{0}$ , for a.e. $C_{1}\in\tilde{\Omega}_{1}$
with $w_{1}(t,C_{1})\in S_{\delta}$, $w_{1}(s,C_{1})\in S_{\delta}$
and therefore $w_{2}(s,C_{1},1)\ge0$ for all $s\in[T_{0},t]$. Since
$w_{1}\left(t,C_{1}\right)$ on $C_{1}\in\tilde{\Omega}_{1}$ has
full support at any $t\geq0$, we have for any $t_{0}\geq T_{0}$,
there is a positive mass on $C_{1}\in\tilde{\Omega}_{1}$ such that
$w_{1}(t_{0},C_{1})\in S_{\delta}$ and hence, as shown, $w_{1}(s,C_{1})\in S_{\delta}$
and $w_{2}(s,C_{1},1)\ge0$ for all $s\in\left[T_{0},t_{0}\right]$.
Note that we have $w_{2}(T_{0},C_{1},1)\le M(T_{0})$ for some finite
$M(T_{0})>0$ for $C_{1}\in\tilde{\Omega}_{1}$ (which follows from
the fact $\left|\frac{\partial}{\partial t}w_{2}\left(t,\cdot,1\right)\right|\leq K$
by Assumption \ref{assump:two-layers}.3 and that $\left|w_{2}(0,C_{1},1)\right|<\left|F(w_{1}(0,C_{1}))\right|+1\leq K$).
Also note that for $w_{1}(s,C_{1})\in S_{\delta}$ and $s\geq T_{0}$,
\[
\frac{\partial}{\partial t}w_{2}(s,C_{1},1)=-G_{2}(s,w_{1}\left(s,C_{1}\right))\stackrel{\left(a\right)}{\leq}-\frac{1}{2}\overline{G}_{2}(w_{1}(s,C_{1}))=-\sqrt{\frac{1}{2}\overline{{\cal F}}\left(w_{1}\left(s,C_{1}\right)\right)}\leq-\sqrt{\frac{1}{2}(\overline{{\cal F}}(u_{1}^{*})-\delta)}
\]
a strictly negative constant, where $\left(a\right)$ is by Eq. (\ref{eq:proof-global-opt-2ms-2.1})
and the fact $\overline{G}_{2}(u_{1})>0$ for all $u_{1}\in{\rm cl}\left(S_{\delta}\right)$.
As such, for any $t_{0}\geq T_{0}$ such that
\[
M\left(T_{0}\right)-\left(t_{0}-T_{0}\right)\sqrt{\frac{1}{2}(\overline{{\cal F}}(u_{1}^{*})-\delta)}<0,
\]
there is a positive mass on $C_{1}\in\tilde{\Omega}_{1}$ such that
firstly $w_{2}(s,C_{1},1)\ge0$ for all $s\in\left[T_{0},t_{0}\right]$
and secondly there exists $t\in\left[T_{0},t_{0}\right]$ in which
\[
w_{2}(t,C_{1},1)\leq M\left(T_{0}\right)-\left(t-T_{0}\right)\sqrt{\frac{1}{2}(\overline{{\cal F}}(u_{1}^{*})-\delta)}<0.
\]
We again obtain a contradiction.

The case $\overline{G}_{2}(u_{1}^{*})<0$ can be treated similarly,
with the use of Eq. (\ref{eq:proof-global-opt-2ms-2.1}) replaced
by Eq. (\ref{eq:proof-global-opt-2ms-2.2}). Both cases lead to a
contradiction, ruling out the possibility that there is a local maximizer
$u_{1}^{*}$ of $\overline{{\cal F}}$ with $\overline{{\cal F}}(u_{1}^{*})>0$.

Next consider the case $\overline{{\cal F}}$ does not have any local
maximizer in $\mathbb{R}^{d}$ but $\overline{{\cal F}}^{\infty}$
has a local maximizer $\tilde{u}_{1}^{*}$ with $\overline{{\cal F}}^{\infty}(\tilde{u}_{1}^{*})>0$.
Under Assumption \ref{assump:Morse-Sard-2} (and with the same argument
in the discussion that follows), there exists $\delta\in\left(0,\overline{{\cal F}}^{\infty}\left(\tilde{u}_{1}^{*}\right)\right)$
arbitrarily small so that for $S_{\delta}$ the connected component
of the set $\{u\in\mathbb{R}^{d}:\;\overline{{\cal F}}(u)>\overline{{\cal F}}^{\infty}(\tilde{u}_{1}^{*})-\delta\}$
which contains $r\tilde{u}_{1}^{*}$ for all $r$ sufficiently large,
there is $\xi>0$ such that $\left|\nabla\overline{{\cal F}}(u)\right|>\xi$
for all $u\in\partial{\rm cl}\left(S_{\delta}\right)$. The rest of
the argument can be repeated as before to yield a contradiction.

In short, we have shown that $\overline{{\cal F}}(u_{1})=\frac{1}{2}\left|\overline{G}_{2}(u_{1})\right|^{2}=0$,
and equivalently,
\[
\mathbb{E}_{Z}\left[\partial_{2}{\cal L}(Y,\overline{y}(X))\varphi_{2}'(\overline{H}_{2}(X))\varphi_{1}(\langle u_{1},X\rangle)\right]=0
\]
for all $u_{1}\in\mathbb{R}^{d}$. The remaining proof follows identically
as in the proof of Theorem \ref{thm:global-optimum-2}.
\end{proof}
\bibliographystyle{amsalpha}
\bibliography{MultiLayerNN_general}

\end{document}